\documentclass[letterpaper, 11pt]{article}
\usepackage[utf8]{inputenc}
\usepackage[margin = 1in]{geometry}
\usepackage{smile}
\usepackage[colorlinks, linkcolor=orange, anchorcolor=blue, citecolor=blue]{hyperref}
\usepackage{xcolor}
\usepackage{natbib}
\usepackage{wrapfig}

\newcommand{\inj}{{\tt inj}}
\newcommand{\diff}{\textrm{d}}
\newcommand{\degree}{\gamma}
\newcommand{\xstar}{x^*_t}

\newcommand{\reach}{\tau}
\newcommand{\density}{p_{\rm data}}
\newcommand{\dist}{P_{\rm data}}
\newcommand{\holder}{\beta}  
\newcommand{\expp}{\mathrm{\bold{Exp}}}
\newcommand{\logg}{\mathrm{\bold{Log}}}
\newcommand{\coef}{\mathrm{\text{coef}}}
\newcommand{\lip}{\mathrm{\text{Lip}}}

\newcommand{\mult}{\mathrm{Mult}}
\newcommand{\vol}{\mu_\cM}  
\newcommand{\aerror}{\epsilon_0} 
\newcommand{\relu}{{\rm ReLU}}
\newcommand{\nnsmall}{\bar{s}_{\rm small}}
\newcommand{\nnlarge}{\bar{s}_{\rm large}}
\newcommand{\bigt}{T}
\newcommand{\tlarge}{t_{\rm large}}
\newcommand{\tsmall}{t_{\rm small}}
\newcommand{\projk}{\Pi_k}  
\newcommand{\projm}{\Pi_\mathcal{M}}  
\newcommand{\tball}[1]{\cB^d_{T_{#1}\cM}}
\newcommand{\tsd}[1]{T^d_{#1}\cM}
\newcommand{\rratio}{\eta}
\newcommand{\lscale}{\ell_0}

\newcommand{\dball}{\cB^d}
\newcommand{\pdata}{p_{\rm data}}

\newcommand{\vradius}{\Delta(t)}
\newcommand{\xradius}{\Delta_0(t)}
\newcommand{\thres}{\bar{r}(t)}
\newcommand{\dgM}{\gamma'}
\newcommand{\nn}{{\frak N}}
\newcommand{\nnprod}{\mathfrak{N}_{\times}}

\title{\LARGE \bf Diffusion Model for Manifold Data: Score Decomposition, Curvature, and Statistical Complexity}
\author{Zixuan Zhang$^1$ \quad Kaixuan Huang$^2$ \quad Tuo Zhao$^1$ \quad Mengdi Wang$^2$ \quad Minshuo Chen$^3$ \\
$^1$Georgia Tech \quad $^2$Princeton University \quad $^3$Northwestern University}
\date{\today}

\begin{document}

\maketitle

\begin{abstract}

Diffusion models have become a leading framework in generative modeling, yet their theoretical understanding---especially for high-dimensional data concentrated on low-dimensional structures---remains incomplete. This paper investigates how diffusion models learn such structured data, focusing on two key aspects: statistical complexity and influence of data geometric properties. By modeling data as samples from a smooth Riemannian manifold, our analysis reveals crucial decompositions of score functions in diffusion models under different levels of injected noise. We also highlight the interplay of manifold curvature with the structures in the score function. These analyses enable an efficient neural network approximation to the score function, built upon which we further provide statistical rates for score estimation and distribution learning. Remarkably, the obtained statistical rates are governed by the intrinsic dimension of data and the manifold curvature. These results advance the statistical foundations of diffusion models, bridging theory and practice for generative modeling on manifolds.

\end{abstract}

\section{Introduction}\label{sec:intro}
Diffusion models have recently emerged as a powerful framework in generative modeling, surpassing the performance of traditional methods like Generative Adversarial Networks (GANs, \citet{goodfellow2014generative}) and Variational AutoEncoders (VAEs, \citet{kingma2013auto}) in many applications \citep{ho2020denoising, song2019generative,  songdenoising,rombach2022high}. Diffusion models have found success across a variety of domains, including image synthesis \citep{dhariwal2021diffusion, huang2025diffusion, song2019generative}, audio generation \citep{kong2020diffwave, yang2023diffsound}, and natural language processing \citep{austin2021structured, lou2023discrete, nie2025large}, and have even extended their influence into areas like computational biology \citep{watson2023novo, guo2024diffusion}, reinforcement learning and control \citep{yang2023diffusion, croitoru2023diffusion, chen2024opportunities}. Despite their empirical success, a comprehensive theoretical understanding of diffusion models---especially in high-dimensional settings---remains limited, and further exploration is essential to advance both their theoretical and practical potential.

At their core, diffusion models learn complex data distributions through a pair of processes: a forward process, where Gaussian noise is incrementally added to data until the data distribution converges to white noise, and a backward process, where a denoising neural network learns to reverse this noise corruption. The denoising neural network estimates the so-called score function, which is the gradient of the log-probability density of the (corrupted) data distribution \citep{ho2020denoising, song2019generative, song2020sliced}. The score function is central to understanding the learned distribution and to the generation of new data samples.

While diffusion models have demonstrated their ability to generate high-quality samples in high-dimensional settings, this success challenges the common expectation that learning in such environments should be hampered by the curse of dimensionality \citep{tsybakov2008introduction, wasserman2006all}. A plausible explanation for this success lies in the observation that many real-world datasets exhibit an underlying low-dimensional structure. Although the data may reside in high-dimensional ambient spaces, the intrinsic dimensionality of the data is often much lower, and these low-dimensional structures can often be understood as manifolds \citep{tenenbaum2000global, roweis2000nonlinear}.

In many domains, data display regularities and symmetries that suggest they can be modeled on low-dimensional manifolds. For instance, the set of all images of an object taken under different lighting conditions or orientations tends to lie on a low-dimensional manifold within the high-dimensional pixel space. By exploiting such low-dimensional structures, a line of work has shown that neural networks can efficiently approximate and estimate functions supported on these data domains, with performance dependent primarily on the intrinsic dimensionality of the data rather than its ambient dimensionality \citep{schmidt2017nonparametric, suzuki2018adaptivity, chen2022nonparametric, suh2023approximation,10.1093/imaiai/iaad018}.

While the low-dimensional manifold hypothesis offers an intuitive explanation for the empirical success of diffusion models, fundamental gaps persist in understanding their behavior on manifold data. Rigorous analyses and precise characterizations of statistical complexity---beyond idealized settings such as access to accurate score estimators \citep{de2022convergence, huang2024denoising}---remain underdeveloped. Notable sample complexity bounds are only established in a few very recent works \citep{pmlr-v238-tang24a, azangulov2024convergence, yakovlev2025generalization, chakraborty2026generalization}. Yet, in these works, the interplay between the manifold geometry (e.g., curvature and boundary effects) and the performance of diffusion models is largely uncharacterized; see a detailed discussion in Section~\ref{sec:related_work}. Despite diffusion models excelling at generating globally coherent structures, they often struggle with local fine-scale details. For instance, diffusion models frequently produce globally coherent images with local geometric inconsistencies---such as distorted fingers in synthesized human hands or hallucinated text in generated documents \citep{lu2025towards, lu2024handrefiner, aithal2024understanding, saharia2022photorealistic}. These challenges motivate two foundational questions.

\noindent {\bf Q1}: {\it Can diffusion models efficiently model manifold data, and how does their sample complexity scale with intrinsic (as opposed to ambient) dimensionality?}

\noindent {\bf Q2}: {\it How do geometric properties---such as curvature---affect the quality of generative performance?}

\subsection{Contribution}
In this paper, we develop a formal statistical framework for diffusion models in the context of data supported on unknown low-dimensional manifolds. We answer the two questions by viewing diffusion models as an unsupervised learner aiming to estimate the underlying data distribution. Different from traditional density estimation methods, diffusion models implicitly learn the distribution by estimating the score function, which presents significant challenges to the analysis. One key issue is that the score function can be highly irregular, even if the manifold and the data distribution are relatively smooth. In particular, it is well-known that the score function has an exploding magnitude as the backward diffusion process approaches zero noise \citep{song2019generative, kim2021soft, pidstrigach2022score}. To address this challenge, we propose a novel score decomposition approach, which projects the score according to the noise level, inspired by prior work on modeling data distributions on low-dimensional subspaces \citep{ oko2023diffusionmodelsminimaxoptimal,pmlr-v202-chen23o}.

Our contributions are three-fold: (1) we identify key intrinsic structures of score functions for manifold data under both the large noise (Lemma~\ref{lemma:score-decomp-large}) and small noise (Lemma~\ref{lemma:score-decomp-small}) regimes; (2) based on the structures, we construct suitable neural network architectures that can effectively approximate the score function in diffusion models (Theorem~\ref{thm:approx}), where we highlight the influence of the curvature of manifold; and (3) we provide statistical rates for estimating the score function and further the data distribution (Theorems~\ref{thm:generalization} and \ref{thm:distribution}). The convergence rates depend on the intrinsic dimension of the data as well as the curvature of the manifold. We summarize our main theoretical results in the following theorem.
\begin{theorem}[Informal]
Under some manifold assumptions,
\begin{enumerate}
\item the score function exhibits distinct decompositions in the large noise and small noise regimes: in the large-noise regime, it decomposes as a weighted sum of localized components that guide noisy data toward the manifold; in the small-noise regime, it decomposes via direct projection onto the manifold, with an additional interaction term that reflects the influence of curvature;
\item based on the decompositions, the score function can be approximated by a ReLU neural network with $$\tilde{\mathcal{O}}( \epsilon^{-d/\holder})$$ non-zero weights, where $\epsilon$ is the target approximation error.
The network size scales exponentially only in the intrinsic dimension $d$, while the dependence on the ambient dimension~$D$ is polynomial with a degree governed by the manifold intrinsic dimension and curvature.
\item furthermore,
the data distribution is estimated in terms of Wasserstein distance at a rate $$\tilde{\mathcal{O}}(n^{-\frac{\holder+1}{d+2\holder}}),$$ where $n$ is the sample size.
\end{enumerate}
\end{theorem}

Our results provide crucial insights into how diffusion models capture low-dimensional data structures. While scaling up models and data sets has driven progress in data distribution modeling \citep{kaplan2020scaling, openai2020scaling}, such approaches can be inefficient if they overlook the underlying structure of the data. For instance, Stable Diffusion employs a pre-trained variational autoencoder to first extract key features and reduce dimensionality before applying diffusion, demonstrating the value of architecture adaptation \citep{esser2024scaling}. Through our theoretical framework---encompassing approximation and estimation---we emphasize the need to tailor diffusion models to unknown low-dimensional structures, thereby tightening approximation and estimation bounds. This work paves the way for more efficient architectures and training strategies that explicitly leverage these structures.

\subsection{Related Work}\label{sec:related_work}
Recent theoretical studies have framed diffusion models as unsupervised distribution learners and samplers, establishing their statistical distribution learning guarantees \citep{oko2023diffusionmodelsminimaxoptimal, pmlr-v202-chen23o,dou2024optimal,mei2025deep} and sampling convergence guarantees \citep{block2020generative,chen2022sampling,lee2023convergence,chen2023probability,benton2023nearly,li2024adapting,li2024towards}. Such results offer invaluable theoretical insights into the efficiency and accuracy of diffusion models for modeling complex data.
Specifically, \citet{oko2023diffusionmodelsminimaxoptimal} analyzes $L_2$-error bounds in score estimation through score matching with neural networks, showing that the learned distribution achieves minimax optimal rates.
\citet{pmlr-v202-chen23o} establishes explicit convergence rates when data reside on
low-dimensional linear subspaces. \citet{mei2025deep}  interprets neural networks in diffusion models as denoising mechanisms, enabling an efficient score approximation in graphical models.

Moreover, several works have investigated diffusion models when data are supported on manifolds. \citet{de2022convergence} makes the first attempt to analyze diffusion models for learning low-dimensional manifold data. Assuming an accurate score estimator, \citet{de2022convergence} provides distribution estimation
guarantees of diffusion models in terms of the Wasserstein distance, though the obtained convergence rate
has an exponential dependence on the manifold diameter. \citet{huang2024denoising}
establishes a sampling theory for diffusion models on manifold data, with sampling speed scaling nearly linearly with the intrinsic data dimension. 
\citet{pmlr-v238-tang24a} proves a convergence rate of the distribution estimator induced by diffusion models only depends on the intrinsic data dimension. However, as \citet{yakovlev2025generalization} observes, the hidden constants in \citet{pmlr-v238-tang24a}'s convergence rates depend exponentially on
the ambient dimension $D$, which makes the derived rate ineffective when the ambient dimension is of logarithm order of sample size.  \citet{azangulov2024convergence} improves the
$D$-dependence to polynomial order, but without characterizing how convergence rates depend on manifold curvature. \citet{yakovlev2025generalization} studies denoising scoring matching under relaxed  manifold assumptions, considering noisy samples from a single-chart low-dimensional manifold. 
Additionally, \citet{li2024adapting} show that the DDPM sampler adapts to unknown low-dimensional structure with iteration complexity scaling in the intrinsic dimension. \citet{li2026scoreslearngeometryrate} reveal a separation between learning the data geometry and the data distribution from a non-statistical learning perspective. They show that recovering the data distribution demands significantly more accurate score function than recovering the data manifold. \citet{farghly2025diffusion} study the influence of manifold curvature on distribution approximation via diffusion models, establishing manifold-adaptive rates for a smoothed score function, though without sample complexity bounds. 
Beyond the manifold setting, \cite{chakraborty2026generalization} establish Wasserstein-$p$ convergence rates for diffusion models that scale with a notion of Wasserstein dimension, a covering-number-based complexity measure generalizing the Minkowski dimension. 

Our work addresses a fundamental gap in these theoretical analyses by explicitly characterizing how geometric properties of the data manifold---particularly its reach---govern both the efficiency of score estimation and subsequent distribution recovery guaranties. This geometric perspective reveals previously unexplored connections between manifold curvature and the performance of diffusion models. 

\paragraph{Notation}

Given a real value $a \in \RR$, we denote $\lfloor a \rfloor$ as the largest integer smaller than $a$. For vectors $u, v \in \RR^d$, we denote $\norm{v}$ as its Euclidean norm and $\inner{u}{v} = u^\top v$ as the inner product. Given an arbitrary integer $k \geq 1$, we denote $v^{\otimes k}$ as the $k$-th order outer product of $v$, which is a tensor with $[v^{\otimes k}]_{j_1, \dots, j_k} = \prod_{i=1}^k v_{j_i}$. Given a multi-index $\xi = [\xi_1, \dots, \xi_d]^\top \in \NN^D$, we denote $v^\xi = \prod_{i=1}^d v_i^{\xi_i}$.
We denote $\cB^d(v, r) = \{x : \norm{x - v} \leq r\}$ as a $d$-dimensional Euclidean ball, where $v$ is the center and $r$ is the radius. We frequently invoke two operations on a set $A$: For a scalar $a > 0$, we denote $a A = \{a x: x \in A\}$ and for a mapping $f$, we denote $f(A) = \{f(x): x \in A\}$. Given a real-valued function $f : \RR^d \mapsto \RR$ and a multi-index $\xi\in\NN^d$, we denote $\partial^\xi f$ as the partial derivative $\frac{\partial^{|\xi|} f}{\partial x_1^{\xi_1} \cdots \partial x_d^{\xi_d}}$ with $|\xi| = \sum_{i=1}^d \xi_i$. Given a distribution $P$, we denote the $L^2(P)$-norm of $f$ as $\norm{f}_{L^2(P)}^2 = \int f^2(x) \diff P(x)$. 

\section{Preliminary: Smooth Manifold and Diffusion Model}\label{sec:pre}
This section provides a brief introduction to manifolds and diffusion models. Readers may refer to \cite{tu2010introduction, flaherty2013riemannian} for a comprehensive introduction to manifolds and \cite{chan2024tutorial, chen2024opportunities, tang2024score} for a technical exposure of diffusion models.

\subsection{Riemannian Manifold and Density}
We denote $\cM$ as a $d$-dimensional Riemannian manifold isometrically embedded in $\RR^D$ with $d \ll D$. On a coarse level, a manifold $\cM$ is a set which is locally Euclidean, meaning that a small local neighborhood on $\cM$ can be continuously mapped into a Euclidean space. This is formalized with the definition of \textit{charts}.
\begin{definition}[Chart]
A chart for $\cM$ is a pair $(U, \phi)$ such that $U \subset \cM$ is open and $\phi : U \mapsto \RR^d,$ where $\phi$ is a homeomorphism (i.e., bijective, $\phi$ and $\phi^{-1}$ are both continuous).
\end{definition}
A manifold typically consists of many local neighborhoods overlapping with each other. We say two charts $(U, \phi)$ and $(V, \psi)$ on $\cM$ are $C^\infty$ compatible if and only if the transition functions,
$$\phi \circ \psi^{-1} : \psi(U \cap V) \mapsto \phi(U \cap V) \quad \textrm{and} \quad \psi \circ \phi^{-1} : \phi(U \cap V) \mapsto \psi(U \cap V),$$ are both $C^\infty$.
\begin{definition}[$C^\infty$ Atlas]
\label{def:atlas}
A $C^\infty$ atlas for $\cM$ is a collection of pairwise $C^\infty$ compatible charts $\{(U_k, \phi_k)\}_{k \in \cA}$ such that $\bigcup_{k \in \cA} U_k = \cM$.
\end{definition}
\begin{definition}[Smooth Manifold]\label{def:smoothmanifold}
A smooth manifold is a manifold together with a $C^\infty$ atlas.
\end{definition}
Classical examples of smooth manifolds are the Euclidean space $\RR^D$, the torus, and the unit sphere. We further define a Riemannian manifold as a pair $(\cM, g)$, where $\cM$ is a smooth manifold and $g$ is a Riemannian metric \cite[Chapter 2]{lee2018introduction}. An \textit{isometric embedding} of the $d$-dimensional $\cM$ in $\RR^D$ is an embedding that preserves the Riemannian metric of $\cM$. For more rigorous statements, see the classic reference \citep{flaherty2013riemannian}.

\subsubsection{Exponential Map}
In this paper, we will frequently work with a common choice of atlas on a Riemannian manifold, constructed from exponential maps. Given a point $x \in \cM$, we denote $T_x\cM \subset \RR^D$ as the tangent space of $\cM$ at $x$; a formal definition is provided in \cite[Section 8.1]{tu2010introduction}. The tangent space $T_x\cM$ is a $d$-dimensional vector space and isomorphic to $\RR^d$. Therefore, we can identify vectors in $T_x\cM$ with their $d$-dimensional representations, which we denote as the space $\tsd{x} \subset \RR^d$. We refer to $T_x\cM$ and $\tsd{x}$ both as tangent spaces. 

To define exponential maps, we begin with generalizing straight lines in Euclidean spaces to manifolds. We denote the Riemannian distance $d_\cM : \cM \times \cM \rightarrow \RR$ on $\cM$ as
\begin{align*}
d_{\cM}(x,y) = \inf\{{\sf length}(\gamma) ~|~ \gamma \textup{ is a } C^1(\cM) \textup{ curve such that } \gamma(0) = x, \gamma(1) = y\},
\end{align*}
where ${\sf length}(\gamma)$ is the length of curve $\gamma$ computed with respect to the Riemannian metric $g$. Riemannian distance $d_{\cM}(x, y)$ is referred to as the length of the shortest path or \textit{geodesic} connecting $x$ and $y$ on the manifold. Exponential map builds upon geodesics.
\begin{definition}[Exponential map]\label{def:exp-map}
Let $x \in \cM$. For a tangent vector $v \in \tsd{x}$, there is a unique geodesic $\gamma$ that starts at $x$ with initial velocity $v$, i.e., $\gamma(0) = x$ and $\gamma'(0) = v$. The exponential map centered at $x$ is given by $\expp_x(v) = \gamma(1)$, for all $v \in \tsd{x}$.
\end{definition} 
As can be seen, the exponential map takes a vector $v \in \RR^d$ on the tangent space $\tsd{x}$ as input. The output, $\expp_x(v) \in \RR^D$, is the point on the manifold obtained by traveling along a (unit speed) geodesic curve that starts at $x$ and has initial direction $v$; see Figure \ref{fig:exponential_map} for an illustration.

It is known that for all $x \in \cM$, there exists a radius $r$ such that the exponential map restricted to a ball $\tball{x}(0, r)$ is a diffeomorphism onto its image, i.e., it is a smooth map with smooth inverse. As the sufficiently small $r$-ball in the tangent space may vary for each $x \in \cM$, we define the injectivity radius of $\cM$ as the minimum $r$ over all $x \in \cM$. 

\begin{wrapfigure}{r}{0.4\textwidth}
\vspace{-0.3in}
\centering
\includegraphics[width=0.39\textwidth]{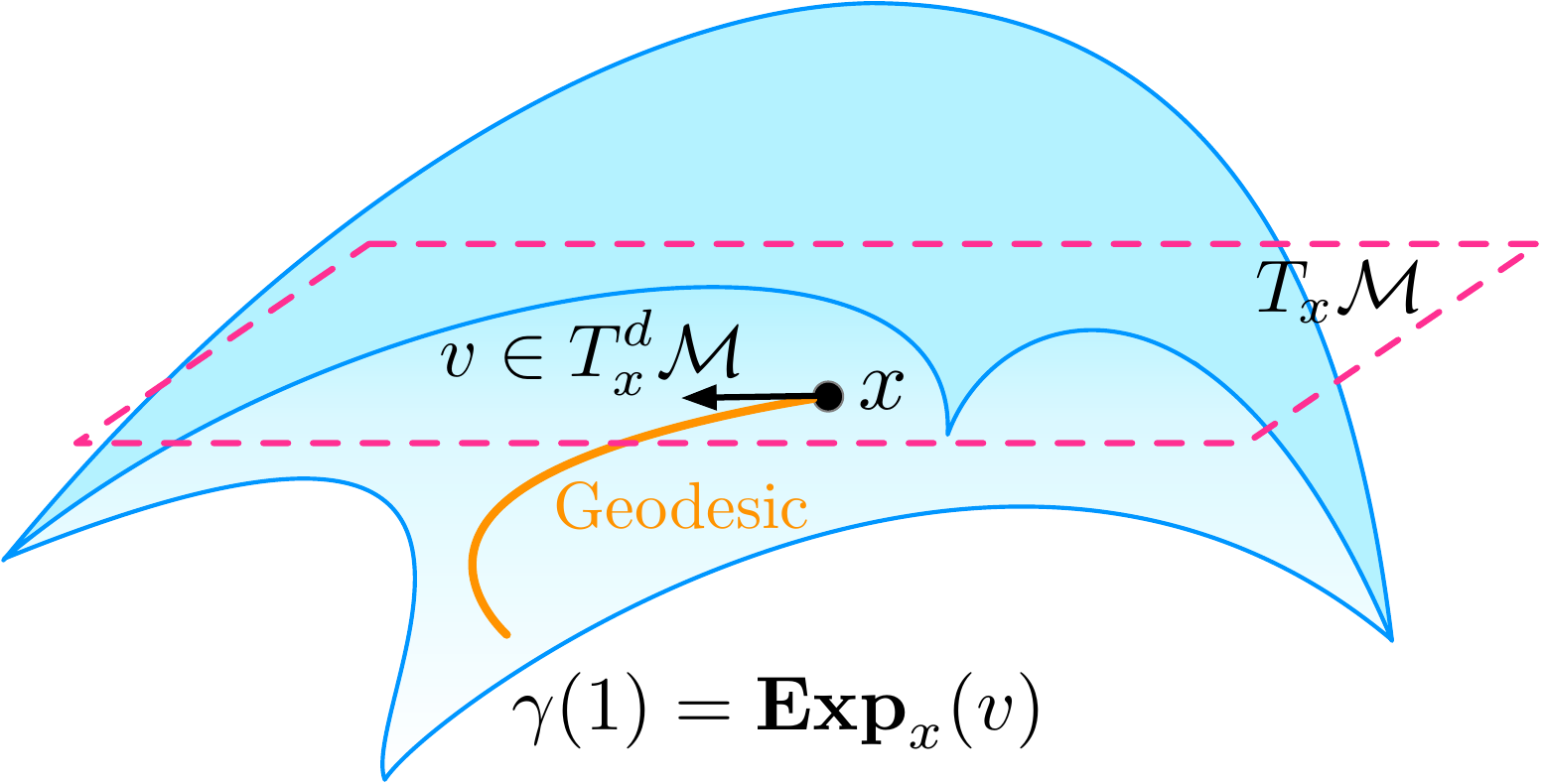}
\vspace{-0.15in}
\caption{Demonstration of tangent space $T_x\cM$, geodesic, and exponential map based on $x \in \cM$ and $v \in \tsd{x}$.}
\vspace{-0.1in}
\label{fig:exponential_map}
\end{wrapfigure}
\begin{definition}[Injectivity radius]\label{def:inj}
For all $x \in \cM$, the injectivity radius at $x$ is defined as $\inj_{\cM}(x) = \sup~ \{r > 0 \mid \expp_{x}: \tball{x}(0, r) \mapsto \cM \textup{ is a diffeomorphism}\}$. Then the injectivity radius of $\cM$ is defined as
\begin{align*}
    \inj(\cM) = \inf~\{\inj_{\cM}(x) \mid x \in \cM\}.
\end{align*}
\end{definition}
As a result, for any $x \in \cM$, the exponential map restricted to a ball of radius $\inj(\cM)$ in $\tsd{x}$ is a well-defined diffeomorphism, which validates $(U_x, \expp_x^{-1})$ as a chart when $U_x = \{\expp_x(v) \mid v \in \tball{x}(0, \inj(\cM))\}$. 
It is convenient to denote the inverse of the exponential map as $\logg_x$, the log map.
Controlling a quantity called reach allows us to lower bound the injectivity radius.
\begin{definition}[Reach \citep{Federer}]
    The reach $\tau$ of a manifold $\cM$ is defined as the quantity
    \begin{align*}
        \tau = \inf~\{r > 0 \mid \exists~ x\neq y \in \cM, v \in \RR^D \text{ such that } r =\|x- v\| = \|y- v\| = \inf_{z \in \cM} \| z - v \| \}.
    \end{align*}
\end{definition}
Intuitively, if the distance of a point $x$ to $\cM$ is smaller than the reach, then there is a unique point in $\cM$ that is closest to $x$. However, if the distance between $x$ and $\cM$ is larger than the reach, then there will no longer be a unique closest point to $x$ in $\cM$. For example, the reach of a sphere is its radius. The reach gives us control over the injectivity radius $\inj({\cM})$.
\begin{proposition}[Proof in \cite{AL}]\label{prop:inj-radius}
    For a manifold $\cM$ with reach $\tau >0$, it holds that $\inj({\cM}) \geq \pi\tau$.
\end{proposition}
\begin{remark}[Atlas using exponential map]\label{remark:exp_atlas}
Constructing an atlas on a compact Riemannian manifold $\cM$ using exponential maps is commonly adopted. In particular, recall that for any point $x \in \cM$, the exponential map is diffeomorphic on a ball $\tball{x}(0, r)$ with $r < \inj(\cM)$. Denote $U_x = \expp_x(\tball{x}(0, r))$ as a local neighborhood, whose union $\cup_{x \in \cM} U_x$ forms an open cover of $\cM$. For a compact $\cM$, there exists a finite sub-cover $\{U_{x_k}\}_{k=1}^{C_{\cM}}$. As shown in \cite{chen2022nonparametric}, $C_{\cM}$ depends on the reach and dimension of the manifold. We suppress the subscript $x_k$ to $k$ and thus, $\{(U_k, \logg_k)\}_{k=1}^{C_\cM}$ becomes an atlas on $\cM$.
\end{remark}

\subsubsection{Differentiable Function on Manifold}
The existence of an atlas allows us to define differentiable functions on $\cM$.
\begin{definition}[$C^s$ Functions on $\cM$]
Let $\cM$ be a $d$-dimensional Riemannian manifold isometrically embedded in $\RR^D$. A function $f: \cM \mapsto \RR$ is $C^s$ if for any chart $(U, \phi)$, the composition $f \circ \phi^{-1}: \phi(U) \mapsto \RR$ is continuously differentiable up to order $s$.
\end{definition}
We note that the definition of $C^s$ functions is independent of the choice of the chart $(U, \phi)$. Suppose $(V, \psi)$ is another chart and $V \cap U \neq \emptyset$. Then we have $f \circ \psi^{-1} = (f \circ \phi^{-1}) \circ (\phi \circ \psi^{-1})$. Since $\cM$ is a smooth manifold, $(U, \phi)$ and $(V, \psi)$ are $C^\infty$ compatible. Thus, $f \circ \phi^{-1}$ is $C^s$ and $\phi \circ \psi^{-1}$ is $C^\infty$, and their composition is $C^s$. The following definition generalizes $C^s$ functions to H\"{o}lder continuous functions.
\begin{definition}[H\"{o}lder Functions on $\cM$]\label{def:holder}
Let $\cM$ be a $d$-dimensional compact Riemannian manifold isometrically embedded in $\RR^D$.
Let $\{(U_k, \textbf{Log}_k)\}_{k\in\mathcal{A}}$ be an atlas of $\cM$ where $\textbf{Log}_k$ is the log map on $U_k$.
For $\beta > 0$, the H\"{o}lder norm of a function $f: \cM \mapsto \RR$ is defined as
\begin{align*}
  \|f\|_{\cH^\beta(\cM)}
  := \max\Bigg\{
    &\max_{\substack{k \in \mathcal{A},\; u \in \NN^d \\ |u| \le \lfloor\beta\rfloor}}
      \sup_{x \in U_k} \big|\partial^u (f \circ \textbf{Exp}_k)\big|,  \\
    &\max_{\substack{k \in \mathcal{A},\; u \in \NN^d \\ |u| = \lfloor\beta\rfloor}}
      \sup_{\substack{x, y \in U_k \\ x \neq y}}
      \frac{\Big|\partial^u (f \circ \textbf{Exp}_k)\big|_{\textbf{Log}_k(x)}
            - \partial^u (f \circ \textbf{Exp}_k)\big|_{\textbf{Log}_k(y)}\Big|}
           {\|\textbf{Log}_k(x) - \textbf{Log}_k(y)\|^{\beta - \lfloor\beta\rfloor}}
  \Bigg\}.
\end{align*}
A function $f$ is $\beta$-H\"{o}lder continuous if $\|f\|_{\cH^\beta(\cM)} < \infty$.
We further denote $\cH^\beta(\cM, C_{\cH})$ as the collection of $\beta$-H\"{o}lder functions on $\cM$ with H\"{o}lder norm bounded by a constant $C_{\cH}$.
\end{definition}

Definition \ref{def:holder} requires that all $\lfloor \beta \rfloor$-th order derivatives of $f \circ \expp_k$ are H\"{o}lder continuous.
We recover the standard H\"{o}lder class on a Euclidean space when $\expp_k$ is the identity map. We next introduce the partition of unity, which plays a crucial role in our construction of neural networks.
\begin{definition}[Partition of Unity, Definition 13.4 in \cite{tu2010introduction}]\label{def:pou}
A $C^\infty$ partition of unity on a manifold $\cM$ is a collection of nonnegative $C^\infty$ functions $\rho_k: \cM \mapsto \RR_+$ for $k \geq 1$ such that
\begin{enumerate}
\item the collection of supports, $\{\textrm{supp} (\rho_k)\}_{k \in \cA}$, is locally finite, i.e., every point on $\cM$ has a neighborhood that meets only finitely many of ${\rm supp} (\rho_k)$'s;
\item $\sum_k \rho_k = 1$.
\end{enumerate}
\end{definition}
For a smooth manifold, a $C^\infty$ partition of unity always exists.
\begin{proposition}[Existence of a $C^\infty$ partition of unity, Theorem 13.7 in \cite{tu2010introduction}]\label{thm:parunity}
Let $\{U_k\}_{k \in \cA}$ be an open cover of a compact smooth manifold $\cM$. Then there is a $C^\infty$ partition of unity $\{\rho_k\}_{k \in \cA}$ where every $\rho_k$ has a compact support such that $\textrm{supp}(\rho_k) \subset U_k$.
\end{proposition}

\subsubsection{Distribution and Density on Manifold}\label{sec:pre_distribution}
To define distributions and density functions on $\cM$, we first define the \textit{volume measure} on a manifold $\cM$ and establish integration. 
\begin{definition}[Volume measure]
For a compact $d$-dimensional Riemannian manifold $\cM$, its volume measure $\mu_{\cM}$ is the $d$-dimensional Hausdorff measure on $\cM$ \citep{Federer}.
\end{definition}

We say that a data distribution $\dist$ supported on $\cM$ has a density $p_{\rm data}$ if the Radon-Nikodym derivative of $\dist$ with respect to $\mu_{\cM}$ is $p_{\rm data}$.
According to \cite{EG}, for any continuous function $f : \cM \mapsto \RR$ supported within the image of the ball $\tball{x}(0,r)$ under the exponential map for $r < \inj(\cM)$, we have
\begin{align}\label{localdensity}
\int f {\rm d} \dist = \int (f \cdot p_{\rm data})  {\rm d} \mu_{\cM} = \int_{\tball{x}(0,r)} (f\cdot p_{\rm data}) \circ \expp_x(v)G_x(v) {\rm d} v.
\end{align}
Here $G_x(v) = \sqrt{\det g_{ij}^x(v)}$ with $g_{ij}^x(v) = \langle \partial \expp_x(v)/\partial e_i, \partial \expp_x(v) / \partial e_j \rangle_g$ for $(e_1,...,e_d)$ being an orthonormal basis of $T_x\cM$.

To extend the integration of $f$ to arbitrary domains, we consider a partition of unity $\{\rho_k(x)\}_{k \in \cA}$, such that each $\rho_k(x)$ is supported on the image of the ball $\tball{x_k}(0,r)$ under the exponential map. Then we have  
\begin{align}\label{globaldensity}
\int f {\rm d} \dist = \int (f\cdot p_{\rm data}) {\rm d} \mu_{\cM} = \sum_{k} \int_{\tball{x_k}(0,r)} (f \cdot p_{\rm data} \cdot \rho_k) \circ \expp_{x_k}(v)G_{x_k}(v) {\rm d} v.
\end{align}
\begin{remark}[Decomposition of Manifold Distribution]\label{remark:decomp_Q}
An important implication of \eqref{globaldensity} is that we may view the manifold distribution $\dist$ as a mixture of local measures supported on charts. To see this, for any measurable subset $A \subset \cM$, we define $P_{\textrm{data}, k}(A) = \int (\mathds{1}_{A} \cdot p_{\rm data} \cdot \rho_k) \circ \expp_{x_k}(v)G_{x_k}(v) {\rm d} v$. Then it holds that $\dist(A) = \sum_{k} P_{\textrm{data}, k}(A)$. Clearly, each $P_{\textrm{data}, k}$ is locally supported within the image of $\tball{x_k}(0,r)$ in the tangent space. Since the support of $P_{\textrm{data}, k}$ may overlap with each other, $P_{\textrm{data}, k}$ is only a measure instead of a probability distribution with unit total mass. For convenience, we denote the normalized version of $P_{\textrm{data}, k}$ as $\mu_k$, which is a probability distribution supported on $U_k$. Then it holds that $\dist = \sum_{k} P_{\textrm{data}, k}(U_k) \cdot \mu_k$, where $P_{\textrm{data}, k}(U_k)$ is the total mass of $\mu_k$. Such a decomposition into local measures is crucial to our analysis on structures in the score function in Section~\ref{sec:decomp}.
\end{remark}

\subsection{Diffusion Model}
We adopt a continuous-time description of diffusion models for the ease of theoretical analysis. There are two coupled processes in diffusion models. The forward process progressively injects independent Gaussian noise to clean data. Following the convention in literature, we use an Ornstein-Ulhenbeck process (also known as the variance preserving process):
\begin{align}\label{eq:forward_sde}
\diff X_t = -\frac{1}{2} X_t \diff t + \diff B_t \quad \text{with} \quad X_0 \sim \dist,
\end{align}
where $(B_t)_{t\geq 0}$ is a Wiener process. We denote the marginal distribution of $X_t$ at time $t$ as $P_t$. Under \eqref{eq:forward_sde}, given $X_0$, the conditional distribution of $X_t | X_0$ is Gaussian ${\sf N}(\alpha_t X_0, h_tI_D)$, where $\alpha_t = \exp(-t/2)$ and $h_t = 1 - \alpha^2_t$. Consequently, under mild conditions, \eqref{eq:forward_sde} transforms the initial distribution $P_{\rm data}$ to $P_{\infty} = {\sf N}(0, I_D)$ at infinite time. In practice, the forward process \eqref{eq:forward_sde} will terminate at a sufficiently large time $T > 0$, where the corrupted marginal distribution $P_T$ is expected to be close to the standard Gaussian distribution.

When generating new samples, diffusion models reverse the time of \eqref{eq:forward_sde}, which leads to the following backward process,
\begin{align}\label{eq:backward_sde}
\diff Y_t & = \left[\frac{1}{2}Y_t + \nabla \log p_{T-t}(Y_t)\right] \diff t + \diff \overline{B}_t \quad \text{with} \quad Y_0 \sim P_T,
\end{align}
where $\nabla \log p_t(\cdot)$ is the so-called score function, i.e., the gradient of log probability density function of $P_t$, and $\overline{B}_t$ is another independent Wiener process. With mild assumptions, the backward process $(Y_t)_{0\leq t\leq T}$ has the same distribution as the time-reversed version of the forward process $(X_{T-t})_{0 \leq t \leq T}$ \citep{anderson1982reverse, haussmann1986time}.

Simulating \eqref{eq:backward_sde}, however, poses substantial difficulties, since both the score function $\nabla \log p_t$ and initial distribution $P_T$ are unknown. To address these challenges, we replace $P_T$ by the standard Gaussian distribution. Moreover, we use a score estimator $\hat{s}$ instead of the ground truth score $\nabla \log p_t$. The estimated score $\hat{s}$ is parameterized by a neural network. With these substitutions, we obtain the following backward SDE,
\begin{align}\label{eq:backward_practice}
\diff  \hat{Y}_t & = \left[\frac{1}{2} \hat{Y}_t + \hat{s}(\hat{Y}_{t}, T-t)\right] \diff t + \diff \overline{B}_t \quad \text{with} \quad \hat{Y}_{0} \sim {\sf N}(0, I_D).
\end{align}
In practice, new data are generated by applying a discretization scheme to simulate \eqref{eq:backward_practice}.

The estimated score $\hat{s}$ is obtained by score matching \citep{song2019generative, ho2020denoising}, which solves a least square problem aggregated over time $t \leq T$. Conceptually, we minimize
\begin{align*}
\min_{s \in \cF}~ \cR(s) = \frac{1}{T-t_0} \int_{t_0}^T \mathbb{E}_{X_0 \sim P_{\rm data}} \EE_{X_t \sim {\sf N}(\alpha_t X_0, h_t I)} \left[\left\| s(X_t,t) - \nabla\log p_t(X_t)\right\|^2 \right]\ud t,
\end{align*}
where $\cF$ is a class of score functions and $t_0$ is an early-stopping time for stabilizing the training \citep{song2019generative, song2020improved}.
Since $\nabla \log p_t$ is not tractable with unknown $P_{\rm data}$, loss $\cR$ cannot be directly implemented and optimized. Thanks to the seminal work \citep{hyvarinen2005estimation}, an equivalent loss (up to a constant) is derived:
\begin{align*}
\min_{s \in \cF}~ \cL(s) & = \EE_{X_0 \sim P_{\rm data}} [\ell(X_0;s)] \quad \text{with} \\
\ell(X_0;s) & = \frac{1}{{\bigt}-t_0} \int_{t_0}^{\bigt} \EE_{X_t \sim {\sf N}(\alpha_t X_0, h_t I)}  \Big\|s(X_t,t) + \frac{X_t - \alpha_t X_0}{h_t} \Big\|^2 \ud t.
\end{align*}
It is shown that $\nabla \cR(s) = \nabla \cL(s)$ for any $s$, and therefore, the minimizer of $\cR$ and $\cL$ is identical. Given a data set $\cD = \{x_1, \dots, x_n\}$ sampled from $P_{\rm data}$, we replace the population expectation by an empirical average:
\begin{align}\label{eq:erm_score}
\hat{s} \in \argmin_{s \in \cF} \hat{\cL}(s) = \frac{1}{n} \sum_{i=1}^n \ell(x_i; s).
\end{align}
In our analysis, we choose $\cF$ as a feedforward neural network architecture (Section~\ref{sec:approx}) taking $X_t$ and time $t$ as input. For technical convenience, we assume $\alpha_t$ and $h_t$ are also inputs to the score network, since they are pre-determined in the forward process. We also focus on the empirical risk minimizer in \eqref{eq:erm_score}.

\section{Structure in Score Function for Manifold Data}\label{sec:decomp}

The score function $\nabla \log p_t$ is the steering power of diffusion models in generating high-fidelity samples.
In this section, we study structures in the score function, when the ground-truth data distribution $P_{\rm data}$ is supported on a $d$-dimensional Riemannian manifold $\cM$. We present key insights on how the original manifold structure and data distribution are progressively restored from pure noise in the backward dynamics, driven by the score function. Formally, we show that the score function decomposes distinctively at different stages of the backward process. These properties inform an efficient representation and estimation of the score function.

\subsection{Warmup: Linear Subspace Data}\label{sec:decomp_warmup}
We start with a simple yet intriguing example---data concentrated on a linear subspace. Consider data $x \in \RR^D$ being represented as $x = Az$, where $A \in \RR^{D\times d}$ is an unknown matrix with orthonormal columns and $z \in \RR^d$ is a latent variable following some underlying distribution. This says that each data point lies in the subspace spanned by the column vectors of matrix $A$. For any $t > 0$, the score function decomposes into two orthogonal components as
\begin{align}\label{eq:score-linear}
\nabla \log p_t(x) = \underbrace{\frac{\alpha_t \EE[X_0 | X_t = x] - \Pi_A(x)}{h_t}}_{s_A(x,t):~\text{on-support~score}} \underbrace{- \frac{x - \Pi_A(x)}{h_t}}_{s_\bot(x,t):~\rm orthogonal~score},
\end{align}
where $\Pi_A(x) = AA^\top x$ is the projection onto the column span of matrix $A$.
The same result is presented in \cite{pmlr-v202-chen23o} with an explicit formula for the conditional expectation $\EE[X_0 | X_t = x]$ using the latent variable. We term the two components as the on-support score $s_A$ and the orthogonal score $s_\bot$, respectively. The on-support score $s_A$ lies in the column space of matrix $A$ as $\EE[X_0 | X_t = x]$ concentrated on the subspace. The orthogonal score function $s_{\perp}(x, t) = -\frac{1}{h_t} (I - AA^\top) x$ is perpendicular to the subspace. As time $t \to 0$, the variance $h_t$ of the added noise tends to zero, and we check that $s_\bot$ blows up. This is consistent with the score blowup issue for learning generic structured data \cite{song2019generative,kim2021soft, pidstrigach2022score}. 

The score decomposition suggests decoupled dynamics in the backward process on the subspace and its orthogonal complement. The on-support score $s_A$ aims to recover the distribution of the latent variable, while the orthogonal score $s_{\perp}$ dictates a contraction towards the subspace whenever the noisy state $x$ is not on the subspace. As $t \to 0$, the strength of the contraction to the subspace amplifies and finally pushes data onto the subspace in the limit.

Structures in the score function motivate an efficient statistical argument to demystify the empirical performance of diffusion models in complex high-dimensional data \cite{oko2023diffusionmodelsminimaxoptimal, pmlr-v202-chen23o}. However, linear subspace is oversimplified for real-world applications, spelling the need for a generalization to generic geometric data, which we present in the next section.

\subsection{Nonlinear Manifold Data}\label{sec:decomp_general}
We generalize the decomposition in Section~\ref{sec:decomp_warmup} to compact $d$-dimensional Riemannian manifolds. On the one hand, the score function for a manifold $\cM$ presents local structures analogous to that for the linear subspace case, given that manifolds locally resemble Euclidean space. On the other hand, the score function possesses time-inhomogeneous behavior, and the score decomposition becomes more involved due to nonlinearity.

We investigate structures of the score function in two stages depending on the amount of noise injected in data. We refer to ``{\it large noise}'' when $t$ is large in the forward process \eqref{eq:forward_sde}, and ``{\it small noise}'' when $t$ is small. The forward process suppresses the data magnitude exponentially fast, i.e., $\alpha_t = e^{-t/2}$, therefore, in the large noise regime, injected noise dominates the corrupted data. The behavior of the score function centers around pushing noisy data towards a close vicinity of the original data manifold. On the other hand, in the small noise regime, the score function recovers geometric details of the manifold and eventually recovers the clean data distribution. Figure~\ref{fig:score_decomp} demonstrates the distinct behaviors in two regimes. A two-time-scale perspective also appears in \cite{dou2024optimal}. We formalize our analysis in the sequel.

\subsubsection{Large Noise: Weighted Decomposition}
Given a large time $t$, due to the dominance of the injected noise, it suffices to consider a coarse linear approximation to the manifold $\cM$ constructed by the tangent spaces at $x_k$ for $k = 1, \dots, C_{\cM}$. 
For the $k$-th chart, we denote the projection of $x$ onto the shrunk tangent space $\alpha_t \cdot T_{x_k}\cM$ as
\begin{align*}
    \projk(x, t) = \argmin_{y \in \alpha_t \cdot T_{x_k}\cM} \|y-x\| \quad \text{for} \quad k=1,\ldots,C_\cM.
\end{align*}
We recall that $\alpha_t$ is the shrinkage ratio introduced by the forward process \eqref{eq:forward_sde} at time $t$. Based on projections onto tangent spaces and the decomposition of manifold distribution into a mixture of local measures in Remark~\ref{remark:decomp_Q}, we present the following weighted score decomposition, whose proof is provided in Appendix~\ref{pf:score_decomp_large}.
\begin{lemma}\label{lemma:score-decomp-large}
For a large time $t>0$ and an arbitrary noisy state $x \in \RR^D$, the score function $ \nabla \log p_t(x)$ takes the form:
\begin{align*}
    \nabla \log p_t(x) = \sum_{k=1}^{C_\cM} w_k(x, t)  \bigg(\underbrace{ \frac{\alpha_t \EE_{X_0 \sim \mu_k}[X_0 | X_t = x] - \projk(x, t)}{h_t}}_{s^k_\cM:~\text{$k$-th~on-support~score}} ~~ \underbrace{-\frac{x-\projk(x, t)}{h_t} }_{s^k_\bot:~\text{$k$-th~orthogonal~score}} \bigg),
\end{align*}
where 
\begin{align*}
    w_k(x, t) = \frac{\int \exp \left( -\frac{\|x - \alpha_t x_0\|^2}{2h_t} \right) \ud P_{\textrm{data}, k}(x_0)}{\sum_{j=1}^{C_\cM} \int \exp \left( -\frac{\|x - \alpha_t x_0\|^2}{2h_t} \right) \ud P_{\textrm{data}, j}(x_0)}.
\end{align*}
\end{lemma}
Lemma~\ref{lemma:score-decomp-large} decomposes the score function according to the atlas on $\cM$, giving rise to a weighted sum of localized terms. We provide interpretation of these terms in below.
\paragraph{Decomposition on Single Chart} For the $k$-th chart, $s_\perp^k$ acts analogously to the orthogonal score in the linear subspace case. While the difference is that $\projk(x, t)$ is a time-dependent projection onto the rescaled tangent space. As a result, the noisy state $x$ moves towards the tangent space $T_{x_k}\cM$ and simultaneously approaches the manifold. Since $\Pi_k(x, t)$ may not belong to the manifold, the on-support term $s_{\cM}^k$ guides the noisy state moving towards $U_k$ by compensating for the deviation of $\projk(x, t)$ from the anticipated mean $\EE_{X_0 \sim \mu_k}[X_0 | X_t = x]$. We explain $w_k$ and $\EE_{X_0 \sim \mu_k}[X_0 | X_t = x]$ in more detail.

\paragraph{Understanding Weight $w_k$} Recall the decomposition of $P_{\rm data}$ into local measures on each chart. The weight $w_k$ is interpreted as the probability that a given noisy state $x$ eventually falls into the $k$-th chart. If $x$ is close to the scaled chart $\alpha_t U_k$, the score function is dominated by the corresponding chart. That is, $w_k \approx 1$ and the score function behaves as if there is only one chart on the manifold. On the contrary, when $x$ is not close to any of $\alpha_t U_k$, the score function behaves as that of a mixture of distributions. Attentive readers may find $w_k$ echos that in Gaussian mixture models. In fact, for a Gaussian mixture $\sum_{k=1}^{C_{\cM}} \pi_k {\sf N}(\mu_k, I)$ with prior probability $\pi_k$, the score function is $\nabla \log p_t(x) = -\sum_{k=1}^{C_{\cM}} c_k(x, t) (x - \alpha_t \mu_k)$, where $c_k(x, t) = \pi_k \exp(-\frac{\norm{x - \alpha_t \mu_k}^2}{2h_t}) / \sum_{j=1}^{C_{\cM}} \pi_j \exp(-\frac{\norm{x - \alpha_t \mu_j}^2}{2h_t})$. Our localized measure $P_{\textrm{data}, k}$ generalizes the prior probability.

\paragraph{Anticipated Mean $\EE_{X_0 \sim \mu_k}[X_0 | X_t = x]$} The anticipated mean can be interpreted as a posterior average, where the prior distribution is $\mu_k$ and the current observation is $X_t = x$. Equivalently, we presume that the noisy state generates clean data belonging to $U_k$. Unfortunately, the anticipated mean is generally not on the manifold or the tangent space due to nonlinearity, leaving $s_{\cM}^k$ and $s_\perp^k$ not orthogonal to each other.

\subsubsection{Small Noise: Local Decomposition}
Driven by the score function in the large noise regime, noisy states gradually reduce noise and evolve towards a close vicinity of the manifold. Then the small noise regime emerges, since the coarse piecewise linear approximation using tangent spaces becomes insufficient for precise manifold recovery. Instead, the score function operates with higher precision, accurately capturing the local geometry of $\cM$ around $x$ through orthogonal projection directly onto $\cM$.

Suppose $\cM$ has a positive reach $\reach$. We define an inflated region around $\cM$ as $\cK(\cM,\reach) = \bigcup_{x \in \cM} \cB(x, \reach)$. At time $t$, for any $x \in \alpha_t \cK(\cM, \reach)$, we define 
\begin{align}\label{eq:def-proj}
\projm(x, t) = \argmin_{y \in \alpha_t \cM} \norm{ y-x}.
\end{align}
The projection \eqref{eq:def-proj} is well-defined as shown in \cite{leobacher2020existence}. The following decomposition of the score function is established using the projection operator, whose proof is provided in Appendix~\ref{pf:score_decomp_small}.

\begin{lemma}\label{lemma:score-decomp-small}
For any fixed time $t > 0$, given an arbitrary $x \in \alpha_t \cK(\cM, \tau)$, the score function $\nabla \log p_t(x)$ decomposes into
\begin{align*}
\nabla \log p_t(x) =  \underbrace{\frac{\alpha_t \EE[X_0 | X_t = x] - \projm(x, t)}{h_t}}_{s_{\cM}:~\text{on-support score}}  ~~\underbrace{-\frac{x -\Pi_{\cM}(x, t) }{h_t}}_{s_{\perp}:~\text{orthogonal score}}.
\end{align*}
Moreover, for the on-support score, it holds that
\begin{align*}
& \hspace{0.45in} s_{\cM}(x, t) = \nabla_x \log \left[\int \exp \left(-\frac{1}{2h_t}(E_1(t) + E_2(t)) \right) \ud \dist(x_0) \right] \quad \text{with} \\
& E_1(t) = \|\projm(x, t) - \alpha_t x_0\|^2 \quad \text{and} \quad E_2(t) = 2\langle x - \projm(x, t) , \projm(x, t) - \alpha_t x_0\rangle.
\end{align*}
\end{lemma}
Compared to the large noise regime in Lemma~\ref{lemma:score-decomp-large}, small noise decomposition does not assert a mixed structure. Alternatively, the decomposition is driven by the unique projection $\projm(x, t)$ onto the manifold, resembling that in the linear subspace case. The reason behind this is that the variance of the added noise to state $x$ is small. It is less likely that $x$ is obtained by corrupting distant clean data points on the manifold. We provide the following further discussions.

\paragraph{Hitting $\alpha_t \cK(\cM, \reach)$ with High Probability} Conditioned on a clean data point $x_0 \in \cM$, the distribution of $X_t$ is Gaussian ${\sf N}(\alpha_t x_0, h_t I)$. When $t$ is sufficiently small, it holds that $\alpha_t \approx 1$ and $h_t \approx 0$. Therefore, we have $\sqrt{D h_t} < \alpha_t \reach$, meaning that $\norm{X_t - \alpha_t x_0} < \alpha_t \reach$ with high probability. Consequently, $\alpha_t \cK(\cM, \reach)$ is hit by noisy states with high probability in the forward process \eqref{eq:forward_sde}. More importantly, when $\reach$ is small, $\cK(\cM, \reach)$ is a narrow band around the manifold, making the requirement on small $t$ stringent. As a result, the learning of the score function becomes challenging. We present explicit approximation and estimation complexities of the score function in relation to the reach $\reach$ in Section~\ref{sec:approx}.

\paragraph{Further Decomposition in $s_{\cM}$} While Lemma~\ref{lemma:score-decomp-small} replicates the same form of score decomposition as the linear subspace, the on-support score possesses intricate structures captured by $E_1$ and $E_2$. Term $E_1$ attempts to distribute the probability mass at the projection point to the manifold according to the Gaussian transition kernel defined in \eqref{eq:forward_sde}. Unfortunately, the nonlinear manifold introduces a nuisance: the movement $x - \projm(x, t)$ is not orthogonal to the movement towards the projection $\projm(x, t) - \alpha_t x_0$. As a result, $E_2$ captures the interaction between the two movements,  and we cannot decouple the dynamics in the backward process as in the linear subspace case. It is worth mentioning that the interaction term $E_2$ heavily depends on the curvature of the manifold. As a sanity check, when the manifold reduces to a linear subspace without curvature, $E_2$ vanishes and Lemma~\ref{lemma:score-decomp-small} reduces to \eqref{eq:score-linear}. Figure~\ref{fig:score_decomp} illustrates such an interaction.

\paragraph{Implication of Score Decomposition} The score decompositions in Lemmas~\ref{lemma:score-decomp-large} and \ref{lemma:score-decomp-small} are key to an efficient neural network approximation developed in Section~\ref{sec:approx}. The on-support scores map $D$-dimensional noisy states to points that can be effectively approximated using $d$-dimensional representations, since the manifold assumes low-dimensional local charts. The orthogonal score exhibits simple structures and can also be approximated efficiently.

\begin{figure}[htb!]
    \centering
    \includegraphics[width=0.99\linewidth]{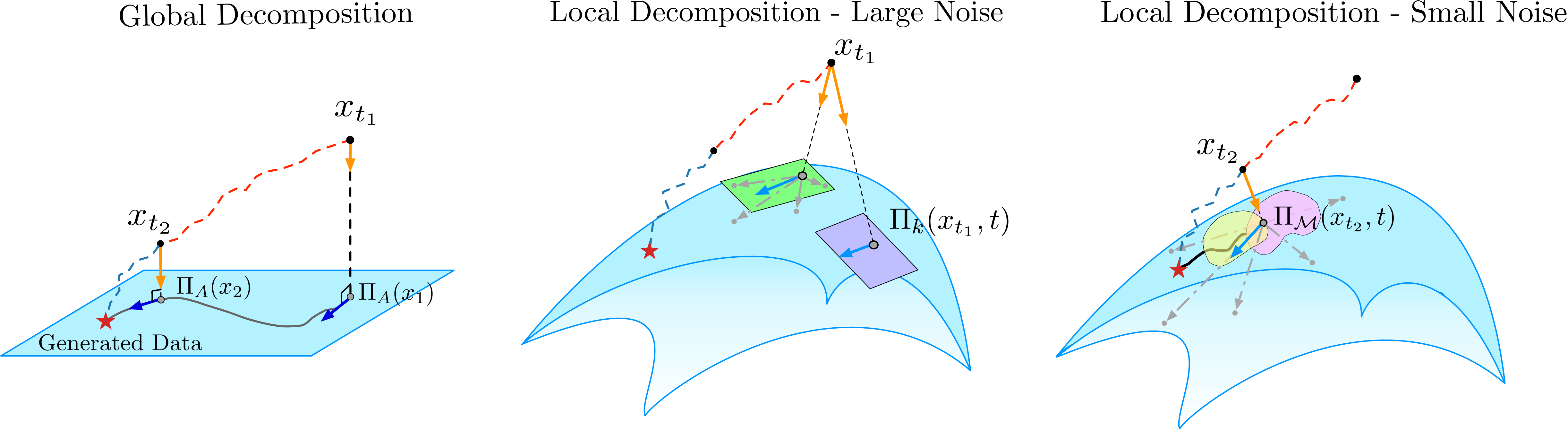}
    \caption{Score decomposition for linear subspace and general manifold. For linear subspace data, the projection $\Pi_A$ is globally defined for all $t$ and noisy state $x$. The score function decomposes into two orthogonal components. For a smooth manifold, the score function behaves distinctly according to the magnitude of the added noise. In the large noise regime, score function decomposes corresponding to tangent spaces on the manifold. In the small noise regime, local decomposition exists and is centered around the projection point $\Pi_{\cM}(x_{t_2}, t_2)$ of a noisy state $x_{t_2}$.}
    \label{fig:score_decomp}
\end{figure}
\section{Approximating Score Function using Neural Network}\label{sec:approx}

In practice, score functions are parameterized by deep neural networks. Leveraging the structures in score functions identified in Section~\ref{sec:decomp}, we develop score approximation theories. We impose some mild regularity conditions on the manifold.
\begin{assumption}[(Data Domain)]\label{assump:manifold}
$\cM$ is a $d$-dimensional compact Riemannian manifold isometrically embedded in $\RR^D$. There exists a constant $B_{\cM} > 0$ such that, for any point $x \in \cM$, it holds that $\|x\|_\infty \leq B_{\cM}$. The reach of $\cM$ is lower bounded by $\reach>0$.
\end{assumption}

A positive reach of $\cM$ ensures geometric regularity by ruling out sharp corners, cusps, or self-intersections. This enables well-defined projections onto the manifold, and ensures that local charts behave nicely with bounded distortion and controlled curvature. Many commonly studied manifolds satisfy such regularity conditions, including the unit sphere $\mathbb{S}^d$ and the torus $\mathbb{T}^d$.

\begin{assumption}\label{assump:exp-ball}
For the atlas in Remark~\ref{remark:exp_atlas}, there exists a constant $\rratio \in (0,1]$ such that choosing $r = 3\reach$ ensures $\cB(x_k, \rratio \reach) \cap\cM$ is a subset of $U_k$.
\end{assumption}

Assumption \ref{assump:exp-ball} is an implicit regularity condition for the geometric structure of $\cM$. While exponential map $\expp_x$ is well-defined on $\cB_{T_{x} \cM}(0,3 \reach)$ by Proposition \ref{prop:inj-radius}, Assumption \ref{assump:exp-ball} further guarantees its inverse map $\logg_x$ is well-defined so that $\expp_x$ is locally a bijection. In addition, we assume the smoothness of
the ground-truth probability density $\density$.
 
\begin{assumption}[(Data Distribution)]\label{assump:density}
    The data distribution $P_{\rm data}$ has a density $\density$ supported on $\cM$ and belonging to $\cH^\holder(\cM)$ for a H\"{o}lder index $\beta > 0$. Moreover, there exists a constant $C_f > 1$ such that $C_f^{-1} \leq \density(x) \leq C_f$ for any $x \in \cM$. 
\end{assumption}
Assumption~\ref{assump:density} generalizes the standard H\"{o}lder densities in nonparametric statistics \cite{tsybakov2008introduction, wasserman2006all} to Riemannian manifolds. We require the density bounded below by a positive constant for technical convenience, which is often adopted in the existing literature \cite{tsybakov2008introduction, oko2023diffusionmodelsminimaxoptimal}.

We parameterize the score function using a class of deep neural networks $\cF$ with entrywise ReLU activation (i.e., $\relu(a) = \max \{a,0\}$):
\begin{equation}\label{def:nn-class}
    \begin{split}
        \cF(L,W,S,B,R) := \big\{ f(x) &= A_L \cdot {\relu}(\cdots {\relu}(A_1 x + b_1)\cdots)+b_L: \\
        & \text{width bounded by } W, ~\sum_{i=1}^L (\|A_i\|_0 + \|b_i\|_0) \leq S, \\
        & \max_{i}\|A_i\|_\infty \leq B, ~ \max_i \|b_i\|_\infty \leq B, ~\sup_z \| f(z)\|_2 \leq R \big\}.
    \end{split}
\end{equation}
Here $\|\cdot\|_0$ and $\|\cdot\|_\infty$ denote the number of nonzero entries and the maximum magnitude of entries. Hyper-parameters control the depth ($L$) and width ($W$) of the network, as well as the sparsity ($S$) and boundedness ($B$ and $R$). The following theorem shows that the network class $\cF$ can accurately represent the score function.
\begin{theorem}\label{thm:approx}[Score Approximation Theory] Suppose Assumptions \ref{assump:manifold}-\ref{assump:density} hold.
For any sufficiently small $\epsilon >0$ and any time $t \in [t_0, \bigt]$ with $t_0 = {\rm poly}(n^{-1})$ and $\bigt=\cO(\log(1/\epsilon))$, there exists a network class $\cF(L,W,S,B,R)$ that yields a function $\bar{s}$ satisfying
\begin{align*}
\left\|\bar{s}(x,t) - \nabla \log p_t(x) \right\|^2_{L^2(P_t)}
= \tilde{\cO}\left(\frac{1}{h_t} D^{2\degree+d+2} \epsilon^2 \right) \quad \text{with} \quad \degree = \left\lceil \frac{\beta \log \frac{1}{\epsilon}}{\log \frac{1}{\epsilon} + \beta \log \reach}\right\rceil.
\end{align*}
The configuration of $\cF$ verifies 
\begin{align*}
  L =\tilde{\cO}\left(\degree^3   \right), \quad
    W = \tilde{\cO}\left(D^{\degree} \degree^3 \epsilon^{-d/\holder}\right), \quad
        S = \tilde{\cO}\left(\degree D^{\degree}\epsilon^{-d/\holder} \right), \quad B= \cO\left( h_{t_0}^{-\degree} \epsilon^{-2\log(1/\epsilon)/\holder} \right).
    \end{align*}
\end{theorem}
The complete proof is provided in Appendix~\ref{sec:proof-approx}, and a proof sketch is presented in Section~\ref{sec:proof-sketch}. The precise requirement on $\epsilon$
is specified in condition~\eqref{cond:epsilon}.
Theorem \ref{thm:approx} establishes an $L^2$ score approximation guarantee over the unbounded domain $\RR^D$, while existing universal approximation theory of neural networks majorly focuses on approximating target functions on a compact domain under the $L^\infty$ norm \cite{yarotsky2017error,schmidt2017nonparametric,chen2022nonparametric}. Our analysis addresses the unboundedness through a truncation argument. Moreover, to achieve an $\epsilon$ approximation error, the network size scales exponentially with the intrinsic dimension $d$, emphasizing the adaptability of neural networks to the manifold structures. Several additional remarks are in turn.

\paragraph{Influence of Manifold Curvature} Theorem~\ref{thm:approx} suggests that the approximation error needs to scale with the reach $\tau$ to capture manifold structures. A slight rewrite of the error guarantee shows an explicit impact of the reach. For any $\epsilon \in (0, 1)$ and $\gamma > 0$, the same network class $\cF$ in Theorem~\ref{thm:approx} gives rise to a $\bar{s}$ with
\begin{align}\label{eq:approx-bound}
\norm{\bar{s}(x, t) - \nabla \log p_t(x)}_{L^2(P_t)}^2 = \tilde{\cO}\left(\frac{D^{2\degree+d+2}}{h_t} \left(\frac{ \epsilon^{2\gamma/\beta}}{\tau^{2\gamma}} + \epsilon^2\right)\right).
\end{align}
Setting $\tau^{-2\gamma}\epsilon^{2\gamma/\beta} = \epsilon^2$ recovers the bound in Theorem~\ref{thm:approx}. As evident from \eqref{eq:approx-bound}, a smaller $\reach$, corresponding to higher manifold curvature, leads to increased approximation error. The reason behind this is that a manifold with a small $\reach$ exhibits rapidly changing local geometry, causing local projections to possess significantly different orientations. These variations are difficult to approximate, especially in the small noise regime. Indeed, $E_2$ in Lemma~\ref{lemma:score-decomp-small} varies significantly as noisy state $x$ moves.

\paragraph{Dependence on Ambient Dimension} The approximation guarantee depends on the ambient dimension $D$ and is determined by the exponent $\gamma$. When the reach $\reach \geq 1$, we can show that $\gamma$ is upper bounded by $\beta$. This indicates that for relatively smoothly varying manifolds, the dependence on the ambient dimension scales at most $D^{2\beta+d+2}$---a polynomial dependence. On the other hand, when $\tau < 1$, $\gamma$ grows as $\tau$ decreases, potentially leading to an exponential dependence on $D$ for extremely small reach. Our bound improves the ambient dimension dependence in \cite{pmlr-v238-tang24a} and explicitly characterizes the interplay between approximation error and the manifold curvature.

\subsection{Large and Small Noise Analysis}\label{sec:approx1}
Theorem \ref{thm:approx} builds upon separate analyses for large noise and small noise regimes. To construct the network class $\cF$, we first construct sub-networks for the large noise and small noise regimes respectively, and then incorporate them by a time switching network. The network architecture is depicted in Figure~\ref{fig:nn}. We highlight the respective approximation guarantees in each regime.

\begin{figure}[htb!]
    \centering
    \includegraphics[width=0.97\linewidth]{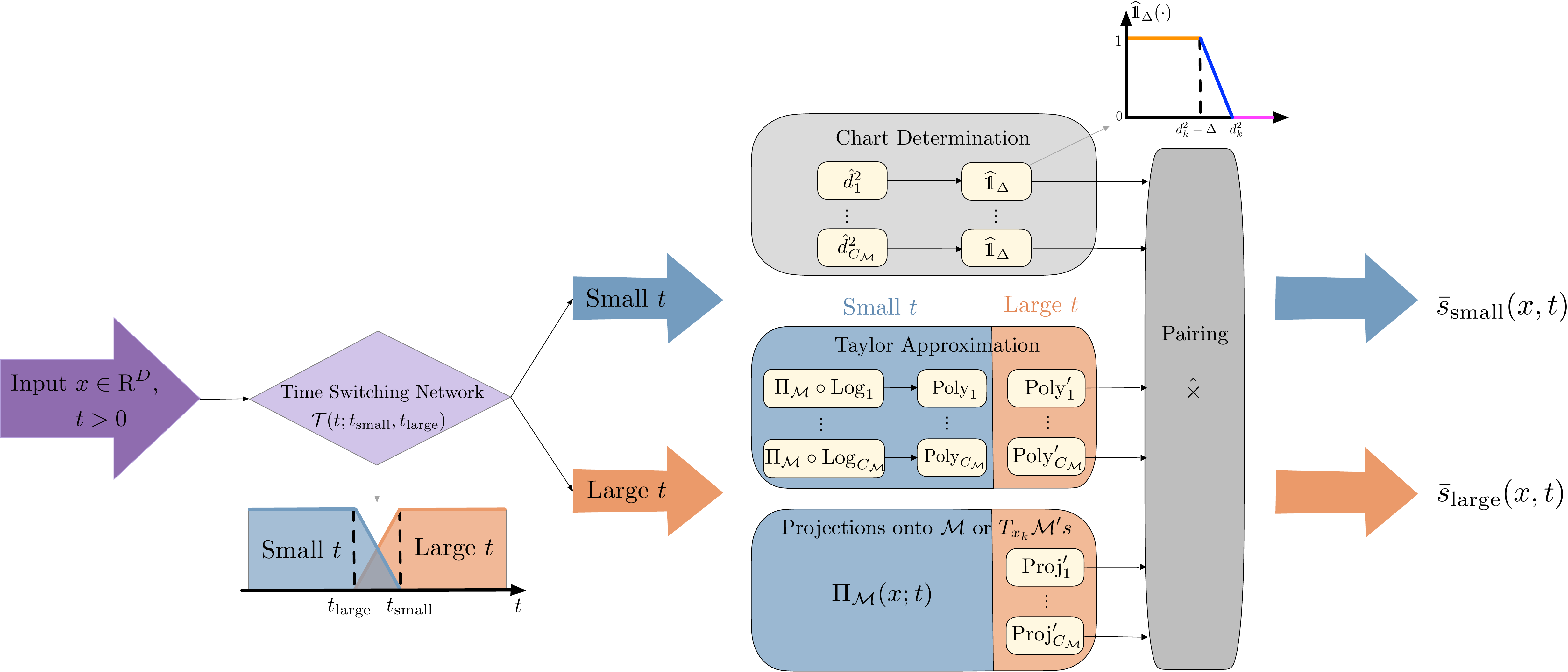}
    \caption{Illustration of the neural network architecture. A time switching network aggregates sub-networks for small noise and large noise respectively. Both regimes involve Taylor approximation of on-support scores, and projection approximation as part of orthogonal scores. These approximators are implemented by neural networks.}
    \label{fig:nn}
\end{figure}

\paragraph{Large Noise} Inspired by Lemma~\ref{lemma:score-decomp-large}, we construct a network consisting of parallel structures for representing the decomposition in each chart. We summarize the result in the following lemma, whenever $t \geq t_{\rm large}$, a threshold for the large noise regime.
\begin{lemma}[Large noise score approximation]\label{lemma:approx-large}
Suppose Assumptions \ref{assump:manifold}-\ref{assump:density} hold. Given an approximation error $\epsilon  \in (0,1)$, there exists a network class $\cF_{\rm large}(L_{\rm large}, W_{\rm large},S_{\rm large},B_{\rm large},\cdot)$  that yields a function $\nnlarge$, such that for any time $t \in [\tlarge,\bigt]$ with $\tlarge=\log \frac{1}{1-\epsilon^{2/\holder}/4}$ and $T = \cO(\log(1/\epsilon) )$, we have
\begin{align*}
\left\|\nnlarge(x,t) - \nabla \log p_t(x) \right\|^2_{L^2(P_t)} = \tilde{\cO} \left( \frac{1}{h_t}\left(\left(\frac{ \epsilon^{2/\holder} }{\reach^2 }\right)^{\degree_0} +\epsilon^{2(\holder+1)/\holder} \right)\right),
\end{align*}
which holds for any $\degree_0>0$.
The network configuration satisfies
\begin{align*}
 &\hspace{0.7in} L_{\rm large} = \tilde{\cO}\left(1\right),  \quad W_{\rm large} = \tilde{\cO}\left(  D^{\degree_0}  \epsilon^{-d/\holder}  \right), \\
&S_{\rm large} = \tilde{\cO}\left( D^{\degree_0}  \epsilon^{-d/\holder} \right), \quad  B_{\rm large} = \cO\left(\epsilon^{-2(\degree_0+\log(1/\epsilon))/\holder}\right).
\end{align*}

\end{lemma}

The proof is provided in Appendix \ref{sec:approx-large}. 
We observe that the network size scales with the intrinsic dimension $d$. Moreover, the approximation error depends on the curvature of the manifold. More specifically, a smaller reach implies more rapidly varying local geometry, which increases the approximation error through the curvature-dependent factor $(\epsilon^{2/\beta}/\tau^2)^{\gamma_0}$. In addition, the network architecture comprises parallel sub-networks corresponding to individual charts, and a smaller reach necessitates finer local approximations that collectively enlarge the network. Compared to the approximate rate in Theorem~\ref{thm:approx}, we note that the large noise regime contributes a minor error term, while the dominating error owes to the small noise regime.

\paragraph{Small Noise} Similar to the large noise regime, we identify a threshold $t_{\rm small}$ for the emergence of the small noise regime. The approximation guarantee is presented in the following lemma.

\begin{lemma}[Small noise score approximation]\label{lemma:approx-small}
Suppose Assumptions \ref{assump:manifold}-\ref{assump:density} hold. For any sufficiently small $\epsilon >0$, there exists a network class $\cF_{\rm small}(L_{\rm small},W_{\rm small},S_{\rm small},B_{\rm small},\cdot)$ that yields a function $\nnsmall$,
such that for all time $t \in [t_0, \tsmall]$, with $t_0 = {\rm poly}(n^{-1})$ and $\tsmall$ satisfying
\begin{align*}
    h_{\tsmall} = \min \left\{ \epsilon^{2/\holder},  \frac{\min\{\rratio^2\reach^2,1\}}{\max\{D\log(1/\epsilon),  \sqrt{D}B\}} \right\},  
\end{align*}
we have
    \begin{align*}
        \left\|\nnsmall (x,t)  - \nabla \log p_t(x) \right\|_{L^2(P_t)}^2  =\tilde{\cO} \left( \frac{D^{2\degree +d+2}}{h_t} \left(  \left( \frac{\epsilon^{2/\holder}}{\reach^2} \right)^{\degree}  +\epsilon^2\right) \right) ,
    \end{align*}
    which holds for any $\degree>0$. The configuration of $\cF_{\rm small}$ satisfies 
    \begin{align*}
    L_{\rm small} =\tilde{\cO}\left(\degree^3   \right), \quad
    W_{\rm small} = \tilde{\cO}\left( D^{\degree} \degree^3 \epsilon^{-d/\holder}\right), \quad
        S_{\rm small} = \tilde{\cO}\left(\degree D^{\degree}\epsilon^{-d/\holder} \right), \quad B_{\rm small}= \cO\left( h_{t_0}^{-\degree} \epsilon^{-2} \right).
    \end{align*}
\end{lemma}

The proof is provided in Appendix \ref{sec:small-approx}. 
Lemma \ref{lemma:approx-small} guarantees that we can find a network class with size depending on the intrinsic dimension $d$ to efficiently approximate the score function in the small noise regime. 
Moreover, both the time threshold $\tsmall$ and the approximation error depend on the manifold curvature. Specifically, a small reach indicates rapidly changing local geometry, which substantially complicates the approximation of the interaction term $E_2$ in Lemma~\ref{lemma:score-decomp-small}, leading to  more stringent requirements on the time $t$ and greater representation difficulty. When $\epsilon$ is sufficiently small, the time threshold becomes $\tsmall = \log \frac{1}{1-\epsilon^{2/\holder}}$, dominated by the approximation error.

Compared to the large noise regime, with a comparable number of nonzero parameters, neural network approximation suffers from an enlarged error. This is due to the finer approximation to the manifold in the small noise regime, instead of using coarse piecewise linear approximations based on tangent spaces. Consequently, when combining the large and small noise regimes, the small noise error is dominating.

The large and small noise regimes overlap in Lemmas~\ref{lemma:approx-large} and \ref{lemma:approx-small} with $t_{\rm large} < t_{\rm small}$, as shown in Figure \ref{fig:nn}. This indicates that when $t \in (t_{\rm large}, t_{\rm small})$, either of the score decompositions can be effectively approximated by neural networks. More importantly, the overlap introduces a ``grace period'' when switching from the large noise regime to the small noise one, where we linearly interpolate between the two regimes. This can be exactly realized by ReLU networks; see details in Section~\ref{sec:proof-sketch}.

\subsection{Proof Sketch of Theorem~\ref{thm:approx}}\label{sec:proof-sketch}

The complete proof of Theorem~\ref{thm:approx} is provided in Appendix \ref{sec:proof-approx}, which boils down to developing neural network approximation theories for the large and small noise regimes, respectively. For either the large or small noise regimes, the proof is constructive, which consists of two major components: \textbf{Step 1. Local polynomial approximation} to the on-support and orthogonal scores characterized in Section~\ref{sec:decomp}, and \textbf{Step 2. Implementing local polynomials by neural networks}. After approximating the large and small noise scores separately, we complete the proof by {\bf Step 3. Introducing a time-switching network} to actively select large and small noise score approximations. 

We focus on the more challenging small noise regime to showcase {\bf Steps 1} and {\bf 2}. The large noise regime replicates similar arguments. Recall from Lemma~\ref{lemma:score-decomp-small} that the on-support score $s_{\cM}$ can be written as $$s_{\cM}(x, t) = \nabla \log \int \exp\left(-\frac{1}{h_t} (E_1(t) + E_2(t))\right) \diff \dist(x_0).$$
Expanding the gradient, we have
\begin{align*}
     s_{\cM}(x,t) 
      = \frac{s_2(x,t)/\sqrt{h_t}}{s_1(x,t)},
\end{align*}
where
\begin{align*}
     s_1(x,t) &= \int_{x_0 \in \cM} \exp \left( - \frac{ E_1(t)+E_2(t)}{2h_t} \right) \ud \dist(x_0) \quad \text{and} \\
     s_2(x,t) & = \int_{x_0 \in \cM} \frac{\alpha_t x_0 - \projm(x,t)}{\sqrt{h_t}}\exp\left( - \frac{E_1(t)+E_2(t) }{2h_t} \right)\ud \dist(x_0).
\end{align*}

It suffices to approximate $s_1$ and $s_2$ separately. Given their similarity, we demonstrate an approximation to $s_1$ for presentation; the same arguments apply to $s_2$.

\noindent $\bullet$ \textbf{Step 1. Local polynomial approximation.}
We consider the atlas in Remark~\ref{remark:exp_atlas} on $\cM$ and adopt the decomposition of $\dist$ into local measures in Remark~\ref{remark:decomp_Q} to cast $s_1$ into
\begin{align} \label{eq:s1-decomp}
\begin{split}
    s_1(x,t)  = \sum_{k=1}^{C_\cM} \int_{\dball(0,r)} &\exp \bigg( - \frac{1}{2h_t} \underbrace{\|\projm(x,t)- \alpha_t \expp_k(v)\|^2}_{E_1(t)} \bigg) \\
    &\cdot \exp \bigg( - \frac{1}{2h_t} \underbrace{2 \langle x - \projm(x,t), \projm(x,t)- \alpha_t \expp_k(v)\rangle}_{E_2(t)}\bigg)F_k(v) \ud v,
\end{split}
\end{align}
where $F_k(v) = \rho_k (\expp_k(v))\cdot p_{\rm data} (\expp_k(v))\cdot G_k(v)$ is a local density-related function according to \eqref{globaldensity}. Here we expand the expressions of $E_1$ and $E_2$ for future reference. A critical advantage of \eqref{eq:s1-decomp} is the reduction of the original integration over $\cM \subset \RR^D$ to local integration over balls in $\RR^d$, manifesting intrinsic dimension dependence.

We further simplify \eqref{eq:s1-decomp} by neglecting minor components for a more efficient approximation. We first show that the magnitude of interaction term $E_2$ can be well-controlled by $E_1$ and the reach of the manifold. In particular, by Lemma~\ref{lemma:cross-term}, it holds that
\begin{align*}
|E_2(t)| \leq \frac{4  \|x-\projm(x,t)\|  }{\alpha_t \reach}E_1(t).
\end{align*}
Therefore, we restrict our attention to $E_1$ being sufficiently small, i.e., $E_1(t) = \tilde{\cO}(h_t)$, so that $s_1$ is not extremely small. Accordingly, within each chart, we consider a subset 
$$\cC_k = \big\{v \in \dball(0, r) : \norm{\alpha_t \expp_k(v) - \projm(x, t)} \leq \Delta\big\}\quad \text{with} \quad \Delta = \cO\big(\sqrt{\alpha_t h_t \log (1/\epsilon)}\big),$$ on which $E_1(t) = \tilde{\cO}(h_t)$. A closer inspection on $\cC_k$ reveals that when the chart center $x_k$ is distant from the given noisy state $x$, $\cC_k$ is empty. To reduce such a redundancy, we identify a selected index set 
\begin{align}\label{eq:def-I(x)}
    \cI(x) = \{k : \norm{x_k - x} \leq L_{\expp_x}r + 2\sqrt{Dh_t \log (1/\epsilon)} + \sqrt{D} Bh_t + \Delta\},
\end{align} where by Lemma~\ref{lemma:localization}, $\cC_k$ asserts an efficient $d$-dimensional representation:
\begin{align}\label{eq:def-vkx}
\cC_k = \dball\big(v_k(x,t), L_{\expp_x}^{-1} \Delta\big),\quad \text{ where } v_k(x,t)=\logg_k(\projm(x, t)/\alpha_t).
\end{align}
To this end, we approximate $s_1$ by
\begin{equation}\label{eq:s1_approx}
\begin{split}
s_1(x, t) \approx \sum_{k\in \cI(x)} \int_{\cC_k} &\exp \bigg( - \frac{1}{2h_t} \underbrace{\|\projm(x,t)- \alpha_t \expp_k(v)\|^2}_{E_1(t)} \bigg) \\
&\cdot \exp \bigg( - \frac{1}{2h_t} \underbrace{2 \langle x - \projm(x,t), \projm(x,t)- \alpha_t \expp_k(v)\rangle}_{E_2(t)}\bigg)F_k(v) \ud v.
\end{split}
\end{equation}
Compared to \eqref{eq:s1-decomp}, we have effectively reduced the integration domain and the number of local charts in \eqref{eq:s1_approx}. The remaining arguments construct polynomials for representing the right-hand side of \eqref{eq:s1_approx}. While the detailed construction is rather technical, we highlight essential building components.
\begin{enumerate}
\item {\bf Local density-related function $F_k$}. We show that $F_k$ is $\beta$-H\"older smooth, inheriting the regularity of the data distribution. Therefore, we use local Taylor polynomials of degree at most $\lfloor \beta \rfloor$ for approximating $F_k$ in Lemma \ref{lemma:Fk-approx-error}; 
\item {\bf Exponential function $\exp(\cdot)$}. Both $E_1$ and $E_2$ are plugged into an exponential function, which is $C^\infty$. We use a Taylor polynomial of degree $\cO(\log (1/\epsilon))$ for approximating the exponential function of $E_1$. A different approach is adopted for $E_2$, where we consider a Taylor polynomial of degree $\gamma$. Note that $E_2$ involves an inner product. When computing powers of $E_2$ in the Taylor polynomial, it is convenient to use tensor inner products. Indeed, for any $j \leq \gamma$, we have $[E_2(t)]^j = \langle [x-\projm(x,t)]^{\otimes j}, [\projm(x,t)-\alpha_t \expp_k(v)]^{\otimes j}\rangle$. Due to the linearity of the tensor inner product, we can extract common factors when summing over $k \in \cI(x)$. Detailed construction of these polynomials are provided in Lemma \ref{lemma:f3-approx-error}.

\item {\bf Exponential map $\expp_k$}. The exponential map $\expp_k$ is also $C^\infty$ on each chart. Different from the previous two components, we use average Taylor polynomials (\citet[Chapter 4.1]{brenner2008mathematical}, see also Definition~\ref{def:avg_poly}) for approximation as detailed in Lemmas \ref{lemma:f5-approx-error} and \ref{lemma:f6-approx-error}. The average Taylor polynomial provides simultaneous accurate approximation to $\expp_k$ and its first-order derivatives. This is critical to ensure that the approximation error of $E_1$ being additionally proportional to $h_t$, leading to a tight approximation error to $s_1$ when $t$ is small.
\end{enumerate}

Substituting the building components into \eqref{eq:s1_approx}, we obtain a polynomial approximation to the integrand, which depends on the noisy state $x$, the projection $\projm(x, t)$, and vector $v$. Integrating against $v$ over the Euclidean ball $\cC_k$ yields another polynomial depending on $x$, $\projm(x, t)$, and center point $v_k(x, t)$, which approximates $s_1$. Note that the integrating tensors in $E_2$ over $\cC_k$ leads to tensors with polynomial entries. We denote them as Tensor-Poly.
The following lemma presents the formal statement.
\begin{lemma}\label{lemma:poly_approx_score}
Suppose Assumptions \ref{assump:manifold}-\ref{assump:density} hold. Consider any time $t \in [t_0, t_{\rm small}]$ with $t_0 = {\rm poly}(n^{-1})$ and $\tsmall = \log \frac{1}{1-\epsilon^{2/\holder}}$. Given a sufficiently small approximation error $\epsilon >0$ and integer $\degree > 0$, consider a polynomial  in $x$, $\projm(x,t)$ and $v_k(x,t)$ defined as
\begin{align*}
    f(x,t) 
     = \sum_{j=0}^{\degree-1} &\frac{1}{ j! h_t^{j/2}}  \bigg \langle [x-\projm(x,t)]^{\otimes j} , \sum_{k \in \cI(x)}  \text{Tensor-Poly}^{k,j}(v_{k}(x,t), h_t, \alpha_t)  \bigg\rangle,
\end{align*}
where $ \text{Tensor-Poly}^{k,j}$ is an order-$j$ tensor with each entry being a polynomial. Then for any $x \in \alpha_t\cK(\cM, 2\sqrt{Dh_t \log(1/\epsilon)}/\alpha_t)$, it holds that
\begin{align*}
    |f(x,t)-s_1(x,t)|  = \tilde{\cO}\left( h_t^{d/2}\max_{j = 0,1,\ldots, \degree} \left\{\frac{\left\| x - \projm(x,t)\right\|^{j}}{(h_t)^{j/2}} \right\}  \left(\frac{D^{(\degree+d)/2}\epsilon^{\degree/\holder}}{\reach^\degree}+\epsilon \right) \right).
\end{align*}

\end{lemma}

\noindent $\bullet$ \textbf{Step 2. Implementing local polynomials by neural networks.}
Given the polynomial approximation $f(x, t)$, we use neural networks to implement it, which is represented as follows, 
\begin{align}\label{eq:nn_approx_s1}
    \bar{s}_1(x,t) 
     = \sum_{j=0}^{\degree-1} \frac{1}{ j!} \nn_{\times}^j \left( \nn_{\rm proj}^j(x, h_t, \alpha_t), \sum_{k=1}^{C_\cM}  \nn_{\times}^1\left(\nn_{\rm{det}}(x-x_k) ,\nn_{\rm poly}^{k,j}(\nn_{v_k}(x,t),h_t,\alpha_t) \right)\right).
\end{align}
There are three key subnetworks: 1) $\nn_{\rm proj}^j$ implements $[x-\projm(x,t)]^{\otimes j}$, 2) $\nn_{\rm det}$ is a chart determination network identifying index in $\cI(x)$, and 3) $\nn_{\rm poly}$ implements the Tensor-Poly. In addition, $\nn_{\times}$ is a network for approximating the multiplication operation. In the sequel, we dive deeper into the construction of these three subnetworks, with full technical details deferred to Appendix \ref{sec:small-approx}.

The three subnetworks take $h_t$, $\alpha_t$, and $v_k(x, t)$ as inputs. Because $h_t = 1-e^{-t}$ and $\alpha_t = e^{-t/2}$ are explicit univariate $C^\infty$ functions, we treat them as known quantities in our analysis.  Their neural network approximations would require only logarithmic network size relative to the approximation accuracy (see \cite[Lemma 3.3]{oko2023diffusionmodelsminimaxoptimal}), contributing negligibly to the overall network complexity. The primary challenge lies in approximating $v_k(x,t)$  defined in \eqref{eq:def-vkx}. The following
proposition proves that $v_k(x,t)$ is $C^\infty$ and thereby can also be efficiently approximated by a network $\nn_{v_k}(x,t)$ as shown in Lemma \ref{lemma:encoder-approx}.

\begin{proposition}\label{prop:proj-smooth}
   Fixing any time $t>0$, the function $v_k(\cdot, t) : \RR^D\to \RR^d$ is $C^\infty$ on $\alpha_t\cK( \cM, \reach)$ for every $x \in \alpha_t\cK( \cM, \reach)$ and every $k \in \cI(x)$.
\end{proposition}
\begin{proof}
    By Lemma \ref{lemma:clip-error}, function $v_k(x, t)$ is well-defined for $x \in \alpha_t\cK( \cM, \reach)$.
    According to \cite{leobacher2020existence}, $\projm(\cdot,t)/\alpha_t$ is a $C^\infty$ function on $\alpha_t\cK( \cM, \reach)$, given that $\cM$ is a smooth manifold. By the definition of atlas in Definition \ref{def:atlas}, the log map $\logg_k$ is $C^\infty$ as well. Therefore, $v_k(\cdot, t)$ is $C^\infty$ as the composition of two $C^\infty$ functions.
\end{proof}
Now we construct three subnetworks based on the elementary network $\nn_{v_k}(x,t)$.
\begin{enumerate}
\item {\bf Projection network $\nn_{{\rm proj}}$}. As shown in Proposition \ref{prop:proj-smooth}, the projection $\projm(x,t)$ is a $C^\infty$-function of $x \in \alpha_t\cK(\cM, \reach)$. Thus each entry of $[x-\projm(x,t)]^{\otimes j}$ is $C^\infty$, which is approximated by $\nn_{{\rm proj}}^j$ entrywise (Lemma \ref{lemma:nn-proj}).

\item {\bf Chart determination network $\nn_{{\rm det}}$}. We identify relevant charts by evaluating the Euclidean distance $\|x-x_k\|$ between the noisy state and the center of each chart. Therefore, we use a neural network to approximate an indicator function $\mathds{1}\{\norm{x - x_k}^2 \leq \thres^2\}$ with $\thres$ given in \eqref{eq:def-I(x)}. This is plausible since the squared Euclidean distance can be well approximated by neural networks. See full details in Lemma~\ref{lemma:nn-chart}.

\item {\bf Tensor-Poly network $\nn_{{\rm poly}}$}. We implement each entry in Tensor-Poly$^{k, j}$ by a neural network. Aggregating these entries yields $\nn_{{\rm poly}}^{k,j}$ (Lemma \ref{lemma:nn-poly}).
\end{enumerate}

Network $\bar{s}_1$ in \eqref{eq:nn_approx_s1} approximates $s_1$. Similarly, we construct a network $\bar{s}_2$ to approximate $s_2$. Then we have $\bar{s}_{\cM} = \bar{s}_2(x,t)/(\sqrt{h_t}\bar{s}_1(x,t))$ as the approximation to the on-support score, where the division is realized by an extra network (Lemma \ref{lemma:approx-small-sm}. Combining $\bar{s}_{\cM}$ with an additive projection network yields the score estimator $\nnsmall$ in the small noise regime. We repeat {\bf Steps 1} and {\bf 2} to construct $\nnlarge$ for the large noise regime in Appendix \ref{sec:approx-large}.

\noindent $\bullet$ {\bf Step 3. Constructing a time-switching network.} To incorporate the large and small noise approximation $\bar{s}_{\rm small}$ and $\bar{s}_{\rm large}$ into a single score network, we implement a time-switching network consisting of two simple ReLU functions. For any $t \in [t_0, T]$, we define two switching function as
\begin{align*}
{\rm SW}_{\rm small}(t) &= \frac{1}{t_{\rm small} -t_{\rm large}}\relu\big( (t_{\rm small} -t_{\rm large}) - \relu(t-t_{\rm large}) + \relu(t-t_{\rm small}) \big), \\
{\rm SW}_{\rm large}(t) &= \frac{1}{t_{\rm small} -t_{\rm large}}\relu\big(  \relu(t-t_{\rm large}) - \relu(t-t_{\rm small}) \big).
\end{align*}
Recall that $t_{\rm small}$ and $t_{\rm large}$ are the thresholds for the large and small noise regimes. The shape of the switching functions is depicted in Figure~\ref{fig:nn}. Using the switching functions, we construct the score network as
\begin{align*}
\tilde{s}(x, t) = {\rm SW}_{\rm small}(t) \cdot  \nnsmall(x,t) + {\rm SW}_{\rm large}(t) \cdot \nnlarge(x,t).
\end{align*}
It can be seen that ${\rm SW}_{\rm small}(t) + {\rm SW}_{\rm large}(t) = 1$ for any $t \in [t_0, T]$. Outside the overlapping interval $[t_{\rm large}, t_{\rm small}]$, one of the switching functions is identically zero. Only within the overlapping interval, does $\tilde{s}(x, t)$ become a convex combination of the large and small noise score networks. It is straightforward to implement the multiplication operation in $\tilde{s}$. This completes the construction of a neural network that approximates the ground-truth score function.

\section{Sample Complexity of Score and Distribution Estimation}\label{sec:stat_rate}
In this section, we provide sample complexity for score function estimation using the network class $\cF$, and further establish distribution estimation guarantees using the estimated score function.

\subsection{Score Estimation}
Recall that we denote $\hat{s}$ as the empirical risk minimizer of denoising score matching in \eqref{eq:erm_score}. By choosing the score network $\cF$ as in Theorem~\ref{thm:approx}, the following theorem establishes the generalization property of $\hat{s}$.

\begin{theorem}\label{thm:generalization}[Score estimation error bound]
     Suppose Assumptions \ref{assump:manifold}-\ref{assump:density} hold. 
     Let early-stopping time $t_0 = n^{-c}$ with some constant $c>0$ and terminal time $\bigt = \cO(\log n)$.
     We choose the network class $\cF = \cF(L,W,S,B,R)$ as in Theorem \ref{thm:approx} with $\epsilon= n^{-\holder/(d+2\beta)}$, and $R= \cO\left( (n^{c} \log n)^{1/2}\right)$. For a sufficiently large $n > \reach^{-(d+2\holder)}$, it holds that
    \begin{align*}
        \frac{1}{{\bigt}-t_0} \EE_{x_0}\left[ \int_{t_0}^{\bigt} \left\| \hat{s}(x,t) - \nabla \log p_t(x) \right\|^2_{L^2(p_t)} \ud t \right] 
        \lesssim D^{ 2\degree+d+2}  n^{-\frac{2\holder}{d+2\holder}},
    \end{align*}
    where $\degree = \lceil \holder (1+\log(\reach^{d+2\holder})/\log n)^{-1}\rceil$.
\end{theorem}

The proof is provided in Appendix \ref{sec:proof-generalization}. 
Theorem \ref{thm:generalization} demonstrates the efficiency of neural networks in learning the score function and their adaptivity to the low-dimensional data structure. We provide two remarks.

\paragraph{Dependence on Dimension} The estimation error of $\hat{s}$ converges at a rate $n^{-\frac{2\beta}{d + 2\beta}}$, depending exponentially only on the intrinsic dimension $d$ and polynomially on the ambient dimension $D$. Familiar readers may find that this estimation error rate matches the mimimax rate of density estimation for $\beta$-H\"{o}lder densities \citep{tsybakov2008introduction}. An intriguing fact about score function estimation is that it inherits the regularity of the clean data distribution, as the score function itself does not directly verify regularity conditions.

\paragraph{Sample Size Depending on Curvature} 
The requirement on the sample size $n > \reach^{-(d+2\beta)}$ is nontrivial when the reach is small, i.e., $\reach < 1$. In the case of $\reach \geq 1$, Theorem~\ref{thm:generalization} holds for any sample size $n$. Consistent with previous discussions, a small $\reach$ induces complicated local structures on the manifold, making it fundamentally challenging for learning. Therefore, the sample size needs to scale with the reach to ensure a plausible estimation.

\subsection{Distribution Estimation}

We transfer the score estimation guarantee to distribution estimation guarantee, since the backward process \eqref{eq:backward_practice} is solely driven by the learned score function:
\begin{align*}
\diff  \hat{Y}_t & = \left[\frac{1}{2} \hat{Y}_t + \hat{s}(\hat{Y}_{t}, T-t)\right] \diff t + \diff \overline{B}_t, \quad \hat{Y}_{0} \sim {\sf N}(0, I_D).
\end{align*}
The backward process is terminated at an early-stopping time $T - t_0$. We denote the estimated distribution $\hat{P}$ as the distribution of $\hat{Y}_{T-t_0}$. The following result bound the deviation of $\hat{P}$ to $\dist$ in Wasserstein distance.
\begin{theorem}[Distribution Recovery Results]\label{thm:distribution}
    Suppose Assumptions \ref{assump:manifold}-\ref{assump:density} hold. Let $t_0 =  n^{-\frac{2(\holder+1)}{d+2\holder}}$, $\bigt= \log n$ and $\degree = \lceil \holder (1+\log(\reach^{d+2\holder})/\log n)^{-1}\rceil$. Then for $n > \reach^{-(d+2\holder)}$, it holds that
    \begin{align*}
        \EE[W_1(\hat{P},P_{\rm data})] \lesssim D^{ \degree+d/2+1}  n^{-\frac{\holder+1}{d+2\holder}}  .
    \end{align*}
\end{theorem}

The proof is given in Appendix \ref{sec:dist}.
The distribution recovery rate matches the minimax optimal rate for learning a $\holder$-H\"older distribution supported on  $[0,1]^d$, which is a special case of $d$-dimensional manifold.
\begin{proposition}[Theorem 3 in \cite{niles2022minimax}] Let $\cP^{\beta}_d$ be a class of distributions satisfying Assumptions \ref{assump:manifold} and \ref{assump:density} with $\cM = [0, 1]^d$. Then for any $\holder>0$ and $d\geq 2$, we have $$ \inf_{\hat{P}_n}\sup_{p \in \cP^{\beta}_d} \EE\left[W_1\left(\hat{P}_n, P \right) \right]  \gtrsim n^{-\frac{\holder+1}{d+2 \holder}},$$
where the infimum is taken over all estimators $\hat{P}_n$ based on $n$ observations.
\end{proposition}

Our theory shows that neural networks efficiently learn the score function, which in turn enables optimal distribution recovery. It validates that the generated samples will faithfully reflect both the geometric structure of the manifold and the distribution of the original data.

\section{Conclusion}\label{sec:conclusion}
In this paper, we have provided a theoretical framework for understanding diffusion models for low-dimensional data. By analyzing the score function's structure and its dependence on the curvature, we have derived approximation and estimation guarantees. Our results demonstrate that diffusion models circumvent the curse of dimensionality by adapting to the underlying structures of the data, with convergence rates governed by intrinsic dimension and manifold curvature. These insights not only explain why diffusion models excel in practice, but also highlight the importance of tailoring architectures to data structure.

\bibliography{ref,ref2}

@article{li2026scoreslearngeometryrate,
      title={When Scores Learn Geometry: Rate Separations under the Manifold Hypothesis}, 
      author={Xiang Li and Zebang Shen and Ya-Ping Hsieh and Niao He},
      year={2026},
      journal={arXiv preprint arXiv:2509.24912}
}

@article{chakraborty2026generalization,
  title={Generalization Properties of Score-matching Diffusion Models for Intrinsically Low-dimensional Data},
  author={Chakraborty, Saptarshi and Berthet, Quentin and Bartlett, Peter L},
  journal={arXiv preprint arXiv:2603.03700},
  year={2026}
}

@inproceedings{li2024towards,
  title={Towards non-asymptotic convergence for diffusion-based generative models},
  author={Li, Gen and Wei, Yuting and Chen, Yuxin and Chi, Yuejie},
  booktitle={The Twelfth International Conference on Learning Representations},
  year={2024}
}

@article{li2024adapting,
  title={Adapting to unknown low-dimensional structures in score-based diffusion models},
  author={Li, Gen and Yan, Yuling},
  journal={Advances in Neural Information Processing Systems},
  volume={37},
  pages={126297--126331},
  year={2024}
}

@article{dou2024optimal,
  title={From optimal score matching to optimal sampling},
  author={Dou, Zehao and Kotekal, Subhodh and Xu, Zhehao and Zhou, Harrison H},
  journal={arXiv preprint arXiv:2409.07032},
  year={2024}
}

@article{huang2024denoising,
  title={Denoising diffusion probabilistic models are optimally adaptive to unknown low dimensionality},
  author={Huang, Zhihan and Wei, Yuting and Chen, Yuxin},
  journal={arXiv preprint arXiv:2410.18784},
  year={2024}
}

@article{pidstrigach2022score,
  title={Score-based generative models detect manifolds},
  author={Pidstrigach, Jakiw},
  journal={Advances in Neural Information Processing Systems},
  volume={35},
  pages={35852--35865},
  year={2022}
}

@article{farghly2025diffusion,
  title={Diffusion models and the manifold hypothesis: Log-domain smoothing is geometry adaptive},
  author={Farghly, Tyler and Potaptchik, Peter and Howard, Samuel and Deligiannidis, George and Pidstrigach, Jakiw},
  journal={arXiv preprint arXiv:2510.02305},
  year={2025}
}

@article{guo2024diffusion,
  title={Diffusion models in bioinformatics and computational biology},
  author={Guo, Zhiye and Liu, Jian and Wang, Yanli and Chen, Mengrui and Wang, Duolin and Xu, Dong and Cheng, Jianlin},
  journal={Nature reviews bioengineering},
  volume={2},
  number={2},
  pages={136--154},
  year={2024},
  publisher={Nature Publishing Group UK London}
}

@article{watson2023novo,
  title={De novo design of protein structure and function with RFdiffusion},
  author={Watson, Joseph L and Juergens, David and Bennett, Nathaniel R and Trippe, Brian L and Yim, Jason and Eisenach, Helen E and Ahern, Woody and Borst, Andrew J and Ragotte, Robert J and Milles, Lukas F and others},
  journal={Nature},
  volume={620},
  number={7976},
  pages={1089--1100},
  year={2023},
  publisher={Nature Publishing Group UK London}
}

@article{austin2021structured,
  title={Structured denoising diffusion models in discrete state-spaces},
  author={Austin, Jacob and Johnson, Daniel D and Ho, Jonathan and Tarlow, Daniel and Van Den Berg, Rianne},
  journal={Advances in neural information processing systems},
  volume={34},
  pages={17981--17993},
  year={2021}
}

@article{nie2025large,
  title={Large language diffusion models},
  author={Nie, Shen and Zhu, Fengqi and You, Zebin and Zhang, Xiaolu and Ou, Jingyang and Hu, Jun and Zhou, Jun and Lin, Yankai and Wen, Ji-Rong and Li, Chongxuan},
  journal={arXiv preprint arXiv:2502.09992},
  year={2025}
}

@article{lou2023discrete,
  title={Discrete diffusion modeling by estimating the ratios of the data distribution},
  author={Lou, Aaron and Meng, Chenlin and Ermon, Stefano},
  journal={arXiv preprint arXiv:2310.16834},
  year={2023}
}

@article{yang2023diffsound,
  title={Diffsound: Discrete diffusion model for text-to-sound generation},
  author={Yang, Dongchao and Yu, Jianwei and Wang, Helin and Wang, Wen and Weng, Chao and Zou, Yuexian and Yu, Dong},
  journal={IEEE/ACM Transactions on Audio, Speech, and Language Processing},
  volume={31},
  pages={1720--1733},
  year={2023},
  publisher={IEEE}
}

@article{kong2020diffwave,
  title={Diffwave: A versatile diffusion model for audio synthesis},
  author={Kong, Zhifeng and Ping, Wei and Huang, Jiaji and Zhao, Kexin and Catanzaro, Bryan},
  journal={arXiv preprint arXiv:2009.09761},
  year={2020}
}

@article{dhariwal2021diffusion,
  title={Diffusion models beat gans on image synthesis},
  author={Dhariwal, Prafulla and Nichol, Alexander},
  journal={Advances in neural information processing systems},
  volume={34},
  pages={8780--8794},
  year={2021}
}

@article{huang2025diffusion,
  title={Diffusion model-based image editing: A survey},
  author={Huang, Yi and Huang, Jiancheng and Liu, Yifan and Yan, Mingfu and Lv, Jiaxi and Liu, Jianzhuang and Xiong, Wei and Zhang, He and Cao, Liangliang and Chen, Shifeng},
  journal={IEEE Transactions on Pattern Analysis and Machine Intelligence},
  year={2025},
  publisher={IEEE}
}

@article{chan2024tutorial,
  title={Tutorial on diffusion models for imaging and vision},
  author={Chan, Stanley},
  journal={Foundations and Trends in Computer Graphics and Vision},
  volume={16},
  number={4},
  pages={322--471},
  year={2024},
  publisher={Emerald Publishing Limited}
}

@article{song2020improved,
  title={Improved techniques for training score-based generative models},
  author={Song, Yang and Ermon, Stefano},
  journal={Advances in neural information processing systems},
  volume={33},
  pages={12438--12448},
  year={2020}
}

@book{brenner2008mathematical,
  title={The mathematical theory of finite element methods},
  author={Brenner, Susanne C},
  year={2008},
  publisher={Springer}
}

@inproceedings{
suh2023approximation,
title={Approximation and non-parametric estimation of functions over high-dimensional spheres via deep Re{LU} networks},
author={Namjoon Suh and Tian-Yi Zhou and Xiaoming Huo},
booktitle={The Eleventh International Conference on Learning Representations },
year={2023},
}

@article{niles2022minimax,
  title={Minimax estimation of smooth densities in Wasserstein distance},
  author={Niles-Weed, Jonathan and Berthet, Quentin},
  journal={The Annals of Statistics},
  year={2022},
  publisher={Institute of Mathematical Statistics}
}

@book{lee2018introduction,
  title={Introduction to Riemannian manifolds},
  author={Lee, John M},
  year={2018},
  publisher={Springer}
}

@article{10.1093/imaiai/iaad018,
    author = {Wang, Jie and Chen, Minshuo and Zhao, Tuo and Liao, Wenjing and Xie, Yao},
    title = {A manifold two-sample test study: integral probability metric with neural networks},
    journal = {Information and Inference: A Journal of the IMA},
    year = {2023},
    month = {06},
    doi = {10.1093/imaiai/iaad018},
}

@article{kaplan2020scaling,
  title={Scaling Laws for Neural Language Models},
  author={Kaplan, Jared and McCandlish, Sam and Henighan, Tom and Brown, Tom B and Chess, Benjamin and Child, Rewon and Gray, Scott and Radford, Alec and Wu, Jeffrey and Amodei, Dario},
  journal={arXiv preprint arXiv:2001.08361},
  year={2020}
}

@article{openai2020scaling,
  title={Scaling Laws for Autoregressive Generative Modeling},
  author={Tom Henighan and Jared Kaplan and Mor Katz and Mark Chen and Christopher Hesse and Jacob Jackson and Heewoo Jun and Tom B. Brown and Prafulla Dhariwal and Scott Gray and Chris Hallacy and Benjamin Mann and Alec Radford and Aditya Ramesh and Nick Ryder and Daniel M. Ziegler and John Schulman and Dario Amodei and Sam McCandlish},
  journal={arXiv preprint arXiv:2010.14701},
  year={2020}
}

@article{saharia2022photorealistic,
  title={Photorealistic text-to-image diffusion models with deep language understanding},
  author={Saharia, Chitwan and Chan, William and Saxena, Saurabh and Li, Lala and Whang, Jay and Denton, Emily L and Ghasemipour, Kamyar and Gontijo Lopes, Raphael and Karagol Ayan, Burcu and Salimans, Tim and others},
  journal={Advances in neural information processing systems},
  volume={35},
  pages={36479--36494},
  year={2022}
}

@article{aithal2024understanding,
  title={Understanding hallucinations in diffusion models through mode interpolation},
  author={Aithal, Sumukh K and Maini, Pratyush and Lipton, Zachary and Kolter, J Zico},
  journal={Advances in Neural Information Processing Systems},
  volume={37},
  pages={134614--134644},
  year={2024}
}

@inproceedings{lu2024handrefiner,
  title={Handrefiner: Refining malformed hands in generated images by diffusion-based conditional inpainting},
  author={Lu, Wenquan and Xu, Yufei and Zhang, Jing and Wang, Chaoyue and Tao, Dacheng},
  booktitle={Proceedings of the 32nd ACM International Conference on Multimedia},
  pages={7085--7093},
  year={2024}
}

@article{lu2025towards,
  title={Towards understanding text hallucination of diffusion models via local generation bias},
  author={Lu, Rui and Wang, Runzhe and Lyu, Kaifeng and Jiang, Xitai and Huang, Gao and Wang, Mengdi},
  journal={arXiv preprint arXiv:2503.03595},
  year={2025}
}

@article{ho2020denoising,
  title={Denoising diffusion probabilistic models},
  author={Ho, Jonathan and Jain, Ajay and Abbeel, Pieter},
  journal={Advances in neural information processing systems},
  volume={33},
  pages={6840--6851},
  year={2020}
}

@article{anderson1982reverse,
  title={Reverse-time diffusion equation models},
  author={Anderson, Brian DO},
  journal={Stochastic Processes and their Applications},
  volume={12},
  number={3},
  pages={313--326},
  year={1982},
  publisher={Elsevier}
}

@article{haussmann1986time,
  title={Time reversal of diffusions},
  author={Haussmann, Ulrich G and Pardoux, Etienne},
  journal={The Annals of Probability},
  pages={1188--1205},
  year={1986},
  publisher={JSTOR}
}

@article{croitoru2023diffusion,
  title={Diffusion models in vision: A survey},
  author={Croitoru, Florinel-Alin and Hondru, Vlad and Ionescu, Radu Tudor and Shah, Mubarak},
  journal={IEEE Transactions on Pattern Analysis and Machine Intelligence},
  volume={45},
  number={9},
  pages={10850--10869},
  year={2023},
  publisher={IEEE}
}

@article{yang2023diffusion,
  title={Diffusion models: A comprehensive survey of methods and applications},
  author={Yang, Ling and Zhang, Zhilong and Song, Yang and Hong, Shenda and Xu, Runsheng and Zhao, Yue and Zhang, Wentao and Cui, Bin and Yang, Ming-Hsuan},
  journal={ACM Computing Surveys},
  volume={56},
  number={4},
  pages={1--39},
  year={2023},
  publisher={ACM New York, NY, USA}
}

@inproceedings{rombach2022high,
  title={High-resolution image synthesis with latent diffusion models},
  author={Rombach, Robin and Blattmann, Andreas and Lorenz, Dominik and Esser, Patrick and Ommer, Bj{\"o}rn},
  booktitle={Proceedings of the IEEE/CVF conference on computer vision and pattern recognition},
  pages={10684--10695},
  year={2022}
}

@article{tang2024score,
  title={Score-based Diffusion Models via Stochastic Differential Equations--a Technical Tutorial},
  author={Tang, Wenpin and Zhao, Hanyang},
  journal={arXiv preprint arXiv:2402.07487},
  year={2024}
}

@article{chen2024opportunities,
  title={Opportunities and challenges of diffusion models for generative AI},
  author={Chen, Minshuo and Mei, Song and Fan, Jianqing and Wang, Mengdi},
  journal={National Science Review},
  volume={11},
  number={12},
  pages={nwae348},
  year={2024},
  publisher={Oxford University Press}
}

@article{song2019generative,
  title={Generative modeling by estimating gradients of the data distribution},
  author={Song, Yang and Ermon, Stefano},
  journal={Advances in neural information processing systems},
  volume={32},
  year={2019}
}

@article{de2022convergence,
  title={Convergence of denoising diffusion models under the manifold hypothesis},
  author={De Bortoli, Valentin},
  journal={arXiv preprint arXiv:2208.05314},
  year={2022}
}

@article{azangulov2024convergence,
  title={Convergence of diffusion models under the manifold hypothesis in high-dimensions},
  author={Azangulov, Iskander and Deligiannidis, George and Rousseau, Judith},
  journal={arXiv preprint arXiv:2409.18804},
  year={2024}
}

@article{yakovlev2025generalization,
  title={Generalization error bound for denoising score matching under relaxed manifold assumption},
  author={Yakovlev, Konstantin and Puchkin, Nikita},
  journal={arXiv preprint arXiv:2502.13662},
  year={2025}
}

@InProceedings{pmlr-v238-tang24a,
  title = 	 {Adaptivity of Diffusion Models to Manifold Structures},
  author =       {Tang, Rong and Yang, Yun},
  booktitle = 	 {Proceedings of The 27th International Conference on Artificial Intelligence and Statistics},
  year = 	 {2024},
  series = 	 {Proceedings of Machine Learning Research},
  publisher =    {PMLR},
}

@article{block2020generative,
  title={Generative modeling with denoising auto-encoders and langevin sampling},
  author={Block, Adam and Mroueh, Youssef and Rakhlin, Alexander},
  journal={arXiv preprint arXiv:2002.00107},
  year={2020}
}

@article{chen2022sampling,
  title={Sampling is as easy as learning the score: theory for diffusion models with minimal data assumptions},
  author={Chen, Sitan and Chewi, Sinho and Li, Jerry and Li, Yuanzhi and Salim, Adil and Zhang, Anru R},
  journal={arXiv preprint arXiv:2209.11215},
  year={2022}
}

@inproceedings{lee2023convergence,
  title={Convergence of score-based generative modeling for general data distributions},
  author={Lee, Holden and Lu, Jianfeng and Tan, Yixin},
  booktitle={International Conference on Algorithmic Learning Theory},
  year={2023},
  organization={PMLR}
}

@article{chen2023probability,
  title={The probability flow ode is provably fast},
  author={Chen, Sitan and Chewi, Sinho and Lee, Holden and Li, Yuanzhi and Lu, Jianfeng and Salim, Adil},
  journal={Advances in Neural Information Processing Systems},
  volume={36},
  pages={68552--68575},
  year={2023}
}

@article{benton2023nearly,
  title={Nearly $ d $-linear convergence bounds for diffusion models via stochastic localization},
  author={Benton, Joe and De Bortoli, Valentin and Doucet, Arnaud and Deligiannidis, George},
  journal={arXiv preprint arXiv:2308.03686},
  year={2023}
}

@InProceedings{pmlr-v202-chen23o,
  title = 	 {Score Approximation, Estimation and Distribution Recovery of Diffusion Models on Low-Dimensional Data},
  author =       {Chen, Minshuo and Huang, Kaixuan and Zhao, Tuo and Wang, Mengdi},
  booktitle = 	 {Proceedings of the 40th International Conference on Machine Learning},
  pages = 	 {4672--4712},
  year = 	 {2023}
}

@misc{oko2023diffusionmodelsminimaxoptimal,
      title={Diffusion Models are Minimax Optimal Distribution Estimators}, 
      author={Kazusato Oko and Shunta Akiyama and Taiji Suzuki},
      year={2023},
      eprint={2303.01861},
      archivePrefix={arXiv},
}

@misc{chen2022distribution,
      title={Distribution Approximation and Statistical Estimation Guarantees of Generative Adversarial Networks}, 
      author={Minshuo Chen and Wenjing Liao and Hongyuan Zha and Tuo Zhao},
      year={2022},
      eprint={2002.03938},
      archivePrefix={arXiv},
      primaryClass={id='cs.LG' full_name='Machine Learning' is_active=True alt_name=None in_archive='cs' is_general=False description='Papers on all aspects of machine learning research (supervised, unsupervised, reinforcement learning, bandit problems, and so on) including also robustness, explanation, fairness, and methodology. cs.LG is also an appropriate primary category for applications of machine learning methods.'}
}

@article{mei2025deep,
  title={Deep networks as denoising algorithms: Sample-efficient learning of diffusion models in high-dimensional graphical models},
  author={Mei, Song and Wu, Yuchen},
  journal={IEEE Transactions on Information Theory},
  year={2025},
  publisher={IEEE}
}

@article{chen2022nonparametric,
  title={Nonparametric regression on low-dimensional manifolds using deep ReLU networks: Function approximation and statistical recovery},
  author={Chen, Minshuo and Jiang, Haoming and Liao, Wenjing and Zhao, Tuo},
  journal={Information and Inference: A Journal of the IMA},
  volume={11},
  number={4},
  pages={1203--1253},
  year={2022},
  publisher={Oxford University Press}
}

@misc{leobacher2020existence,
      title={Existence, Uniqueness and Regularity of the Projection onto Differentiable Manifolds}, 
      author={Gunther Leobacher and Alexander Steinicke},
      year={2020},
      eprint={1811.10578},
      archivePrefix={arXiv},
      primaryClass={math.DG}
}

@article{hyvarinen2005estimation,
  title={Estimation of non-normalized statistical models by score matching.},
  author={Hyv{\"a}rinen, Aapo and Dayan, Peter},
  journal={Journal of Machine Learning Research},
  volume={6},
  number={4},
  year={2005}
}

@inproceedings{song2020sliced,
  title={Sliced score matching: A scalable approach to density and score estimation},
  author={Song, Yang and Garg, Sahaj and Shi, Jiaxin and Ermon, Stefano},
  booktitle={Uncertainty in Artificial Intelligence},
  pages={574--584},
  year={2020},
  organization={PMLR}
}

@article{AL,
	title        = {Nonasymptotic rates for manifold, tangent space and curvature estimation},
	author       = {Eddie Aamari and Cl{\'e}ment Levrard},
	year         = 2019,
	journal      = {The Annals of Statistics}
}

@article{article,
	title        = {Starshaped sets},
	author       = {Hansen, G. and Herburt, Irmina and Martini, H. and Moszyńska, M.},
	year         = 2020,
	month        = 12,
	journal      = {Aequationes mathematicae},
	volume       = 94,
	doi          = {10.1007/s00010-020-00720-7}
}

@book{Conway:1987:SLG:39091,
	title        = {Sphere-packings, Lattices, and Groups},
	author       = {Conway, J. H. and Sloane, N. J. A. and Bannai, E.},
	year         = 1987,
	publisher    = {Springer-Verlag},
	address      = {Berlin, Heidelberg},
	isbn         = {0-387-96617-X}
}

@book{EG,
	title        = {Measure theory and fine properties of functions},
	author       = {Evans, L. C. and R. F. Gariepy},
	year         = 1992,
	publisher      = {Chapman and Hall/CRC},
    address={Boca Raton, FL}
}

@article{Federer,
	title        = {Curvature measures},
	author       = {Herbert Federer},
	year         = 1959,
	journal      = {Transactions of the AMS},
	pages        = {418--494}
}

@book{flaherty2013riemannian,
	title        = {Riemannian Geometry},
	author       = {Flaherty, F. and do Carmo, M.P.},
	year         = 2013,
	publisher    = {Birkh{\"a}user Boston},
	series       = {Mathematics: Theory \& Applications},
}

@article{goodfellow2014generative,
	title        = {Generative adversarial nets},
	author       = {Goodfellow, Ian and Pouget-Abadie, Jean and Mirza, Mehdi and Xu, Bing and Warde-Farley, David and Ozair, Sherjil and Courville, Aaron and Bengio, Yoshua},
	year         = 2014,
	booktitle    = {Advances in neural information processing systems},
	pages        = {2672--2680},
	eprint       = {1406.2661},
	archiveprefix = {arXiv},
	primaryclass = {stat.ML}
}

@article{kingma2013auto,
	title        = {Auto-Encoding Variational Bayes. ICLR 2014 2014},
	author       = {Kingma, DP and Welling, M},
	year         = 2013,
	journal      = {arXiv preprint arXiv:1312.6114}
}

@article{niyogi2008finding,
	title        = {Finding the homology of submanifolds with high confidence from random samples},
	author       = {Niyogi, Partha and Smale, Stephen and Weinberger, Shmuel},
	year         = 2008,
	journal      = {Discrete \& Computational Geometry},
	publisher    = {Springer},
	volume       = 39,
	number       = {1-3},
	pages        = {419--441}
}

@article{roweis2000nonlinear,
	title        = {Nonlinear dimensionality reduction by locally linear embedding},
	author       = {Roweis, Sam T and Saul, Lawrence K},
	year         = 2000,
	journal      = {science},
	publisher    = {American Association for the Advancement of Science},
	volume       = 290,
	number       = 5500,
	pages        = {2323--2326}
}

@book{S,
	title        = {Optimal Transport for Applied Mathematicians},
	author       = {Filippo Santambrogo}
}

@article{tenenbaum2000global,
	title        = {A global geometric framework for nonlinear dimensionality reduction},
	author       = {Tenenbaum, Joshua B and De Silva, Vin and Langford, John C},
	year         = 2000,
	journal      = {science},
	publisher    = {American Association for the Advancement of Science},
	volume       = 290,
	number       = 5500,
	pages        = {2319--2323}
}

@book{tsybakov2008introduction,
	title        = {Introduction to nonparametric estimation},
	author       = {Tsybakov, Alexandre B},
	year         = 2008,
	publisher    = {Springer Science \& Business Media}
}

@book{tu2010introduction,
	title        = {An Introduction to Manifolds},
	author       = {Tu, L.W.},
	year         = 2010,
	publisher    = {Springer New York},
	series       = {Universitext},
	isbn         = 9781441973993,
	lccn         = 2010936466
}

@book{wasserman2006all,
	title        = {All of nonparametric statistics},
	author       = {Wasserman, Larry},
	year         = 2006,
	publisher    = {Springer Science \& Business Media}
}

@article{yarotsky2017error,
	title        = {Error bounds for approximations with deep ReLU networks},
	author       = {Yarotsky, Dmitry},
	year         = 2017,
	journal      = {Neural Networks},
	publisher    = {Elsevier},
	volume       = 94,
	pages        = {103--114}
}

@article{laurent2000adaptive,
  title={Adaptive estimation of a quadratic functional by model selection},
  author={Laurent, Beatrice and Massart, Pascal},
  journal={Annals of statistics},
  pages={1302--1338},
  year={2000},
  publisher={JSTOR}
}

@inproceedings{esser2024scaling,
  title={Scaling rectified flow transformers for high-resolution image synthesis},
  author={Esser, Patrick and Kulal, Sumith and Blattmann, Andreas and Entezari, Rahim and M{\"u}ller, Jonas and Saini, Harry and Levi, Yam and Lorenz, Dominik and Sauer, Axel and Boesel, Frederic},
  booktitle={Forty-first International Conference on Machine Learning},
  year={2024}
}

@inproceedings{songdenoising,
  title={Denoising Diffusion Implicit Models},
  author={Song, Jiaming and Meng, Chenlin and Ermon, Stefano},
  booktitle={International Conference on Learning Representations},
  year={2021}
}

@article{kim2021soft,
  title={Soft truncation: A universal training technique of score-based diffusion model for high precision score estimation},
  author={Kim, Dongjun and Shin, Seungjae and Song, Kyungwoo and Kang, Wanmo and Moon, Il-Chul},
  journal={arXiv preprint arXiv:2106.05527},
  year={2021}
}

@article{schmidt2017nonparametric,
  title={Nonparametric regression using deep neural networks with ReLU activation function},
  author={Schmidt-Hieber, Johannes},
  journal={The Annals of Statistics},
  volume={48},
  number={4},
  pages={1875--1897},
  year={2020},
  publisher={Institute of Mathematical Statistics}
}

@inproceedings{suzuki2018adaptivity,
  title={Adaptivity of deep Re{LU} network for learning in Besov and mixed smooth Besov spaces: optimal rate and curse of dimensionality},
  author={Taiji Suzuki},
  booktitle={International Conference on Learning Representations},
  year={2019}}
\bibliographystyle{plainnat}

\newpage
\appendix
\section{Proofs in Section~\ref{sec:decomp}}
We present the proofs of score decomposition for both large noise and small noise regimes.

\subsection{Proof of Lemma \ref{lemma:score-decomp-large}}\label{pf:score_decomp_large}
Lemma \ref{lemma:score-decomp-large} decomposes the 
score function via projections of the input $x\in \RR^D$ onto tangent spaces of $\cM$.  
The proof manipulates on the marginal density function $p_t$ via the conditional transition kernel in the forward process. Conditioned on the initial clean data $x_0 \in \cM$, at time $t$, the marginal distribution of the noisy state is Gaussian ${\sf N}(\alpha_tx_0, h_tI_D)$. Therefore, the marginal density function $p_t$ satisfies the following display, 
\begin{align*}
     p_t(x) =  (2 \pi h_t)^{-D/2}\int_{x_0 \in \cM} \exp \left( -\frac{\|x - \alpha_t x_0\|^2}{2h_t} \right) \ud \dist(x_0) .
\end{align*}
We consider the atlas on $\cM$ in Remark \ref{remark:exp_atlas}. Then we use the associated partition of unity $\{\rho_k\}_{k=1}^{C_\cM}$ to rewrite $p_t$ and $\nabla p_t$ as follows:
\begin{align}\label{eq:pt}
    p_t(x) =  (2 \pi h_t)^{-D/2}\sum_{k=1}^{C_\cM} \int_{x_0 \in U_k} \exp \left( -\frac{\|x - \alpha_t x_0\|^2}{2h_t} \right) \rho_k(x_0) \pdata(x_0) \ud \vol(x_0) 
\end{align}
and 
\begin{align*}
    \nabla p_t(x) &=  (2 \pi h_t)^{-D/2}\int_{x_0 \in \cM} -\frac{x-\alpha_tx_0}{h_t}\exp \left( -\frac{\|x - \alpha_t x_0\|^2}{2h_t} \right) \ud \dist(x_0) \notag \\
    & =  (2 \pi h_t)^{-D/2}\sum_{k=1}^{C_\cM} \int_{x_0 \in U_k}-\frac{x-\alpha_tx_0}{h_t} \exp \left( -\frac{\|x - \alpha_t x_0\|^2}{2h_t} \right) \rho_k(x_0) \pdata(x_0) \ud \vol(x_0).
\end{align*}
Utilizing the projection $\projk$ onto the $k$-th tangent space, we have 
\begin{align}
    \nabla p_t(x)
    & =  (2 \pi h_t)^{-D/2}\sum_{k=1}^{C_\cM} \int_{x_0 \in U_k}-\frac{x-\projk(x,t)}{h_t} \exp \left( -\frac{\|x - \alpha_t x_0\|^2}{2h_t} \right) \rho_k(x_0) \pdata(x_0) \ud \vol(x_0)\notag \\
    &\quad   +(2 \pi h_t)^{-D/2}\sum_{k=1}^{C_\cM} \int_{x_0 \in U_k} \frac{\alpha_t x_0- \projk(x,t)}{h_t}\exp \left( -\frac{\|x - \alpha_t x_0\|^2}{2h_t} \right) \rho_k(x_0) \pdata(x_0) \ud \vol(x_0). \label{eq:grad-pt}
\end{align}
Substituting \eqref{eq:pt} and \eqref{eq:grad-pt} into $\nabla \log p_t = \nabla p_t / p_t$ gives rise to 
\begin{align}\label{eq:sdecomp}
     \nabla \log p_t(x) &= \sum_{k=1}^{C_\cM} w_k(x,t) \left[-\frac{x-\projk(x,t)}{h_t} + \frac{\alpha_t y_k(x,t) - \projk(x,t)}{h_t} \right],
\end{align}
where 
\begin{align*}
    w_k(x,t) =& \frac{\int_{x_0 \in U_k} \exp \left( -\frac{\|x - \alpha_t x_0\|^2}{2h_t} \right) \rho_k(x_0) \pdata(x_0) \ud \vol(x_0)}{\sum_{j=1}^{C_\cM} \int_{x_0 \in U_j} \exp \left( -\frac{\|x - \alpha_t x_0\|^2}{2h_t} \right) \rho_j(x_0) \pdata(x_0) \ud \vol(x_0)} \\
    =& \frac{\int \exp \left( -\frac{\|x - \alpha_t x_0\|^2}{2h_t} \right) \ud P_{\textrm{data}, k}(x_0)}{\sum_{j=1}^{C_\cM} \int \exp \left( -\frac{\|x - \alpha_t x_0\|^2}{2h_t} \right) \ud P_{\textrm{data}, j}(x_0)}
\end{align*}
and
\begin{align*}
    y_k(x,t) &= \frac{\int_{x_0 \in U_k} x_0 \exp \left( -\frac{\|x - \alpha_t x_0\|^2}{2h_t} \right) \rho_k(x_0) \pdata(x_0) \ud \vol(x_0)}{\int_{x_0 \in U_k} \exp \left( -\frac{\|x - \alpha_t x_0\|^2}{2h_t} \right) \rho_k(x_0) \pdata(x_0) \ud \vol(x_0)}\\
    &= \frac{\int_{x_0 \in U_k} x_0 \exp \left( -\frac{\|x - \alpha_t x_0\|^2}{2h_t} \right) \ud P_{\textrm{data}, k}(x_0)}{\int_{x_0 \in U_k} \exp \left( -\frac{\|x - \alpha_t x_0\|^2}{2h_t} \right) \ud P_{\textrm{data}, k}(x_0)}.
\end{align*}
Moreover, we replace $P_{\textrm{data}, k}$ by its normalized version $\mu_k$ defined in Remark \ref{remark:decomp_Q}. This yields
\begin{align*}
    y_k(x,t) 
    = \frac{\int_{x_0 \in U_k} x_0 \exp \left( -\frac{\|x - \alpha_t x_0\|^2}{2h_t} \right) \ud \mu_k (x_0)}{\int_{x_0 \in U_k} \exp \left( -\frac{\|x - \alpha_t x_0\|^2}{2h_t} \right) \ud \mu_k(x_0)}
     = \EE_{X_0 \sim \mu_k}[X_0 | X_t = x].
\end{align*}
Substituting $y_k(x, t) = \EE_{X_0 \sim \mu_k}[X_0 | X_t = x]$ into \eqref{eq:sdecomp} completes the proof.

\subsection{Proof of Lemma \ref{lemma:score-decomp-small}}\label{pf:score_decomp_small}
For any fixed time $t>0$ and noisy state $x \in \alpha_t \cK(\cM, \reach)$, we represent the score function $\nabla \log p_t(x)$ as
\begin{align}
    \nabla \log p_t(x) = \frac{\nabla p_t(x)}{p_t(x)}  &=  \frac{(2 \pi h_t)^{-D/2}\int_{x_0 \in \cM} - \frac{x-\alpha_tx_0}{h_t} \exp \left( -\frac{\|x - \alpha_t x_0\|^2}{2h_t} \right) \ud \dist(x_0)}{(2 \pi h_t)^{-D/2}\int_{x_0 \in \cM} \exp \left( -\frac{\|x - \alpha_t x_0\|^2}{2h_t} \right) \ud \dist(x_0)} \notag \\
    &= \frac{\int_{x_0 \in \cM} \left(- \frac{x-\projm(x,t)}{h_t}  - \frac{\projm(x,t)-\alpha_tx_0}{h_t}   \right)\exp \left( -\frac{\|x - \alpha_t x_0\|^2}{2h_t} \right) \ud \dist(x_0)}{\int_{x_0 \in \cM} \exp \left( -\frac{\|x - \alpha_t x_0\|^2}{2h_t} \right) \ud \dist(x_0)} \notag \\
    &= - \frac{x-\projm(x,t)}{h_t} + \frac{\int_{x_0 \in \cM} \left(  \frac{\alpha_t x_0 - \projm(x,t)}{h_t}   \right)\exp \left( -\frac{\|x - \alpha_t x_0\|^2}{2h_t} \right) \ud \dist(x_0)}{\int_{x_0 \in \cM} \exp \left( -\frac{\|x - \alpha_t x_0\|^2}{2h_t} \right) \ud \dist(x_0)} \notag\\
    & =  \underbrace{-\frac{x -\Pi_{\cM}(x, t) }{h_t}}_{s_{\perp}:~\text{orthogonal score}}  + \underbrace{\frac{\alpha_t \EE[X_0 | X_t = x] - \projm(x, t)}{h_t}}_{s_{\cM}:~\text{on-support score}} , \label{eq:score-m-decomp}
\end{align}
where in the last equality \eqref{eq:score-m-decomp}, we invoke
\begin{align*}
    \EE[X_0 | X_t = x] = \frac{\int_{x_0 \in \cM}   x_0 \exp \left( -\frac{\|x - \alpha_t x_0\|^2}{2h_t} \right) \ud \dist(x_0)}{\int_{x_0 \in \cM} \exp \left( -\frac{\|x - \alpha_t x_0\|^2}{2h_t} \right) \ud \dist(x_0)}.
\end{align*}
Furthermore, we decompose $p_t$ by
\begin{align*}
p_t(x) & = \int_{x_0 \in \cM} \frac{1}{(2 \pi h_t)^{D/2}}\exp \left( -\frac{\|x-\alpha_tx_0\|^2}{2h_t} \right) \ud \dist(x_0)\\
& = \frac{1}{(2 \pi h_t)^{D/2}} \int_{x_0 \in \cM} \exp \left(  - \frac{ \|x - \projm(x,t) + \projm(x,t) - \alpha_t x_0\|^2 }{2h_t} \right) \ud \dist(x_0)\\
& = \frac{1}{(2 \pi h_t)^{D/2}} \exp \left( -\frac{\|x-\projm(x,t)\|^2}{2h_t} \right) \\ & \quad \cdot \int_{x_0 \in \cM} \exp \left( - \frac{ \|\projm(x,t) - \alpha_t x_0\|^2 + 2 \langle x - \projm(x,t), \projm(x,t) - \alpha_t x_0\rangle }{2h_t} \right) \ud \dist(x_0).
\end{align*}
Here in the last equality, we expand the squared distance and observe that $\projm(x,t)$ is independent of the integrand $x_0$. Taking logarithm and gradient with respect to $x$ on $p_t$ gives rise to
\begin{equation}\label{eq:nabla_log_pt}
\begin{split}
\nabla \log p_t(x) & = \underbrace{\nabla_x  \left( - \frac{ \|x - \projm(x,t)\|^2}{2h_t} \right)}_{(\spadesuit)} + g(x, t),
\end{split}
\end{equation}
where
\begin{align*}
    g(x, t) = \nabla_x \log \int_{x_0 \in \cM}  \exp \left(  \frac{ -\|\projm(x,t) - \alpha_t x_0\|^2 + 2 \langle x - \projm(x,t), \projm(x,t) - \alpha_t x_0\rangle }{2h_t} \right) \ud \dist(x_0).
\end{align*} 
For term $(\spadesuit)$,  we use the standard identity for the squared distance to a set with positive reach \citep{leobacher2020existence} on $\alpha_t\cK(\cM,\reach)$,
\begin{align*}
(\spadesuit) = -\frac{x- \projm(x,t) }{h_t}.
\end{align*}
Substituting $(\spadesuit)$ into \eqref{eq:nabla_log_pt}, we derive
\begin{align*}
    \nabla \log p_t(x) = -\frac{x- \projm(x,t) }{h_t} + g(x,t).
\end{align*}
Comparing with \eqref{eq:score-m-decomp}, we conclude that $g(x, t) = s_\cM(x,t)$, which completes the proof.

\section{Proofs in Section~\ref{sec:approx}}\label{sec:proof-approx}

With the goal of proving Theorem~\ref{thm:approx}, we present the analysis for small noise first in Appendix~\ref{sec:small-approx} and large noise in Appendix~\ref{sec:approx-large}. In Appendix~\ref{sec:approx_thm_proof}, we combine the results in large and small noise regimes to complete the proof of Theorem~\ref{thm:approx}. To ease the presentation, we frequently defer technical lemmas to Appendix~\ref{appendix:lemmas4B}.

\subsection{Small Noise: Proof of Lemma~\ref{lemma:approx-small}}\label{sec:small-approx}

The small noise regime is determined by the following conditions on time $t$,
\begin{align}\label{cond:time}
    h_t &\leq \min \left \{  \epsilon^{2/\holder},\frac{\rratio\reach}{16\sqrt{D}B}, \frac{\min\{\reach,\reach^2\}}{\min\{\reach,\reach^2\}+ 256 \log(1/\epsilon) \max\{D, 16L_\expp^2L_\logg^2 \}/\eta^2 } \right\},
\end{align}
where $\epsilon\in(0,1)$ denotes the desired approximation error and $\eta \in (0,1]$ is specified in Assumption \ref{assump:exp-ball}. The conditions in \eqref{cond:time} identify the small time regime where the variance $h_t$ of the added noise is smaller than the target accuracy $\epsilon$, while simultaneously ensuring that the projections of input points onto the manifold are well-defined. Notably, in the limit of small approximation error $\epsilon \to 0$, the constraints in \eqref{cond:time} are dominated by the condition $h_t \leq  \epsilon^{2/\holder}$.
Conditions in \eqref{cond:time}  imply the following two inequalities:
\begin{align}\label{eq:time-implications}
    2\sqrt{Dh_t\log(1/\epsilon)} \leq \frac{\eta\alpha_t}{8}\min\{\sqrt{\reach}, \reach\} \quad \text{and} \quad \sqrt{D}Bh_t \leq 
\frac{\eta\reach}{16}.
\end{align}
These bounds are used later to ensure that projections from a truncated region onto $\cM$ are well-defined.

Our goal is to construct a network $\nnsmall(x,t)$ that approximates the score function $\nabla\log p_t(x)$ in the $L^2(P_t)$ norm. For a given time $t>0$, we define the truncated region 
$$\cK_t(\delta) = \cK(\alpha_t\cM, 2\sqrt{D h_t \log(1/\delta)}).$$  Our approach involves a dual treatment of $\nnsmall$ depending on the spatial domain: within the region $\cK_t(\delta)$, we design $\nnsmall$ to achieve pointwise approximation in terms of the $L_\infty$ norm, while outside this region, we only require $\nnsmall$ to remain bounded. 
The approximation error can then be decomposed as
\begin{align}\label{eq:truncation}
\begin{split}
    \norm{\nnsmall(\cdot,t) - \nabla \log p_t}_{L^2(P_t)}^2 &\leq \sup_{x \in \cK_t(\delta)}  \norm{\nnsmall(x,t) - \nabla \log p_t(x)}^2  \\
&\quad+ \underbrace{\int_{\RR^D \setminus \cK_t(\delta)} \norm{\nnsmall(x,t)-\nabla \log p_t(x)}^2 \ud P_t(x)}_{=~\tilde{\cO}(\delta^2 / h_t) \text{ by Lemma~\ref{lemma:lowp-bound}}}.
\end{split}
\end{align}
By Lemma \ref{lemma:truncate-x}, we show that $\cK_t(\delta)$ is a high probability region, i.e., $\PP[X_t \in \cK_t(\delta)] \geq 1 - \delta^D$ for $\delta \in (0, e^{-2})$. Accordingly, the second error term on the complement of $\cK_t(\delta)$ only contributes a minor error, which scales as $\tilde{\cO}(\delta^2 / h_t)$ by Lemma~\ref{lemma:lowp-bound}. In the sequel, we focus on bounding the first error term.

In the small noise regime, the radius of $\cK_t(\delta)$ satisfies $2\sqrt{D h_t \log(1/\delta)} \leq \alpha_t \reach$ once we set $\delta=\epsilon$ in \eqref{eq:time-implications}, and thereby the projection of $x \in \cK_t(\delta)$ onto $\alpha_t \cM$ is well-defined \citep{leobacher2020existence}. In this case, Lemma \ref{lemma:score-decomp-small} decomposes the score function $\nabla \log p_t(x)$ into the on-support score $s_\cM(x,t)$ and the orthogonal score $s_\bot(x,t)$. Our proof focuses primarily on approximating the more complicated on-support score $s_\cM(x,t)$. Nonetheless, the orthogonal score $s_{\bot}$ can be approximated using the same components developed for $s_{\cM}$.

We represent the on-support score $s_{\cM}$ by writing it as a fraction:
\begin{align}\label{eq:def-sM}
     s_{\cM}(x,t) 
      = \frac{s_2(x,t)/\sqrt{h_t}}{s_1(x,t)},
\end{align}
where we denote the projection $\xstar = \projm(x,t)/\alpha_t$ and define
\begin{align}\label{eq:def-s1}
     s_1(x,t) &= \int_{x_0 \in \cM} \exp \left( - \frac{ \|\alpha_t \xstar- \alpha_t x_0\|^2 + 2 \langle x - \alpha_t \xstar , \alpha_t \xstar - \alpha_t x_0\rangle }{2h_t} \right) \ud \dist(x_0) 
\end{align}
and
\begin{align*}
     s_2(x,t) = \int_{x_0 \in \cM} -\frac{\alpha_t}{\sqrt{h_t}}(\xstar-x_0)\exp \left( - \frac{ \|\alpha_t \xstar- \alpha_t x_0\|^2 + 2 \langle x - \alpha_t \xstar , \alpha_t \xstar - \alpha_t x_0\rangle }{2h_t} \right) \ud \dist(x_0).
\end{align*}

We approximate $s_1$ and $s_2$ separately, and the approximation to $s_1$ consists of two steps: 1) constructing local polynomials to approximate $s_1$ (Appendix~\ref{subsec:local-poly}), and 2) implementing local polynomials using neural networks (Appendix~\ref{subsec:nn}). The approximation error of $s_1$ is provided in Lemma~\ref{lemma:s1-approx} (Appendix~\ref{sec:approx-error-analysis}). The approximation of $s_2$ follows a similar procedure and is postponed to Lemma~\ref{lemma:s2-approx} (Appendix~\ref{sec:s2-error-analysis}).

\subsubsection{Local Polynomial Construction}\label{subsec:local-poly}

By Remark~\ref{remark:exp_atlas}, we equip $\cM$ with an atlas $\{(U_k, \logg_k) \}_{k=1}^{C_\cM}$, where $U_k = \expp_k(\dball(0,r))$. We choose the radius $r$ to satisfy
\begin{align}\label{eq:cond-r}
    r < \min\{3\reach, \rratio L_\logg \reach, \rratio \reach/(4 L_\expp)  \}.
\end{align}
Here $L_\expp = \max_{k} L_{\expp_k} $ for $L_{\expp_k}>0$ being the Lipschitz constant of $\expp_k$, and $L_\logg = \max_{k} L_{\logg_k}$ for $L_{\logg_k} >0$ being the Lipschitz constant of $\logg_k$. A precise value of the radius $r$ will be set later according to the approximation accuracy.
The number of charts in the atlas is bounded in Lemma \ref{lemma:total-cover-num}, which reads
\begin{equation}\label{eq:CM-bound}
    C_\cM \leq \frac{L_{\logg}^d T_d}{r^d}\int_\cM \diff \mu_\cM,
\end{equation}
with $T_d$ the thickness of the charts, defined as the average number of charts that contain a point on $\cM$ \citep{Conway:1987:SLG:39091}.

Using the decomposition of $\dist$ in \eqref{globaldensity} with respect to the atlas $\{(U_k, \logg_k)\}_{k=1}^{C_{\cM}}$, we have
\begin{align*}
    s_1(x,t) & = \sum_{k=1}^{C_\cM} \int_{\dball(0,r)} \exp \left( - \frac{ \|\alpha_t \xstar- \alpha_t \expp_k(v)\|^2 + 2 \langle x - \alpha_t \xstar , \alpha_t \xstar - \alpha_t \expp_k(v)\rangle }{2h_t} \right) F_k(v)\ud v,
\end{align*}
where $F_k(v) = \rho_k (\expp_k(v))\cdot p_{\rm data} (\expp_k(v))\cdot G_k(v) $. Now let $\vradius > 0$ and $\thres > 0$ be two truncation radii that will be chosen later. Given an arbitrary $x\in\cK_t(\delta)$, we define the low-dimensional representation $v_k(x,t)$ and the index set $\cI(x)$ as
\begin{equation*}
v_k(x,t) = \logg_k(\xstar) = \logg_k (\projm(x,t)/\alpha_t)\quad \text{and} \quad \cI(x) = \{k \leq C_{\cM} : \|x_k-x\| \leq \thres\}.   
\end{equation*}
For any $x \in \cK_t(\delta)$ and $u, v \in \RR^d$, we also introduce the shorthands:
\begin{align*}
\cT_k(x, v, t) & = \inner{x - \alpha_t x_t^*}{\expp_k(v_k(x, t)) - \expp_k(v)} \quad \text{and} \\
\cD_k(u, v) & = \norm{\expp_k(u) - \expp_k(v)}^2.
\end{align*}
Then we define $f_1$ as
\begin{align}\label{eq:f-localization}
    f_1(x, t) = \sum_{k \in \cI(x)} \int_{\cB^d(v_k(x, t), \Delta(t))} \exp\left(-\frac{\alpha_t}{h_t} \cT_k(x, v, t)\right) \exp\left(-\frac{\alpha_t^2}{2h_t} \cD_k(v_k(x, t), v)\right) F_k(v) \diff v.
\end{align} 
Here, we substitute $u = v_k(x,t)$ into $\cD_k(u, v, t)$. Function $f_1$ resembles a similar formula to $s_1$. Yet, function $f_1$ only integrates over a small ball of radius $\vradius$. Moreover, $f_1$ neglects many charts with a center $x_k$ that is distant from the given noisy state $x$. Thanks to Lemma~\ref{lemma:localization}, we show that $f_1$ can still approximate $s_1$ as long as $\vradius$ and $\thres$ are set suitably. Specifically, for any given $\epsilon_1 \in (0, 1)$, we choose
\begin{align*}
\vradius & = 2 L_{\logg} \sqrt{(h_t/\alpha_t^2)\big(\log(1/\epsilon_1)+(d/2)\log(1/h_t)\big)} \quad \text{and} \\
\thres & = r L_\expp + 2\sqrt{D h_t \log(1/\delta)} +\sqrt{D}B h_t + L_\expp \vradius.
\end{align*}
Then for any $t < t_{\rm small}$, it holds that
\begin{align}\label{eq:f1_error}
\sup_{x \in \cK_t(\delta)} |f_1(x, t) - s_1(x, t)| = \cO(h_t^{d/2} \epsilon_1).
\end{align}
We remark that function $f_1$ possesses simplicity and efficiency for developing a neural network approximation, since it involves less charts.

In order to implement $f_1$ using neural networks, we first construct a polynomial approximation to $f_1$, which boils down to constructing polynomial approximation to components in $f_1$.

\noindent $\bullet$ {\bf 1. Local density-related function $F_k$.}
Recall that $F_k(v) = (\rho_k \cdot p_{\rm data}) \circ \expp_k(v)\cdot G_k(v) $ for $v \in \dball(0,r)$. 
We claim that $F_k$ is a non-negative  $\holder$-H\"older function on $\dball(0, r)$ with a bounded H\"older norm $C_F>0$. To see this, we observe that $p_{\rm data} \in \cH^\holder (\cM)$, while the partition of unity $\rho_k$, the exponential map $\expp_k$, and the square root of the determinant of the Jacobian matrix $G_k$ are all $C^\infty$. It is clear that $C_F$ depends on the bounds on the derivatives of $\rho_k$, $\expp_k$, and $G_k$. Accordingly, we adopt Taylor polynomials of degree $\lfloor \holder \rfloor$ to approximate $F_k$. Let $\theta = [\theta_1, \dots, \theta_d]^\top \in \NN^d$ be a multi-index, and we define
\begin{equation}\label{eq:def-Fk-hat}
    \hat{F}_k(v) = \sum_{| {\theta}|\leq \lfloor \holder \rfloor} a_{\theta}^k v^{\theta} \quad \text{with}\quad a_{\theta}^k = \frac{1}{\theta_1! \cdots \theta_d!} \partial^\theta F_k (0).
\end{equation}
Invoking Theorem 3 in \cite{chen2022nonparametric}, we have
\begin{equation}\label{eq:Fk-approx-error}
    \sup_{v\in \dball(0, r)} |\hat{F}_k(v) -  {F}_k(v)| \leq  C_F d^\holder r^{\holder}.
\end{equation}
This indicates that $\hat{F}_k(v)$ can uniformly approximate $F_k(v)$ up to an arbitrarily small approximation error as long as the radius $r$ is small enough. Therefore, we define $f_2$ by replacing $F_k$ in $f_1$ by $\hat{F}_k$:
\begin{align*}
f_2(x, t) = \sum_{k \in \cI(x)} \int_{\cB^d(v_k(x, t), \Delta(t))} \exp\left(-\frac{\alpha_t}{h_t} \cT_k(x, v, t)\right) \cdot \exp\left(-\frac{\alpha_t^2}{2h_t} \cD_k(v_k(x, t), v)\right) \cdot \hat{F}_k(v) \diff v.
\end{align*}
The approximation error of $f_2$ to $f_1$ is established in Lemma \ref{lemma:Fk-approx-error}. In particular, for any approximation error $\epsilon_2 < (\min\{3\reach, \rratio L_\logg \reach, \rratio \reach/(4 L_\expp)  \})^\holder$, we set the radius of charts to be $$r = \epsilon_2^{1/\holder}.$$ 
This choice of $r$ validates the condition in \eqref{eq:cond-r}. Then for any $x \in \cK_t(\delta)$, we have
\begin{align}\label{eq:f2-f1_error}
|f_2(x,t)-f_1(x,t)| = \cO(h_t^{d/2} |\cI(x)| \epsilon_2).
\end{align}

\noindent $\bullet$ {\bf 2. Exponential function $\exp(\cdot)$.} We approximate the two exponential functions in $f_2$ by their Taylor series expanded around $0$, i.e., $\exp(-a) = \sum_{j=0}^\infty \frac{(-1)^j}{j!} a^j$ for any $a \in \RR$. We truncate the degree at $\gamma$ and $\gamma'$, respectively, which gives rise to
\begin{align*}
f_3(x, t) = \sum_{k \in \cI(x)} \int_{\cB^d(v_k(x, t), \Delta(t))} \left[ \sum_{j=0}^{\degree-1} \frac{(-1)^j}{j!} \frac{(\cT_k(x, v, t))^j}{(h_t/\alpha_t)^j} \right] \cdot \left[ \sum_{l=0}^{\gamma'-1}  \frac{(-1)^l}{2^l l!} \frac{(\cD_k(v_k(x, t), v))^l}{(h_t/\alpha_t^2)^l} \right] \hat{F}_k(v) \diff v.
\end{align*}
Controlling the approximation error of $f_3$ to $f_2$ is relatively simple, as the exponential function is $C^\infty$. Specifically, for any given $\epsilon_3 \in (0, 1)$, we set $\dgM ={\cO} (\log(1/\epsilon_3)+ \log(1/\epsilon_1) + \log(1/h_t))$; recall that $\epsilon_1$ is the approximation error of $f_1$ in \eqref{eq:f1_error}. Then by Lemma \ref{lemma:f3-approx-error}, for any $x \in \cK_t(\delta)$, it holds that
\begin{align}\label{eq:f3-f2_error}
|f_3(x,t)-f_2(x,t)| =\tilde{\cO}\left(h_t^{d/2}\left( \frac{1}{\degree !} \left(  \frac{16 L_\logg^2\sqrt{D h_t \log(1/\delta)} (\log(1/\epsilon_1)+d\log(1/h_t)/2)}{\alpha_t \reach}  \right)^{\degree}  +\epsilon_3 \right)\right).
\end{align}
We leave $\gamma$ unspecified, which will be chosen later depending on the curvature of the manifold.

\noindent $\bullet$ {\bf 3. Exponential map $\expp_k$.} Different from the previous two components, we use average Taylor polynomials to approximate the exponential map $\expp_k$ in $\cT_k$ and $\cD_k$. The use of average Taylor polynomials allows us to simultaneously approximate $\expp_k$ and all its first-order derivatives. For attentive readers, we provide an introduction to average Taylor polynomials in Appendix \ref{sec:avg-taylor-poly}.

It suffices to approximate each coordinate in $\expp_k$. For the $i$-th coordinate map $\expp_{k, i}$ with $ i \in \{1,\ldots,D\}$, we approximate it by an average Taylor polynomial $\hat{\expp}_{k,i}^{S}$ of degree $S$ averaged over $\dball(0,r)$. Let $\theta = [\theta_1, \dots, \theta_d]^\top\! \in \NN^d$ be a multi-index again. We define 
\begin{align}\label{eq:exp-approx-form}
	\hat{\expp}_{k,i}^{S}(v)\overset{(i)}{=}\int_{\dball(0,r)} \left( \sum_{|\theta|<S} \frac{\partial^{\theta}\expp_{k,i}(z)}{|\theta|!} (v-z)^{\theta} \right)\phi(z)dz \overset{(ii)}{=} \sum_{|\theta|<S} c^{k,i}_{\theta} v^\theta,
\end{align}
where equality $(i)$ follows the averaged Taylor polynomial in Definition~\ref{def:avg_poly}, with $\phi(z)$ an arbitrary cut-off function, and quality $(ii)$ implies that the averaged Taylor polynomial can be written as a sum of monomials as established in Lemma \ref{lemma:atp-form}. Now we instantiate $S = \beta + 1$ and define
\begin{align*}
\hat{\Delta}_{\expp_{k, i}}(u, v) = \hat{\expp}_{k, i}^{\beta+1}(u) - \hat{\expp}_{k, i}^{\beta+1}(v)
\end{align*}
as the coordinate-wise difference. Then we approximate $\cD_k(v_k(x, t), v)$ by 
\begin{align*}
\hat{\cD}_k(v_k(x, t), v) = \sum_{i=1}^D \left(\hat{\Delta}_{\expp_{k, i}}(v_k(x, t), v)\right)^2.
\end{align*}
Lemma~\ref{lemma:prop-Dk-hat} proves that $\hat{\cD}_k$ well approximates $\cD_k$ with respect to the Sobolev $W^{1, \infty}$ norm. Substituting $\hat{\cD}_k$ into $f_3$ leads to
\begin{align*}
f_4(x, t) = \sum_{k \in \cI(x)} \int_{\cB^d(v_k(x, t), \Delta(t))} \left[ \sum_{j=0}^{\degree-1} \frac{(-1)^j}{j!} \frac{(\cT_k(x, v, t))^j}{(h_t/\alpha_t)^j} \right] \cdot \left[ \sum_{l=0}^{\gamma'-1}  \frac{(-1)^l}{2^l l!} \frac{(\hat{\cD}_k(v_k(x, t), v))^l}{(h_t/\alpha_t^2)^l} \right] \hat{F}_k(v) \diff v.
\end{align*}
Exploiting the approximation error of $\hat{\cD}_k$ to $\cD_k$ in Lemma~\ref{lemma:prop-Dk-hat}, we show in Lemma~\ref{lemma:f5-approx-error} that  for any $x \in \cK_t(\delta)$, it holds that
\begin{align}\label{eq:f4-f3_error}
|f_4(x,t)-f_3(x,t)| = \tilde{\cO}\big(h_t^{d/2}|\cI(x)| \epsilon_2\big).
\end{align}

Next, we approximate $\cT_k(x, v, t) = \inner{x - \alpha_t x^*_t}{\expp_k(v_k(x, t)) - \expp_k(v)}$ using $\hat{\Delta}_{\expp_{k, i}}$. We collect $\hat{\Delta}_{\expp_{k, i}}$ for $i = 1, \dots, D$ into a $D$-dimensional vector:
\begin{align*}
\hat{\Delta}_{\expp_k}(v_k(x, t), v) = \left[\hat{\Delta}_{\expp_{k, 1}}(v_k(x, t), v), \dots, \hat{\Delta}_{\expp_{k, D}}(v_k(x, t), v)\right]^\top.
\end{align*}
Then $\cT_{k}$ is realized by
\begin{align}
\hat{\cT}_k(x, v, t) = \inner{x - \alpha_t x_t^*}{\hat{\Delta}_{\expp_k}(v_k(x, t), v)}. \nonumber
\end{align}
Note that $f_4$ involves powers of $\cT_k$. We adopt the tensor notation to write 
\begin{align*}
\left(\hat{\cT}_k(x, v, t)\right)^j = \inner{[x - \alpha_t \xstar]^{\otimes j}}{[\hat{\Delta}_{\expp_k}(v_k(x, t), v)]^{\otimes j}}.
\end{align*}
Substituting $(\hat{\cT}_k)^j$ into $f_4$, we derive
\begin{align*}
& \quad~ f_5(x, t) \nonumber \\
& = \sum_{k \in \cI(x)} \int_{\cB^d(v_k(x, t), \Delta(t))} \left[ \sum_{j=0}^{\degree-1} \frac{(-1)^j}{j!} \frac{(\hat{\cT}_k(x, v, t))^j}{(h_t/\alpha_t)^j} \right] \cdot \left[ \sum_{l=0}^{\gamma'-1}  \frac{(-1)^l}{2^l l!} \frac{(\hat{\cD}_k(v_k(x, t), v))^l}{(h_t/\alpha_t^2)^l} \right] \hat{F}_k(v) \diff v  \\
& = \sum_{k \in \cI(x)} \sum_{j=0}^{\gamma-1} \frac{(-1)^j}{j!(h_t)^{j/2}} \left\langle [x - \alpha_t x_t^*]^{\otimes j},  \text{Tensor-Poly}^{k,j}(v_k(x, t), h_t, \alpha_t)\right\rangle,
\end{align*}
where
\begin{align*}
& \text{Tensor-Poly}^{k,j}(v_k(x, t),  h_t, \alpha_t) \\
& \qquad = \int_{\cB^d(v_k(x, t), \Delta(t))}  \frac{[\hat{\Delta}_{\expp_k}(v_k(x, t), v)]^{\otimes j}}{(h_t/\alpha_t^2)^{j/2}} \cdot \left[ \sum_{l=0}^{\gamma'-1}  \frac{(-1)^l}{2^l l!} \frac{(\hat{\cD}_k(v_k(x, t), v))^l}{(h_t/\alpha_t^2)^l} \right] \hat{F}_k(v) \diff v.
\end{align*}
Expressing the chart determination via indicator functions and rearranging terms yields
\begin{align}\label{eq:f5_poly}
 f_5(x, t) 
 =  \sum_{j=0}^{\gamma-1} \frac{1}{j!} \left\langle  \frac{(-1)^j}{h_t^{j/2}}[x - \alpha_t x_t^*]^{\otimes j},  \sum_{k=1}^{C_{\cM}} \ind(k \in \cI(x))\cdot \text{Tensor-Poly}^{k,j}(v_k(x, t), h_t, \alpha_t)\right\rangle.
\end{align}

Notably, the integration in $\text{Tensor-Poly}^{k,j}$ is performed over polynomials within a Euclidean ball centered at $v_k(x,t)$, thus $\text{Tensor-Poly}^{k,j}$ is still a polynomial with respect to $v_k(x,t)$. 
By Lemma \ref{lemma:f6-approx-error}, we establish the approximation error: for any $x \in \cK_t(\delta)$, we have 
\begin{align}\label{eq:f5-f4_error}
    |f_5(x,t)-f_4(x,t)| = \tilde{\cO} \left ( D^{\degree/2} h_t^{d/2} \max_{ j  = 0,1,\ldots, \degree} \left\{\frac{\left\| x - \alpha_t \xstar\right\|^{j}}{(h_t)^{j/2}} \right\} \epsilon_2 \right).
\end{align} 
Finally, we can conclude the result in Lemma \ref{lemma:poly_approx_score} by summing up the approximation errors in \eqref{eq:f1_error}, \eqref{eq:f2-f1_error}, \eqref{eq:f3-f2_error}, \eqref{eq:f4-f3_error} and \eqref{eq:f5-f4_error}. Letting $\epsilon_1 =\epsilon_2=\epsilon_3=\delta= \epsilon$, for any $x\in\cK_t(\epsilon)$, we have 
    \begin{align}\label{eq:approx-pre-result}
        |f_5(x,t) - s_1(x,t)| &\lesssim
         D^{\degree/2} h_t^{d/2} |\cI(x)|\max_{j = 0,1,\ldots, \degree} \left\{\frac{\left\| x - \alpha_t \xstar\right\|^{j}}{(h_t)^{j/2}} \right\}  \notag  \\
        & \quad \cdot \left( \frac{1}{\degree !} \left(  \frac{\sqrt{ \log(1/\epsilon)} (\log(1/\epsilon)+d\log(1/h_{t})/2)}{\alpha_t \reach}  \right)^{\degree}  \epsilon^{\degree/\holder}+\epsilon\right),
     \end{align}
where $\lesssim$ hides the dependence on $d,\holder,T_d, C_F$, the upper bounds on the derivatives of exponential maps, and logarithmic factors. 
Applying Lemma \ref{lemma:total-cover-num} on the region $\{z\in\cM: \|z-x\|\leq \thres+r\}$, we can bound the number of selected charts $|\cI(x)|$ by 
\begin{equation*}
       |\cI(x)| \leq \frac{L_{\logg}^d T_d}{r^d}\int_{z\in\cM} \ind(\|z-x\|\leq \thres+r)\diff \mu_\cM(z).
\end{equation*}
Recall our choice of $\thres$:
\begin{align*}
    \thres = L_\expp r + 2\sqrt{D h_t \log(1/\epsilon)} +\sqrt{D}B h_t + 2L_\expp L_\logg \sqrt{(h_t/\alpha_t^2)(\log(1/\epsilon+d\log(1/h_t)/2))}.
\end{align*}
Under the condition in \eqref{cond:time} that $h_t \leq \epsilon^{2/\holder} = r^2$, we have $\thres \lesssim \sqrt{D} r$, where $\lesssim $ hides logarithmic factors and dependence on $d$, $B$ and  the upper bounds on the derivatives of exponential maps. This further yields
\begin{equation*}
       |\cI(x)| \lesssim \frac{ \left( \thres+r\right)^d}{r^d} \lesssim D^{d/2}.
\end{equation*}
Therefore, plugging $|\cI(x)|$ in \eqref{eq:approx-pre-result} yields 
\begin{align*}
        |f_5(x,t) - s_1(x,t)| &\lesssim D^{(\degree+d)/2}h_t^{d/2} \max_{j = 0,1,\ldots, \degree} \left\{\frac{\left\| x - \alpha_t \xstar\right\|^{j}}{(h_t)^{j/2}} \right\}  \left( \frac{(\log(1/\epsilon)+d\log(1/h_t)/2)^{2\degree}}{\reach^{\degree}} \epsilon^{\degree/\holder}  +\epsilon\right).
    \end{align*}

\subsubsection{Neural Network Implementation}\label{subsec:nn}

We construct our neural network based on implementing function $f_5$ in \eqref{eq:f5_poly}. We start with developing two elementary networks: 1) $\nn_{v_k}$ approximates the low-dimensional representation $v_k$, and 2) $\nn_{\times}$ realizes multiplication. We then construct three key networks: 1) projection network $\nn_{\rm proj}$, 2) chart determination network $\nn_{\rm chart}$, and 3) Tensor-Poly network $\nn_{\rm poly}$ to approximate key components in $f_5$.
Accordingly, the constructed network takes the form:
\begin{align}
     \bar{s}_1(x,t) 
     = \sum_{j=0}^{\degree-1} \frac{1}{ j!} \nn_{\times}^j \left( \nn_{\rm proj}^j(x,h_t,\alpha_t), \sum_{k=1}^{C_\cM}  \nn_{\times}^1\left(\nn_{\rm{det}}(x-x_k) ,\nn_{\rm {poly}}^{k,j}(\nn_{v_k}(x,t),h_t,\alpha_t) \right)\right).\label{eq:s1-nn}
\end{align}
In the sequel, we present a detailed construction of the two elementary networks followed by the three key networks.

\noindent {\bf Elementary Network Construction}. We first approximate the low-dimensional representation $v_k(x, t)$ of a noisy state $x$ at time $t$. By definition, 
$ v_k(x,t)$ projects $x$ onto the shrunk manifold $\alpha_t\cM$ first, then maps the projection to the low-dimensional space $\RR^d$. Thanks to Proposition \ref{prop:proj-smooth}, $v_k$ is $C^\infty$ in $x$ and thereby can be efficiently approximated by a network. By Lemma \ref{lemma:encoder-approx}, for any approximation error $\epsilon_v\in(0,1)$, there exists a network $\nn_{v_k} = [\nn_{v_k,1},\nn_{v_k,2},\ldots,\nn_{v_k,d}]^\top$ satisfying
\begin{align}\label{eq:error-nn-v}
    \sup_{x\in\cK_t(\delta)}\left\| \nn_{v_k}(x,t)- v_k(x,t) \right\| \leq \epsilon_v,
\end{align}
with each coordinate map $\nn_{v_k,i} \in \cF(L_v,W_v, S_v,B_v, \cdot)$ for
$$L_v=\cO(\log(1/\epsilon_v)), ~~ W_v=\cO\left(\epsilon_v^{-d/\holder}\log(1/\epsilon_v) \right),~~ S_v=\cO\left(\epsilon_v^{-d/\holder}\log(1/\epsilon_v) \right),~~ B_v=\cO(\max\{B_{\cM},\reach^2\}).$$

We then implement the multiplication operation. Invoking Lemma \ref{lemma:nn-prod}, we construct a network $\nn_{\times}^j \in \cF(L_{\times}, W_{\times}, S_{\times}, B_{\times}, \cdot)$ that realizes the $j$-th order tensor inner product, i.e., given $\epsilon_{\times}\in (0,1)$, for any $u_1,u_2 \in [-2\sqrt{D\log(1/\epsilon)},2\sqrt{D\log(1/\epsilon)}]^D$, it holds that
\begin{align}\label{eq:error-nn-prod}
        \left| \nn_{\times}^j\left(u_1^{\otimes j}, u_2^{\otimes j} \right) - \left\langle u_1^{\otimes j}, u_2^{\otimes j} \right\rangle \right| \leq D^j \epsilon_{\times}.
\end{align}
The corresponding network configuration is
\begin{align*}
       L_{\times} = \cO(\log(D/\epsilon_{\times} )
    ), \quad W_{\times}=\cO(D^j),\quad S_{\times} = \cO(D^j\log(D/\epsilon_{\times} )),\quad \text{and}\quad B_{\times}=\cO(D^{2} \log(1/\epsilon)). 
\end{align*}

\noindent {\bf Key Networks Construction}. It remains to construct three key networks. We begin with the projection network.

\noindent $\bullet$ {\bf 1. Projection network $\nn_{\rm proj}$}.
We construct a network $\nn_{\rm proj}^j$ that implements each entry in the tensor $[x - \alpha_t \xstar]^{\otimes j}$. As a function of input $x$, the tensor $[x - \alpha_t \xstar]^{\otimes j}$ is $C^\infty$ for fixed time $t$, since the projection $\alpha_t x_t^* = \projm(x, t)$ is $C^\infty$ for $x \in \cK_t(\epsilon)$ due to   
Proposition \ref{prop:proj-smooth} and restrictions to the small noise regime \eqref{eq:time-implications}. Therefore, each entry of $[x - \alpha_t \xstar]^{\otimes j}$ can be approximated by a network as follows.
Lemma \ref{lemma:nn-proj} establishes approximation guarantees for $ [x - \alpha_t \xstar]^{\otimes j}$. Given an order $j$ and $\epsilon_{\rm proj}\in (0,1)$, there exists $\nn_{\rm proj}^j$ such that for any  time $t \in [t_0, t_{\rm small}]$,
\begin{align}\label{eq:error-nn-proj}
      \sup_{ x \in \cK_t(\epsilon)}\left\| \nn_{\rm proj}^j(x, h_t, \alpha_t) - \frac{[x - \alpha_t \xstar]^{\otimes j}}{h_t^{j/2}} \right\|_\infty \leq h_{t}^{d/2}\epsilon_{\rm proj}.
  \end{align}
Here $\nn_{\rm proj}^j$ is a network in $\cF(L_{\rm proj}, W_{\rm proj}, S_{\rm proj}, B_{\rm proj}, \cdot)$, where 
  \begin{align*}
      & L_{\rm proj} = \cO \left(\log^3\big(1/(h_{t_0}^{j/2}\epsilon_{\rm proj})\big)\right), \quad W_{\rm proj} = \cO\left(\epsilon_{\rm proj}^{-d/\holder} \log\big(1/(h_{t_0}^{j/2}\epsilon_{\rm proj})\big)  + \log^4(1/h_{t_0})\right), \\
      & S_{\rm proj} = \cO\left(\epsilon_{\rm proj}^{-d/\holder} \log(1/(h_{t_0}^{j/2}\epsilon_{\rm proj}))  + \log^4(1/h_{t_0})\right), \quad \text{and} \quad B = h_{t_0}^{-j}.
  \end{align*}

\noindent $\bullet$ {\bf 2. Chart determination network $\nn_{\rm{det}}$}.  The chart  determination network $\nn_{\rm{det}}$ identifies the index set $\cI(x) =\{k: \norm{x - x_k}\leq \thres\}$. For an index $k \in \cI(x)$, it must verify the condition $\mathds{1}\{\norm{x - x_k}^2 \leq \thres^2\} = 1$. As a result, $\nn_{\rm det}$ implements the indicator function by (i) approximating the squared Euclidean distance using a network $\hat{d}$, and (ii) implementing the indicator function using another network $\hat{\mathds{1}}$. The composition $\nn_{\rm det} = \hat{\mathds{1}}\circ \hat{d}$ yields the desired approximation to the index set.
The detailed network construction and approximation guarantee are provided in Lemma \ref{lemma:nn-chart}. Given $\epsilon_{\rm{det}} \in (0,1)$, $x \in \cK_t(\epsilon)$, and any $k \in [C_{\cM}]$, it holds that 
\begin{align}\label{eq:error-nn-det}
     \nn_{\rm{det}}(x - x_k) =
\begin{cases}
1 & \text{if}~\norm{x - x_k}^2 \leq (1-\epsilon_{\rm det})\thres^2 \\
0 & \text{if}~\norm{x - x_k}^2 \geq \thres^2,
\end{cases}
\end{align}
Here $\nn_{\rm{det}}$ is a network in $\cF(L_{\rm{det}}, W_{\rm{det}}, S_{\rm{det}}, B_{\rm{det}}, \cdot)$, where 
\begin{align*}
    &L_{\rm{det}} = \cO(\log(D/\epsilon)+\log(1/\epsilon_{\rm{det}})), \quad W_{\rm{det}} = \cO(D), \\
    &S_{\rm{det}} = \cO(\log(D/\epsilon)+\log(1/\epsilon_{\rm{det}})),  \quad  \text{and} \quad
    B_{\rm{det}} = \cO(\log(D/\epsilon)+\log(1/\epsilon_{\rm{det}})). 
\end{align*}

\noindent $\bullet$ {\bf 3. Tensor-Poly network $\nn_{\rm {poly}}$}. We implement the entries of $\text{Tensor-Poly}^{k,j}$ in \eqref{eq:f5_poly} by a single network $\nn_{\rm {poly}}^{k,j}$. Since each entry of $\text{Tensor-Poly}^{k,j}$ is a polynomial with a $d$-dimensional input and degree at most $(\holder+1)(j+\dgM+1)$, where we set $\dgM ={\cO} \left(\log(1/(h_{t_0}\epsilon))\right)$, it can be efficiently approximated by a neural network. By Lemma \ref{lemma:nn-poly}, given $\epsilon_{\rm {poly}}\in (0,1)$, there exists a network $\nn_{\rm {poly}}^{k,j}$ such that for any $t \in [t_0, t_{\rm small}]$, it holds that
\begin{align}\label{eq:error-nn-poly}
\sup_{v \in \dball(0,r)}\Big\| \nn_{\rm {poly}}^{k,j}(v,h_t,\alpha_t) - \text{Tensor-Poly}^{k,j}(v,h_t,\alpha_t) \Big\|_\infty \leq  h_{t_0}^{ d/2}\epsilon_{\rm {poly}}.
  \end{align}
Here each entry of $\nn_{\rm {poly}}^{k,j}$ is a network in $\cF(L_{\rm poly}, W_{\rm poly}, S_{\rm poly}, B_{\rm poly}, \cdot)$ with
  \begin{align*}
      &L_{\rm poly} = \cO(\log(1/(h_{t_0}\epsilon_{\rm {poly}}) ), \quad W_{\rm poly} = \cO\left((\degree+\dgM+1)^d \right),
     \\ &S_{\rm poly} =\cO\left((\degree+\dgM+1)^d\log(1/(h_{t_0}\epsilon_{\rm {poly}})) \right), ~~  \text{and}  ~~B_{\rm poly} =\cO(1)
\end{align*}

\subsubsection{Bounding Approximation Error of $\bar{s}_1$}\label{sec:approx-error-analysis}
We state the approximation guarantee of the network implementation $\bar{s}_1$ in the following lemma.
\begin{lemma}[Approximation error of $s_1(x,t)$]\label{lemma:s1-approx}
    Given $\epsilon \in \big(0,(\min\{1,3\reach, \rratio L_\logg \reach, \rratio \reach/(4 L_\expp)  \})^\holder \big)$ and time $t\geq t_0$ satisfying \eqref{cond:time}, let $r = \epsilon^{1/\holder}$.
    Then for any fixed $\degree>0$, there exists a network $\bar{s}_1 \in \cF(L_1,W_1,S_1,B_1,\cdot)$  with 
    \begin{align*}
        & L_1 =\cO\left(\degree^3 \log^3(1/(h_{t_0}\epsilon))  \right),\quad
        W_1 = \cO \left(D^{\degree} \vee \epsilon^{-d/\holder} \degree^3\log^3(1/(h_{t_0}\epsilon))  \right), \\
        &S_1 = \cO\left( \epsilon^{-d/\holder} \left(\degree\log(1/(h_{t_0}\epsilon))  + \log^4(1/h_{t_0}) \right) +  \degree D^{\degree}\log(1/\epsilon ) \right),  \quad
        B_1 = \cO\left( h_{t_0}^{-\degree}\right),
    \end{align*}
    such that for all $x\in\cK_t(\epsilon)$, we have
    \begin{align*}
        |\bar{s}_1(x,t) - s_1(x,t)| \lesssim D^{\degree+d/2}h_t^{d/2} \max_{j = 0,1,\ldots, \degree} \left\{\frac{\left\| x - \alpha_t \xstar\right\|^{j}}{(h_t)^{j/2}} \right\} \left( \frac{(\log(1/\epsilon)+d\log(1/h_t)/2)^{2\degree}}{\reach^\degree} \epsilon^{\degree/\holder}  +\epsilon\right).
    \end{align*}
    Here $\lesssim $ hides logarithmic factors and dependence on $d$, $\holder$, $B$, $C_F$ and  the upper bounds on the derivatives of exponential maps. 
\end{lemma}
\begin{proof}[Proof of Lemma \ref{lemma:s1-approx}]
First, by Lemma \ref{lemma:poly_approx_score}, $s_1$ can be approximated by $f_5$ in \eqref{eq:f5_poly} with the following approximation error
\begin{align*}
        |f_5(x,t) - s_1(x,t)| &\lesssim D^{(\degree+d)/2}h_t^{d/2} \max_{j = 0,1,\ldots, \degree} \left\{\frac{\left\| x - \alpha_t \xstar\right\|^{j}}{(h_t)^{j/2}} \right\}  \left( \frac{(\log(1/\epsilon)+d\log(1/h_t)/2)^{2\degree}}{\reach^{\degree}} \epsilon^{\degree/\holder}  +\epsilon\right),
    \end{align*}
which holds for all $x\in\cK_t(\epsilon)$. 
Replacing the low-dimensional representation $v_k(x,t)$ by its network implementation $\nn_{v_k}$ in $f_5(x,t)$ \eqref{eq:f5_poly} gives rise to
\begin{align}\label{eq:f6}
f_6(x, t) 
  = \sum_{k \in \cI(x)} \sum_{j=0}^{\gamma-1} \frac{(-1)^j}{j!(h_t)^{j/2}} \left\langle [x - \alpha_t x_t^*]^{\otimes j},  \text{Tensor-Poly}^{k,j}(\nn_{v_k}(x,t), h_t, \alpha_t)\right\rangle.
\end{align}
Lemma \ref{lemma:f7-approx-error} bounds the error of $f_6$ approximating $f_5$:
\begin{align*}
    |f_6(x,t)-f_5(x,t)| &\lesssim D^{\degree+d/2} h_t^{d/2}  (\log(1/\epsilon)+d\log(1/h_t)/2)^{2\degree}\epsilon_v.
\end{align*}
This error is controlled by the approximation cost $\epsilon_v$ given in \eqref{eq:error-nn-v}, together with the Lipschitz property of $f_5(x,t)$ with respect to $v_k(x,t)$ (Lemma \ref{lemma:f6-lipschitz}). Now we combine the results in \eqref{eq:error-nn-v}-\eqref{eq:error-nn-poly}. Taking $\epsilon_v=\epsilon_\text{proj} = \epsilon_{\rm det}=\epsilon_\text{poly} = \epsilon$ and $\epsilon_\times = h_t^{d/2}\epsilon$, we have
    \begin{align*}
        |\bar{s}_1(x,t) - f_6(x,t) |\lesssim h_t^{d/2} D^\degree\left( \log(1/\epsilon) +d \log(1/h_t)/2 \right)^{\degree} \epsilon + D\degree h_t^{d/2}\epsilon.
    \end{align*}
   Here $\bar{s}_1$ defined in \eqref{eq:s1-nn} can be written as a feedforward network $\bar{s}_1 \in \cF(L_1,W_1,S_1,B_1,\cdot)$, where 
    \begin{align*}
        & L_1 =\cO\left(\degree^3 \log^3(1/(h_{t_0}\epsilon))  \right),\quad
        W_1 = \cO \left(D^{\degree} \vee \epsilon^{-d/\holder} \degree^3\log^3(1/(h_{t_0}\epsilon))  \right), \\
        &S_1 = \cO\left( \epsilon^{-d/\holder} \left(\degree\log(1/(h_{t_0}\epsilon))  + \log^4(1/h_{t_0}) \right) +  \degree D^{\degree}\log(1/\epsilon ) \right),  \quad
        B_1 = \cO\left( h_{t_0}^{-\degree}\right).
    \end{align*}
    Here $\cO(\cdot)$ hides dependence on $d$, $\holder$, $B$, $C_F$, and the upper bounds on the derivatives of exponential maps.
 Therefore, using the triangle inequality, we finally obtain
    \begin{align*}
        |\bar{s}_1(x,t) - s_1(x,t)| \lesssim D^{\degree+d/2}h_t^{d/2} \max_{j = 0,1,\ldots, \degree} \left\{\frac{\left\| x - \alpha_t \xstar\right\|^{j}}{(h_t)^{j/2}} \right\} \left( \frac{(\log(1/\epsilon)+d\log(1/h_t)/2)^{2\degree}}{\reach^\degree} \epsilon^{\degree/\holder}  +\epsilon\right).
    \end{align*}
    The proof is complete.
\end{proof}

\subsubsection{Bounding Approximation Error of $\bar{s}_2$}\label{sec:s2-error-analysis}

It remains to construct a network $\bar{s}_{2} =\{\bar{s}_{2,i} \}_{i=1}^D$ that  approximates each entry of $s_2$:
\begin{align*}
     s_2(x,t) = \int_{x_0 \in \cM} -\frac{\alpha_t}{\sqrt{h_t}}(\xstar-x_0)\exp \left( - \frac{ \|\alpha_t \xstar- \alpha_t x_0\|^2 + 2 \langle x - \alpha_t \xstar , \alpha_t \xstar - \alpha_t x_0\rangle }{2h_t} \right) \ud \dist(x_0).
\end{align*}
The construction follow the same approximation procedure as that to $s_1$. We provide the approximation guarantees as follows. 

\begin{lemma}[Approximation result for $s_2(x,t)$]\label{lemma:s2-approx}
    Given $\epsilon \in \big(0,(\min\{1,3\reach, \rratio L_\logg \reach, \rratio \reach/(4 L_\expp)  \})^\holder \big)$, and time $t\geq t_0$ satisfying \eqref{cond:time}, let $r = \epsilon^{1/\holder}$.
    Then for any fixed $\degree>0$, there exists a network $\bar{s}_{2} =\{\bar{s}_{2,i} \}_{i=1}^D$ such that for all $x\in\cK_t(\epsilon)$, we have
    \begin{align*}
        \|\bar{s}_2(x,t) - s_2(x,t)\| \lesssim D^{\degree+(d+1)/2}h_t^{d/2} \max_{j = 0,1,\ldots, \degree} \left\{\frac{\left\| x - \alpha_t \xstar\right\|^{j}}{(h_t)^{j/2}} \right\} \left( \frac{(\log(1/\epsilon)+d\log(1/h_t)/2)^{2\degree}}{\reach^\degree} \epsilon^{\degree/\holder}  +\epsilon\right).
    \end{align*}
    Here each entry $\bar{s}_{2,i} $ is a network in $ \cF(L_2, W_2, S_2, B_2,\cdot)$ with 
    \begin{align*}
        & L_2 =\cO\left(\degree^3 \log^3(1/(h_{t_0}\epsilon))  \right),\quad
        W_2 = \cO \left(D^{\degree} \vee \epsilon^{-d/\holder} \degree^3\log^3(1/(h_{t_0}\epsilon))  \right), \\
        &S_2 = \cO\left( \epsilon^{-d/\holder} \left(\degree\log(1/(h_{t_0}\epsilon))  + \log^4(1/h_{t_0}) \right) +  \degree D^{\degree}\log(1/\epsilon ) \right)  \quad
        B_2 = \cO\left( h_{t_0}^{-\degree}\right).
    \end{align*}
\end{lemma}
\begin{proof}[Proof of Lemma \ref{lemma:s2-approx}]
Compared to $s_1(x,t)$, the main difference in $s_2(x,t)$ is the extra $-\alpha_t(\xstar-x_0)/\sqrt{h_t}$ in the integrand. It suffices to slightly enlarge the truncation radius to to cancel out the scale of this extra term,
$$\vradius= 2L_\logg \sqrt{(h_t/\alpha_t^2)(\log(1/\epsilon)+(d+1)\log(1/h_t)/2)},$$ 
then the analysis in Lemma \ref{lemma:s1-approx} can be carried through. Consequently, the degrees of the polynomial approximators with respect to each entry of $s_2$ are increased by $1$, while this does not affect the order of approximation error and network size. Therefore, for $i =1,\ldots,D$ and any $x\in\cK_t(\epsilon)$, we have
\begin{align*}
        |\bar{s}_{2,i}(x,t) - s_{2,i}(x,t)| &\lesssim D^{\degree+d/2}h_t^{d/2} \max_{j = 0,1,\ldots, \degree} \left\{\frac{\left\| x - \alpha_t \xstar\right\|^{j}}{(h_t)^{j/2}} \right\} \left( \frac{(\log(1/\epsilon)+d\log(1/h_t)/2)^{2\degree}}{\reach^\degree} \epsilon^{\degree/\holder}  +\epsilon\right).
    \end{align*}
The network configuration of $\bar{s}_{2,i}$ is in the same order as $\bar{s}_1$ given in Lemma \ref{lemma:s1-approx}. Finally, we can conclude the proof by applying $ \|\bar{s}_{2,i}(x,t) - s_{2,i}(x,t)\| \leq \sqrt{D} \max_{i=1,\ldots,D} |\bar{s}_{2,i}(x,t) - s_{2,i}(x,t)|. $
\end{proof}

\subsubsection{Constructing Network Approximation to $s_{\cM}(x,t)$ and $\nabla\log p_t(x)$}\label{sec:sM-approx}

In this step, we construct a neural network to approximate the on-manifold score $s_\cM$ in \eqref{eq:def-sM}, using networks $\bar{s}_1$ (Lemma \ref{lemma:s1-approx}) and $\bar{s}_2$ (Lemma \ref{lemma:s2-approx}).
The formal result is summarized as follows.
\begin{lemma}\label{lemma:approx-small-sm}
Given $\epsilon \in \big(0,(\min\{1,3\reach, \rratio L_\logg \reach, \rratio \reach/(4 L_\expp)  \})^\holder \big)$ and time $t\geq t_0$ satisfying \eqref{cond:time}, let $r = \epsilon^{1/\holder}$.  Then there exists a network $\bar{s}_\cM \in \cF(L,W,S,B,\cdot)$  with 
    \begin{align*}
        & L =\cO\left(\degree^3 \log^3(1/(h_{t_0}\epsilon))  \right),\quad
        W = \cO \left(D^{\degree} \vee \epsilon^{-d/\holder} \degree^3\log^3(1/(h_{t_0}\epsilon))  \right), \\
        &S = \cO\left( \epsilon^{-d/\holder} \left(\degree\log(1/(h_{t_0}\epsilon))  + \log^4(1/h_{t_0}) \right) +  \degree D^{\degree}\log(1/\epsilon ) \right),  \quad
        B = \cO\left( h_{t_0}^{-\degree}\right),
    \end{align*}
    such that for all $x\in\cK_t(\epsilon)$, we have
    \begin{align*}
    \left\| \bar{s}_\cM (x,t)- s_\cM (x,t) \right\|  
    &\lesssim \frac{ D^{\degree+d/2+1}}{\sqrt{h_t}} \exp \left( \frac{5(d+1) \log(1/h_t) \|x- \alpha_t\xstar \|}{\reach \alpha_t} \right) \\
    &\quad\cdot \max_{j = 0,1,\ldots, \degree} \left\{\frac{\left\| x - \alpha_t \xstar\right\|^{j}}{(h_t)^{j/2}} \right\} \left(  \frac{(\log(1/\epsilon)+d\log(1/h_t)/2)^{2\degree}}{\reach^{\degree}} \epsilon^{\degree/\holder}  +\epsilon\right).
\end{align*}
     Here $\lesssim$ hides dependence on $d, \holder,B, C_f$ and the upper bounds on the derivatives of exponential maps, and $\degree>0$ is some constant to be chosen.
\end{lemma}
\begin{proof}[Proof of Lemma~\ref{lemma:approx-small-sm}]
The idea is to first probe the approximation ability of the ratio \begin{align}\label{eq:nn-sm}
\bar{s}_{\cM}(x, t) = \frac{\bar{s}_2(x,t)/\sqrt{h_t}}{\bar{s}_1(x,t)}
\end{align}
to the targeted score $s_{\cM}$. Then we implement \eqref{eq:nn-sm} using a neural network by realizing the division operation.

We invoke Lemma \ref{lemma:final-approx-error} with approximation errors to $s_1$ and $s_2$ given in Lemma \ref{lemma:s1-approx} and \ref{lemma:s2-approx}, respectively. This yields the following error bound:
\begin{align*}
     \left\| s_{\cM}(x,t) -  \frac{\bar{s}_2(x,t)/\sqrt{h_t}}{\bar{s}_1(x,t)} \right\|  \lesssim &  \frac{ D^{\degree+d/2+1}}{\sqrt{h_t}} \exp \left( \frac{5(d+1) \log(1/h_t) \|x- \alpha_t\xstar \|}{\reach \alpha_t} \right) \notag\\
    \cdot &\max_{j = 0,1,\ldots, \degree} \left\{\frac{\left\| x - \alpha_t \xstar\right\|^{j}}{(h_t)^{j/2}} \right\} \left(  \frac{(\log(1/\epsilon)+d\log(1/h_t)/2)^{2\degree}}{\reach^{\degree}} \epsilon^{\degree/\holder}  +\epsilon\right).
\end{align*}
Then by Lemma F.7 in \cite{oko2023diffusionmodelsminimaxoptimal}, there exists a feedforward network $\bar{\phi}$ with no more than $\cO(\log^2(1/\epsilon))$ layers, width bounded by $\cO(\log^3(1/\epsilon))$, at most non-zero $\cO(\log^4(1/\epsilon))$ neurons and weight parameters bounded by $\cO(\epsilon^{-2})$, such that
\begin{align*}
    \left\| \bar{\phi}\left(\bar{s}_1(x,t), \bar{s}_2(x,t)\right) -  \frac{\bar{s}_2(x,t)/\sqrt{h_t}}{\bar{s}_1(x,t)} \right\| \leq \epsilon.
\end{align*}
Taking $\bar{s}_\cM = \bar{\phi}(\bar{s}_1, \bar{s}_2)$ leads to
\begin{align*}
    \left\| \bar{s}_\cM (x,t)- s_\cM (x,t) \right\|  &\leq \left\| \bar{\phi}\left(\bar{s}_1(x,t), \bar{s}_2(x,t)\right) -  \frac{\bar{s}_2(x,t)/\sqrt{h_t}}{\bar{s}_1(x,t)} \right\| + \left\|   \frac{\bar{s}_2(x,t)/\sqrt{h_t}}{\bar{s}_1(x,t)} -s_{\cM}(x,t) \right\| \\
    &  \lesssim   \frac{ D^{\degree+d/2+1}}{\sqrt{h_t}} \exp \left( \frac{5(d+1) \log(1/h_t) \|x- \alpha_t\xstar \|}{\reach \alpha_t} \right) \\
    &\quad\cdot \max_{j = 0,1,\ldots, \degree} \left\{\frac{\left\| x - \alpha_t \xstar\right\|^{j}}{(h_t)^{j/2}} \right\} \left(  \frac{(\log(1/\epsilon)+d\log(1/h_t)/2)^{2\degree}}{\reach^{\degree}} \epsilon^{\degree/\holder}  +\epsilon\right).
\end{align*}
Combined with the network size of $\bar{s}_1,\bar{s}_2$ given in Lemma \ref{lemma:s1-approx} and \ref{lemma:s2-approx}, we deduce that
$\bar{s}_\cM \in \cF(L,W,S,B,\cdot)$  with 
\begin{align*}
        & L =\cO\left(\degree^3 \log^3(1/(h_{t_0}\epsilon))  \right),\quad
        W = \cO \left(D^{\degree} \vee \epsilon^{-d/\holder} \degree^3\log^3(1/(h_{t_0}\epsilon))  \right), \\
        &S = \cO\left( \epsilon^{-d/\holder} \left(\degree\log(1/(h_{t_0}\epsilon))  + \log^4(1/h_{t_0}) \right) +  \degree D^{\degree}\log(1/\epsilon ) \right),  \quad \text{and} \quad
        B = \cO\left( h_{t_0}^{-\degree}\right).
    \end{align*}
The proof is complete.
\end{proof}

To this end, we are ready to prove the score approximation theory for the small noise regime (Lemma \ref{lemma:approx-small}). Recall the score decomposition in Lemma \ref{lemma:score-decomp-small}:
\begin{align*}
    \nabla\log p_t(x) = s_\cM(x,t) - \frac{x-\projm(x,t)}{h_t}.
\end{align*}
Adopting $\bar{s}_\cM$ given in Lemma \ref{lemma:approx-small-sm} and the projection network given in Lemma \ref{lemma:nn-proj}, we can construct a network $\nnsmall \in \cF(L,W,S,B,\cdot)$, where 
\begin{align*}
        & L =\cO\left(\degree^3 \log^3(1/(h_{t_0}\epsilon))  \right),\quad
        W = \cO \left(D^{\degree} \vee \epsilon^{-d/\holder} \degree^3\log^3(1/(h_{t_0}\epsilon))  \right), \\
        &S = \cO\left( \epsilon^{-d/\holder} \left(\degree\log(1/(h_{t_0}\epsilon))  + \log^4(1/h_{t_0}) \right) +  \degree D^{\degree}\log(1/\epsilon ) \right),  \quad
        B = \cO\left( h_{t_0}^{-\degree}\right).
    \end{align*}
such that for any $t\geq t_0$ satisfying \eqref{cond:time} and $x\in\cK_t(\epsilon)$,
\begin{align*}
    \left\| \nnsmall (x,t)- \nabla\log p_t(x)\right\|  
    &\lesssim \frac{ D^{\degree+d/2+1}}{\sqrt{h_t}} \exp \left( \frac{5(d+1) \log(1/h_t) \|x- \alpha_t\xstar \|}{\reach \alpha_t} \right) \\
    &\quad\cdot \max_{j = 0,1,\ldots, \degree} \left\{\frac{\left\| x - \alpha_t \xstar\right\|^{j}}{(h_t)^{j/2}} \right\} \left(  \frac{(\log(1/\epsilon)+d\log(1/h_t)/2)^{2\degree}}{\reach^{\degree}} \epsilon^{\degree/\holder}  +\epsilon\right),
\end{align*}
The remaining part repeats the argument in the beginning of Appendix \ref{sec:small-approx} with more details.
Recall that we decompose the $L^2$ approximation error of $\nnsmall$ as as
    \begin{align*}
        \left\|\nnsmall(x,t) - \nabla \log p_t(x) \right\|^2_{L^2(P_t)} = \left( \int_{x\in \cK_t(\epsilon)} + \int_{x\in\RR^D \setminus \cK_t(\epsilon)} \right) \left\|\nnsmall(x,t) - \nabla \log p_t(x) \right\|^2 p_t(x) \ud x.
    \end{align*}
We can bound the integral within $\cK_t(\epsilon)$ via the approximation error of $\nnsmall$:
    \begin{align*}
       \int_{x\in \cK_t(\epsilon)} \left\|\nnsmall(x,t) - \nabla \log p_t(x) \right\|^2 p_t(x) \ud x 
  &\lesssim  \frac{D^{2\degree+d+2}}{h_t}\left( \frac{\epsilon^{2\degree/\holder}}{\reach^{2\degree}} +\epsilon^2\right).
    \end{align*}
    For the integral over $\RR^D\setminus \cK_t(\epsilon)$, we have
    \begin{align*}
      \int_{x\in\RR^D \setminus \cK_t(\epsilon)}  \left\|\nnsmall(x,t) - \nabla \log p_t(x) \right\|^2 p_t(x) \ud x &\leq 2 \int_{x\in\RR^D \setminus \cK_t(\epsilon)} \left( \left\|\nnsmall(x,t)  \right\|^2 + \left\|\nabla \log p_t(x) \right\|^2 \right) p_t(x) \ud x.
\end{align*}
By Lemma \ref{lemma:truncate-x}, we 
   can bound the probability mass outside the truncation region $\cK_t(\epsilon)$, which yields,
    \begin{align*}
         \int_{x\in\RR^D \setminus \cK_t(\epsilon)} \left\|\nnsmall(x,t)  \right\|^2 p_t(x) \ud x \leq \left\|\nnsmall(x,t)  \right\|_\infty^2 \PP\left(x \notin \cK_t(\epsilon) \right) \leq  \left\|\nnsmall(x,t)  \right\|_\infty^2 \epsilon^D.
    \end{align*}
    Since $\|\nabla p_t(x)\| = \cO(\sqrt{\log(1/\epsilon)/h_t})$ for $x \in \cK_t(\epsilon)$, it suffices to clip $\nnsmall(x,t)$ by $R = \cO(\sqrt{\log(1/\epsilon)/h_t})$ via  a feedforward layer as constructed in Lemma F.5 of \cite{oko2023diffusionmodelsminimaxoptimal}. 
Moreover, we have
\begin{align}\label{eq:score-tail-bound}
      \int_{x\in\RR^D \setminus \cK_t(\epsilon)}  \left\|\nabla \log p_t(x) \right\|^2  p_t(x) \ud x \leq \frac{1}{h_t} \int_{x\in\RR^D \setminus \cK_t(\epsilon)}  \frac{2\|x-\alpha_t\xstar\|^2+2DB^2}{h_t}  p_t(x) \ud x  \lesssim \frac{\epsilon^2}{h_t},
\end{align}
for sufficiently large $D >0$. Finally, we can conclude that the network $\nnsmall(x,t)$ has the following $L^2$ approximation error for any $t \in [t_0,t_{\rm small}]$, 
\begin{align*}
        \left\|\nnsmall(x,t) - \nabla \log p_t(x) \right\|^2_{L^2(P_t)} & \lesssim  \frac{D^{2\degree+d+2}}{h_t}\left( \frac{\epsilon^{2\degree/\holder}}{\reach^{2\degree}} +\epsilon^2\right).
    \end{align*}

\subsection{Large Noise: Proof of Lemma~\ref{lemma:approx-large}}\label{sec:approx-large}

Score approximation is less challenging in the large noise regime, where clean data is corrupted by Gaussian noise with large variance $h_t$. The resulting score function is smoother than the small noise counterpart.
The large noise regime is characterized by
\begin{align}\label{cond:time-large}
    h_t \geq  \epsilon^{2/\holder}/4,
\end{align}
where $\epsilon>0$ is the approximation error we aim to achieve. We will choose $\epsilon$ later so that the large noise and small noise regimes overlap. Our goal is to construct a network $\nnlarge(x,t)$ approximating the score function $\nabla \log p_t(x)$ in this regime. By the same truncation argument in \eqref{eq:truncation}, it suffices to construct $\nnlarge(x,t)$ for a pointwise approximation to $\nabla \log p_t$ for any $x\in\cK_t(\epsilon)$.

In this regime, the precise projection $\projm(x,t)$ is no longer necessary. Instead, it suffices to decompose the score into components with respect to each local tangent space of $\cM$ and orthogonal components perpendicular to these tangent spaces. The tangent space projections $\projk$ have explicit forms of linear operations and thus avoid the complex approximation of $\projm$. Lemma \ref{lemma:score-decomp-large} provides the following score decomposition:
\begin{align*}
    \nabla \log p_t(x) 
 = \sum_{k=1}^{C_\cM} w_k(x, t)  \bigg( \frac{\alpha_t \EE_{X_0 \sim \mu_k}[X_0 | X_t = x] - \projk(x, t)}{h_t} -\frac{x-\projk(x, t)}{h_t} \bigg),
\end{align*}
where 
\begin{align*}
    w_k(x, t) =  \frac{\int_{x_0 \in U_k} \exp \left( -\frac{\|x - \alpha_t x_0\|^2}{2h_t} \right) \rho_k(x_0) \pdata(x_0) \ud \vol(x_0)}{\sum_{j=1}^{C_\cM} \int_{x_0 \in U_j} \exp \left( -\frac{\|x - \alpha_t x_0\|^2}{2h_t} \right) \rho_j(x_0) \pdata(x_0) \ud \vol(x_0)} 
\end{align*}
We firstly focus on the approximation of the denominator in the weighting function $w_k(x, t)$. For notational simplcity, we denote
\begin{align}\label{eq:def-s3}
    s_3(x,t) := \sum_{j=1}^{C_\cM} \int_{x_0 \in U_j} \exp \left( -\frac{\|x - \alpha_t x_0\|^2}{2h_t} \right) \rho_j(x_0) \pdata(x_0) \ud \vol(x_0).
\end{align}
The approximation of $s_3$ follows a similar process to Section \ref{sec:small-approx}: we first approximate the components of $s_3$ by local polynomials, then implement the local polynomials by neural networks. We defer technical lemmas to Appendix \ref{sec:proof-large}. For the numerator, we can take a more direct approach to merge the on-support and orthogonal components together (see details in Section \ref{sec:large-score-approx}).

\subsubsection{Local Polynomial Construction}

Recall that we construct the atlas as $\{(U_k,\logg_k)\}_{k=1}^{C_\cM}$, where $U_k = \expp_k(\dball(0,r))$. For the target approximation error $\epsilon>0$, we choose the radius $r = \epsilon^{1/\holder}$, which matches the choice in the small noise regime.
We use the same notation $\thres$ for the chart determination threshold, but set $\thres = 2 \sqrt{h_t \log(1/\aerror)} + h_tB + L_\expp \epsilon$, where $\aerror>0$ is some hyperparameter that we will instantiate later. We define $g_1$ as 
\begin{align*}
        g_1(x,t) = \sum_{k: \| x - x_k\| \leq \thres} \int_{x_0 \in U_k} \exp \left( -\frac{\|x - \alpha_t x_0\|^2}{2h_t} \right) \rho_k(x_0) p_{\rm data}(x_0) \ud \vol(x_0).
\end{align*}
Lemma \ref{lemma:large-truncation} demonstrates that $g_1(x,t)$ well approximates $s_3(x,t)$ while involves much less charts. For any time $t$ satisfying \eqref{cond:time-large}, it holds that
\begin{align}\label{eq:g1-approx}
    \left\|g_1(\cdot,t)-s_3(\cdot,t)\right\|_{L^\infty} \leq \aerror.
\end{align}
Recall we denote the projection onto the $k$-th tangent space as   $\projk(x,t) := \argmin_{y \in \alpha_t \cdot T_{x_k}\cM} \| y-x\|$, which inspires the following decomposition,
\begin{align*}
        \|x - \alpha_t x_0\|^2 = \|x - \projk(x,t)\|^2 + \cT_k(x, x_0, t)  + \alpha_t^2 \cD_k(x,x_0, t),
    \end{align*}    
where we reload the notations from Section \ref{sec:small-approx} as
\begin{align*}
    \cT_k(x, x_0, t) &= \langle x - \projk(x,t), \projk(x,t) - \alpha_t x_0 \rangle \quad \text{and} \\
    \cD_k(x,x_0, t) &= \|\projk(x,t)/\alpha_t - x_0\|^2.
\end{align*}
Utilizing the decomposition, we can rewrite $g_1$ as 
\begin{align*}
    g_1(x,t) = \sum_{k: \| x - x_k\| \leq \thres}  \exp \left( -\frac{\|x - \projk(x,t)\|^2}{2h_t} \right) I_k(x,t),
\end{align*}
where we define 
\begin{align}\label{eq:def-Ik}
    I_k(x,t):=  \int_{x_0 \in U_k} \exp \left( -\frac{1}{h_t} \cT_k(x, x_0, t)-\frac{\alpha_t^2}{2h_t} \cD_k(x,x_0, t)\right) \rho_k(x_0) p_{\rm data}(x_0) \ud \vol(x_0).
\end{align}
Next, we replace the exponential functions in $I_k$ by polynomials, which yields 
\begin{align}\label{def:Ik-poly-approximator}
    \text{Poly}^k(x, t) :=& \int_{x_0 \in U_k} \sum_{l=0}^{\degree_0 -1} \sum_{j=0}^{S-1}  \frac{(-1)^{l+j}\alpha_t^{l+2j}}{2^j h_t^{l+j} l!j!} \cT_k^l(x,x_0,t)  \cD_k^{j}(x,x_0,t) \rho_k(x_0) p_{\rm data}(x_0) \ud \vol(x_0).
    \end{align}
The degree $\degree_0>0$ will be instantiated later. Lemma \ref{lemma:large-poly} establishes the approximation error of $\text{Poly}^k$ to $I_k$ with $S= 4 e^2 \log(1/\aerror) + 4e^2(B+L_\expp+1)^2$. For any $x\in \cK_t(\epsilon)$, it holds that
\begin{align}\label{eq:large-poly-error}
    \left|\text{Poly}^k(x,t)-I_k(x,t)\right| =\cO\left( \frac{ (\sqrt{\log(1/\aerror) }+ B)^{\degree_0} }{\reach^{\degree_0} } \epsilon^{(d+\degree_0)/\holder}+ \epsilon^{d/\holder} \aerror \right).
\end{align}
Lemma \ref{lemma:Ik_simplify} further shows that $\text{Poly}^k$ is a polynomial with inputs $h_t$, $\alpha_t$, $x \in \RR^D$ and  $P_k^\top x \in \RR^d$, where $P_k$ is a matrix in $\RR^{D\times d}$. In addition, $\text{Poly}^k$ has a degree up to  $\degree_0$ with respect to $x$ and degree up to $S$ with respect to $P_k^\top x$. For notational clarity, we write $\text{Poly}^k(x,t)$ as $\text{Poly}^k(x,h_t, \alpha_t)$.
Substituting $\text{Poly}^k$ for $I_k$ in $g_1$ and expressing the chart determination via indicator functions yields
\begin{align}\label{def:large-g2}
    g_2(x,t) = \sum_{k=1}^{C_\cM} \ind( \| x - x_k\| \leq \thres)\cdot \exp \left( -\frac{\|x - \projk(x,t)\|^2}{2h_t} \right) \text{Poly}^k(x,h_t, \alpha_t).
\end{align}

\subsubsection{Neural Network Implementation}\label{sec:nn-large}

We construct neural networks to implement function $g_2$ in \eqref{def:large-g2}. The construction involves three key networks: (1) chart determination network $\nn_{\rm det}$, (2) polynomial network $\nn_{\rm poly}$ and (3) weighting network $\nn_{\rm exp}$. We also utilize the elementary network $\nnprod$ used in Section \ref{subsec:nn} to realize multiplication.  Accordingly, the constructed network takes the form:
\begin{align}\label{eq:def-s3-nn}
        \bar{s}_3(x;t) = \sum_{k=1}^{C_\cM} \nnprod \left(\nn_{\rm{det}}(x-x_k), \nnprod \left(\nn^k_{\rm exp}(x,h_t, \alpha_t), \nn_{\rm poly}^k(x, h_t, \alpha_t) \right) \right).
    \end{align}
In the sequel, we provide detailed implementations of the three key networks:

\noindent $\bullet$ {\bf 1. Chart determination network $\nn_{\rm {det}}$}. We implement  $\nn_{\rm {det}}(x-x_k)$ to  approximate $ \ind( \| x - x_k\| \leq \thres)$, following the same construction as the chart determination network described in Section~\ref{subsec:nn} but using a different $\thres$. By Lemma \ref{lemma:nn-chart}, for $\epsilon_{\rm det} \in (0,1)$, $x \in \cK_t(\epsilon)$, and any $k \in [C_{\cM}]$, 
we have
\begin{align*}
     \nn_{\rm{det}}(x - x_k) =
\begin{cases}
1 & \text{if}~\norm{x - x_k}^2 \leq (1-\epsilon_{\rm det})\thres^2 \\
0 & \text{if}~\norm{x - x_k}^2 \geq \thres^2,
\end{cases}
\end{align*}
Here $\nn_{\rm{det}}$ is a network in $\cF(L_{\rm{det}}, W_{\rm{det}}, S_{\rm{det}}, B_{\rm{det}}, \cdot)$, where 
\begin{align*}
    &L_{\rm{det}} = \cO(\log(D/\epsilon)+\log(1/\epsilon_{\rm{det}})), \quad W_{\rm{det}} = \cO(D), \\
    &S_{\rm{det}} = \cO(\log(D/\epsilon)+\log(1/\epsilon_{\rm{det}})),  \quad  \text{and} \quad
    B_{\rm{det}} = \cO(\log(D/\epsilon)+\log(1/\epsilon_{\rm{det}})). 
\end{align*}

\noindent $\bullet$ {\bf 2. Polynomial network $\nn_{\rm {poly}}$}.
We proceed to implement $\text{Poly}^k$ by a network. Applying Theorem 3 in \cite{chen2022nonparametric} and Lemmas F.6 and F.7 in \cite{oko2023diffusionmodelsminimaxoptimal}, for $\epsilon_{\rm poly}\in(0,1)$ and each $k=1,\ldots,C_\cM$, there exists a network $\nn_{\rm {poly}}^k$ such
    \begin{align*}
        \left|  \nn_{\rm {poly}}^k(x,h_t,\alpha_t) - \text{Poly}^k(x,h_t,\alpha_t) \right| \leq \epsilon_{\rm poly},
    \end{align*}
    which holds for all $x \in \cK_t(\epsilon)$ and $h_t, \alpha_t \in [\epsilon_{\rm poly},1]$.
Here $\nn_{\rm{poly}}^k$ is a network in $\cF(L_{\rm{poly}}, W_{\rm{poly}}, S_{\rm{poly}}, B_{\rm{poly}}, \cdot)$ with 
\begin{align*}
    &L_{\rm{poly}} = \cO\left(\log^2(1/\epsilon_{\rm poly} )\right), \quad W_{\rm{poly}} = \cO\left(D^{\degree_0} S^d \log^3(1/\epsilon_{\rm poly} ) \right), \\
    &S_{\rm{poly}} = \cO\left(D^{\degree_0} S^d\log^4(1/\epsilon_{\rm poly})\right),  \quad  \text{and} \quad
    B_{\rm{poly}} = \cO\left(h_t^{-(S+\degree_0)}\right). 
\end{align*}
    
\noindent $\bullet$ {\bf 3. Network $\nn_{\rm exp}$}. 
It remains to implement $\exp(-\|x-\projk(x,t)\|^2/(2h_t))$ in \eqref{def:large-g2} by networks. By Lemma F.6, F.7 and F.12 in \cite{oko2023diffusionmodelsminimaxoptimal}, for $\epsilon_{\rm exp}\in(0,1)$, there exists a network $\nn_{\rm exp}^k(x,h_t,\alpha_t)$ such that for any $x\in\cK_t(\epsilon)$,
    \begin{align*}
        \left|  \nn_{\rm exp}^k(x,h_t,\alpha_t) -  \exp \left( -\frac{\|x - \projk(x,t)\|^2}{2h_t} \right) \right| \leq \epsilon_{\rm exp}.
    \end{align*}
    where  $\nn_{\rm exp}^k$ has (i) no more than $c_{\rm exp}\log^2(1/\epsilon_{\rm exp} )$ layers with width bounded by $c_{\rm exp}\log^3(1/\epsilon_{\rm exp} )$, and (ii) at most $c_{\rm exp} \log^4(1/\epsilon_{\rm exp}) $ neurons and weight parameters, where the constant $c_{\rm exp}$ depends on $B$ and $D$ at most polynomially.
    $x \in \cK_t(\epsilon)$ and $h_t, \alpha_t \in [\epsilon_{\rm poly},1]$.
Here $\nn_{\rm{exp}}^k$ is a network in $\cF(L_{\rm{exp}}, W_{\rm{exp}}, S_{\rm{exp}}, B_{\rm{exp}}, \cdot)$ with 
\begin{align*}
    &L_{\rm{exp}} = \cO\left(\log^2(1/\epsilon_{\rm exp} )\right), \quad W_{\rm{exp}} = \cO\left(\log^3(1/\epsilon_{\rm exp}) \right), \\
    &S_{\rm{exp}} = \cO\left(\log^4(1/\epsilon_{\rm exp})\right),  \quad  \text{and} \quad
    B_{\rm{exp}} = \cO\left(\log(1/\epsilon)\right). 
\end{align*}

The next lemma shows that  $s_3(x,t)$ defined in \eqref{eq:def-s3} can be well approximated by a network. 
\begin{lemma}\label{lemma:s3-large-time}
    Given $\epsilon, \aerror\in (0,1)$ and $\degree_0>0$, there exists a network $\bar{s}_3 \in \cF(L_3,W_3,S_3,B_3,\cdot)$ with 
    \begin{align*}
 L_3 = \tilde{\cO}(1), \quad W_3 = \tilde{\cO}(D^{\degree_0}\epsilon^{-d/\holder}), \quad
        S_3 = \tilde{\cO}(D^{\degree_0}\epsilon^{-d/\holder}), \quad  B_3= \tilde{\cO}(\epsilon^{-2(\log(1/\aerror)+\degree_0)/\holder}) ,
\end{align*}
    such that for any time $t>0$ satisfying $h_t \geq \epsilon^{2/\holder}/4$ and $x\in\cK_t(\epsilon)$,
    \begin{align*}
        \left|\bar{s}_3(x;t) - s_3(x,t)  \right| =  \tilde{\cO}\left( \frac{ (\sqrt{\log(1/\aerror) }+ B)^{\degree_0} }{\reach^{\degree_0} } h_t^{d/2}\epsilon^{\degree_0/\holder}+ \aerror \right) .
    \end{align*}
\end{lemma}
\begin{proof}[Proof of Lemma \ref{lemma:s3-large-time}]
Recall that $s_3(x,t)$ is a summation over all the charts, i.e.,
\begin{align*}
    s_3(x,t) = \sum_{k=1}^{C_\cM} \int_{x_0 \in U_k} \exp \left( -\frac{\|x - \alpha_t x_0\|^2}{2h_t} \right) \rho_k(x_0) p_{\rm data}(x_0) \ud \vol(x_0). 
\end{align*}
Notably, \eqref{eq:g1-approx} suggests that it suffices to focus on the charts satisfying $\|x-x_k\| \leq \thres$. We implement the chart determination by network $\nn_{\rm{det}}$. According to Lemma \ref{lemma:nn-chart}, we have
\begin{align*}
    &\quad~\bigg|   \sum_{k=1}^{C_\cM} \nn_{\rm{det}}(x-x_k)  \int_{x_0 \in U_k} \exp \left( -\frac{\|x - \alpha_t x_0\|^2}{2h_t} \right) \rho_k(x_0) p_{\rm data}(x_0) \ud \vol(x_0)  - s_3(x,t) \bigg| \\
    &\leq \left| \sum_{k:  (1-\epsilon_{\rm{det}})\thres<\|x_k-x\| < \thres} \int_{x_0 \in U_k} \exp \left( -\frac{\|x - \alpha_t x_0\|^2}{2h_t} \right) \rho_k(x_0) p_{\rm data}(x_0) \ud \vol(x_0) \right| . 
\end{align*}
Adopting the analysis in Lemma \ref{lemma:large-truncation}, we can derive $ \|x - \alpha_t x_0\| \geq (2- 4\epsilon_{\rm{det}})\sqrt{h_t \log(1/\aerror)}$
for $k$-th chart satisfying $ \|x_k-x\| > (1-\epsilon_{\rm{det}})\thres$ and $x_0 \in U_k$. Taking $\epsilon_{\rm{det}} = (2-\sqrt{2})/4$ yields
\begin{align*}
\bigg|   \sum_{k=1}^{C_\cM} \nn_{\rm{det}}(x-x_k)  \int_{x_0 \in U_k} \exp \left( -\frac{\|x - \alpha_t x_0\|^2}{2h_t} \right) \rho_k(x_0) p_{\rm data}(x_0) \ud \vol(x_0)  - s_3(x,t) \bigg|\leq 2\aerror. 
\end{align*} 
By the definition of $I_k$ in \eqref{eq:def-Ik}, we can rewrite the above inequality as
\begin{align}\label{eq:s3-1}
\bigg|   \sum_{k=1}^{C_\cM} \nn_{\rm{det}}(x-x_k) \exp \left( -\frac{\|x - \projk(x,t)\|^2}{2h_t} \right) I_k(x,t)  - s_3(x,t) \bigg|\leq 2\aerror. 
\end{align} 
Moreover, we apply $h_t \geq \epsilon^{2/\holder}$ to \eqref{eq:large-poly-error}, which yields
\begin{align}\label{eq:s3-2}
\begin{split}
    &\quad~\left|\sum_{k=1}^{C_\cM} \nn_{\rm{det}}(x-x_k)   \exp \left( -\frac{\|x - \alpha_t x_k\|^2}{2h_t} \right) \text{Poly}^k(x,h_t,\alpha_t) -I_k(x,h_t,\alpha_t) \right|\\
     & =  \tilde{\cO}\left( \frac{ (\sqrt{\log(1/\aerror) }+ B)^{\degree_0} }{\reach^{\degree_0} } h_t^{d/2}\epsilon^{\degree_0/\holder}+ \aerror \right).
\end{split}
\end{align}
Next,  Lemma \ref{lemma:nn-prod} derives the following approximation error for network implementation:
\begin{align*}
    &\quad\left|\nnprod \left(\nn^k_{\rm exp}(x,\hat{h}_t, \hat{\alpha}_t), \nn_{\rm poly}^k(x, \hat{h}_t, \hat{\alpha}_t) \right)  -  \exp \left( -\frac{\|x - \projk(x,t)\|^2}{2h_t} \right) \text{Poly}^k(x,h_t,\alpha_t)  \right| \notag\\
    &\leq \epsilon_{\rm prod} + 4 \max\{\epsilon_{\rm poly},\epsilon_{\rm exp}\}, 
\end{align*}
where $\epsilon_{\rm prod} \in (0,1)$ is the target approximation error for the multiplication network $\nnprod(\cdot,\cdot)$. Taking  $\epsilon_{\rm prod} = \aerror/(4 C_\cM)$ and $\epsilon_{\rm poly} = \epsilon_{\rm exp} = \aerror/(32 C_\cM)$ gives
\begin{align}\label{eq:s3-3}
    \left|\bar{s}_3(x,t) - \sum_{k=1}^{C_\cM} \nn_{\rm{det}}(x-x_k)   \exp \left( -\frac{\|x - \projk(x,t)\|^2}{2h_t} \right) \text{Poly}^k(x,h_t,\alpha_t)  \right| \leq \aerror.
\end{align}
Now we combine \eqref{eq:s3-1}, \eqref{eq:s3-2} and \eqref{eq:s3-3} to derive
    \begin{align*}
        \left|\bar{s}_3(x,t) -  s_3(x,t) \right|  &\leq \left|\bar{s}_3(x,t) - \sum_{k=1}^{C_\cM} \nn_{\rm{det}}(x-x_k)   \exp \left( -\frac{\|x - \projk(x,t)\|^2}{2h_t} \right) \text{Poly}^k(x,h_t,\alpha_t)  \right|\\
        &\quad+ \sum_{k=1}^{C_\cM} \nn_{\rm{det}}(x-x_k)   \exp \left( -\frac{\|x - \projk(x,t)\|^2}{2h_t} \right) \left|\text{Poly}^k(x,h_t,\alpha_t) -I_k(x,h_t,\alpha_t) \right|\\
        &\quad+ \left| \sum_{k=1}^{C_\cM} \nn_{\rm{det}}(x-x_k)   \exp \left( -\frac{\|x - \projk(x,t)\|^2}{2h_t} \right) I_k(x,h_t,\alpha_t) -  s_3(x,t) \right|\\
        &=  \tilde{\cO}\left( \frac{ (\sqrt{\log(1/\aerror) }+ B)^{\degree_0} }{\reach^{\degree_0} } h_t^{d/2}\epsilon^{\degree_0/\holder}+ \aerror \right).
    \end{align*}
According to the network implementation details in Section \ref{sec:nn-large} and Lemmas F.1-3 in \cite{oko2023diffusionmodelsminimaxoptimal}, $\bar{s}_3$ is
a network in  $\cF(L_3,W_3,S_3,B_3,\cdot)$ with 
    \begin{align*}
 L_3 = \tilde{\cO}(1), \quad W_3 = \tilde{\cO}(C_\cM D^{\degree_0}), \quad
        S_3 = \tilde{\cO}(C_\cM D^{\degree_0}), \quad  B_3= \tilde{\cO}(\epsilon^{-2(\log(1/\aerror)+\degree_0)/\holder}) ,
\end{align*}
where $\tilde{\cO}$ hides the logarithmic factors and dependence  depends on $d$, $B$, $\holder$ and $L_\expp$. Since we take the partition accuracy $r=\epsilon^{1/\holder}$, we have $C_\cM = \cO(r^{-d}) = \cO(\epsilon^{-d/\holder})$.
\end{proof}

\subsubsection{Constructing Network Approximation to $\nabla\log p_t(x)$}\label{sec:large-score-approx}

Recall the score function $\nabla\log p_t(x)$ is formulated as
\begin{align*}
    \nabla \log p_t(x) 
 = \sum_{k=1}^{C_\cM} w_k(x, t)  \bigg( \frac{\alpha_t \EE_{X_0 \sim \mu_k}[X_0 | X_t = x] - \projk(x, t)}{h_t} -\frac{x-\projk(x, t)}{h_t} \bigg).
\end{align*}
Lemma \ref{lemma:s3-large-time} establishes the approximation results for the denominator of $w_k(x, t)$. Furthermore,  we take a more direct approach to approximate the numerator, merge the on-support and orthogonal components together. Specifically, we rewrite $\nabla\log p_t(x)$ as
\begin{align*}
    \nabla \log p_t(x) &= \sum_{k=1}^{C_\cM} w_k(x, t)  \bigg( -\frac{x-\alpha_t \EE_{X_0 \sim \mu_k}[X_0 | X_t = x]} {h_t}  \bigg)\\
    &=\sum_{k=1}^{C_\cM}\frac{\int_{x_0 \in U_k} -\frac{x - \alpha_t x_0}{h_t}\exp \left( -\frac{\|x - \alpha_t x_0\|^2}{2h_t} \right) \rho_k(x_0) \pdata(x_0) \ud \vol(x_0)}{\sum_{j=1}^{C_\cM} \int_{x_0 \in U_j} \exp \left( -\frac{\|x - \alpha_t x_0\|^2}{2h_t} \right) \rho_j(x_0) \pdata(x_0) \ud \vol(x_0)} 
\end{align*}
Now we denote
\begin{align*}
    s_4(x,t) := \sum_{k=1}^{C_\cM}\int_{x_0 \in U_k} -\frac{x - \alpha_t x_0}{\sqrt{h_t}}\exp \left( -\frac{\|x - \alpha_t x_0\|^2}{2h_t} \right) \rho_k(x_0) \pdata(x_0) \ud \vol(x_0).
\end{align*}
Therefore, we have $\nabla\log p_t(x) =  s_4(x,t)/(\sqrt{h_t}s_3(x,t))$.  Note that $s_4(x,t) $ differs from $s_3(x,t) $ with an extra term $(x-\alpha_t x_0)/\sqrt{h_t}$ inside the integral. Then following the analysis in Lemma \ref{lemma:s3-large-time}, we can construct a network $\bar{s}_4$ to approximate $s_4(x,t)$ with slightly larger network size. Consequently, there exists a network $\bar{s}_4 \in \cF(L_4,W_4,S_4,B_4,\cdot)$ such that
\begin{align*}
     \left\|\bar{s}_4(x,h_t,\alpha_t) -s_4(x,t) \right\| =  \tilde{\cO}\left( \frac{ (\sqrt{\log(1/\aerror) }+ B)^{\degree_0} }{\reach^{\degree_0} } h_t^{d/2}\epsilon^{\degree_0/\holder}+ \aerror \right),
\end{align*}
where we take
       \begin{align*}
 L_4 = \tilde{\cO}(1), \quad W_4 = \tilde{\cO}(D^{\degree_0+1}\epsilon^{-d/\holder}), \quad
        S_4 = \tilde{\cO}(D^{\degree_0+1}\epsilon^{-d/\holder}), \quad  B_4 = \tilde{\cO}(\epsilon^{-2(\log(1/\aerror)+\degree_0)/\holder}).
\end{align*}

Finally, we are ready to prove the score approximation theory for the large noise regime (Lemma \ref{lemma:approx-large}). 
First, we establish the approximation error of $\bar{s}_4/(\sqrt{h_t}\bar{s}_3)$ to  the score function $\nabla \log p_t(x)$:
\begin{align*}
        \sqrt{h_t}\left\| \nabla \log p_t(x) -  \frac{\bar{s}_4(x,t)/\sqrt{h_t}}{\bar{s}_3(x,t)} \right\| &=  \left\| \frac{s_4(x,t)}{s_3(x,t)} -  \frac{\bar{s}_4(x,t)}{\bar{s}_3(x,t)} \right\|\notag \\
        &\leq  \left\| \frac{s_4(x,t)}{s_3(x,t)} -  \frac{s_4(x,t) }{ \bar{s}_3(x,t)}\right\| 
    + \left\|  \frac{s_4(x,t) - \bar{s}_4(x,t)}{\bar{s}_3(x,t)} \right\|  \notag\\
        &\leq \frac{\|\sqrt{h_t}\nabla \log p_t(x) \| \cdot \left| \bar{s}_3(x,t)-s_3(x,t)\right|}{\| \bar{s}_3(x,t)\|}+ \frac{\left\| s_4(x,t) - \bar{s}_4(x,t) \right\| }{\| \bar{s}_3(x,t)\|}.
    \end{align*}
For any $x\in\cK_t(\epsilon)$, we can derive the bound for $\nabla \log p_t(x)$,
\begin{align*}
    \| \nabla \log p_t(x) \| \leq  \sup_{x_0\in \cM} \left\| \frac{x-\alpha_t x_0}{h_t} \right\| \leq \frac{1}{\sqrt{h_t}}\left( 2\sqrt{D\log(1/\epsilon)} + \alpha_t B/\sqrt{h_t} \right),
\end{align*}
as well as the lower bound on $s_3(x,t)$,
\begin{align*}
    s_3(x,t) 
    \geq \int_{x_0 \in \cM: \|x-\alpha_t x_0 \|\leq \sqrt{h_t}} \exp \left( -\frac{\|x - \alpha_t x_0\|^2}{2h_t} \right) \ud \dist(x_0)\gtrsim C_f^{-1} e^{-1/2}\left(\frac{\sqrt{h_t}}{\alpha_t} \right)^d.
\end{align*}
Combining all the pieces together, we obtain
\begin{align*}
        \sqrt{h_t}\left\| \nabla \log p_t(x) -  \frac{\bar{s}_4(x,t)/\sqrt{h_t}}{\bar{s}_3(x,t)} \right\| =  \tilde{\cO} \left( \left(\frac{ \sqrt{\log(1/\aerror) }\epsilon^{1/\holder} }{\reach }\right)^{\degree_0} \frac{h_t^{d/2}}{h_t^{(d+1)/2}}+ \frac{\aerror}{h_t^{d/2}} \right).
    \end{align*}
    Taking $\aerror = h_t^{(d-1)/2}\epsilon^{(\holder+1)/\holder}$  yields
    \begin{align*}
        \sqrt{h_t}\left\| \nabla \log p_t(x) -  \frac{\bar{s}_4(x,t)/\sqrt{h_t}}{\bar{s}_3(x,t)} \right\|=  \tilde{\cO} \left( \frac{1}{\sqrt{h_t}}\left(\left(\frac{ \sqrt{\log(1/\epsilon) }\epsilon^{1/\holder} }{\reach }\right)^{\degree_0} +\epsilon^{(\holder+1)/\holder} \right)\right).
    \end{align*}
    Furthermore, by Lemma F.7 in \cite{oko2023diffusionmodelsminimaxoptimal}, there exists a feedforward network $\bar{\phi}$ with no more than $\cO(\log^2(1/\epsilon))$ layers, width bounded by $\cO(\log^3(1/\epsilon))$, at most non-zero $\cO(\log^4(1/\epsilon))$ neurons and weight parameters bounded by $\cO(\epsilon^{-2(\holder+1)/\holder})$, such that
\begin{align*}
    \left\| \bar{\phi}\left(\bar{s}_3(x,t), \bar{s}_4(x,t)\right) -  \frac{\bar{s}_4(x,t)/\sqrt{h_t}}{\bar{s}_3(x,t)} \right\| \leq \aerror^{(\holder+1)/\holder}.
\end{align*}
Now let $\nnlarge = \bar{\phi}\left(\bar{s}_3(x,t), \bar{s}_4(x,t)\right)$, so that for any $x \in \cK_t(\epsilon)$
\begin{align*}
    \left\|\nnlarge (x,t)- \nabla \log p_t(x) \right\|  &\leq \left\| \bar{\phi}\left(\bar{s}_3(x,t), \bar{s}_4(x,t)\right) -  \frac{\bar{s}_4(x,t)/\sqrt{h_t}}{\bar{s}_3(x,t)} \right\| + \left\|   \frac{\bar{s}_4(x,t)/\sqrt{h_t}}{\bar{s}_3(x,t)} -\nabla \log p_t(x) \right\| \\
    & =  \tilde{\cO} \left( \frac{1}{\sqrt{h_t}}\left(\left(\frac{ \sqrt{\log(1/\epsilon) }\epsilon^{1/\holder} }{\reach }\right)^{\degree_0} +\epsilon^{(\holder+1)/\holder} \right)\right).
\end{align*}
Combined with the network size of $\bar{s}_3$ and $\bar{s}_4$, we can conclude
$\nnlarge\in \cF(L_{\rm large},W_{\rm large},S_{\rm large},B_{\rm large},\cdot)$  with 
\begin{align*}
 L_{\rm large} = \tilde{\cO}\left(1 \right), ~ W_{\rm large} = \tilde{\cO}\left(  D^{\degree_0+1}\epsilon^{-d/\holder}  \right), ~
        S_{\rm large} =\tilde{\cO}\left(   D^{\degree_0+1} \epsilon^{-d/\holder} \right), ~B_{\rm large}= \tilde{\cO}\left(\epsilon^{-2(\degree_0+\log(1/\epsilon))/\holder}\right). 
\end{align*}

Finally, we derive the $L^2$ approximation error of $\nnlarge$, which can be written as
    \begin{align*}
        \left\|\nnlarge(x,t) - \nabla \log p_t(x) \right\|^2_{L^2(p_t)} &= \int_{x \in \RR^D}  \left\|\nnlarge(x,t) - \nabla \log p_t(x) \right\|^2 p_t(x) \ud x \\
        &\leq \left( \int_{x\in \cK_t(\epsilon)} + \int_{x\in\RR^D \setminus \cK_t(\epsilon)} \right) \left\|\nnlarge(x,t) - \nabla \log p_t(x) \right\|^2 p_t(x) \ud x.
    \end{align*}
    Here we can bound the integral within $\cK_t(\epsilon)$ via the approximation error of $\nnlarge$. 
    \begin{align*}
     \int_{x\in \cK_t(\epsilon)} \left\|\nnlarge(x,t) - \nabla \log p_t(x) \right\|^2 p_t(x) \ud x =  \tilde{\cO} \left( \frac{1}{h_t}\left(\left(\frac{ {\log(1/\epsilon) }\epsilon^{2/\holder} }{\reach^2 }\right)^{\degree_0} +\epsilon^{2(\holder+1)/\holder} \right)\right).
    \end{align*}
    For the integral over $\RR^D\setminus \cK_t(\epsilon)$, we have
\begin{align*}
      \int_{x\in\RR^D \setminus \cK_t(\epsilon)}  \left\|\nnlarge(x,t) - \nabla \log p_t(x) \right\|^2 p_t(x) \ud x &\leq 2 \int_{x\in\RR^D \setminus \cK_t(\epsilon)} \left( \left\|\nnlarge(x,t)  \right\|^2 + \left\|\nabla \log p_t(x) \right\|^2 \right) p_t(x) \ud x.
\end{align*}
Similar to the small noise case (Appendix \ref{sec:sM-approx}), we clip $\nnlarge(x,t)$ by $R =\cO(\sqrt{\log(1/\epsilon)/h_t})$ via  a feedforward layer as constructed in Lemma F.5 of \cite{oko2023diffusionmodelsminimaxoptimal}. Them Lemma \ref{lemma:truncate-x} yields
    \begin{align*}
         \int_{x\in\RR^D \setminus \cK_t(\epsilon)} \left\|\nnlarge(x,t)  \right\|^2 p_t(x) \ud x \leq \epsilon^{2(\holder+1)/\holder} \log(1/\epsilon)/h_t.
    \end{align*}
    Moreover, \eqref{eq:score-tail-bound} shows 
\begin{align*}
      \int_{x\in\RR^D \setminus \cK_t(\epsilon)}  \left\|\nabla \log p_t(x) \right\|^2  p_t(x) \ud x   \lesssim \frac{\epsilon^{2(\holder+1)/\holder}}{h_t}.
\end{align*}
Therefore, we can conclude that the network $\nnlarge(x,t)$ has the following $L^2$ approximation error for any $t \in [t_{\rm large},\bigt]$, 
\begin{align*}
        \left\|\nnlarge(x,t) - \nabla \log p_t(x) \right\|^2_{L^2(p_t)} = \tilde{\cO} \left( \frac{1}{h_t}\left(\left(\frac{ \epsilon^{2/\holder} }{\reach^2 }\right)^{\degree_0} +\epsilon^{2(\holder+1)/\holder} \right)\right).
    \end{align*}

\subsection{Proof of Theorem~\ref{thm:approx}}\label{sec:approx_thm_proof}

We prove Theorem \ref{thm:approx} by constructing a single network that well approximates the score function $\nabla \log p_t(x)$ for any time $t \in [t_0, \bigt]$. First, we impose the following conditions on the approximation error $\epsilon>0$ so that the small noise and large noise regimes overlap:
\begin{align}\label{cond:epsilon}
    \epsilon \leq \min\left\{\frac{ \sqrt{\eta \reach}}{rD^{1/4}\sqrt{B}}, \frac{\eta\min\{\reach, \sqrt{\reach}\}}{16\max\{\sqrt{D}, 4L_\expp L_\logg\}},1 \right\}^\holder.
\end{align}
This requirement on $\epsilon$ implies that the small noise regime conditions in \eqref{cond:time} are dominated by $h_t\leq \epsilon^{2/\holder}$. Therefore, given $t_{\rm small} = \log\frac{1}{1-\epsilon^{2/\holder}}$ and $t_{\rm large} = \log\frac{1}{1-\epsilon^{2/\holder}/4}$, the small noise regime is exactly $[t_0, t_{\rm small}]$, which overlaps with the large noise regime $[t_{\rm large},\bigt]$ in $[t_{\rm large}, t_{\rm small}]$.

Next, we construct the time switching network to incorporate networks $\nnsmall$ in Lemma \ref{lemma:approx-small} and $\nnlarge$ in Lemma \ref{lemma:approx-large}.
For any $t \in [t_0, T]$, we define two switching function as
\begin{align*}
{\rm SW}_{\rm small}(t) &= \frac{1}{t_{\rm small} -t_{\rm large}}\relu\big( (t_{\rm small} -t_{\rm large}) - \relu(t-t_{\rm large}) + \relu(t-t_{\rm small}) \big), \\
{\rm SW}_{\rm large}(t) &= \frac{1}{t_{\rm small} -t_{\rm large}}\relu\big(  \relu(t-t_{\rm large}) - \relu(t-t_{\rm small}) \big).
\end{align*}
Here ${\rm SW}_{\rm small}(t), {\rm SW}_{\rm large}(t) \in [0,1]$, ${\rm SW}_{\rm small}(t) = 0$ for all $t \geq t_{\rm small}$, ${\rm SW}_{\rm large}(t) =0$ for all $t \leq t_{\rm large}$, and ${\rm SW}_{\rm small}(t)+ {\rm SW}_{\rm large}(t)=1$ for all $t$. Moreover, we note that
\begin{align*}
    {\rm SW}_{\rm small}(t), {\rm SW}_{\rm large}(t)  \in \cF\left(3,2,8, \max\{t_{\rm small},(t_{\rm small}-t_{\rm large})^{-1}\}, \cdot\right).
\end{align*}
Now we use the switching functions to construct the score network as
\begin{align*}
\tilde{s}(x, t) = {\rm SW}_{\rm small}(t) \cdot  \nnsmall(x,t) + {\rm SW}_{\rm large}(t) \cdot \nnlarge(x,t).
\end{align*}
Here we set $\degree_0 = \degree$ in Lemma \ref{lemma:approx-large}. We will instantiate the choice of $\degree>0$ later.
Utilizing Lemma \ref{lemma:nn-prod}, we can easily implement multiplication and aggregation by feedforward networks. Thereby, the construction $\tilde{s}(x, t)$ yields a network $\bar{s} \in \cF(L,W,S,B,\cdot)$ with
\begin{align*}
L =\tilde{\cO}\left(\degree^3   \right), \quad
    W = \tilde{\cO}\left(D^{\degree} \degree^3 \epsilon^{-d/\holder}\right), \quad
        S = \tilde{\cO}\left(\degree D^{\degree}\epsilon^{-d/\holder} \right), \quad B= \cO\left( h_{t_0}^{-\degree} \epsilon^{-2\log(1/\epsilon)/\holder} \right).
    \end{align*}
    Next, we derive the $L^2$ error of $\bar{s}(x,t)$ approximating the score function $\nabla \log p_t(x)$ when $t \in [t_0,\bigt]$.
    \begin{align*}
        \left\|\bar{s}(x,t) - \nabla \log p_t(x) \right\|^2_{L^2(p_t)} &\leq \ind\left( t_0 \leq t \leq t_{\rm small} \right) \left\|\nnsmall- \nabla \log p_t(x) \right\|^2_{L^2(p_t)} \\
        &\quad+  \ind\left( t_{\rm large} < t \leq \bigt \right) \left\|\nnlarge- \nabla \log p_t(x) \right\|^2_{L^2(p_t)}.
    \end{align*}
Applying the approximation errors in   Lemmas \ref{lemma:approx-small} and \ref{lemma:approx-large}, 
    \begin{align*}
        \left\|\nnsmall (x,t)  - \nabla \log p_t(x) \right\|_{L^2(P_t)}^2  =\tilde{\cO} \left( \frac{D^{2\degree+d+2}}{h_t} \left(  \left( \frac{D\epsilon^{2/\holder}}{\reach^2} \right)^{\degree}  +\epsilon^2\right) \right), \quad \text{for any } t\in [t_0, t_{\rm small}],
    \end{align*}
    and 
    \begin{align*}
        \left\|\nnlarge(x,t) - \nabla \log p_t(x) \right\|^2_{L^2(P_t)} = \tilde{\cO} \left( \frac{1}{h_t}\left(\left(\frac{ \epsilon^{2/\holder} }{\reach^2 }\right)^{\degree} +\epsilon^{2(\holder+1)/\holder} \right)\right),\quad \text{for any } t\in [t_{\rm large}, \bigt],
\end{align*}
we can derive
   \begin{align*}
        \left\|\bar{s}(x,t) - \nabla \log p_t(x) \right\|^2_{L^2(p_t)} 
        =\tilde{\cO} \left( \frac{D^{2\degree+d+2}}{h_t} \left(  \left( \frac{\epsilon^{2/\holder}}{\reach^2} \right)^{\degree}  +\epsilon^2\right) \right), \quad \text{for any } t\in [t_0, \bigt].
    \end{align*}

In particular, when $\epsilon< \reach^\holder$, if we take $\degree = \lceil \holder (1+\log(\reach^\holder)/\log(1/\epsilon))^{-1}\rceil$ so that $(\epsilon^{1/\holder}/\reach)^{\degree} \leq \epsilon$, we will have
    \begin{align*}
        \left\|\bar{s}(x,t) - \nabla \log p_t(x) \right\|^2_{L^2(p_t)}=\tilde{\cO} \left( \frac{1}{h_t} D^{2\degree+d+2}\epsilon^2\right).
    \end{align*}

\section{Proofs in Section \ref{sec:stat_rate}}

\subsection{Proof of Theorem \ref{thm:generalization}}\label{sec:proof-generalization}
In this proof, we view the training samples as random quantities and slightly abuse the notation to denote $\cD = \{X_{i}\}_{i=1}^n$, the estimated score function is the empirical risk minimizer defined by
\begin{align*}
    \hat{s} = \argmin_{s\in\cF} \hat{\cL}(s), \quad \text{where }  \hat{\cL}(s) := \frac{1}{n}\sum_{i=1}^n \ell(X_{i};s).
\end{align*}
The population risk is denoted as $\cL(s) = \EE_{X 
\sim \dist}[\ell(X;s)]$. According to Theorem \ref{thm:approx},  $\hat{s}$ can be taken so that $\|\hat{s}(\cdot,t)\|_\infty \lesssim 1/\sqrt{h_t}$. Thereby we limit $\cF$ into $\tilde{\cF}$:
\begin{align}\label{eq:def-F-limit}
    \tilde{\cF}(L,W,S,B,C_R) := \left\{s \in \cF(L,W,S,B,R): \|s(\cdot,t)\|_\infty\leq \frac{C_R}{\sqrt{h_t}} \right \}.
\end{align}

Recall from Section~\ref{sec:pre} that $\cL(s)$ is equivalent (up to a constant) to
\begin{align*}
    \cR(s) = \frac{1}{T-t_0} \int_{t_0}^T \EE_{X \sim \dist} \left[\EE_{X_t|X}\left[ \left\| s(X_t,t) - \nabla\log p_t(X_t)\right\|^2\right] \right]\ud t.
\end{align*}
Similarly, we denote the empirical version of $\cR(s)$ as
\begin{align*}
    \hat{\cR}(s) = \frac{1}{T-t_0} \int_{t_0}^T \frac{1}{n} \sum_{i=1}^n \EE_{X_t|X_i} \left[\left\| s(X_t,t) - \nabla\log p_t(X_t)\right\|^2 \right]\ud t.
\end{align*}
It is convenient to recenter the empirical loss $\ell$ by defining $\lscale$ as
\begin{align*}
    \lscale(x;s) = \ell(x;s) -  \ell(x;s^*) \quad\text{with} \quad s^*(x, t) = \nabla \log p_t(x).
\end{align*}
Then we define the centered empirical risk as
\begin{align*}
\hat{\cL}_0(s) = \frac{1}{n}\sum_{i=1}^n \lscale(X_i; s).
\end{align*}
Notably, minimizing $\hat{\cL}_0(s) $ is equivalent to  minimizing $\hat{\cL}(s)$. Moreover, the population centered risk $\cL_0(s) = \EE_{X \sim \dist} [\lscale(X; s)]$ is simply $\cR(s)$, i.e., $\cL_0(s) = \cR(s)$.

We denote $\EE_{\cD}$ as the expectation over the randomness in the training samples $\cD$. To bound the expected generalization error $\EE_{\cD}[\cR(\hat{s})]$, we consider the following decomposition
\begin{align*}
    \EE_{\cD}[\cR(\hat{s})] = \underbrace{\EE_{\cD}[\cR(\bar{s})]}_{\mathrm{(I)}} + \underbrace{\EE_{\cD}[\cR(\hat{s})] - \EE_{\cD}[\cR(\bar{s})] }_{\mathrm{(II)}}.
\end{align*}
Here, network $\bar{s}$ is given in Theorem \ref{thm:approx} as a constructed approximator to the ground truth score function. For any $\epsilon >0 $ satisfting \eqref{cond:epsilon}, the approximation guarantee in Theorem \ref{thm:approx} gives rise to
\begin{align*}
        ({\rm I}) \lesssim  D^{ 2\degree+d+2} \epsilon^2, \quad \quad \text{with} \quad \degree = \left\lceil \frac{\holder \log \frac{1}{\epsilon}}{\log \frac{1}{\epsilon} + \holder \log \reach}\right\rceil.
    \end{align*}  
We will choose $\epsilon$ later to optimally balance the error terms. The second term $\mathrm{(II)}$, measuring the difference in generalization errors of $\hat{s}$ and $\bar{s}$, can be further decomposed as
\begin{align*}
   (\textrm{II}) &\overset{(i)}{=} \EE_{\cD}\left[\cL_0(\hat{s}) - \cL_0(\bar{s})\right] \\
   & = \underbrace{\EE_{\cD}\left[\cL_0(\hat{s}) - \hat{\cL}_0(\hat{s})\right]}_{(\textrm{II-A})} + \underbrace{\EE_{\cD}\left[\hat{\cL}_0(\hat{s})-\hat{\cL}_0(\bar{s}) \right]}_{(\textrm{II-B})} + \EE_{\cD}\left[\hat{\cL}_0(\bar{s}) -  \cL_0(\bar{s})\right] \\
   & \overset{(ii)}{=} \underbrace{\EE_{\cD}\left[\cL_0(\hat{s}) - \hat{\cL}_0(\hat{s})\right]}_{(\textrm{II-A})} + \underbrace{\EE_{\cD}\left[\hat{\cL}_0(\hat{s})-\hat{\cL}_0(\bar{s}) \right]}_{(\textrm{II-B})},
\end{align*}
where equality $(i)$ invokes the identity $\cR(s) = \cL_0(s)$, and equality $(ii)$ follows since $\bar{s}$ is independent of $\cD$. By the definition of $\hat{s}$, we have
\begin{align*}
    \hat{\cL}_0(\hat{s})-\hat{\cL}_0(\bar{s}) =\hat{\cL}(\hat{s})-\hat{\cL}(\bar{s}) \leq 0.
\end{align*}
Therefore, we have $(\textrm{II-B}) \leq 0$ and it remains to bound $(\textrm{II-A})$. We introduce a set of i.i.d. ghost samples $\bar{\cD} = \{\bar{X}_i\}_{i=1}^n$ following the same distribution but independent of $\cD$. For an arbitrary $a \in (0, 1)$, we have
\begin{align*}
   (\textrm{II-A}) & = \EE_{\cD}\left[  \cL_0(\hat{s}) - \hat{\cL}_0(\hat{s}) - a \cR(\hat{s}) \right] + a \EE_{\cD}[\cR(\hat{s})] \\
   &=  \EE_{\cD}\left[ \EE_{X\sim \dist} [\lscale(X;\hat{s})] -\frac{1}{n}\sum_{i=1}^n \lscale(X_{i};\hat{s}) - a \cR(\hat{s}) \right] + a \EE_{\cD}[\cR(\hat{s})] \\
   & \overset{(i)}{=} \EE_{\cD, \bar{\cD}}\left[\frac{1}{n}\sum_{i=1}^n [\lscale(\bar{X}_i; \hat{s}) - \lscale(X_{i};\hat{s})] - a \cR(\hat{s}) \right] + a \EE_{\cD}[\cR(\hat{s})] \\
   &\overset{(i)}{\leq} \underbrace{\EE_{\cD,\bar{\cD}} \left[\sup_{s\in\tilde{\cF}} \frac{1}{n}\sum_{i=1}^n \left[\lscale(\bar{X}_i;s)-\lscale(X_{i};s) \right] - a \cR(s) \right]}_{(\spadesuit)} + a \EE_{\cD}[\cR(\hat{s})],
\end{align*}
where equality $(i)$ holds since $\bar{\cD}$ is independent of $\cD$, and inequality $(ii)$ holds as $\hat{s} \in \tilde{\cF}$. We temporally neglect the $a \EE_{\cD}[\cR(\hat{s})]$ term and bound $(\spadesuit)$. Denote $\cG = \{\ell_0(\cdot; s): s \in \tilde{\cF}\}$ be a function class induced by the score network $\tilde{\cF}$. For any $\delta > 0$, we construct a covering on $\cG$ with respect to the $L^\infty$ norm. That is, we select a collection of representative score networks $s_j$ for $j = 1, \dots, \cN(\delta, \cG, \norm{\cdot}_\infty)$, such that for any $s \in \tilde{\cF}$, there exists an $s_j$ with $\norm{\lscale(\cdot;s)-\lscale(\cdot;s_j)}_\infty=\norm{\ell(\cdot; s) - \ell(\cdot; s_j)}_\infty \leq \delta$. The count $\cN(\delta, \cG, \norm{\cdot}_\infty)$ is known as the covering number of $\cG$. Using the covering of $\cG$, we can replace the supremum in $(\spadesuit)$ by a maximum over a finite set.

For any $s \in \tilde{\cF}$ and its close representation $s_j$ in the covering, we have
\begin{align*}
    |\cR(s) - R(s_{j}) |= |\cL_0(s) - \cL_0(s_{j}) |= \left|\EE_{\dist}[\lscale(X; s) - \lscale(X; s_{j})] \right| \leq \delta.
\end{align*}
Therefore, we have
\begin{align}
(\spadesuit) \leq \EE_{\cD,\bar{\cD}} \left[\max_{j=1, \dots, \cN(\delta, \cG, \norm{\cdot}_\infty)}  \frac{1}{n}\sum_{i=1}^n \left[\ell_0(\bar{X}_i;s_j)-\ell_0(X_{i};s_j) \right] -a \cR(s_j)\right] +(a+2)\delta. \notag 
\end{align}
To further simplify the notation, we denote a variable $h_i(s) = \lscale(\bar{X}_i;s)- \lscale(X_i;s)$ for an arbitrary fixed $s\in\tilde{\cF}$. Note that $\EE[h_i]=0$ and
\begin{align*}
    \Var[h_i(s)]  &\leq \EE\left[\left|\lscale(\bar{X}_i;s)- \lscale(X_i;s)\right|^2\right] \\
    &=  \EE\left[\left|\ell(\bar{X}_i;s)- \ell(\bar{X}_i;s^*)-[\ell(X_i;s)-\ell(X_i;s^*)]\right|^2\right] \\
    &\leq  2\EE\left[\left|\ell(\bar{X}_i;s)- \ell(\bar{X}_i;s^*)\right|^2\right] + 2\EE\left[\left|\ell(X_i;s) -\ell(X_i;s^*)\right|^2\right].
\end{align*}
Applying Lemma \ref{lemma:var-hi} with $C_\ell = 16(C_fC_R^2+D)$,  we have
\begin{align}\label{eq:variance_bound}
     \Var[h_i(s)] \leq C_\ell \cR(s).
\end{align}
and 
\begin{align*}
    |h_i(s)| \leq |\ell(\bar{X}_i;s)|+| \ell(\bar{X}_i;s^*)|+|\ell(X_i;s)|+|\ell(X_i;s^*)| \leq C_\ell.
\end{align*}
Using \eqref{eq:variance_bound}, we deduce
\begin{align}\label{eq:spade_bound}
(\spadesuit) \leq \EE_{\cD,\bar{\cD}} \left[ \max_{j=1, \dots, \cN(\delta, \cG, \norm{\cdot}_\infty)} \frac{1}{n}\sum_{i=1}^n \big(h_i(s_j) - \frac{a}{C_\ell} \Var[h_i(s_j)]\big) \right] + (a+2)\delta.
\end{align}
For a fixed $s_j$ and any $\lambda \in (0, 3n/C_\ell)$, by Taylor expansion, we have
\begin{align*}
    \EE \left[ \exp\left( \frac{\lambda}{n}h_i(s_j) \right) \right] &=  \EE \left[ 1 + \frac{\lambda}{n}h_i(s_j) + \sum_{k=2}^\infty\frac{(\lambda/n)^k h_i(s_j)^k}{k!} \right] \\
    &\leq \EE \left[ 1 + \frac{\lambda}{n}h_i(s_j) + \frac{h_i^2(s_j) (\lambda/n)^2}{2} \sum_{k=2}^\infty\frac{(\lambda/n)^{k-2} C_\ell^{k-2}}{3^{k-2}} \right].
\end{align*}
where the last inequality follows from $|h_i(s)| \leq C_{\ell}$. Summing up the geometric series, we obtain
\begin{align}\label{eq:E_exp_bound}
    \EE \left[ \exp\left( \frac{\lambda}{n}h_i(s_j) \right) \right] 
    &\leq \EE \left[ 1 + \frac{\lambda}{n}h_i(s_j) +  \frac{ 3\lambda^2 h_i^2(s_j)}{6n^2-2\lambda n C_\ell} \right] \nonumber \\
    &= 1+  \frac{3 \lambda^2 \EE [h_i^2(s_j)]}{6n^2-2\lambda n C_\ell}  \nonumber \\
    &\leq \exp\left(  \frac{3 \lambda^2 \Var [h_i(s_j)]}{6n^2-2\lambda n C_\ell}  \right),
\end{align}
where the equality follows from $\EE[h_i(s_j)] = 0$ and the last inequality invokes the fact $1+x \leq \exp(x)$. Now consider the first term in the right-hand side of \eqref{eq:spade_bound}. We have
\begin{align*}
    & \quad \exp\left( \lambda \EE_{\cD,\bar{\cD}} \left[ \max_{j=1, \dots, \cN(\delta, \cG, \norm{\cdot}_\infty)} \frac{1}{n}\sum_{i=1}^n \big(h_i(s_j) - \frac{a}{C_\ell} \Var[h_i(s_j)]\big) \right]\right) \\
    &\overset{(i)}{\leq} \EE_{\cD,\bar{\cD}} \left[\exp\left( \max_{j=1, \dots, \cN(\delta, \cG, \norm{\cdot}_\infty)}  \frac{\lambda}{n}\sum_{i=1}^n \big(h_i(s_j) - \frac{a}{C_\ell}\Var[h_i(s_j)] \big) \right)\right] \\
    & \leq \sum_{j=1, \ldots, \cN(\delta, \cG, \norm{\cdot}_\infty)} \EE_{\cD,\bar{\cD}} \left[  \exp\left( \frac{\lambda}{n}\sum_{i=1}^n \big(h_i(s_j) - \frac{a}{C_\ell} \Var[h_i(s_j)] \big)  \right) \right]\\
    &\overset{(ii)}{\leq}  \sum_{j=1, \dots, \cN(\delta, \cG, \norm{\cdot}_\infty)}  \exp\left( \sum_{i=1}^n \left(  \frac{3 \lambda^2}{6n^2-2\lambda n C_\ell}  -\frac{a \lambda}{nC_\ell} \right) \Var[h_i(s_j)]  \right),
\end{align*}
where $(i)$ utilizes Jensen's inequality and $(ii)$ invokes \eqref{eq:E_exp_bound}. We choose $\lambda = \frac{6 a n}{(2a+3)C_\ell} < \frac{3n}{C_\ell}$ so that $\frac{3 \lambda^2}{6n^2-2\lambda n C_\ell} =  \frac{a\lambda}{n C_\ell}$, which leads to 
\begin{align*}
     \exp\left( \lambda \EE_{\cD,\bar{\cD}} \left[ \max_{j=1, \dots, \cN(\delta, \cG, \norm{\cdot}_\infty)} \frac{1}{n}\sum_{i=1}^n \big(h_i(s_j) - \frac{a}{C_\ell} \Var[h_i(s_j)]\big) \right]\right) \leq \cN(\delta, \cG, \norm{\cdot}_\infty)
\end{align*}
This further implies     
\begin{align*}
(\spadesuit) \leq \frac{1}{\lambda} \log \cN(\delta) + (a+2)\delta.
\end{align*}
Bounding the covering number as \cite[Lemma C.2]{oko2023diffusionmodelsminimaxoptimal}, we obtain
\begin{align*}
(\spadesuit) \lesssim \frac{(2a+3)C_\ell}{6 a n}  SL\log(\delta^{-1}LW(B\vee 1)n) + (a+2)\delta.
\end{align*}
Substituting $(\spadesuit)$ into $(\textrm{II-A})$, we have
\begin{align*}
(\textrm{II-A}) - a\EE_\cD[\cR(\hat{s})] \lesssim \frac{(2a+3)C_\ell}{6 a n}  SL\log(\delta^{-1}LW(B\vee 1)n) + (a+2)\delta.
\end{align*}
Combining the bounds of $(\textrm{II-A})$ and $(\textrm{II-B})$ yields
\begin{align*}
(\textrm{II}) - a \EE_{\cD}[\cR(\hat{s})] \lesssim \frac{(2a+3)C_\ell}{3 a n}  SL\log(\delta^{-1}LW(B\vee 1)n) + (a+2)\delta.
\end{align*}
Further combining $(\textrm{I})$ and $(\textrm{II})$ gives rise to
\begin{align*}
   \EE_{\cD}[\cR(\hat{s})] - a \EE_{\cD}[\cR(\hat{s})] & = (\textrm{I})+(\textrm{II}) - a \EE_{\cD}[\cR(\hat{s})] \\
   & \lesssim D^{ 2\degree+d+2} \epsilon^2 + \frac{(2a+3)C_\ell}{6 a n}  SL\log(\delta^{-1}LW(B\vee 1)n) + (a+2)\delta .
\end{align*}
We set $a=1/2$ and derive
\begin{align}\label{eq:E_R_shat}
\EE_{\cD}[\cR(\hat{s})] \lesssim D^{ 2\degree+d+2} \epsilon^2 + \frac{C_\ell}{n}  SL\log(\delta^{-1}LW(B\vee 1)n) + \delta .
\end{align}
Plugging the network configuration from Theorem~\ref{thm:approx} into \eqref{eq:E_R_shat} gives rise to
\begin{align*}
         \EE_{\cD}[\cR(\hat{s})]
        &\lesssim D^{ 2\degree+d+2} \epsilon^2 + \frac{\degree^4 D^{\degree+1}\epsilon^{-d/\holder}}{n}\left(
        \log \left( \degree^3 D^{\degree} /\delta\right) +\log(1/(t_0\epsilon))  \right) + \delta.
\end{align*}
We take $\delta=1/n$, $\epsilon = n^{-\holder/(d+2\holder)}$ and $t_0 = n^{-c}$ for some constant $c>0$. We simplify the bound on $\EE_{\cD}[\cR(\hat{s})]$ as
\begin{align*}
\EE_{\cD}[\cR(\hat{s})]
&\lesssim D^{ 2\degree+d+2}  n^{-\frac{2\holder}{d+2\holder}}  + \degree^4 D^{\degree+1} n^{-\frac{2\holder}{d+2\holder}} \lesssim D^{ 2\degree+d+2}n^{-\frac{2\holder}{d+2\holder}} .
\end{align*}
Moreover, we can rewrite $\degree$ as 
\begin{align*}
    \degree = \left\lceil \frac{\holder \log \frac{1}{\epsilon}}{\log \frac{1}{\epsilon} + \holder \log \reach}\right\rceil = \left\lceil \frac{\holder }{1+ (d+2\holder)\log\reach/\log n}\right\rceil.
\end{align*}
The proof is complete.

\begin{lemma}\label{lemma:var-hi}
    Suppose Assumption \ref{assump:density} holds. Take $t_0 = e^{-T}$. Then for any $s\in \tilde{\cF}$, with $\tilde{\cF}$ defined in \eqref{eq:def-F-limit}, we have
    \begin{align*}
        \sup_{x\in\cM} \ell(x;s) \leq 4(C_fC_R^2+D),
    \end{align*}
    and 
    \begin{align*}
        \EE_{X_0\sim P_{\rm data}}|\ell(X_0;s)-\ell(X_0;s^*)|^2 = 16(C_f C_R^2+D) \cdot \cR(s),
    \end{align*}
    where $s^*(x, t) = \nabla \log p_t(x)$ is the ground-truth score function.
\end{lemma}
\begin{proof} 
We first derive the uniform bound on $\ell(\cdot;s)$. Since $s\in\tilde{\cF}$, 
we have $\|s(\cdot,t)\|_\infty \leq C_R/\sqrt{h_t}$. Then Assumption \ref{assump:density} guarantees
    \begin{align}\label{eq:s-l2-bound}
        \EE_{X_t|x_0} \|s(X_t,t)\|^2 \leq C_f \EE_{X_t} \|s(X_t,t)\|^2\leq C_f C_R^2/h_t.
    \end{align}
    By the definition of the loss $\ell$, we have
    \begin{align*}
        \ell(x_0;s) &= \frac{1}{T-t_0}\int_{t_0}^T \EE_{X_t|X_0=x_0 }\left\|s(X_t,t)-\nabla \log p_t(X_t|X_0) \right\|^2 \ud t\\
        &\leq \frac{1}{T-t_0}\int_{t_0}^T  \EE_{X_t \sim N(\alpha_t x_0,h_tI_D) } \left(2\left\|s(X_t,t) \right\|^2 + 2\left\| \frac{X_t-\alpha_t x_0}{h_t}\right\|^2 \right)\ud t.
    \end{align*}
    Using the $L^2$-bound on $s$ gives
    \begin{align*}
        \ell(x_0;s) 
        &\leq \frac{1}{T-t_0}\int_{t_0}^T   \left(\frac{2C_f C_R^2}{h_t} + 2\EE_{z \sim N(0,I_D)} \frac{\|z\|^2}{h_t} \right)\ud t = \frac{2(C_f C_R^2+D)}{T-t_0} \log(\frac{e^T-1}{e^{t_0}-1}).
    \end{align*}
    Letting $t_0 = e^{-T}$, we further have $\frac{1}{T-t_0} \log\left( \frac{e^T-1}{e^{t_0}-1}\right) \leq 2$. This yields
    \begin{align*}
    \sup_{x\in\cM} \ell(x;s) \leq 4(C_fC_R^2+D).
    \end{align*}

Next, we bound the difference between $s$ and $s^*$ given the same input $x_0\in\cM$.
    \begin{align*}
        \ell(x_0;s)-\ell(x_0;s^*)  &= \frac{1}{T-t_0}\int_{t_0}^T \EE_{X_t|x_0}[\|s(X_t,t) - \nabla \log p_t(x|x_0) \|^2 - \|s^*(X_t,t) - \nabla \log p_t(x|x_0) \|^2 ] \ud t\\
        &=\frac{1}{T-t_0}\int_{t_0}^T \EE_{X_t|x_0}\left[(s(X_t,t) - s^*(X_t,t))^\top (s(X_t,t) + s^*(X_t,t)- 2\nabla \log p_t(x|x_0) )\right]\ud t.
    \end{align*}
    By Cauchy-Schwartz inequality, we have
    \begin{align*}
        |\ell(x_0;s)-\ell(x_0;s^*)| \leq& \sqrt{\frac{1}{T-t_0}\int_{t_0}^T\EE_{X_t|x_0} \|s(X_t,t) - s^*(X_t,t)\|^2 \ud t} \\
        &\cdot \sqrt{\frac{1}{T-t_0}\int_{t_0}^T\EE_{X_t|x_0} \|s(X_t,t) + s^*(X_t,t)- 2\nabla \log p_t(x|x_0) \|^2 \ud t} .
    \end{align*}
    Utilizing the $L^2$-bounds of $s$ in \eqref{eq:s-l2-bound}, we can get
    \begin{align*}
        &\frac{1}{T-t_0}\int_{t_0}^T\EE_{X_t|x_0} \|s(X_t,t) + s^*(X_t,t)- 2\nabla \log p_t(x|x_0) \|^2 \ud t \\
        &\leq \frac{1}{T-t_0}\int_{t_0}^T \left( \frac{8C_f C_R^2}{h_t} + 2 \EE_{X_t|x_0} \left\| \frac{2(X_t -\alpha_t x_0)}{h_t}  \right\|^2\right) .
    \end{align*}
    Recall $X_t |x_0 \sim N(\alpha_t x_0, h_tI_D)$. This gives
    \begin{align*}
         &\frac{1}{T-t_0}\int_{t_0}^T\EE_{X_t|x_0} \|s(X_t,t) + s^*(X_t,t)- 2\nabla \log p_t(x|x_0) \|^2 \ud t \\&\leq \frac{1}{T-t_0}\int_{t_0}^T \left( \frac{8C_f C_R^2}{h_t} + 2 \EE_{z \sim N(0,I_D)} \left\| \frac{2\sqrt{h_t}z}{h_t}  \right\|^2\right) \ud t\\
         &= \frac{8}{T-t_0}\int_{t_0}^T \left( \frac{C_f C_R^2}{h_t} +  \EE_{z \sim N(0,I_D)} \frac{\|z\|^2}{h_t}  \right) \ud t\\
         & = \frac{8}{T-t_0}\int_{t_0}^T \frac{C_R^2+D}{h_t} \ud t\\
         &=\frac{8(C_f C_R^2+D)}{T-t_0} \log\left( \frac{e^T-1}{e^{t_0}-1}\right).
    \end{align*}
    Letting $t_0 = e^{-T}$, we further have $\frac{1}{T-t_0} \log\left( \frac{e^T-1}{e^{t_0}-1}\right) \leq 2$. This gives rise to
    \begin{align*}
        |\ell(x_0;s)-\ell(x_0;s^*)|^2 \leq 16\frac{C_f C_R^2+D}{T-t_0}\int_{t_0}^T\EE_{X_t|x_0} \|s(X_t,t) - s^*(X_t,t)\|^2 \ud t .
    \end{align*}
    Finally, taking expectation over $X_0\sim P_{\rm data}$ yields
    \begin{align*}
        \EE_{X_0\sim P_{\rm data}}|\ell(X_0;s)-\ell(X_0;s^*)|^2 &\leq 16\frac{C_f C_R^2+D}{T-t_0}\int_{t_0}^T \EE_{X_0\sim P_{\rm data}}\EE_{X_t|X_0} \|s(X_t,t) - \nabla\log p_t(X_t)\|^2 \ud t \\
        &= 16(C_f C_R^2+D) \cR(s).
    \end{align*}
    The proof is complete.
\end{proof}

\subsection{Proof of Theorem~\ref{thm:distribution}}\label{sec:dist}

Built upon the score estimation results in Theorem \ref{thm:generalization}, we aim to establish the convergence rate of the estimated distribution $\hat{P} = \mathrm{Law}(\hat{Y}_{\bigt-t_0})$ in $W_1$-distance.
Firstly, we introduce $\{\bar{Y}_t\}_{t=0}^{\bigt-t_0}$ that replaces $\hat{Y}_0 \sim N(0,I_D)$ by $\Bar{Y}_0 \sim p_{\bigt}$, i.e.
\begin{align*}
\diff  \Bar{Y}_t & = \left[\frac{1}{2} \Bar{Y}_t + \hat{s}(\Bar{Y}_{t}, T-t)\right] \diff t + \diff \overline{B}_t, \quad \Bar{Y}_0 \sim p_{\bigt}.
\end{align*}
By the triangle inequality, we have the following decomposition,
\begin{align*}
    \EE[W_1(\hat{P},P_{\rm data})] = \EE[W_1(\hat{Y}_{\bigt-t_0},X_0)] &\leq \EE[W_1(\hat{Y}_{\bigt-t_0},\Bar{Y}_{\bigt-t_0})] +  \EE[W_1(\Bar{Y}_{\bigt-t_0},Y_{\bigt-t_0})] +\EE[W_1(Y_{\bigt-t_0},X_0)]\\
    &= \EE[W_1(\hat{Y}_{\bigt-t_0},\Bar{Y}_{\bigt-t_0})] +  \EE[W_1(\Bar{Y}_{\bigt-t_0},Y_{\bigt-t_0})] +\EE[W_1(X_{t_0},X_0)].
\end{align*}
Let $X \sim P_{\rm data}$ and $Z \sim N(0,I_D)$. Then we have
\begin{align*}
    W_1(X_{t_0},X_0) \leq \EE \| \alpha_t X + \sqrt{h_t} Z - X \| \leq (1-\alpha_t) \EE \|X\| + \sqrt{h_t} \EE \|Z\| \leq 
    \sqrt{D}(B_{\cM}+1)\sqrt{h_t} \lesssim \sqrt{t_0},
\end{align*}
where we apply Assumption \ref{assump:manifold} and $h_t = 1 - \alpha_t^2$. Moreover, combining Theorem D.7 in \cite{oko2023diffusionmodelsminimaxoptimal} with Lemmas  \ref{lemma:approx-large}-\ref{lemma:approx-small} and Theorem \ref{thm:generalization}, we have $ \EE[W_1(\hat{Y}_{\bigt-t_0},\Bar{Y}_{\bigt-t_0})] \lesssim \exp(-\bigt)$, and 
\begin{align*}
    \EE[W_1(\Bar{Y}_{\bigt-t_0},Y_{\bigt-t_0})] \lesssim  \sqrt{n^{-\frac{2}{d+2\holder}} D^{2\degree+d+2}  n^{-\frac{2\holder}{d+2\holder}} } + n^{-\frac{\holder+1}{d+2\holder}} \lesssim D^{ \degree + d/2 +1}  n^{-\frac{\holder+1}{d+2\holder}}.
\end{align*}
Now combining all the pieces together, we obtain
\begin{align*}
    \EE[W_1(\hat{P},P_{\rm data})]  \lesssim \sqrt{t_0} + \exp(-\bigt) + D^{ \degree + d/2 +1}  n^{-\frac{\holder+1}{d+2\holder}}.
\end{align*}
It suffices to take $t_0 =  n^{-\frac{2(\holder+1)}{d+2\holder}}$ and $\bigt= \log n$ to conclude the proof.
\section{Omitted Lemmas in Appendix~\ref{sec:proof-approx}}\label{appendix:lemmas4B}
We present formal statements and proofs of supporting lemmas in Appendix~\ref{sec:proof-approx}.

\subsection{Lemma Statements and Proofs in Appendix~\ref{subsec:local-poly}}

\begin{lemma}\label{lemma:lowp-bound}
For any fixed time $t>0$ and $\delta \in (0,1)$, take $\cK_t(\delta) = \cK(\alpha_t\cM, 2\sqrt{D h_t \log(1/\delta)}).$ Let $\bar{s}(x,t)$ be any function such that $ \norm{\bar{s}(\cdot,t) }_{L^\infty(\cK_t(\delta))} =\cO\left(\sqrt{\log(1/\delta)/h_t}\right)$. Then we have
    \begin{align*}
        \int_{\RR^D \backslash \cK_t(\delta)} \norm{\bar{s}(x,t)-\nabla \log p_t(x)}^2 \ud P_t(x) =\cO \left(  \frac{D\log(1/\delta)}{h_t}\delta^2\right).
    \end{align*}
\end{lemma}
\begin{proof}[Proof of \ref{lemma:lowp-bound}]
    We begin by applying the the inequality $\norm{a-b}^2 \leq 2(\norm{a}^2 +\norm{b}^2)$ to bound the integral:
    \begin{align*}
      \int_{\RR^D \backslash \cK_t(\delta)} \norm{\bar{s}(x,t)-\nabla \log p_t(x)}^2 \ud P_t(x)  &\leq 2 \int_{x\in\RR^D \setminus \cK_t(\delta)} \left( \left\|\bar{s}(x,t)  \right\|^2 + \left\|\nabla \log p_t(x) \right\|^2 \right)\ud P_t(x).
\end{align*}
Lemma \ref{lemma:truncate-x} bounds the probability mass outside the truncation region. This yields $\cK_t(\delta)$, which yields,
    \begin{align*}
         \int_{\RR^D \setminus \cK_t(\delta)} \left\|\bar{s}(x,t)  \right\|^2 \ud P_t(x) \leq\norm{\bar{s}(\cdot,t) }_{L^\infty(\cK_t(\delta))}  \PP\left(x \notin \cK_t(\delta) \right) \leq  \left\|\bar{s}(x,t)  \right\|_\infty^2 \delta^D.
    \end{align*}
    Utilizing $\norm{\bar{s}(\cdot,t) }_{L^\infty(\cK_t(\delta))}= \cO \left(\sqrt{\log(1/\delta)/h_t} \right)$, we have
    \begin{align}\label{eq:lowp-eq1}
         \int_{\RR^D \backslash \cK_t(\delta)} \left\|\bar{s}(x,t) \right\|^2 \ud P_t(x) = \cO \left( \frac{\log(1/\delta)}{h_t}\delta^D\right).
    \end{align}
    Recall the marginal density function $p_t$ satisfies
    \begin{align*}
     p_t(x) =  (2 \pi h_t)^{-D/2}\int_{x_0 \in \cM} \exp \left( -\frac{\|x - \alpha_t x_0\|^2}{2h_t} \right) \ud \dist(x_0) .
\end{align*}
and 
\begin{align*}
    \nabla p_t(x) &=  (2 \pi h_t)^{-D/2}\int_{x_0 \in \cM} -\frac{x-\alpha_tx_0}{h_t}\exp \left( -\frac{\|x - \alpha_t x_0\|^2}{2h_t} \right) \ud \dist(x_0).
\end{align*}
Then we can bound $ \left\|\nabla \log p_t(x) \right\|^2  $ as follows:
\begin{align}\label{eq:lowp-eq2}
      \int_{\RR^D \setminus \cK_t(\delta)}  \left\|\nabla \log p_t(x) \right\|^2 \ud P_t(x) \leq \frac{2}{h_t} \int_{\RR^D \setminus \cK_t(\delta)}  \frac{\|x\|^2+DB^2}{h_t}  \ud P_t(x) = \cO\left( \frac{D\delta^2}{h_t} \right).
\end{align}
Combing \eqref{eq:lowp-eq1} and \eqref{eq:lowp-eq2} concludes the proof.
\end{proof}

\begin{lemma}\label{lemma:total-cover-num}
     Suppose Assumption \ref{assump:exp-ball} holds. Let the radius $r < \min\{3\reach ,\eta L_\logg \reach)\}$. Consider the atlas $\{(U_k, \logg_k) \}_{k=1}^{C_\cM}$ of $\cM$ given in Remark \ref{remark:exp_atlas}, where $U_k = \expp_k( \dball(0,r))$. Then the total number of charts $C_{\cM}$ satisfies
    \begin{align*}
         C_\cM \leq \frac{L_{\logg}^d T_d}{r^d}\int_\cM \diff \mu_\cM.
    \end{align*}
    Here $L_\logg >0$ is the upper bound for the Lipschitz constants of $\logg_k$'s, and $T_d$ is the thickness of the charts. 
\end{lemma}
\begin{proof}[Proof of Lemma~\ref{lemma:total-cover-num}]
    First, we show $\cB(x_k, L_\logg^{-1}r) \cap \cM \subseteq U_k$, $k=1,\ldots,C_\cM$. For any $z \in \cB(x_k, L_\logg^{-1}r) \cap \cM$, we have 
\begin{equation}\label{eq:set-inclusion}
    \| \logg_k(z)  \| = \| \logg_k(z) - \logg_k(x_k)  \| \leq L_\logg \|z-x_k\| \leq r.
\end{equation}
Notably, $\logg_k$ is well-defined on $ \cB(x_k, L_\logg^{-1}r)$ according to Assumption \ref{assump:exp-ball}. Since the exponential map $\expp_k$ is a diffeomorphism on $\dball(0,r)$, \eqref{eq:set-inclusion} implies that $z\in U_k$ and thus we can conclude $\cB(x_k, L_\logg^{-1}r) \cap \cM \subseteq U_k$. Therefore, we can bound $C_\cM$ with the covering number of a $\ell^2$-covering in $\RR^D$ with radius $L_\logg^{-1}r$. Applying the covering number in \citet{chen2022nonparametric} with radius $L_\logg^{-1}r$, we have
\begin{equation*}
    C_\cM \leq \frac{L_{\logg}^d T_d}{r^d}\int_\cM \diff \mu_\cM,
\end{equation*}
where $T_d$ is the thickness of the charts.
\end{proof}

\begin{lemma}\label{lemma:localization}
    For any fixed time $t>0$ satisfying \eqref{cond:time}, let $x \in \cK(\alpha_t\cM, 2\sqrt{D h_t \log(1/\delta)})$, $\vradius= 2 L_\logg \sqrt{(h_t/\alpha_t^2)(\log(1/\epsilon_1)+d\log(1/h_t)/2)}$, and $\thres = L_\expp r + 2\sqrt{D h_t \log(1/\delta)} +\sqrt{D}B h_t + L_\expp \vradius $. For any $\epsilon_1>0$, we have
    \begin{align*}
        |f_1(x,t)-s_1(x,t)| \leq 2  h_t^{d/2} \epsilon_1.
    \end{align*}
\end{lemma}
\begin{proof}
Firstly, we clip the integral region by a truncation radius $\xradius>0$. The resulting truncated integral is
\begin{align*}
     s_{\text{trunc}}(x,t) =\sum_{k=1}^{C_\cM}  \int_{\{v:\|\expp_k(v) -\xstar\|\leq \xradius\}} &\exp \left( - \frac{ \|\alpha_t \xstar- \alpha_t \expp_k(v)\|^2  }{2h_t} \right) \\
    &\cdot \exp \left( - \frac{  \langle x - \alpha_t \xstar , \alpha_t \xstar - \alpha_t \expp_k(v)\rangle }{h_t} \right) \cdot F_k(v)\ud v.
\end{align*}
Lemma \ref{lemma:clip-error} derives an upper bound for the clipping error. Take   
\begin{align*}
    \xradius = 2 L_\expp L_\logg \sqrt{(h_t/\alpha_t^2)(\log(1/\epsilon_1)+d\log(1/h_t)/2)},
\end{align*} 
and then we have $| s_{\text{trunc}}(x,t) -s_1(x,t)| \leq h_t^{d/2}\epsilon_1$.

The function $s_{\text{trunc}}(x,t)$ calculates the sum of truncated integrals over all the charts of $\cM$. However, it is sufficient to concentrate on those charts that are in proximity to $x$, out of which the partition of unity $\rho_k(\xstar) =0$ so that it does not contribute to the integrals in $s_{\text{trunc}}(x,t)$. Specifically, given an input $ x\in \cK(\alpha_t\cM, 2\sqrt{D h_t \log(1/\delta)})$, we select charts  whose center $x_k$ satisfies 
\begin{equation*}
    \|x_k - x\| \leq  L_\expp r + 2\sqrt{D h_t \log(1/\delta)} +\sqrt{D}B h_t + \xradius = \thres.
\end{equation*} 

For the rest of the charts satisfying $ \|x_k - x\| >  \thres$, 
taking $x_0 \in \cM$ satisfying $\|x_0-\xstar\| \leq \xradius$, we have
\begin{equation}\label{eq:outside-charts}
\begin{split}
    \|x_0-x_k\| &\geq \|x - x_k\| -\|x - x_0\| \\
     &\geq \|x - x_k\| - \|x-\alpha_t\xstar\| - \|\alpha_t \xstar - \xstar\|  - \| \xstar- x_0\| \\
    &> \thres -2 \sqrt{D h_t \log(1/\delta)} - \sqrt{D}B h_t - \xradius \\
     &> L_\expp r.
\end{split}    
\end{equation}
This indicates that $x_0 \notin U_k = \expp_k(\cB_{T_{x_k} \cM}(0,r))$, and thus $\rho_k(x_0) = 0$ as well as $F_k(x_0) = 0$. Hence the integrals given in $f_1(x,t)$ are equal to zero on these charts, so that we can restrict the choice of $k$ and only compute integrals on the selected charts.

Furthermore, we show that $\xstar \in \cM$ has a unique preimage under the exponential map with respect to the selected charts. To be specific, for the $k$-th chart satisfying  $ \|x_k - x\| \leq \thres$, $\logg_k(\xstar) = \expp_k^{-1}(\xstar)$ is well-defined. By Assumption \ref{assump:exp-ball}, it suffices to verify that
\begin{align*}
    \|\xstar -x_k \| 
    \leq  \|\xstar -x \| + \|x -x_k \|
    \leq 2 \sqrt{D h_t \log(1/\delta)}+ \sqrt{D}B h_t +  \thres < \eta \reach .
\end{align*}
This holds for $L_\expp r < \rratio \reach /4$, and time $t>0$ such that $2 \sqrt{D h_t \log(1/\delta)} + \sqrt{D}B h_t \leq \rratio \reach/4$ and $\xradius \leq \rratio\reach/4$. By \eqref{cond:time}, these conditions are guaranteed to hold if we specify $\delta=\epsilon$ to be the ultimate approximation error and $\epsilon_1=\epsilon$. 
Thereby it is valid to define $v_k(x,t) = \logg_k(\xstar)$ on the selected charts where  $ \|x_k - x\| \leq \thres$.

Next, we further shrink each integral region into an $l_2$-balls in $\RR^d$, on which the integral is easier to compute. This would give the formulation of $f_1$. Intuitively, as a result of the smoothness of the exponential map $\expp_k$, any set $V_k = \{ v\in \RR^d : \|\expp_k(v) - \expp_k(v_k(x,t))\| \leq \xradius\}$ has a smooth boundary and thus there exists two $l_2$-balls with small difference in radius such that one of them is included in $V_k$ and the other one contains $V_k$. One can verify that
\begin{equation}\label{eq:set-contain}
    V_k^l \subseteq V_k \subseteq V_k^u,
\end{equation}
where $ V_k^l : =\{ v\in \RR^d : \|v - v_k(x,t)\| \leq \vradius \} $ with $\vradius = (L_\expp)^{-1}\xradius$, and $ V_k^u = \{ v\in \RR^d : \|v - v_k(x,t)\| \leq L_\logg\xradius \} $. As shown in Lemma \ref{lemma:trans-region-error}, we have $|f_1(x,t)-s_{\text{trunc}}(x,t)| \leq h_t^{d/2}\epsilon_1$. Finally, we concludes the proof by adding together the approximation errors,
\begin{align*}
    |f_1(x,t)-s_1(x,t)| \leq |f_1(x,t)-s_{\text{trunc}}(x,t)| + |s_{\text{trunc}}(x,t)-s_1(x,t)| \leq 2 h_t^{d/2}\epsilon_1.
\end{align*}
The proof is complete.
\end{proof}

\begin{lemma}\label{lemma:clip-error}
    For a given $\epsilon_1>0$ and $t>0$ satisfying \eqref{cond:time}, let $\{g_k\}_{k=1}^{C_\cM}$ be a series of  functions satisfying $\|g_k\|_\infty \leq C_g$ for some constant $C_g>0$, and $\xradius = c_0 \sqrt{(h_t/\alpha_t^2)(\log(1/\epsilon_1)+d\log(1/h_t)/2)}$ for some constant $c_0 \geq 2$. Then for any $ x\in \cK(\alpha_t\cM, 2\sqrt{D h_t \log(1/\delta)})$, we have 
\begin{align*}
    \sum_{k=1}^{C_\cM}  \int_{\{v: \|\expp_k(v) -\xstar\|> \xradius\} } &g_k(v) \exp \left( - \frac{ \|\alpha_t \xstar- \alpha_t \expp_k(v)\|^2  }{2h_t} \right) \\
    &\cdot \exp \left( - \frac{  \langle x - \alpha_t \xstar , \alpha_t \xstar - \alpha_t \expp_k(v)\rangle }{h_t} \right) F_k(v)\ud v  \leq  C_g h_t^{d/2}\epsilon_1 .
\end{align*}
\end{lemma}
\begin{proof}
For notational simplicity, we denote
\begin{align*}
    (\clubsuit) := \sum_{k=1}^{C_\cM}  \int_{\{v: \|\expp_k(v) -\xstar\|> \xradius\}} &g_k(v) \exp \left( - \frac{ \|\alpha_t \xstar- \alpha_t \expp_k(v)\|^2  }{2h_t} \right) \\
    &\cdot \exp \left( - \frac{  \langle x - \alpha_t \xstar , \alpha_t \xstar - \alpha_t \expp_k(v)\rangle }{h_t} \right) F_k(v)\ud v.
\end{align*}
Since $g_k$ is bounded by $C_g>0$, we have
    \begin{align*}
    (\clubsuit)
    \leq \sum_{k=1}^{C_\cM}  \int_{\{v: \|\expp_k(v) -\xstar\|> \xradius\}} & C_g \exp \left( - \frac{ \|\alpha_t \xstar- \alpha_t \expp_k(v)\|^2  }{2h_t} \right) \\
    &\cdot \exp \left( - \frac{  \langle x - \alpha_t \xstar , \alpha_t \xstar - \alpha_t \expp_k(v)\rangle }{h_t} \right)F_k(v)\ud v.
    \end{align*}
    By the Cauchy-Schwartz inequality, we have
    \begin{align*}
        &\exp \left( - \frac{ \|\alpha_t \xstar- \alpha_t \expp_k(v)\|^2 +2 \langle x - \alpha_t \xstar , \alpha_t \xstar - \alpha_t \expp_k(v)\rangle  }{2h_t} \right)  \\
        \leq &  \exp \left(  \frac{ \|\alpha_t \xstar- \alpha_t \expp_k(v)\|\cdot( 2 \| x - \alpha_t \xstar\|-\| \alpha_t \xstar - \alpha_t \expp_k(v)\| )}{2h_t} \right) .
    \end{align*}
    If $ \| \xstar - \expp_k(v)\| > \rratio\reach $,  we have $ \| \alpha_t \xstar - \alpha_t \expp_k(v)\| > 4 \| x - \alpha_t \xstar\|$. This is because the projection distance $ \| x - \alpha_t \xstar\| $ can be bounded as follows, 
    \begin{equation*}
        \| x - \alpha_t \xstar\| \leq 2 \sqrt{D h_t \log(1/\delta)} \leq \alpha_t \rratio \reach/4 <  \alpha_t \| \xstar - \expp_k(v)\| /4,
    \end{equation*}
    which holds for time $t$ satisfying $h_t \leq \reach^2/(64 D \log(1/\delta)/\rratio^2+\reach^2)$. We further obtain that
     \begin{align*}
        2 \| x - \alpha_t \xstar\|-\| \alpha_t \xstar - \alpha_t \expp_k(v)\|
        &\leq   -\frac{1}{2}\| \alpha_t \xstar - \alpha_t \expp_k(v)\| ,
    \end{align*}
    which yields
    \begin{align*}
    \exp \left( - \frac{ \|\alpha_t \xstar- \alpha_t \expp_k(v)\|^2 +2 \langle x - \alpha_t \xstar , \alpha_t \xstar - \alpha_t \expp_k(v)\rangle  }{2h_t} \right)       
        \leq \exp \left(  -\frac{ \|\alpha_t \xstar- \alpha_t \expp_k(v)\|^2}{4h_t} \right).
    \end{align*}
    Now we consider the case when $ \| \xstar - \expp_k(v)\| \leq \rratio \reach$. By Assumption \ref{assump:exp-ball}, $\expp_k(v)$ has a unique preimage under $\expp_{\xstar}$. 
    By Lemma \ref{lemma:exp-cross-term}, the cross-term can be upper bounded as follows:
    \begin{align*}
         \exp \left(- \frac{  \langle x - \alpha_t \xstar, \alpha_t  \xstar - \alpha_t \expp_k(v)\rangle }{h_t} \right) \leq  \exp \left(  \frac{ \|\alpha_t \xstar -  \alpha_t\expp_k(v)\|^2}{4h_t} \right).
    \end{align*}
    This implies that
    \begin{align*}
        &\exp \left( - \frac{ \|\alpha_t \xstar- \alpha_t \expp_k(v)\|^2 +2 \langle x - \alpha_t \xstar , \alpha_t \xstar - \alpha_t \expp_k(v)\rangle  }{2h_t} \right) \leq   \exp \left(  -\frac{ \|\alpha_t \xstar- \alpha_t \expp_k(v)\|^2 }{4h_t} \right) .
    \end{align*}
    Therefore, combining the two cases, we can derive the following bound for $(\clubsuit)$,
    \begin{align*}
    (\clubsuit)
    &\leq  \sum_{k=1}^{C_\cM}  \int_{\|\expp_k(v) -\xstar\|> \xradius}  \exp \left(  -\frac{ \|\alpha_t \xstar- \alpha_t \expp_k(v)\|^2}{4h_t} \right) F_k(v)\ud v \notag\\
    &\leq  \sum_{k=1}^{C_\cM}  \int \exp \left(  -\frac{ \alpha_t^2 \xradius^2 }{4h_t} \right) F_k(v)\ud v.
    \end{align*}
    Notice that 
    \begin{equation}\label{eq:sum-Fk}
        \sum_{k=1}^{C_\cM}  \int F_k(v)\ud v =  \sum_{k=1}^{C_\cM}  \int \rho_k (\expp_k(v)) p_{\rm data} (\expp_k(v)) G_k(v) \ud v =  \sum_{k=1}^{C_\cM}  \int_{x_0 \in \cM}  p_{\rm data} (x_0) \ud x_0 = 1,
    \end{equation}
    we can conclude the proof by substituting $\xradius = c_0  \sqrt{(h_t/\alpha_t^2)(\log(1/\epsilon_1)+d\log(1/h_t)/2)}$, which gives
    $  (\clubsuit)  \leq  \exp \left(  -\frac{ \alpha_t^2 \xradius^2 }{4h_t} \right) \leq h_t^{d/2}\epsilon_1.$
\end{proof}

\begin{lemma}\label{lemma:trans-region-error}
    Given any $\epsilon_1>0$ and any time $t$ satisfying $h_t \leq \reach^2/(256 D\log(1/\delta)+\reach^2)$, let $\vradius = 2 L_\logg \sqrt{(h_t/\alpha_t^2)(\log(1/\epsilon_1)+d\log(1/h_t)/2)}$, and $\{g_k\}_{k=1}^{C_\cM}$ be a series of  functions on $\RR^d$ where $|g_k| \leq C_g$ for some constant $C_g>0$. Then for any $x \in \cK(\alpha_t\cM, 2\sqrt{D h_t \log(1/\delta)})$, we have
\begin{align*}
   \sum_{k\in\cI(x)}\int_{\Omega_k} &g_k(v) \exp \left( - \frac{  \langle x - \alpha_t \xstar, \alpha_t  \expp_k(v_k(x,t)) - \alpha_t \expp_k(v)\rangle }{h_t} \right) \\
     &\hspace{-0in} \cdot \exp \left( - \frac{ \|\alpha_t \expp_k(v_k(x,t))- \alpha_t \expp_k(v)\|^2  }{2h_t} \right)  F_k(v) \ud v \leq  C_g h_t^{d/2} \epsilon_1.
\end{align*}
Here we denote $\Omega_k = \{v \in \RR^d: \vradius <\|v -v_k(x,t)\|\leq L_\expp L_\logg \vradius\}$.
\end{lemma}
\begin{proof}
    First, we apply the upper bound for $g_k$ and the same arguments in Lemma \ref{lemma:exp-cross-term},
    \begin{align*}
   &\sum_{k\in\cI(x)}\int_{\Omega_k } g_k(v) \exp \left( - \frac{  \langle x - \alpha_t \xstar, \alpha_t  \expp_k(v_k(x,t)) - \alpha_t \expp_k(v)\rangle }{h_t} \right)\\
     &\hspace{1in}\cdot \exp \left( - \frac{ \|\alpha_t \expp_k(v_k(x,t))- \alpha_t \expp_k(v)\|^2  }{2h_t}  \right) F_k(v) \ud v \\
      \leq &  C_g\sum_{k\in\cI(x)}\int_{\Omega_k} \exp \left( - \frac{ \|\alpha_t \expp_k(v_k(x,t))- \alpha_t \expp_k(v)\|^2  }{4h_t} \right)F_k(v) \ud v.
    \end{align*}
    Notice that the Lipschitz property of  $\logg_k(\cdot)$ indicates
    \begin{equation}\label{eq:log-lip}
        \|v_k(x,t) - v \| = \| \logg_k(\xstar) - \logg_k(x_0)  \| \leq L_{\logg} \|\xstar - x_0\|.
    \end{equation}
    Thereby this implies  
    \begin{align*}
        \|\expp_k(v)- \expp_k(v_k(x,t))\|= \|x_0-\xstar\|> ( L_\logg)^{-1}\vradius,
    \end{align*}
    for any $v\in \Omega_k$. It further gives rise to
    \begin{align*}
   \sum_{k\in\cI(x)}&\int_{ \Omega_k} \exp \left( - \frac{ \|\alpha_t \expp_k(v_k(x,t))- \alpha_t \expp_k(v)\|^2  }{4h_t} \right)F_k(v) \ud v \\
    \leq \sum_{k\in\cI(x)}&\int  \exp \left( - \frac{ \alpha_t^2 (L_\logg)^{-2}\vradius^2  }{4h_t} \right) F_k(v)\ud v.
    \end{align*}
    Next, we apply \eqref{eq:sum-Fk} again to bound the above integral:
    \begin{align*}
    \sum_{k\in\cI(x)}\int_{ \Omega_k} \exp \left( - \frac{ \|\alpha_t \expp_k(v_k(x,t))- \alpha_t \expp_k(v)\|^2  }{4h_t} \right)F_k(v) \ud v 
    &\leq  \exp \left( - \frac{ \alpha_t^2 (L_\logg)^{-2}\vradius^2  }{4h_t} \right)\\
    & \leq h_t^{d/2}\epsilon_1,
    \end{align*}
    where the last inequality uses $\vradius = 2 L_\logg \sqrt{(h_t/\alpha_t^2)(\log(1/\epsilon_1)+d\log(1/h_t)/2)}$.
\end{proof}

\begin{lemma}[Upper bound for the cross term]\label{lemma:cross-term}
    For $x\in \alpha_t\cK(\cM,\reach)$, let $\xstar = \projm(x,t)/\alpha_t$. For any $x_0 \in \expp_{\xstar}(\cB_{T_{\xstar} \cM}(0, \inj(\cM)))$ and $t>0$, we have 
\begin{align*}
    \big\| -  \langle x - \alpha_t\xstar,  \xstar - x_0\rangle \big\| \leq \frac{2}{\reach} \|x - \alpha_t\xstar\| \cdot \|\xstar -  x_0\|^2. 
\end{align*}
\end{lemma}
\begin{proof}
    Consider a curve $\gamma: (-a,a) \to \cM$ from some constant $a>0$, such that $\gamma(0) = \xstar$, $\gamma(\tau) = x_0$ and $\gamma'(0) \in T_{\xstar}\cM$. Moreover, $\|\gamma'(s)\| =1$ for any $s\in(-a,a)$.
    Then we can derive the cross term as following:
    \begin{align*}
        -  \langle x - \alpha_t\xstar,  \xstar - x_0\rangle & = \langle x - \alpha_t\xstar,  \gamma(\tau) - \gamma(0)\rangle \\
         & = \left\langle x - \alpha_t\xstar, \int_0^\tau  \gamma'(s) \ud s \right\rangle \\
         & = \left\langle x - \alpha_t\xstar, \int_0^\tau  [\gamma'(s)-\gamma'(0)] \ud s + \tau \gamma'(0) \right\rangle.
    \end{align*}
    Since $ x - \alpha_t\xstar$ is perpendicular to $T_{\xstar} \cM$, we have $\langle x - \alpha_t\xstar, \gamma'(0) \rangle =0 $. It follows that
    \begin{align*}
         \big\|- \langle x - \alpha_t\xstar,  \xstar - x_0\rangle \big\|  
        &= \left\|  \Big\langle x - \alpha_t\xstar, \int_0^\tau  [\gamma'(s)-\gamma'(0)] \ud s  \Big\rangle \right\| \\
        &\leq \left\|  x - \alpha_t\xstar\right\|\cdot \left|  \int_0^\tau  [\gamma'(s)-\gamma'(0)] \ud s  \right|,
    \end{align*}
    where the last inequality uses Cauchy-Schwartz inequality.
    Moreover, by Proposition 6.1 in \cite{niyogi2008finding}, we have $\|\gamma''(s)\| \leq 1/\reach$ where $\reach$ is the reach of manifold $\cM$. Then we arrive at
    \begin{align*}
         \big\|- \langle x - \alpha_t\xstar,  \xstar - x_0\rangle \big\|  
        \leq \left\|  x - \alpha_t\xstar\right\|\cdot \left|\int_0^\tau  \frac{s}{\reach} \ud s  \right| \leq \frac{\tau^2}{2\reach} \left\|  x - \alpha_t\xstar\right\|.
    \end{align*}
    Furthermore, we apply Proposition 6.3 in \cite{niyogi2008finding} to get
    \begin{equation*}
        \tau \leq \reach -\reach \sqrt{1-\frac{2 \| \xstar -x_0\|}{\reach}} \leq 2 \| \xstar -x_0\|.
    \end{equation*}
    Combining the above inequalities concludes the proof.
\end{proof}

\begin{lemma}\label{lemma:exp-cross-term}
    Let $x \in \cK(\alpha_t\cM, 2 \sqrt{D h_t \log(1/\delta)})$ and  $v\in \dball( v_k(x,t), \inj(\cM))$. Then for any time $t$ satisfying $h_t \leq \reach^2/(256 D\log(1/\delta)+\reach^2)$,  we have 
\begin{align*}
    - \frac{  \langle x - \alpha_t \xstar, \alpha_t  \expp_k(v_k(x,t)) - \alpha_t \expp_k(v)\rangle }{h_t}  \leq   \frac{ \|\alpha_t \expp_k(v_k(x,t)) -  \alpha_t\expp_k(v)\|^2}{4h_t}.
\end{align*}
\end{lemma}
\begin{proof}

Using the upper bound for the cross term given in Lemma \ref{lemma:cross-term}, we get
    \begin{align*}
         &- \frac{  \langle x - \alpha_t \xstar, \alpha_t  \expp_k(v_k(x,t)) - \alpha_t \expp_k(v)\rangle }{h_t}  \\
         \leq ~& \frac{2 \alpha_t}{\reach h_t} \|x - \alpha_t \xstar\| \cdot \| \expp_k(v_k(x,t)) -  \expp_k(v)\|^2 \\
         \leq  ~&  \frac{16\sqrt{D h_t \log(1/\delta)}}{\alpha_t\reach} \cdot \frac{\|\alpha_t \expp_k(v_k(x,t)) -  \alpha_t\expp_k(v)\|^2}{4h_t} .
    \end{align*}
    The last inequality uses $ \|x -\alpha_t \xstar\| \leq  2 \sqrt{D h_t \log(1/\delta)}$. Finally, we conclude the proof by applying $h_t \leq \reach^2/(256 D\log(1/\delta)+\reach^2)$, which is equivalent to $16\sqrt{D h_t \log(1/\delta)} \leq \alpha_t\reach$.
\end{proof}

\begin{lemma}\label{lemma:Fk-approx-error}
      For any time $t$ satisfying $h_t \leq r^2/(256 D\log(1/\delta)+\reach^2)$, let $x \in \cK(\alpha_t\cM, 2\sqrt{D h_t \log(1/\delta)})$ and $\{g_k\}_{k=1}^{C_\cM}$ be a series of  functions on $\RR^d$ where $|g_k| \leq C_g$ for some constant $C_g>0$. Then we have 
\begin{align*}
    &\Bigg|\sum_{k\in\cI(x)}\int_{\dball(v_k(x,t), \vradius)}  g_k(v)  \exp \left( - \frac{ 2 \langle x - \alpha_t \xstar, \alpha_t  \expp_k(v_k(x,t)) - \alpha_t \expp_k(v)\rangle }{2h_t} \right)\\
     & \hspace{1.7in} \cdot \exp \left( - \frac{ \|\alpha_t \expp_k(v_k(x,t))- \alpha_t \expp_k(v)\|^2  }{2h_t} \right) \cdot \left[\hat{F}_k(v) -F_k(v)\right]\ud v \Bigg|\\
     \leq &  \left(4 \pi L_\logg^2 h_t/(\alpha_t^2) \right)^{d/2}  C_F C_g  d^\holder r^\holder |\cI(x)| .
\end{align*}
\end{lemma}
\begin{proof}
For notational simplicity, denote 
\begin{align*}
    (\clubsuit)=\Bigg|\sum_{k\in\cI(x)}\int_{\dball(v_k(x,t), \vradius)}  g_k(v)  \exp \left( - \frac{ 2 \langle x - \alpha_t \xstar, \alpha_t  \expp_k(v_k(x,t)) - \alpha_t \expp_k(v)\rangle }{2h_t} \right)&\\
     \cdot   \exp \left( - \frac{ \|\alpha_t \expp_k(v_k(x,t))- \alpha_t \expp_k(v)\|^2  }{2h_t} \right) \cdot \left[\hat{F}_k(v) -F_k(v)\right]\ud v \Bigg|.
\end{align*}
    Due to Lemma \ref{lemma:exp-cross-term} and the condition that $g_k$ is upper bounded, we have
    \begin{align*}
     (\clubsuit)
     \leq C_g\sum_{k\in\cI(x)}\int_{\dball(v_k(x,t), \vradius)}   \exp \left( - \frac{ \|\alpha_t \expp_k(v_k(x,t))- \alpha_t \expp_k(v)\|^2  }{4h_t} \right) \cdot \left|\hat{F}_k(v) -F_k(v)\right|\ud v.
    \end{align*}
    Moreover, we apply the bound for $\left|\hat{F}_k(v) -F_k(v)\right|$ stated in \eqref{eq:Fk-approx-error}, which leads to
    \begin{align*}
       (\clubsuit)
        \leq C_g C_F d^\holder r^\holder \sum_{k\in\cI(x)}\int_{\dball(v_k(x,t), \vradius)}   \exp \left( - \frac{ \|\alpha_t \expp_k(v_k(x,t))- \alpha_t \expp_k(v)\|^2  }{4h_t} \right) \ud v.
    \end{align*}
    Since the Logarithm map is $L_\logg$-Lipschitz, we further have
    \begin{align*}
        (\clubsuit)
        &\leq C_g C_F d^\holder r^\holder\sum_{k\in\cI(x)}\int_{\dball(v_k(x,t), \vradius)}   \exp \left( - \frac{ (L_\logg)^{-2} \alpha_t^2 \| v-v_k(x,t)\|^2  }{4h_t} \right) \ud v \\
        &= C_g C_F d^\holder r^\holder \sum_{k\in\cI(x)} \left( \frac{4 \pi L_\logg^2 h_t}{\alpha_t^2}\right)^{d/2} \PP\left( \|Z\| \leq \frac{\alpha_t \vradius}{\sqrt{2 h_t}L_\logg}\right),
    \end{align*}
    where the last equality holds for the standard Gaussian random variable $Z\sim N(0,I_d)$. Furthermore, notice that the probability is always no more that $1$ and the number of charts is $C_\cM$, we get
    \begin{align*}
   \sum_{k\in\cI(x)}\int_{\dball(v_k(x,t), \vradius)}   \exp \left( - \frac{ \|\alpha_t \expp_k(v_k(x,t))- \alpha_t \expp_k(v)\|^2  }{4h_t} \right) \ud v 
    &\leq |\cI(x)|  \left( \frac{4 \pi L_\logg^2 h_t}{\alpha_t^2}\right)^{d/2}.
    \end{align*}
    In conclusion, we put all the inequalities together and then arrive at
    \begin{align*}
    (\clubsuit)
     \leq   \left(4 \pi L_\logg^2 h_t/(\alpha_t^2) \right)^{d/2}  C_F C_g  d^\holder r^\holder |\cI(x)| .
    \end{align*}
    The proof is complete.
\end{proof}

\begin{lemma}\label{lemma:f3-approx-error}
Fix any $\epsilon_1 \in [0,e^{-1}]$ and  time $t$ satisfying  $h_t \leq \reach^2/(64 D\log(1/\delta)+\reach^2)$. Let $\vradius = 2 L_\logg \sqrt{(h_t/\alpha_t^2)(\log(1/\epsilon_1)+d\log(1/h_t)/2)}$. For any $\epsilon_3 > 0$ and $x \in \cK(\alpha_t\cM, 2 \sqrt{D h_t \log(1/\delta)})$, if we take $\degree \in (0,\holder] $ as some fixed constant and $\dgM \geq 4e^2 L_\expp^2 L_\logg^2 d\log(1/h_t )/2 + 4e^2 L_\expp^2 L_\logg^2 \log(1/\epsilon_1) + \log(1/\epsilon_3)$, we have that 
 \begin{align*}
     |f_3(x,t)-f_2(x,t)| &\leq 2 C_F \frac{(4 L_\expp^2 L_\logg^2\pi  h_t(\log(1/\epsilon_1)+d\log(1/h_t)/2))^{d/2}}{\Gamma(d/2+1)(L_\expp \alpha_t )^d}|\cI(x)|\\
     &\quad \cdot\left( \frac{1}{\degree !} \left(  \frac{16 L_\expp^2 L_\logg^2\sqrt{D h_t \log(1/\delta)} (\log(1/\epsilon_1)+d\log(1/h_t)/2)}{\alpha_t \reach}  \right)^{\degree}  +\epsilon_3 \right).
 \end{align*}
\end{lemma}
\begin{proof}
    Given that the exponential function is $C^\infty$, we can approximate it up to any order. To facilitate the approximation, we firstly bound the cross term inside the exponential function by Lemma \ref{lemma:cross-term}. Recall by definition
    \begin{align*}
        \cT_k =   \langle x - \alpha_t\xstar,  \expp_k(v_k(x,t)) -  \expp_k(v)\rangle \quad \text{and} \quad \cD_k = \|\expp_k(v_k(x,t))-  \expp_k(v)\| ^2.
    \end{align*}
    Here we write $\cT_k=\cT_k(x,v,t)$ and $\cD_k=\cD_k(v_k(x,t),v,t)$ for notational simplicity.
    By Lemma \ref{lemma:cross-term}, for $x \in \cK(\alpha_t\cM, 2 \sqrt{D h_t \log(1/\delta)})$, we have
    \begin{align*}
       \left| \cT_k\right| 
        &\leq \frac{2 }{\reach} \|x - \alpha_t \xstar\| \cdot \| \expp_k(v_k(x,t)) -  \expp_k(v)\|^2 \\
        &\leq  \frac{8\sqrt{D h_t \log(1/\delta)}}{ 2\reach} \| \expp_k(v_k(x,t)) -  \expp_k(v)\|^2 .
    \end{align*}
Now we apply $8\sqrt{D h_t \log 1/\delta} \leq \alpha_t \reach$ in the last inequality, i.e.  $h_t \leq \reach^2/(64 D\log(1/\delta)+\reach^2)$,
\begin{align*}
    \left|\cT_k\right| \leq  \frac{\alpha_t}{2} \| \expp_k(v_k(x,t)) -  \expp_k(v)\|^2 = \frac{\alpha_t}{2}|\cD_k |,
\end{align*}
Take $\vradius = 2 L_\logg \sqrt{(h_t/\alpha_t^2)(\log(1/\epsilon_1)+d\log(1/h_t)/2)}$. Then for $v \in \dball(v_k(x,t), \vradius)$, we get
    \begin{align}
       \left| \frac{\alpha_t}{h_t} \cT_k\right|  \leq 2  L_\expp^2 L_\logg^2(\log(1/\epsilon_1)+d\log(1/h_t)/2) .\label{eq:exp1}
    \end{align}
    This further indicates the following upper bound for the exponential of the cross term:
    \begin{align*}
        \exp \left( -\frac{\alpha_t}{h_t} \cT_k \right) \leq \exp \left( 2L_\expp^2 L_\logg^2(\log(1/\epsilon_1)+d\log(1/h_t)/2) \right) 
        = h_t^{-L_\expp^2 L_\logg^2 d} \epsilon_1^{-2 L_\expp^2 L_\logg^2}.
    \end{align*}
    Next, if we approximate the exponential function with the cross term up to an order $\degree$,
    the tail can be bounded by Lemma \ref{lemma:smooth-approx} as follows:
    \begin{align}
    \begin{split}
        \bigg|\exp \left(    -\frac{\alpha_t}{h_t} \cT_k  \right) -g_\cT(\cT_k)\bigg| 
        & \leq \frac{1}{\degree !}\left(  \frac{4 L_\expp^2\alpha_t \sqrt{D  \log(1/\delta)}  \vradius^2}{\sqrt{h_t} \reach}  \right)^{\degree}  \exp \left(  \frac{\alpha_t^2}{2h_t} | \cD_k| \right),\label{eq:exp2}
    \end{split}
    \end{align}
    where we denote 
    \begin{equation}\label{eq:def-gC}
        g_\cT(\cT_k) := \sum_{j=0}^{\degree-1} \frac{(-1)^j}{j!(h_t/\alpha_t)^j}  \cT_k^j.
    \end{equation}
    Based on \eqref{eq:exp1}, we can derive the following upper bound for $g_\cT(\cT_k)$:
    \begin{align}
        |g_\cT(\cT_k) | \leq \sum_{j=0}^{\degree-1} \frac{1}{j!}  \left|\frac{\alpha_t}{h_t} \cT_C\right|^j 
        \leq \left(2L_\expp^2 L_\logg^2(\log(1/\epsilon_1)+d\log(1/h_t)/2)\right)^\degree. \label{eq:exp4}
    \end{align}
    Likewise, we can approximate the exponential function with $\cD_k$ up to an order $\dgM \geq \log(1/\epsilon_3) + 4e^2 L_\expp^2 L_\logg^2 d\log(1/h_t ) + 8e^2 L_\expp^2 L_\logg^2\log(1/\epsilon_1) $. We again apply Lemma \ref{lemma:smooth-approx} by setting $\epsilon = h_t^{L_\expp^2 L_\logg^2 d}\epsilon_1^{2L_\expp^2 L_\logg^2}\epsilon_3$ to bound the approximation error:
    \begin{align}
         \bigg| \exp \left( - \frac{ \alpha_t^2  }{2h_t} \cD_k^2 \right) -g_\cD(\cD_k)\bigg| \leq h_t^{L_\expp^2 L_\logg^2 d}\epsilon_1^{2L_\expp^2 L_\logg^2}\epsilon_3, \label{eq:exp3}
    \end{align}
    where we denote
    \begin{equation*}
        \quad g_\cD(\cD_k) :=  \sum_{l=0}^{\dgM-1}  \frac{(-1)^l}{ 2^ll! (h_t/\alpha_t^2)^l} \cD_k^{l} .
    \end{equation*}
    Next, we proceed to bound the difference between $f_3$ and $f_2$.
    \begin{align*}
        |f_3(x,t)-f_2(x,t)| &
        \leq 2 C_F\sum_{k\in\cI(x)}  \int_{\dball(v_k(x,t), \vradius)}   \bigg| g_\cT(\cT_k) g_\cD(\cD_k)-  \exp \left( -\frac{\alpha_t}{h_t}\cT_k-\frac{\alpha_t^2}{2h_t}\cD_k \right)  \bigg| \ud v\\
        &\leq 2 C_F\sum_{k\in\cI(x)}  \bigg[ \int_{\dball(v_k(x,t), \vradius)}  \bigg|\exp \left( -\frac{\alpha_t}{h_t}\cT_k \right) -g_\cT(\cT_k) \bigg| \cdot \exp \left( -\frac{\alpha_t^2}{2h_t}\cD_k \right) \ud v \\
        &\hspace{1in}+\int_{\dball(v_k(x,t), \vradius)}  \bigg| \exp \left(-\frac{\alpha_t^2}{2 h_t}\cD_k\right)-g_\cD(\cD_k) \bigg| \cdot g_\cT(\cT_k) \ud v \bigg] .
    \end{align*}
    Plug in \eqref{eq:exp2}, \eqref{eq:exp4} and \eqref{eq:exp3}, then we get
    \begin{align*}
        |f_3(x,t)-f_2(x,t)| 
        &\leq  2 C_F\sum_{k\in\cI(x)}  \int_{\dball(v_k(x,t), \vradius)}  \left( \frac{1}{\degree !} \left(  \frac{4 L_\expp^2\alpha_t \sqrt{D  \log(1/\delta)} \vradius^2}{\sqrt{h_t} \reach}  \right)^{\degree} +\epsilon_3 \right)\ud v\\
        &\leq  2 C_F \frac{\pi^{d/2}(\vradius)^d}{\Gamma(d/2+1)}  \left( \frac{1}{\degree !} \left(  \frac{4 L_\expp^2\alpha_t \sqrt{D  \log(1/\delta)} \vradius^2}{\sqrt{h_t} \reach}  \right)^{\degree} +\epsilon_3 \right)|\cI(x)| .
    \end{align*}
    We conclude the proof by applying $\vradius = 2 L_\logg \sqrt{(h_t/\alpha_t^2)(\log(1/\epsilon_1)+d\log(1/h_t)/2)}$.
\end{proof}

\begin{lemma}\label{lemma:prop-Dk-hat}
    Fix any $\epsilon_2 >0$, and take $r = \epsilon_2^{1/\holder}$.
    We have for any $v,v_k(x,t) \in \dball(0,r)$,
    \begin{equation*}
        \left|\hat{\cD}_k(v_k(x,t), v, t) \right| \leq \left( L_\expp + \sqrt{D} C_{d,\holder} \epsilon_2  \right)^2 \| v-v_k(x,t) \|^2,
    \end{equation*}
    and 
    \begin{equation*}
         \left|\hat{\cD}_k(v_k(x,t), v, t) - \cD_k(v_k(x,t), v, t) \right| \leq D  C_{d,\holder} \epsilon_2 (  C_{d,\holder} \epsilon_2  + 2L_\expp) \|v-v_k(x,t) \|^2.
    \end{equation*}
    Here $C_{d,\holder} >0$ is a constant depending on $d$ and $\holder$.
\end{lemma}
\begin{proof}
    In this proof, we omit the order $S=\holder+1$ in $\hat{\expp}_{k}^{\holder+1}$ for notational simplicity. We first apply Lemma \ref{lemma:bramble} with $p=1$, 
    \begin{equation}\label{eq:lip-exp-error}
        \max_{\|\theta\| \leq 1} \left| \partial^\theta \left(\hat{\expp}_{k,i} - \expp_{k,i} \right) \right|_{L_\infty(\dball(0,r) )}\leq  C_{d,\holder} \epsilon_2 .
    \end{equation}
    Since $\expp_k$ is $L_\expp$-Lipschitz, we can further upper bound the Lipschitz constant of $\hat{\expp}_k$:
    \begin{equation}\label{eq:lip-exp-hat}
        \lip\left( \hat{\expp}_{k}\right) \leq  L_\expp + \sqrt{D}  C_{d,\holder} \epsilon_2 .
    \end{equation}
    Therefore, by utilizing this Lipschitz continuity of $\hat{\expp}_k$, we can get
    \begin{equation*}
        \left|\hat{\cD}_k(v_k(x,t), v, t) \right| = \left\| \hat{\expp}_k(v_k(x,t))- \hat{ \expp}_k(v) \right\|^2 \leq \left( L_\expp + \sqrt{D}  C_{d,\holder} \epsilon_2  \right)^2 \| v-v_k(x,t) \|^2.
    \end{equation*}
    Next, we bound the approximate error for $\hat{D}_k$. To begin with, we rewrite the approximate error as follows:
    \begin{align*}
        &\quad~ \left|\hat{\cD}_k(v_k(x,t), v, t) - \cD_k(v_k(x,t), v, t) \right| \\
        &=  \left | \sum_{i=1}^D \left( \hat{\expp}_{k,i}(v)- \hat{ \expp}_{k,i}(v_k(x,t)) \right)^2 -\left( \expp_{k,i}(v)-  \expp_{k,i}(v_k(x,t))  \right)^2 \right|\\
        &\leq \sum_{i=1}^D \underbrace{\left |  \hat{\expp}_{k,i}(v)- \hat{ \expp}_{k,i}(v_k(x,t)) - \expp_{k,i}(v)+  \expp_{k,i}(v_k(x,t))  \right|}_{\text{(I)}} \\
        &\qquad\quad  \cdot \underbrace{\left |  \hat{\expp}_{k,i}(v)- \hat{ \expp}_{k,i}(v_k(x,t)) +\expp_{k,i}(v)-  \expp_{k,i}(v_k(x,t))   \right|}_{\text{(II)}}.
    \end{align*}
    Notice that (I) can be upper bounded by the Lipschitz continuity of $\hat{\expp}_{k,i} - \expp_{k,i}$:
    \begin{align*}
        \text{(I)} &= \left |  \left(\hat{\expp}_{k,i}(v)- \expp_{k,i}(v)\right)- \left (\hat{ \expp}_{k,i}(v_k(x,t)) -  \expp_{k,i}(v_k(x,t)) \right)  \right| \\
        &\leq \lip\left(\hat{\expp}_{k,i}- \expp_{k,i}\right) \|v-v_k(x,t) \|.
    \end{align*}
    On the other hand, term (II) can be bounded by the smoothness of $\hat{\expp}_{k,i}$ and $\expp_{k,i}$:
    \begin{align*}
        \text{(II)} &\leq \left( \lip\left( \hat{\expp}_{k,i}\right) + \lip\left( \expp_{k,i}\right) \right) \|v-v_k(x,t) \|\\
        &\leq \left( \lip\left( \hat{\expp}_{k,i}-\expp_{k,i}\right) + 2\lip\left( \expp_{k,i}\right) \right) \|v-v_k(x,t) \|.
    \end{align*}
    Moreover, \eqref{eq:lip-exp-error} implies that
    \begin{equation*}
        \lip\left( \hat{\expp}_{k,i}-\expp_{k,i}\right) \leq  C_{d,\holder} \epsilon_2 .
    \end{equation*}
    Therefore, we obtain
    \begin{align*}
         \left|\hat{\cD}_k(v_k(x,t), v, t) - \cD_k(v_k(x,t), v, t) \right| 
         &\leq D  C_{d,\holder} \epsilon_2 (  C_{d,\holder} \epsilon_2  + 2L_\expp) \|v-v_k(x,t) \|^2.
    \end{align*}
    The proof is complete.
\end{proof}

\begin{lemma}\label{lemma:f5-approx-error}
 Fix any $\epsilon_2 \in (0,\epsilon_1)$ where $\epsilon_1$ is given in Lemma \ref{lemma:clip-error}. Take $r = \epsilon_2^{1/\holder}$. Given time $t$ satisfying  $h_t \leq \reach/(256 D\log(1/\delta)+\reach)$, let $\vradius = 2 L_\logg \sqrt{(h_t/\alpha_t^2)(\log(1/\epsilon_1)+d\log(1/h_t)/2)}$.  For any  $\epsilon_2\in (0, 1)$ and $x \in \cK(\alpha_t\cM, 2 \sqrt{D h_t \log(1/\delta)})$, 
  we have
 \begin{equation*}
     |f_4(x,t)-f_3(x,t)| \leq   3\dgM D  C_{d,\holder} C_F\frac{\left( \pi  h_t/2 \right)^{d/2}}{ L_\expp^d \alpha_t^d } \left(\frac{L_\expp^2 L_\logg^2(\log(1/\epsilon_1)+d\log(1/h_t)/2)}{\sqrt{\reach}} \right)^\holder \epsilon_2 |I(x)|.
\end{equation*}
\end{lemma}
\begin{proof}
 As shown in \eqref{eq:Fk-approx-error}, $\hat{F}_k(v)$ can be always bounded by $2C_F$ as long as $r$ is sufficiently small. This implies that
\begin{align*}
    &\quad~|f_4(x,t)-f_3(x,t)| \\
    &\leq 2C_F \sum_{k\in\cI(x)}  \int_{\dball(v_k(x,t), \vradius)}   \left| \sum_{l=0}^{\dgM-1}  \frac{(-1)^l}{2^l\cdot l!} \frac{ \hat{\cD}_k(v_k(x,t), v, t)^l-{\cD}_k(v_k(x,t), v, t)^l }{(h_t/\alpha_t^2)^l} \right|  g_\cT(\cT_k) \ud v.
\end{align*}
where $ g_\cT$ is defined in \eqref{eq:def-gC}. Combined with \eqref{eq:exp4}, the above inequality can be further derived as 
\begin{align*}
    |f_4(x,t)-f_2(x,t)| &\leq 2C_F\left(\frac{L_\expp^2 L_\logg^2(\log(1/\epsilon_1)+d\log(1/h_t)/2)}{\sqrt{\reach}} \right)^\degree \\
    &\quad\cdot\sum_{k\in\cI(x)}  \int_{\dball(v_k(x,t), \vradius)}   \left| \sum_{l=0}^{\dgM-1}  \frac{(-1)^l}{2^l\cdot l!} \frac{ \hat{\cD}_k(v_k(x,t), v, t)^l-{\cD}_k(v_k(x,t), v, t)^l  }{(h_t/\alpha_t^2)^l} \right|   \ud v.
\end{align*}
By Lemma \ref{lemma:prop-Dk-hat}, we can choose $\epsilon_2  \leq L_\expp/ C_{d,\holder} $ such that 
\begin{align*}
   \left|  \hat{\cD}_k(v_k(x,t), v, t)-{\cD}_k(v_k(x,t), v, t) \right|  \leq 3 D  C_{d,\holder} L_\expp  \epsilon_2  \|v-v_k(x,t)\|^2.
\end{align*}
Notice that for any $p,q>0$, we have
\begin{align*}
    (p+q)^l -p^l = \sum_{i=0}^l {l \choose i} q^i p^{l-i} - p^l
    = \sum_{i=1}^l {l \choose i} q^i p^{l-i}
    \leq lq \sum_{i=1}^{l} {l-1 \choose i-1} q^{i-1} p^{l-i}
    = lq (p+q)^{l-1}.
\end{align*}
This suggests that we can bound $ \left|  \hat{\cD}_k(v_k(x,t), v, t)^l-{\cD}_k(v_k(x,t), v, t)^l \right|$ by
\begin{align*}
    \left|  \hat{\cD}_k(v_k(x,t), v, t)^l-{\cD}_k(v_k(x,t), v, t)^l \right|
    &\leq 3l D  C_{d,\holder} L_\expp  \epsilon_2  \|v-v_k(x,t)\|^2 \cdot (2L_\expp)^{2(l-1)}  \|v-v_k(x,t)\|^{2(l-1)}\\
    &= \frac{3}{2}l D  C_{d,\holder} (2L_\expp)^{2l} \epsilon_2  \|v-v_k(x,t)\|^{2l},
\end{align*}
where we plug in $p=  {\cD}_k(v_k(x,t), v, t)$ and $ p +q = \hat{\cD}_k(v_k(x,t), v, t)$ and use $ | p+q| = |\hat{\cD}_k(v_k(x,t), v, t)| \leq (2L_\expp)^2 \|v-v_k(x,t)\|^2 $. Thereby we can get
\begin{align*}
      &\quad~ \int_{\dball(v_k(x,t), \vradius)}   \sum_{l=0}^{\dgM-1}  \frac{(-1)^l}{2^l\cdot l!} \frac{ \left|  \hat{\cD}_k(v_k(x,t), v, t)^l-{\cD}_k(v_k(x,t), v, t)^l \right| }{(h_t/\alpha_t^2)^l}    \ud v \\
      &\leq \frac{ 3}{2}\dgM D  C_{d,\holder}  \epsilon_2  \int_{\dball(v_k(x,t), \vradius)}   \exp \left(-\frac{(2L_\expp)^2\alpha_t^2\|v -v_k(x,t)\|^2}{2h_t}\right)  \ud v  .
\end{align*}
Notably, the integral can be written as an integration of normal density and thus we have
\begin{align}
  \int_{\dball(v_k(x,t), \vradius)}   \exp \left(-\frac{(2L_\expp)^2\alpha_t^2\|v -v_k(x,t)\|^2}{2h_t}\right)  \ud v \leq \left( \frac{\pi h_t}{2L_\expp^2\alpha_t^2} \right)^{d/2}.\label{eq:int-bound}
\end{align}
Finally, we plug in the choice of $\vradius$ and obtain 
\begin{align*}
    |f_4(x,t)-f_3(x,t)| 
    \leq   3\dgM D  C_{d,\holder} C_F   \frac{\left( \pi h_t/2 \right)^{d/2}}{ L_\expp^d \alpha_t^d } \left(\frac{L_\expp^2 L_\logg^2(\log(1/\epsilon_1)+d\log(1/h_t)/2)}{\sqrt{\reach}} \right)^\degree \epsilon_2  |\cI(x)|. 
\end{align*}
The proof is complete.
\end{proof}

\begin{lemma}\label{lemma:prop-P-hat}
    Fix any $\epsilon_2 >0$, and take $r = \epsilon_2^{1/\holder}$.
    We have for any $v,v_k(x,t) \in \dball(0,r)$,
    \begin{equation*}
        \norm{[\hat{\Delta}_{\expp_k}(v_k(x,t), v, t)]^{\otimes j} }_\infty \leq \left( L_\expp + \sqrt{D} C_{d,\holder}\epsilon_2 \right)^j \| v-v_k(x,t)\|^j,
    \end{equation*}
    and 
    \begin{equation*}
        \norm{ [\hat{\Delta}_{\expp_k}(v_k(x,t), v, t)]^{\otimes j}  -[\Delta_{\expp_k}(v_k(x,t), v, t)]^{\otimes j} }_\infty \leq j C_{d,\holder}\epsilon_2(C_{d,\holder}\epsilon_2 + 2L_\expp)^{j-1} \|v-v_k(x,t)\|^j,
    \end{equation*}
    which hold for any $z, z + v-v_k(x,t) \in \dball(0,r)$.
\end{lemma}
\begin{proof}
    In the proof, we omit the order $S=\holder+1$ in $\hat{\expp}_{k}^{\holder+1}$ for notational simplicity. Let $I\in\{1,\ldots,D\}^d$ be a multi-index. We utilize the Lipschitz continuity of $\hat{\expp}_k$ given in \eqref{eq:lip-exp-hat} and get
    \begin{equation*}
        \left|[\hat{\Delta}_{\expp_k}(v_k(x,t), v, t)]_I^{\otimes j} \right| = \prod_{i\in I} \left| \hat{\expp}_{k,i}(v)- \hat{ \expp}_{k,i}(v) \right| \leq \left( L_\expp + \sqrt{D}C_{d,\holder} \epsilon_2 \right)^j \| v-v_k(x,t)\|^j.
    \end{equation*}
    Next, we bound the approximate error for $[\hat{\Delta}_{\expp_k}(v_k(x,t), v, t)]^{\otimes j}$. Notice that the approximate error can be bounded as follows:
    \begin{align*}
        &\left| [\hat{\Delta}_{\expp_k}(v_k(x,t), v, t)]^{\otimes j}_I  -[\Delta_{\expp_k}(v_k(x,t), v, t)]^{\otimes j}_I\right|\\
        =~&  \left |\prod_{i\in I} \left( \hat{\expp}_{k,i}(v_k(x,t))- \hat{ \expp}_{k,i}(v) \right)- \prod_{i\in I} \left( \expp_{k,i}(v_k(x,t))-  \expp_{k,i}(v) \right) \right|\\
        \leq~& j \lip\left(\hat{\expp}_{k,i}- \expp_{k,i}\right) \left( \lip\left( \hat{\expp}_{k,i}\right) + \lip\left( \expp_{k,i}\right) \right)^{j-1} \|v-v_k(x,t)\|^j.
    \end{align*}
    where the last inequality follows $\prod_{i=1}^j a_i-\prod_{i=1}^j b_i = \sum_{i=1}^j a_1\cdots a_{i-1}(a_i-b_i)b_{i+1}\cdots b_j$.
    Following the same analysis in Lemma \ref{lemma:prop-Dk-hat}, we can get
    \begin{align*}
         \left| [\hat{\Delta}_{\expp_k}(v_k(x,t), v, t)]_I^{\otimes j}  -[\Delta_{\expp_k}(v_k(x,t), v, t)]_I^{\otimes j}\right| \leq j C_{d,\holder}\epsilon_2(C_{d,\holder} \epsilon_2 + 2L_\expp)^{j-1} \|v-v_k(x,t)\|^j.
    \end{align*}
    We conclude the proof by applying the entry-wise results to the $\ell^\infty$-bounds.
\end{proof}

\begin{lemma}\label{lemma:f6-approx-error}
 Fix any $\epsilon_2 \in (0,2L_\expp/C_{d,\holder})$ and take $r = \epsilon_2^{1/\holder}$. Let $x \in \cK(\alpha_t\cM, 2 \sqrt{D h_t \log(1/\delta)})$, $\vradius = 2 L_\logg \sqrt{(h_t/\alpha_t^2)(\log(1/\epsilon_1)+d\log(1/h_t)/2)}$.   For any time $t$ satisfying $h_t \leq \reach/(256 D\log(1/\delta)+\reach)$, 
 we have 
\begin{align*}
    |f_5(x,t)-f_4(x,t)|  
    &\leq \frac{C_{d,\holder}C_F (\pi h_t/2)^{d/2}}{L_\expp^{d+1} (  \alpha_t )^{d}}\degree D^{\degree/2} ( 36 L_\expp^2 L_\logg^2 (\log(1/\epsilon_1)+d\log(1/h_t)/2))^{\degree/2} \\
    & \quad \cdot  \max_{j = 0,1,\ldots, \degree} \left\{\frac{\left\| x - \alpha_t \xstar\right\|^{j}}{(h_t)^{j/2}} \right\} \epsilon_2.    
\end{align*}
\end{lemma}
\begin{proof}
For notational simplicity, define
\begin{align*}
    \tilde{\cT}^{k,j}(v_k(x,t),v,t) :=  [\hat{\Delta}_{\expp_k}(v_k(x,t), v, t)]^{\otimes j}  -[\Delta_{\expp_k}(v_k(x,t), v, t)]^{\otimes j}.
\end{align*}
As shown in Lemma \ref{lemma:f3-approx-error}, $\hat{F}_k(v)$ can be always bounded by $2C_F$. By Lemma \ref{lemma:prop-Dk-hat}, we can infer the following bound:
\begin{align*}
    \left| \sum_{l=0}^{\dgM-1}  \frac{(-1)^l}{2^l l!} \frac{  \hat{\cD}_k(v_k(x,t), v, t)^l  }{(h_t/\alpha^2_t)^l} \right|
    &\leq \exp\left(- \frac{L_\expp^2 \alpha_t^2\|v-v_k(x,t)\|^2}{h_t}\right),
\end{align*}
which holds for $\epsilon_2  \leq L_\expp/\sqrt{D} $ and $\|v -v_k(x,t)\|\leq \vradius$. Therefore, the difference between $f_5$ and $f_4$ can be  bounded as follows:
\begin{align*}
     |f_5(x,t)-f_4(x,t)| 
    \leq 2C_F  &\sum_{k\in\cI(x)}   \sum_{j=0}^{\degree-1} \frac{(\alpha_t)^j}{j! (h_t)^j} \left\|  [x - \alpha_t \xstar]^{\otimes j} \right\| \\
    &\cdot \int_{\dball(v_k(x,t), \vradius)}    \left\| \tilde{\cT}^{k,j}(v_k(x,t),v,t) \right\| \exp\left( -\frac{L_\expp^2 \alpha_t^2\|v-v_k(x,t)\|^2}{h_t}\right)\ud v,
\end{align*}
where we apply Cauchy-Schwartz inequality. Here the norm of tensor can be computed as
\begin{equation*}
    \left\|  [x - \alpha_t \xstar]^{\otimes j} \right\|  = \left\| x - \alpha_t \xstar\right\|^j.
\end{equation*}
Moreover, by Lemma \ref{lemma:prop-P-hat}, choosing $\epsilon_2 \leq 2L_\expp/C_{d,\holder}$, we have
\begin{align*}
    \left\|\tilde{\cT}^{k,j}(v_k(x,t),v,t)  \right\| &\leq D^{\degree/2} \norm{[\hat{\Delta}_{\expp_k}(v_k(x,t), v, t)]^{\otimes j}  -[\Delta_{\expp_k}(v_k(x,t), v, t)]^{\otimes j}}_\infty \\
    &\leq
     j D^{\degree/2} C_{d,\holder} \epsilon_2(3L_\expp)^{j-1} \|v-v_k(x,t) \|^j,
\end{align*}
Combine the above inequalities, we can derive that
\begin{align*}
     |f_5(x,t)-f_4(x,t)| 
    &\leq 2C_F \degree D^{\degree/2} C_{d,\holder}\epsilon_2 \frac{(3L_\expp)^{\degree-1}\alpha_t^\degree (\vradius)^\degree}{h_t^{\degree/2}}  \max_{j = 0,1,\ldots, \degree} \left\{\frac{\left\| x - \alpha_t \xstar \right\|^{j}}{h_t^{j/2}} \right\} \\
    &\quad \cdot \sum_{k\in\cI(x)}  \int_{\dball(v_k(x,t), \vradius)}  \exp\left( -\frac{L_\expp^2 \alpha_t^2\|v-v_k(x,t)\|^2}{h_t}\right) \ud v.
\end{align*}
Finally, we plug in \eqref{eq:int-bound} and the choice of $\vradius$. Then we obtain that
\begin{align*}
    |f_5(x,t)-f_4(x,t)|  
    &\leq \frac{C_{d,\holder}C_F (\pi h_t/2)^{d/2}}{L_\expp^{d+1} (  \alpha_t )^{d}} \cdot  \degree D^{\degree/2} ( 36 L_\expp^2 L_\logg^2 (\log(1/\epsilon_1)+d\log(1/h_t)/2))^{\degree/2}\\
    &\quad \cdot\max_{j = 0,1,\ldots, \degree} \left\{\frac{\left\| x - \alpha_t \xstar \right\|^{j}}{(h_t)^{j/2}} \right\} |\cI(x)| \epsilon_2.    
\end{align*}
The proof is complete.
\end{proof}

\subsection{Helper Lemmas in Section \ref{subsec:nn}}

\begin{lemma}[Network Implementation for Low-dimensional Representation]\label{lemma:encoder-approx}
    Fix time $t>0$ such that $ h_t \leq \reach^2/(4D \log(1/\delta) + \reach^2)$. Let $\thres\leq 1/2$. Given an input $x \in \cK(\alpha_t\cM, 2 \sqrt{D h_t \log(1/\delta)})$, for $k$-th chart whose center satisfies $\|x_k - x\| \leq \thres$,  $v_k$ can be well approximated by a network $\nn_{v_k} = (\nn_{v_k,1},\nn_{v_k,2},\ldots,\nn_{v_k,d})$, that is, for $\epsilon_v>0$, we have
    \begin{equation*}
        \left\| \nn_{v_k}(x,t)-v_k(x,t) \right\| \leq \epsilon_v.
    \end{equation*}
    Here each $\nn_{v_k,i}$ is a feedforward network  with depth at most $c\log(1/\epsilon_v)$, width and number of weights bounded by $c\epsilon_v^{-d/\holder}\log(1/\epsilon_v)$ and the range of weights bounded by $c\max\{B_{\cM},\reach^2\}$, where $c>0$ is some constant at most polynomially depending on $D,d,\holder$ and the smoothness of $\logg_k \circ \projm$.
\end{lemma}
\begin{proof}[Proof of Lemma \ref{lemma:encoder-approx}]Take the approximation degree to be $(D+1)\holder/d$ in Theorem 1 in \cite{chen2022nonparametric}. There exists a ReLU feedforward network $\nn_{v_k,i}$ with depth at most $c\log(1/\epsilon_v)$, width and number of weights bounded by $c\epsilon_v^{-d/\holder}\log(1/\epsilon_v)$ and the range of weights bounded by $c\max\{B_{\cM},\reach^2\}$, such that
\begin{equation*}
    \left| \nn_{v_k,i}(x,t)-v_{k,i}(x,t) \right| \leq \frac{\epsilon_v}{d}.
\end{equation*}
Then we have
\begin{equation*}
    \left\| \nn_{v_k}(x,t)-v_k(x,t) \right\| = \sqrt{\sum_{i=1}^d \left| \nn_{v_k,i}(x,t)-v_{k,i}(x,t)  \right|^2} \leq \epsilon_v.
\end{equation*}
The proof is complete by setting $\nn_{v_k} = (\nn_{v_k,1},\nn_{v_k,2},\ldots,\nn_{v_k,d})$.
\end{proof}

\begin{lemma}\label{lemma:f7-approx-error}
For time $t$ satisfying  $h_t \leq \reach^2/(64 D\log(1/\delta)+\reach^2)$ and $x \in \cK(\alpha_t\cM, 2 \sqrt{D h_t \log(1/\delta)})$, let $\vradius = 2 L_\logg \sqrt{(h_t/\alpha_t^2)(\log(1/\epsilon_1)+d\log(1/h_t)/2)}$. Then for the approximation error $\epsilon_v>0$ given in Lemma \ref{lemma:encoder-approx} and $x\in\cK_t(\delta)$, we have 
\begin{align*}
    |f_6(x,t)-f_5(x,t)| &\leq  C_\text{poly} \left( \frac{h_t}{\alpha_t^2}\right)^{\frac{d}{2}}  \left(4 D^2 \log\frac{1}{\delta}\right)^{\frac{\degree}{2}}\bigg(4 L_\logg^2 \bigg(\log\frac{1}{\epsilon_1}+\frac{d}{2}\log\frac{1}{h_t} \bigg)\bigg)^{  (\degree+d)/2+ \dgM-1} |\cI(x)| \epsilon_v.
\end{align*}
Here $C_{poly}$ is a the same constant introduced in Lemma \ref{lemma:f6-lipschitz}, which only depends on $d$, $R$, $\holder$, $C_F$ and the smoothness of exponential maps $\expp_k$.
\end{lemma}
\begin{proof}
    According to Lemma \ref{lemma:f6-lipschitz}, given a fixed $x \in \cK(\alpha_t\cM, 2 \sqrt{D h_t \log(1/\delta)})$, the Lipschitz constant of $f_5(x,t)$ with respect to $v_k(x,t)$ are bounded by
    \begin{align*}
          \lip(f_5)&\leq \sum_{k\in\cI(x)}   \sum_{j=0}^{\degree-1} \frac{D^{j/2}}{j! (h_t)^{j/2}}  \left\| x - \alpha_t \xstar\right\|^j \cdot \max_{I}\lip(\text{Tensor-Poly}_I^{k,j}) \\
          &\leq \sum_{k\in\cI(x)}   \sum_{j=0}^{\degree-1} \frac{(4 D^2 \log(1/\delta))^{j/2}}{j!} \cdot\max_{I}\lip(\text{Tensor-Poly}_I^{k,j}) \\
          &\leq C_\text{poly} \left( \frac{h_t}{\alpha_t^2}\right)^{d/2}  (4 D^2 \log(1/\delta))^{\degree/2}\left({4 L_\logg^2 (\log(1/\epsilon_1)+d\log(1/h_t)/2)} \right)^{  (\degree+d)/2+ \dgM-1} |\cI(x)| .
    \end{align*}
   Utilizing the above Lipschitz property, we have
\begin{align*}
    |f_6(x,t)-f_5(x,t)| \leq  \lip(f_5) \| \nn_{v_k}(x,t)-v_k(x,t)\|.
\end{align*}
Plugging in the approximation error $\| \nn_{v_k}(x,t)-v_k(x,t)\| \leq \epsilon_v$ given in Lemma \ref{lemma:encoder-approx} concludes the proof.
\end{proof}

\begin{lemma}\label{lemma:f6-lipschitz}
    Given any time $t$ satisfying $h_t \leq \reach^2/(64 D\log(1/\delta)+\reach^2)$ and $\epsilon_1 \in (0,1)$, let
    $\vradius = 2 L_\logg \sqrt{(h_t/\alpha_t^2)(\log(1/\epsilon_1)+d\log(1/h_t)/2)}$. Then for $x \in \cK(\alpha_t\cM, 2 \sqrt{D h_t \log(1/\delta)})$, each entry of 
    $\text{Tensor-Poly}^{k,j}$ is a polynomial with respect to $v_k(x,t)$ with degree at most $(\holder+1)(j+\dgM+1)$. Moreover, there exists a constant $C_\text{poly}>0$ only depending on $d$, $r$, $\holder$, $C_F$ and the smoothness of exponential maps $\expp_k$, such that 
    each entry of 
    $\text{Tensor-Poly}^{k,j}$ is Lipschitz with respect to $v_k(x,t)$ with Lipschitz constant and its coefficients with respect to $v_k(x,t)$ bounded by
    \begin{align*}
        C_\text{poly} \left( \frac{h_t}{\alpha_t^2}\right)^{d/2} \left({4 L_\logg^2(\log(1/\epsilon_1)+d\log(1/h_t)/2)} \right)^{(j+d)/2+\dgM-1}. 
    \end{align*}
\end{lemma}
\begin{proof}
Recall the formulation of $f_5(x,t)$ in \eqref{eq:f5_poly}, where we defined
    \begin{align*}
        & \text{Tensor-Poly}^{k,j}(v_k(x, t), t) \\
& \qquad = \int_{\cB^d(v_k(x, t), \vradius)} & \frac{[\hat{\Delta}_{\expp_k}(v_k(x,t), v, t)]^{\otimes j}}{(h_t/\alpha_t^2)^{j/2}} \cdot \left[ \sum_{l=0}^{\gamma'-1}  \frac{(-1)^l}{2^l l!} \frac{(\hat{\cD}_k(v_k(x,t), v, t))^l}{(h_t/\alpha_t^2)^l} \right] \cdot \hat{F}_k(v) \diff v.
    \end{align*}
Substituting the integration variable $v$ by $u = v-v_k(x,t)$, we can rewrite $\text{Tensor-Poly}^{k,j}$ as
    \begin{align*}
        \text{Tensor-Poly}^{k,j}(v_k(x,t), t) 
        = \int_{\cB^d(0, \vradius)} &\frac{[\hat{\Delta}_{\expp_{k}}(v_k(x,t), u, t)]^{\otimes j}}{u^j} \frac{u^j}{(h_t/\alpha_t^2)^{j/2}} \cdot\hat{F}_k(u+v_k(x,t) )\\
        & \cdot   \bigg[ \sum_{l=0}^{\dgM-1}  \frac{(-1)^l}{2^l l!} \bigg( \frac{  \sum_{i=1}^D \big(\hat{\Delta}_{\expp_{k, i}}(v_k(x,t), u, t)\big)^2  }{u^{2}}\bigg)^l\frac{  u^{2l} }{(h_t/\alpha_t^2)^l} \bigg]  \ud u,
    \end{align*}
    where $
        \hat{\Delta}_{\expp_k}(v_k(x,t), u, t) = \hat{\expp}_{k}^{\beta+1}(v_k(x, t)) - \hat{\expp}_{k}^{\beta+1}(u+v_k(x, t))$.
    Denoting $\cP^{k,j}(v_k(x,t),t)$ as follows,
    \begin{align*}
        \cP^{k,j}(v_k(x,t),t)
        =\int_{\cB^d(0, \vradius)} &\frac{[\hat{\Delta}_{\expp_{k}}(v_k(x,t), u, t)]^{\otimes j}}{u^j} \cdot\hat{F}_k(u+v_k(x,t) )\\
        & \cdot   \bigg[ \sum_{l=0}^{\dgM-1}  \frac{(-1)^l}{2^l l!} \bigg( \frac{  \sum_{i=1}^D \big(\hat{\Delta}_{\expp_{k, i}}(v_k(x,t), u, t)\big)^2  }{u^2} \bigg)^l\bigg]  \ud u,
    \end{align*}
    we notice that for a multi-index $I\in \{1,\ldots,D\}^2$, $\text{Tensor-Poly}_I^{k,j}$ and $\cP^{k,j}_I(v_k(x,t))$ have the same degree with respect to $v_k(x,t)$. The bound for the coefficients of $\text{Tensor-Poly}_I^{k,j}$, denoted as $\coef(\text{Tensor-Poly}_I^{k,j})$, can be controlled by 
    \begin{equation*}
        \coef(\cP^{k,j}_I) \left(\frac{\vradius}{\sqrt{h_t}/\alpha_t} \right)^{j + 2(\dgM-1)}.
    \end{equation*}
    Likewise, the Lipschitz constant of $\text{Tensor-Poly}_I^{k,j}$, denoted as $\lip(\text{Tensor-Poly}_I^{k,j})$, can be bounded by 
    \begin{equation*}
        \lip(\cP_I^{k,j}) \left(\frac{\vradius}{\sqrt{h_t}/\alpha_t} \right)^{j + 2(\dgM-1)}.
    \end{equation*}
    By Lemma \ref{lemma:prop-Dk-hat} and \ref{lemma:prop-P-hat}, and the calculation of polynomial size in Section \ref{sec:poly-size}, $\cP_I^{k,j}$ has degree at most $(\holder+1)(j + \dgM +1)$ with respect to $v_k(x,t)$, and there exists  a constant $C_\text{poly}>0$, which only depends on $d$, $r$, $\holder$, $C_F$ and the smoothness of exponential maps $\expp_k$, such that 
    \begin{equation*}
        \coef(\cP^{k,j}_I) \leq C_\text{poly}(\vradius)^d,
    \end{equation*}
    and $\cP^{k,j}_I$ is approximately Lipschitz with constant
    \begin{align*}
        \lip(\cP^{k,j}_I) \leq C_\text{poly}(\vradius)^d.
    \end{align*}
    Therefore, $\text{Tensor-Poly}_I^{k,j}$ has degree at most $(\holder+1)(j+\dgM+1)$ and its coefficients and Lipschitz constant can be respectively bounded by
    \begin{align*}
        C_\text{poly} \left( \frac{h_t}{\alpha_t^2}\right)^{d/2} \left({4 L_\logg^2 (\log(1/\epsilon_1)+d\log(1/h_t)/2)} \right)^{(j+d)/2+\dgM-1}, 
    \end{align*}
    where we plug in $\vradius=2 L_\logg \sqrt{(h_t/\alpha_t^2)(\log(1/\epsilon_1)+d\log(1/h_t)/2)}$. 
\end{proof}

\begin{lemma}[Network Implementation for Tensor Product]\label{lemma:nn-prod}
    Let $C\geq 1 $. Given any $j \in \{1,\ldots,\degree-1\}$, and $\epsilon_{\times}\in (0,1)$, there exists a  ReLU feedforward networks $\nn_{\times}^j$, with $\cO(\log(1/\epsilon_{\times} ) + \log C
    )$ layers, width $\cO(D^j)$, $\cO(D^j(\log(1/\epsilon_{\times} )+ \log C))$ non-zero neurons, and weight parameters upper bounded by $C^{2}$ such that
    \begin{align*}
        \left| \nn_{\times}^j(\cT_1,\cT_2) - \left\langle v_1^{\otimes j}, v_2^{\otimes j} \right\rangle \right| \leq D^j \epsilon_{\times} + 2 C D^j\epsilon_{\text{error}},
    \end{align*}
    which holds for all $v_1,v_2 \in [-C,C]^D$, and $j$-th order width-$D$ tensors $\cT_1, \cT_2$ satisfying $\|\cT_i  - v_i\|_\infty \leq \epsilon_{\text{error}}$, $i=1,2$.
\end{lemma}
\begin{proof}[Proof of Lemma \ref{lemma:nn-prod}]
    Recall the definition of the product of $j$-th order tensors, we have
    \begin{align*}
        \left\langle v_1^{\otimes j}, v_2^{\otimes j} \right\rangle = \left\langle v_1, v_2 \right\rangle^j = \sum_{I_1,\ldots,I_j =1}^D v_{1,I_1}v_{2,I_1}\cdots v_{1,I_j}v_{2,I_j} = \sum_{I=1}^{D^j} v_{1,I}^{\otimes j}v_{2,I}^{\otimes j}.
    \end{align*}
    By Lemma F.6 in \cite{oko2023diffusionmodelsminimaxoptimal}, there exists a ReLU feedforward network $\mult$ with $\cO(\log(1/\epsilon_{\times} ) + \log C
    )$ layers, width at most $96$, $\cO(\log(1/\epsilon_{\times} )+ \log C)$ non-zero neurons, and weight parameters upper bounded by $C^{2}$, such that
    \begin{align*}
        \left| \mult(\cT_{1,I},\cT_{2,I}) -  v_{1,I}^{\otimes j}v_{2,I}^{\otimes j} \right| \leq \epsilon_{\times}  + 2 C\epsilon_{\text{error}},
    \end{align*}
    which holds for any $I \in \{ 1,\ldots,D^j\}$. Therefore, taking
    \begin{align*}
        \nn_{\times}^j ( \cT_1, \cT_2) = \sum_{I=1}^{D^j} \mult(\cT_{1,I},\cT_{2,I}),
    \end{align*}
    we have \begin{align*}
         \left| \nn_{\times}^j(\cT_1,\cT_2) - \left\langle v_1^{\otimes j}, v_2^{\otimes j} \right\rangle \right| \leq \sum_{I=1}^{D^j} \left| \mult(\cT_{1,I},\cT_{2,I}) -  v_{1,I}^{\otimes j}v_{2,I}^{\otimes j} \right| \leq D^j \epsilon_{\times}  + 2 C D^j \epsilon_{\text{error}}.
    \end{align*}
    By Lemma F.1, F.2 and F.3 in \cite{oko2023diffusionmodelsminimaxoptimal}, $\nn_{\times}^j$ can be exactly implemented by a ReLU feedforward network with $\cO(\log(1/\epsilon_{\times} ) + \log C
    )$ layers, width $\cO(D^j)$, $\cO(D^j(\log(1/\epsilon_{\times} )+ \log C))$ non-zero neurons, and weight parameters upper bounded by $C^{2}$.
\end{proof}

\begin{lemma}[Network Implementation for Projection Term]\label{lemma:nn-proj}
  Given any $j \in \{0,\ldots,\degree-1\}$, and $\epsilon_{\rm proj}\in (0,1)$, there exists a tensor consisting of ReLU feedforward networks, $\nn_{\rm proj}^j = \{\nn^j_{\text{proj},I}\}_{I=1,2,\ldots,D^j} $, such that for any  time $t \in [t_0, T]$,
  \begin{align*}
      \sup_{ x \in \alpha_t \cK(\cM, \reach)}\left| \nn^j_{\text{proj},I}(x,h_t,\alpha_t) - \frac{[x - \alpha_t \xstar]_I^{\otimes j}}{h_t^{j/2}} \right| \leq h_{t}^{d/2}\epsilon_{\rm proj}.
  \end{align*}
  Here each network $\nn^j_{\text{proj},I}(x,h_t,\alpha_t) \in \cF(L,W,S,B,\cdot)$, where $L = c(\log^3(1/(h_{t_0}^{j/2}\epsilon_{\rm proj})))$, $W=S= c(\epsilon_{\rm proj})^{-d/\holder} \log(1/(h_{t_0}^{j/2}\epsilon_{\rm proj}))  + \log^4(1/h_{t_0})$, $B = h_{t_0}^{-j}$, and the constant $c$ depends on $D, d,\holder$, $B$ and the smoothness of projection maps.
\end{lemma}
\begin{proof}[Proof of Lemma \ref{lemma:nn-proj}]
    Fix $I \in \{1,2,\ldots,D^j\}$. By Proposition \ref{prop:proj-smooth}, $ [x - \alpha_t \xstar]_I^{\otimes j}$  is $C^\infty$. Then we apply Lemma 2 in \cite{chen2022distribution} with approximation error $ h_t^{j/2} \epsilon_{\text{poly}}$ and approximation degree $$\tilde{\beta} =\frac{D\holder\log(h_t^{(j+d)/2} \epsilon_{\text{poly}})}{d \log(\epsilon_{\text{poly}})}.$$ This gives that there exists a ReLU feedforward network, $\tilde{\nn}_{\text{proj},I}^j$ such that
  \begin{align*}
      \left| \tilde{\nn}^j_{\text{proj},I}(x,h_t,\alpha_t) - [x - \alpha_t \xstar]_I^{\otimes j} \right| \leq  h_t^{(j+d)/2} \epsilon_{\text{poly}},
  \end{align*}
which holds for any $ x \in \alpha_t\cK(\cM,  \reach)$ and $\alpha_t\in [0,1]$. Here each network $\tilde{\nn}^j_{\text{proj},I}$ has no more than $c(\log(1/(h_t^{(j+d)/2}\epsilon_{\rm proj})) + 1)$ layers, and at most $c'(h_t^{(j+d)/2}\epsilon_{\rm proj})^{-D/\tilde{\beta}} (\log(1/(h_t^{(j+d)/2}\epsilon_{\rm proj})) + 1)$ neurons and weight parameters, where the constant $c$  depends on $D, d,\holder$, $B$ and the smoothness of projection maps. Note that given the choice of $\tilde{\beta}$, we have
\begin{align*}
    \left(h_t^{(j+d)/2}\epsilon_{\rm proj}\right)^{-\frac{D}{\tilde{\beta}}} =     \left(h_t^{(j+d)/2}\epsilon_{\rm proj}\right)^{-\frac{d \log\left(\epsilon_{\text{poly}}\right)}{\holder \log\left(h_t^{(j+d)/2} \epsilon_{\text{poly}}\right)}   }  = \epsilon_{\rm proj}^{-\frac{d}{\holder}   }.
\end{align*}
Then by Lemma F.6-F.8 in \cite{oko2023diffusionmodelsminimaxoptimal}, there exists networks $\nn^j_{\text{proj},I} \in \cF(L,W,S,B,\cdot)$ such that
\begin{align*}
      \left| \nn^j_{\text{proj},I}(x,h_t,\alpha_t) - \frac{[x - \alpha_t \xstar]_I^{\otimes j}}{h_t^{j/2}} \right| \leq h_t^{d/2}\epsilon_{\rm proj},
  \end{align*}
where the network configuration is  $L= c(\log^3(1/(h_{t_0}^{j/2}\epsilon_{\rm proj})))$, $W=S= c(\epsilon_{\rm proj})^{-d/\holder} \log(1/(h_{t_0}\epsilon_{\rm proj}))  + \log^4(1/h_{t_0})$, and $B = h_{t_0}^{-j}$.
\end{proof}

\begin{lemma}[NN Implementation for Chart Determination] \label{lemma:nn-chart}
Given $x\in \RR^D$ where $\|x\|_\infty \leq B_{\rm trunc}$, and functions $\{g_k:{\RR^{D+2} \to \RR^m}\}_{k=1,\ldots,C_\cM}$, for any $\epsilon_{\rm{det}} \in (0,1)$, we can construct a network $\nn_{\rm{det}}$ 
such that for any  $k \in [C_\cM]$,
\begin{align*}
     \nn_{\rm{det}}(x - x_k) =
\begin{cases}
1 & \text{if}~\norm{x - x_k}^2 \leq (1-\epsilon_{\rm det})\thres^2 \\
0 & \text{if}~\norm{x - x_k}^2 \geq \thres^2,
\end{cases}
\end{align*}
and
\begin{align*}
    \left\|   \sum_{k=1}^{C_\cM} \nn_{\rm{det}}(x-x_k) g_k(x,h_t,\alpha_t) - \sum_{k:\|x_k-x\| \leq \thres} g_k(x,h_t,\alpha_t) \right\| \leq \left\| \sum_{k:  (1-\epsilon_{\rm{det}})\thres<\|x_k-x\| < \thres} g_k(x,h_t,\alpha_t)   \right\| . 
\end{align*}
Here $\nn_{\rm{det}}$ is a network with number of layers, number of neurons and weight parameters bounded by $c\log(D(B+B_{\rm trunc})^2/(\epsilon_{\rm det} \thres^2))$, and width bounded by $cD$, where  $c>0$ is some absolute constant.
\end{lemma}
\begin{proof}[Proof of Lemma \ref{lemma:nn-chart}]
We start with implementing the squared Euclidean distance. By Lemma 1 in \cite{chen2022nonparametric}, there exists a network $\hat{\sf sq}$ that approximates the square function ${\sf sq}(x)=x^2$, $x\in[-1,1]$, with an arbitrary error $\epsilon_{\sf sq} >0$. The network c has a depth and the number of neurons and weight parameters no more than $c\log (1/\epsilon_{\sf sq})$ for an absolute constant $c$, and the width of $\hat{\sf sq}$ is an absolute constant.

Now we utilize $\hat{\sf sq}$ to construct a network $\hat{d}$ implementing $\norm{x - x_k}^2$. Given $x\in \RR^D$ where $\|x\|_\infty \leq B_{\rm trunc}$, we can write the distance as $\norm{x - x_k}^2 = \sum_{i=1}^D (B+B_{\rm trunc})^2 \left(\frac{[x - x_k]_i}{B + B_{\rm trunc}}\right)^2$. Then we define 
\begin{align*}
\hat{d}(x - x_k) = \sum_{i=1}^D (B + B_{\rm trunc})^2 ~ \hat{\sf sq}\left(\frac{[x - x_k]_i}{B + B_{\rm trunc}}\right).
\end{align*}
Examining the summation in $\hat{d}$, we note that $\hat{d}$ achieves the following  approximation error
\begin{align*}
    \left|\hat{d}(x - x_k) - \norm{x - x_k}^2 \right|\leq D (B + B_{\rm trunc})^2 \epsilon_{\sf sq}.
\end{align*}

Next, to implement $\hat{\mathds{1}}$, we consider a basic step function $F(a) = 2\relu(a-0.5\tilde{r}) - 2\relu(a-\tilde{r})$ for some $\tilde{r} >0$, which can be directly formulated as a feedforward layer. It is straightforward to check
\begin{align*}
    F_l(a) = \underbrace{F\circ \cdots \circ F}_{l} (a) = 
\begin{cases}
0 & \text{if}~a < (1-2^{-l})\tilde{r}, \\
2^l(a-\tilde{r}) + \tilde{r}   & \text{if}~ (1-2^{-l})\tilde{r} \leq a \leq \tilde{r} \\
1 & \text{if}~a > \tilde{r}
\end{cases}.
\end{align*}
Let $\tilde{r} = (\thres)^2 - D(B + B_{\rm trunc})^2 \epsilon_{\sf sq} $ and  $\epsilon_{\sf sq} = \epsilon_{\rm det} \thres^2/(2D(B+B_{\rm trunc})^2)$. We choose $l$ such that $(1-2^{-l})\tilde{r} \geq (1-\epsilon_{\rm det}) (\thres)^2 + D(B + B_{\rm trunc})^2 \epsilon_{\sf sq} $, which yields $l = \lceil \log(1/\epsilon_{\rm det}) \rceil$. Now take
$\hat{\mathds{1}}_{\epsilon_{\rm det}} = 1- F_l(a)/\tilde{r}$. Then the composition of $\hat{\mathds{1}}_{\epsilon_{\rm det}} \circ \hat{d}$ verifies
\begin{align*}
\left[\hat{\mathds{1}}_{\epsilon_{\rm det}} \circ \hat{d}\right] (x - x_k) =
\begin{cases}
1 & \text{if}~\norm{x - x_k}^2 \leq (1-\epsilon_{\rm det})\thres^2 \\
0 & \text{if}~\norm{x - x_k}^2 \geq \thres^2.
\end{cases}
\end{align*}
Hereby, we construct $\nn_{\rm{det}}(x-x_k) = \left[\hat{\mathds{1}}_{\epsilon_{\rm det}} \circ \hat{d}\right] (x - x_k)$. The network implementation $\nn_{\rm{det}}$ of the chart determination yields an error
\begin{align*}
    &\quad~ \left\|   \sum_{k=1}^{C_\cM} \nn_{\rm{det}}(x-x_k) g_k(x,h_t,\alpha_t) - \sum_{k:\|x_k-x\| \leq \thres} g_k(x,h_t,\alpha_t) \right\| \\
    &= \left\|   \sum_{k=1}^{C_\cM} \left(\nn_{\rm{det}}(x-x_k) - 1\right) \mathds{1}\left( (1-\epsilon_{\rm det})\thres^2 <\norm{x - x_k}^2 < \thres^2\right) g_k(x,h_t,\alpha_t) \right\| \\
    &\leq \left\| \sum_{k:  (1-\epsilon_{\rm det})\thres<\|x_k-x\| < \thres} g_k(x,h_t,\alpha_t)   \right\| . 
\end{align*}
Based on our construction, $\nn_{\rm{det}}(x-x_k)$ has number of layers, width, and number of neurons and weight parameters bounded by $c\log(D(B+B_{\rm trunc})^2/(\epsilon_{\rm det} \thres^2))$ for some absolute constant $c>0$.
\end{proof}

\begin{lemma}[Network Implementation for Polynomials]\label{lemma:nn-poly}
  Given any $k \in \{1,2,\ldots,C_\cM\}$, $j \in \{0,\ldots,\degree-1\}$, and $\epsilon_{\rm {poly}}\in (0,1)$, there exists a network $\nn_{\rm {poly}}^{k,j} = \{\nn_{\rm {poly},I}^{k,j}\}_{I\in\{1,\ldots,D\}^j} $, such that for any $v \in [-C,C]^d$ and time $t \in [t_0,T]$,
  \begin{align*}
      \left| \nn_{\rm {poly},I}^{k,j}(v,h,\alpha) - \text{Tensor-Poly}_I^{k,j}(v,h,\alpha) \right| \leq  h_{t_0}^{ d/2}\epsilon_{\rm {poly}}.
  \end{align*}
Here each network $\nn_{\rm {poly},I}^{k,j}$ has (i) no more than $c_{\rm {poly}}(\log(1/(h_{t_0}\epsilon_{\rm {poly}})) + 1)$ layers with width bounded by $c_{\rm {poly}}(j+\dgM+1)^d$, and (ii) at most $c_{\rm {poly}}(j+\dgM+1)^d(\log(1/(h_{t_0}\epsilon_{\rm {poly}})) + 1)$ neurons and weight parameters bounded by $c_{\rm {poly}}C$, where the constant $c_{\rm {poly}}$ depends on $d,\holder,C_F$ and the smoothness of exponential maps $\expp_k$.
\end{lemma}
\begin{proof}[Proof of Lemma \ref{lemma:nn-poly}]
    By Lemma \ref{lemma:f6-lipschitz}, $ \text{Tensor-Poly}_I^{k,j}(v,h,\alpha) $ is a polynomial with degree at most $(\holder+1)(j+\dgM+1)$, and thus a H\"older function. 
    Then applying Theorem 3 in \cite{chen2022nonparametric} directly yields the results. 
\end{proof}

\subsection{Helper Lemmas in Section \ref{sec:sM-approx}}

\begin{lemma}\label{lemma:final-approx-error}
    Let $\bar{s}_1(x,t)$ and $\bar{s}_2(x,t)$ satisfy that given $x \in \RR^D$ and $t$ satisfying  \eqref{cond:time},
    \begin{equation*}
        |\bar{s}_1(x,t) - s_1(x,t)| \leq \epsilon_1, ~~\text{and}~~ \|\bar{s}_2(x,t) - s_2(x,t)\| \leq \epsilon_2
    \end{equation*}
    for some sufficiently small $\epsilon_1, \epsilon_2>0$.
    Then we have
    \begin{equation*}
        \left\| s_{\cM}(x,t) -  \frac{\bar{s}_2(x,t)/\sqrt{h_t}}{\bar{s}_1(x,t)} \right\| \leq  \frac{C_0\sqrt{D}\alpha_t^d}{h_t^{(d+1)/2}} \left( \log\frac{1}{h_t} \right)^{\frac{d+1}{2}} \exp \left( \frac{5(d+1) \log(1/h_t) \|x- \alpha_t\xstar \|}{\reach \alpha_t} \right)(\epsilon_1 + \epsilon_2),
    \end{equation*}
     where $C_0 > 0$ is an absolute constant depending on $d, B_{\cM}, C_f$ and $C_G$ defined in Lemma \ref{lemma:s1-lb}.
\end{lemma}
\begin{proof}
    By the definition of $s_\cM(x,t)$ given in \eqref{eq:def-sM}, we have
    \begin{align*}
        \sqrt{h_t}\left\| s_{\cM}(x,t) -  \frac{\bar{s}_2(x,t)/\sqrt{h_t}}{\bar{s}_1(x,t)} \right\| &= \sqrt{h_t} \left\| \frac{s_2(x,t)/\sqrt{h_t}}{s_1(x,t)} -  \frac{\bar{s}_2(x,t)/\sqrt{h_t}}{\bar{s}_1(x,t)} \right\| \\
        &\leq \left\| \frac{s_2(x,t)}{s_1(x,t)} -  \frac{s_2(x,t)}{\bar{s}_1(x,t)} \right\| + \left\| \frac{s_2(x,t)}{\bar{s}_1(x,t)} -  \frac{\bar{s}_2(x,t)}{\bar{s}_1(x,t)}  \right\| \\
        &\leq \frac{\|s_2(x,t)\|\cdot \|\bar{s}_1(x,t)-s_1(x,t)\| }{\|s_1(x,t) \bar{s}_1(x,t)\|} + \frac{\|\bar{s}_2(x,t)-s_2(x,t)\|}{\| \bar{s}_1(x,t)\|}\\
        &\leq \frac{\|s_2(x,t)\|\epsilon_1 }{\|s_1(x,t) \bar{s}_1(x,t)\|} + \frac{\epsilon_2}{\| \bar{s}_1(x,t)\|}.
    \end{align*}
    Applying the lower bound on $s_1(x,t)$ presented in Lemma \ref{lemma:s1-lb} and the upper bound on $s_2(x,t)$ in Lemma \ref{lemma:s2-ub}, we  have
    \begin{align*}
        \left\| s_{\cM}(x,t) -  \frac{\bar{s}_2(x,t)/\sqrt{h_t}}{\bar{s}_1(x,t)} \right\| \leq  \frac{C_0\sqrt{D}\alpha_t^d}{h_t^{(d+1)/2}} \left( \log\frac{1}{h_t} \right)^{\frac{d+1}{2}} \exp \left( \frac{5(d+1) \log(1/h_t) \|x- \alpha_t\xstar \|}{\reach \alpha_t} \right)(\epsilon_1 + \epsilon_2),
    \end{align*}
    where $C_0 > 0$ is an absolute constant depending on $d, B_{\cM}, C_f$ and $C_G$ defined in Lemma \ref{lemma:s1-lb}.
\end{proof}

\begin{lemma}[Lower Bound of $s_1(x,t)$]\label{lemma:s1-lb}
    Given time $t$ satisfying  $h_t \leq  \reach^2 /(1/36 + \reach^2)$, and $x \in \cB(\alpha_t\cM, \alpha_t \reach)$, we have
    \begin{equation*}
        s_1(x,t) \geq C_f^{-1} C_G^{-1}  \frac{(\pi h_t)^{d/2} e^{-1/8}}{\Gamma(d/2+1) (2\alpha_t)^d}  \exp \left( - \frac{ \| x - \alpha_t \xstar \|  }{2\reach \alpha_t }\right).
    \end{equation*}
    Here $C_G>0$ is a constant such that $C_G^{-1} \leq  |G_k(v)| \leq C_G$ for all $v\in \dball(0,r)$ and $k =1,\ldots,C_\cM$.
\end{lemma}
\begin{proof}
    Recall the definition of $s_1(x,t)$:
    \begin{align*}
     s_1(x,t) 
     &= \sum_{k=1}^{C_\cM} \int_{\dball(0,r)} \exp \left( - \frac{ \|\alpha_t \xstar- \alpha_t \expp_k(v)\|^2 + 2 \langle x - \alpha_t \xstar , \alpha_t \xstar - \alpha_t \expp_k(v)\rangle }{2h_t} \right) \\
     &\qquad\qquad\qquad\qquad \cdot \rho_k ( \expp_k(v) ) \cdot p_{\rm data}( \expp_k(v)) \cdot G_k(v)\ud v .
    \end{align*}
    By Assumption \ref{assump:density}, $p_{\rm data}$ is lower bounded by $C_f^{-1}$. Moreover, supposing $\min_{v \in \dball(0,r)}G_k(v) \geq C_G^{-1}$  for any $k =1,\ldots,C_\cM$, we can derive
    \begin{align*}
     s_1(x,t) \geq C_f^{-1} C_G^{-1} \sum_{k=1}^{C_\cM} \int_{\dball(0,r)} &\exp \left( - \frac{ \|\alpha_t \xstar- \alpha_t \expp_k(v)\|^2 }{2h_t} \right)\\  
     &\cdot\exp \left( \frac{ 2 \langle x - \alpha_t \xstar , \alpha_t \xstar - \alpha_t \expp_k(v)\rangle }{2h_t} \right)\rho_k ( \expp_k(v) ) \ud v .
    \end{align*}
    Noting that the integrand is non-negative, we can shrink the integral region as follows:
    \begin{align*}
     s_1(x,t) \geq C_f^{-1} C_G^{-1} \sum_{k=1}^{C_\cM} \int_{\|\expp_k(v) -\xstar\| \leq \frac{\sqrt{h_t}}{2 \alpha_t}} &\exp \left( - \frac{ \|\alpha_t \xstar- \alpha_t \expp_k(v)\|^2 }{2h_t} \right)\\  
     &\cdot\exp \left( \frac{ 2 \langle x - \alpha_t \xstar , \alpha_t \xstar - \alpha_t \expp_k(v)\rangle }{2h_t} \right)\rho_k ( \expp_k(v) ) \ud v .
    \end{align*}
    When the time $t$ satisfies $h_t \leq \reach^2 /(1/36 + \reach^2)$, we have $\sqrt{h_t}/(2\alpha_t) \leq 3 \reach$. Then given $\|\expp_k(v) -\xstar\| \leq \frac{\sqrt{h_t}}{2 \alpha_t}$, we can apply Lemma \ref{lemma:cross-term} to get
    \begin{align*}
      - \frac{  \langle x - \alpha_t \xstar , \alpha_t \xstar - \alpha_t \expp_k(v)\rangle }{h_t} 
     \geq -\frac{ 2\| x - \alpha_t \xstar \| \cdot \| \alpha_t \xstar - \alpha_t \expp_k(v) \|^2 }{\reach \alpha_t h_t}
     \geq - \frac{ \| x - \alpha_t \xstar \|  }{2\reach \alpha_t }.
    \end{align*}
    This further yields
    \begin{align*}
      s_1(x,t) 
     &\geq \sum_{k=1}^{C_\cM} \int_{\|\expp_k(v) -\xstar\| \leq \frac{\sqrt{h_t}}{2 \alpha_t}} \exp \left( - \frac{1}{8} - \frac{ \| x - \alpha_t \xstar \|  }{2\reach \alpha_t }\right) \rho_k ( \expp_k(v) ) \ud v \\
     &\geq \frac{\pi^{d/2}}{\Gamma(d/2+1)}\left( \frac{\sqrt{h_t}}{2 \alpha_t}\right)^d  \exp \left( - \frac{1}{8} - \frac{ \| x - \alpha_t \xstar \|  }{2\reach \alpha_t }\right). 
    \end{align*}
    Here in the last inequality we use $\sum_{k=1}^{C_\cM}  \rho_k =1$. In conclusion, we obtain the following lower bound for $s_1(x,t)$:
    \begin{align*}
     s_1(x,t) \geq C_f^{-1} C_G^{-1}  \frac{(\pi h_t)^{d/2} e^{-1/8}}{\Gamma(d/2+1) (2\alpha_t)^d}  \exp \left( - \frac{ \| x - \alpha_t \xstar \|  }{2\reach \alpha_t }\right) .
    \end{align*}
    The proof is complete.
\end{proof}

\begin{lemma}[Upper Bound of $s_2(x,t)$]\label{lemma:s2-ub}
     Given time $t$ satisfying  $h_t \leq \min\{\reach^2/(64 DL_\logg^2 \log(1/\delta)/9+\reach^2), \reach^2 /(4 c_0^2 ((d+1)\log(1/h_t)/2)/L_\expp^2+ \reach^2) \}$, and $x \in \cK(\alpha_t\cM, 2 \sqrt{D h_t \log(1/\delta)})$, we have
    \begin{equation*}
        \| s_2(x,t) \| \leq  \left[ \frac{ C_f C_G\pi^{d/2} \left(2 (d+1)\log(1/h_t) \right)^{\frac{d+1}{2}}}{\Gamma(d/2+1) \alpha_t^d}  \exp \left( \frac{ 4 (d+1)\log(1/h_t) \cdot \| x - \alpha_t \xstar \|  }{\reach \alpha_t } \right) +  \sqrt{D}B  \right] h_t^{d/2}.
    \end{equation*}
\end{lemma}
\begin{proof}
    Recall the definition of $s_2(x,t)$:
    \begin{align*}
        s_2(x,t) 
        = \sum_{k=1}^{C_\cM} \int_{\dball(0,r)} &-\frac{\alpha_t}{\sqrt{h_t}}(\xstar-\expp_k(v))  \exp \left( - \frac{ \|\alpha_t \xstar- \alpha_t \expp_k(v)\|^2}{2h_t} \right) \\
     & \cdot \exp \left( - \frac{ 2 \langle x - \alpha_t \xstar , \alpha_t \xstar - \alpha_t \expp_k(v)\rangle }{2h_t} \right) 
 F_k (v) \ud v .
    \end{align*}
    Notice that we have $\sup_{x,y\in\cM}\|x-y\| \leq \sqrt{D}B$ according to Assumption $\ref{assump:manifold}$. Then by Lemma \ref{lemma:clip-error}, taking $\epsilon_1 = \sqrt{h_t}/\alpha_t$ and $\tilde{\Delta} =  \sqrt{2(d+1)h_t \log(1/h_t)/\alpha_t^2}$ yields
    \begin{align*}
        \| s_2(x,t) \| \leq  \sum_{k=1}^{C_\cM} \int_{\|\expp_k(v) -\xstar\| \leq \tilde{\Delta}}  &\frac{\alpha_t}{\sqrt{h_t}}\|\xstar-\expp_k(v)\| \exp \left( - \frac{ \|\alpha_t \xstar- \alpha_t \expp_k(v)\|^2}{2h_t} \right) \\
     \cdot &\exp \left( - \frac{ 2 \langle x - \alpha_t \xstar , \alpha_t \xstar - \alpha_t \expp_k(v)\rangle }{2h_t} \right)  F_k ( v ) \ud v +  \sqrt{D}B h_t^{d/2} ,
    \end{align*}
   Next, we bound the cross term by Lemma \ref{lemma:cross-term}:
    \begin{align*}
        &\quad\sum_{k=1}^{C_\cM} \int_{\|\expp_k(v) -\xstar\| \leq \tilde{\Delta}} \frac{\alpha_t}{\sqrt{h_t}}\|\xstar-\expp_k(v)\|   \exp \left( - \frac{ 2 \langle x - \alpha_t \xstar , \alpha_t \xstar - \alpha_t \expp_k(v)\rangle }{2h_t} \right) F_k(v) \ud v \\
        &\leq \sum_{k=1}^{C_\cM} \int_{\|\expp_k(v) -\xstar\| \leq \tilde{\Delta}} \frac{\alpha_t}{\sqrt{h_t}}\|\xstar-\expp_k(v)\|   \exp \left( \frac{ 2\| x - \alpha_t \xstar \| \cdot \| \alpha_t \xstar - \alpha_t \expp_k(v) \|^2 }{\reach \alpha_t h_t} \right) F_k(v) \ud v  \\
        &\leq \sum_{k=1}^{C_\cM} \int_{\|\expp_k(v) -\xstar\| \leq \tilde{\Delta}} \frac{\alpha_t \tilde{\Delta}}{\sqrt{h_t}}  \exp \left( \frac{ 2 \alpha_t^2 \tilde{\Delta}^2\| x - \alpha_t \xstar \|  }{\reach \alpha_t h_t} \right)  F_k(v) \ud v. 
    \end{align*}
    Since $p_{\rm data}$ is upper bounded by $C_f$ (Assumption \ref{assump:density}) and $\sup_{v \in \dball(0,r)}G_k(v) \leq C_G$  for any $k =1,\ldots,C_\cM$, we have
    \begin{align*}
        &\quad \sum_{k=1}^{C_\cM} \int_{\|\expp_k(v) -\xstar\| \leq \tilde{\Delta}} \frac{\alpha_t \tilde{\Delta}}{\sqrt{h_t}}  \exp \left( \frac{ 2 \alpha_t^2 \tilde{\Delta}^2\| x - \alpha_t \xstar \|  }{\reach \alpha_t h_t} \right)  F_k(v) \ud v \\
        &\leq \sum_{k=1}^{C_\cM} \int_{\|\expp_k(v) -\xstar\| \leq \tilde{\Delta}}  C_f C_G \frac{\alpha_t \tilde{\Delta}}{\sqrt{h_t}} \exp \left( \frac{ 2 \alpha_t^2 \tilde{\Delta}^2\| x - \alpha_t \xstar \|  }{\reach \alpha_t h_t} \right)  \rho_k ( \expp_k(v) ) \ud v \\
        &\leq C_f C_G \frac{\pi^{d/2}\tilde{\Delta}^d }{\Gamma(d/2+1)}\cdot \frac{\alpha_t \tilde{\Delta}}{\sqrt{h_t}} \cdot \exp \left( \frac{ 2 \alpha_t^2 \tilde{\Delta}^2\| x - \alpha_t \xstar \|  }{\reach \alpha_t h_t} \right).
    \end{align*}
    Putting together the above inequalities leads to
    \begin{align*}
        \| s_2(x,t) \| &\leq C_f C_G \frac{\pi^{d/2}\tilde{\Delta}^d }{\Gamma(d/2+1)}\cdot \frac{\alpha_t \tilde{\Delta}}{\sqrt{h_t}} \cdot \exp \left( \frac{ 2 \alpha_t^2 \tilde{\Delta}^2\| x - \alpha_t \xstar \|  }{\reach \alpha_t h_t} \right) +  \sqrt{D}B h_t^{d/2}\\
        &\leq   \frac{ C_f C_G(\pi h_t)^{d/2} \left(2 (d+1)\log(1/h_t) \right)^{\frac{d+1}{2}}}{\Gamma(d/2+1) \alpha_t^d}  \exp \left( \frac{ 4 (d+1)\log(1/h_t) \cdot \| x - \alpha_t \xstar \|  }{\reach \alpha_t } \right) +  \sqrt{D}B h_t^{d/2}.
    \end{align*}
    The proof is complete.
\end{proof}

\subsection{Helper Lemmas in Section \ref{sec:approx-large}}\label{sec:proof-large}

\begin{lemma}\label{lemma:large-truncation}
For any $\epsilon,\aerror \in (0,1)$ and time $t >0$ satisfying $h_t \geq \epsilon^{2/\holder}/4$. Let $\thres = 2 \sqrt{h_t \log(1/\aerror)} + h_tB + L_\expp \epsilon^{1/\holder}$. Then we have
    \begin{align*}
        \left| \sum_{k: \| x - x_k\| > \thres} \int_{x_0 \in U_k} \exp \left( -\frac{\|x - \alpha_t x_0\|^2}{2h_t} \right) \rho_k(x_0) p_{\rm data }(x_0) \ud \vol(x_0) \right| 
        \leq  \aerror.
        \end{align*}
\end{lemma}
\begin{proof}
Recall the atlas $\{(U_k,\logg_k)\}_{k=1}^{\cM}$ where $U_k = \expp_k(\dball(0,\epsilon^{1/\holder})$. It satisfies $U_k \subseteq \cB^D(x_k, L_\expp \epsilon^{1/\holder})$. For the $k$-th chart satisfying $\| x - x_k\| > \thres$ and any $x_0 \in U_k$, we have
    \begin{align*}
         \|x - \alpha_t x_0\| &\geq  \|x - x_k\|  - \|x_k - \alpha_t x_k\| -  \|\alpha_t x_k - \alpha_t x_0\| \\
         &\geq  \thres - (1-\alpha_t)B - \alpha_t L_\expp \epsilon^{1/\holder} \\
         &\geq \thres - h_t B - L_\expp \epsilon^{1/\holder} ,
    \end{align*}
    where the last inequality uses $1-\alpha_t = (1-\alpha_t^2)/(1+\alpha_t) =  h_t /(1+\alpha_t)  \leq h_t$.  Since we take  $\thres = 2 \sqrt{h_t \log(1/\aerror)} + h_tB + L_\expp \epsilon^{1/\holder}$, we have 
    \begin{align*}
        \|x - \alpha_t x_0\| \geq 2\sqrt{h_t \log(1/\aerror)}.
    \end{align*}
    This further yields
    \begin{align*}
         \exp \left( -\frac{\|x - \alpha_t x_0\|^2}{2h_t} \right)\leq \exp \left( -\frac{ 4h_t \log(1/\epsilon)}{2h_t} \right)\leq \aerror.
    \end{align*}
    Therefore, we can conclude 
    \begin{align*}
        \left| \sum_{k: \| x - x_k\| > \thres} \int_{x_0 \in U_k} \exp \left( -\frac{\|x - \alpha_t x_0\|^2}{2h_t} \right) \rho_k(x_0) p_{\rm data}(x_0) \ud \vol(x_0) \right| 
        \leq  \aerror.
    \end{align*}
    The proof is complete.
\end{proof}

\begin{lemma}\label{lemma:large-poly}
    For any $\epsilon, \aerror \in (0,1)$,  and time $t >0$ satisfying $h_t \geq \epsilon^{2/\holder}/4$. Let  $\thres = 2 \sqrt{h_t \log(1/\aerror)} + h_tB + L_\expp \epsilon^{1/\holder}$, and $S= 4 e^2 \log(1/\aerror) + 4e^2(B+L_\expp+1)^2$. Then for $\text{Poly}^k$ defined in \eqref{def:Ik-poly-approximator} with any $\degree_0>0$, we have
    \begin{align*}
    \sup_{x\in \cK_t(\aerror)}\left|\text{Poly}^k(x,t)-I_k(x,t)\right| = \cO\left( \frac{ (\sqrt{\log(1/\aerror) }+ B)^{\degree_0} }{\reach^{\degree_0} } \epsilon^{\degree_0/\holder}+ \aerror \right),
\end{align*}
\end{lemma}
\begin{proof}
We begin with approximating the exponential functions that take $\cT_k(x,x_0,t)$ and $\cD_k(x,x_0,t)$ as inputs. By the definition of $\projk(x,t)$, we have $x - \projk(x,t) \perp \alpha_t \cdot T_{x_k}\cM$, which yields
    \begin{align*}
       \cT_k(x,x_0,t) = \langle x - \projk(x,t), \projk(x,t)  - \alpha_t x_k + \alpha_t x_k- \alpha_t x_0 \rangle = \langle x - \projk(x,t), \alpha_t x_k- \alpha_t x_0 \rangle.
    \end{align*}
    Then by Lemma \ref{lemma:cross-term}, for any $x_0 \in U_k$, it holds that
    \begin{align*}
        \left| \langle x - \projk(x,t), \alpha_t x_k- \alpha_t x_0 \rangle \right| \leq \frac{2\alpha_t}{\reach} \| x - \projk(x,t)\| \cdot \|x_k- x_0\|^2 \leq \frac{2\alpha_t \epsilon^{2/\holder}}{\reach} \| x - \projk(x,t)\|.
    \end{align*}
     For the selected charts satisfying $\| x - x_k\| \leq \thres$, we can derive 
    \begin{align*}
        \| x- \projk(x,t) \| \leq \| x- \alpha_t x_k \| \leq  \| x- x_k \| + \| x_k - \alpha_t x_k \| \leq \thres + h_t B.
    \end{align*}
    Plugging in $\thres = 2 \sqrt{h_t \log(1/\aerror)} + h_tB + L_\expp \epsilon^{1/\holder}$, we can derive
    \begin{align*}
        \left| \frac{1}{h_t} \cT_k(x,x_0,t) \right| \leq \frac{2\alpha_t \epsilon^{2/\holder}( 2\sqrt{h_t\log(1/\aerror)}+2h_t B + L_\expp \epsilon^{1/\holder})}{h_t\reach} \leq \frac{ 8(\sqrt{\log(1/\aerror) }+ B+L_\expp)\epsilon^{1/\holder} }{\reach} ,
    \end{align*}
    where we apply $h_t \geq \epsilon^{2/\holder}/4$ in the last inequality.
    Now we approximate the exponential function with a polynomial with degree $\degree_0$.
    By Lemma \ref{lemma:smooth-approx}, the approximation error can be bounded as
    \begin{align}\label{eq:Lexp1}
        \left | \exp \left( -\frac{1}{h_t} \cT_k(x,x_0,t) \right) - \sum_{l=0}^{\degree_0-1} \frac{(-1)^l\alpha_t^l}{h_t^l l!} \cT^l_k(x,x_0,t) \right| \leq \left( \frac{ 8(\sqrt{\log(1/\aerror) }+ B+L_\expp)\epsilon^{1/\holder} }{\reach} \right)^{\degree_0}.
    \end{align}
    Likewise, we can derive the upper bound for $\cD_k(x,x_0,t)$. For any $x_0 \in U_k$, it holds that
    \begin{align*}
        \left|\frac{\alpha_t^2}{2h_t}\cD_k(x,x_0,t) \right| &\leq \frac{1}{2h_t}\left( \|\projk(x,t) -  \alpha_t x_k\|+ \alpha_t^2 \| x_k-x_0\|\right)^2 \\
        &\leq \frac{1}{2h_t}\left( \|x -  x_k\| + \| x_k- \alpha_t x_k\|+ \alpha_t^2\| x_k-x_0\|\right)^2 \\
        &\leq \frac{1}{2h_t}\left( \thres + h_t B + \epsilon^{1/\holder}\right)^2.
    \end{align*}
    Utilizing $\thres = 2\sqrt{h_t\log(1/\aerror)}+h_t B + L_\expp \epsilon^{1/\holder}$ and $h_t \geq \epsilon^{2/\holder}/4$, we get
    \begin{align*}
        \left|\frac{\alpha_t^2}{2h_t}\cD_k(x,x_0,t) \right| 
        &\leq 2\left( \sqrt{\log(1/\aerror)}+  B + L_\expp+1\right)^2.
    \end{align*}
    Then we can derive the following polynomial approximation result. Set $S= 4 e^2 \log(1/\aerror) + 4e^2(B+L_\expp+1)^2$. Lemma \ref{lemma:smooth-approx} shows that
    \begin{align}\label{eq:Lexp2}
        \left | \exp \left(  -\frac{\alpha_t^2}{2h_t}\cD_k(x,x_0,t)  \right) - \sum_{j=0}^{S-1} \frac{(-1)^j\alpha_t^{2j}}{2^j h_t^j j!} \cD_k(x,x_0,t) ^{2j} \right| \leq \aerror.
    \end{align}
    Finally, combining the approximation results in \eqref{eq:Lexp1} and \eqref{eq:Lexp2}, we can derive the approximation error of $\text{Poly}^k$ to $I_k$. Since  $\text{Poly}^k$ and $I_k$ are integrals over $U_k = \expp_k(\dball(0,\epsilon^{1/\holder}))$, we have
    \begin{align*}
    \left|\text{Poly}^k(x,t)-I_k(x,t)\right|= \cO\left( \frac{ (\sqrt{\log(1/\aerror) }+ B)^{\degree_0} }{\reach^{\degree_0} } \epsilon^{(d+\degree_0)/\holder}+ \epsilon^{d/\holder}\aerror \right),
\end{align*}
which holds for any chart $k\in [C_\cM]$.
\end{proof}

\begin{lemma}\label{lemma:Ik_simplify}
 For any $k=1,\ldots,C_\cM$, let $P_k \in \RR^{D\times d}$ be a matrix with columns forming an orthonormal basis of $T_{x_k}\cM$. For any $x\in \RR^D$ and $t>0$, $\text{Poly}^k(x, t)$ defined in \eqref{def:Ik-poly-approximator} is a polynomial with respect to $x$ up to order $\degree_0$, and $P_k^\top x \in \RR^d$ up to order $S$, with the form 
 \begin{align*}
        \text{Poly}^k(x,t)
        =  \sum_{l=0}^{ \degree_0}\sum_{j=0}^{S-1}  \frac{1}{h_t^{l+j}} \sum_{p = 0}^{2(l+j)} \alpha_t^p \sum_{|\theta| \leq l, |\gamma| \leq j} a_{l,j,p,\theta,\xi} x^\theta (P_k^\top x)^\xi.
    \end{align*}  
    Here $\{a_{l,j,p,\theta,\xi} \}$ are constants in $\RR$.
\end{lemma}
\begin{proof}
Recall the definition of $\text{Poly}^k(x, t)$ in \eqref{def:Ik-poly-approximator}
\begin{align*}
    \text{Poly}^k(x, t) :=& \int_{x_0 \in U_k} \sum_{l=0}^{\degree_0 -1} \sum_{j=0}^{S-1}  \frac{(-1)^{l+j}\alpha_t^{l+2j}}{2^j h_t^{l+j} l!j!} \cT_k^l(x,x_0,t)  \cD_k^{j}(x,x_0,t) \rho_k(x_0) p_{\rm data}(x_0) \ud \vol(x_0),
    \end{align*}
    where by definition, we have
    \begin{align*}
    \cT_k(x,x_0,t)= \langle x-\projk(x,t), \projk(x,t)-\alpha_t x_0 \rangle, \quad \text{and} \quad
    \cD_k(x,x_0,t)=\|\projk(x,t)/\alpha_t-x_0\|^2.
    \end{align*}
    Since $x-\projk(x,t) \perp \alpha_t\cdot T_{x_k}\cM$, we have 
    \begin{align*}
        \cT_k(x,x_0,t)= \langle x-\projk(x,t), \projk(x,t)-\alpha_t x_k +\alpha_t x_k -\alpha_t x_0 \rangle = \langle x-\projk(x,t), \alpha_t x_k-\alpha_t x_0 \rangle.
    \end{align*}
     Note that the projection $\projk$ admits a linear formulation.
     By definition, the columns of $P_k \in \RR^{D\times d}$ form an orthonormal basis of $T_{x_k}\cM$. Then we have
    \begin{align}\label{eq:proj-k}
        \projk(x,t) = P_k P_k^\top(x-\alpha_t x_k) + \alpha_t x_k.
    \end{align}
    Therefore, we can rewrite $\cT_k(x,x_0,t)$ and $\cD_k(x,x_0,t)$ as
    \begin{align*}
        \cT_k(x,x_0,t) =  (x_k-x_0)^\top \left( I -P_k P_k^\top \right)(x-\alpha_t x_k),
    \end{align*}
    and
    \begin{align*}
         \cD_k(x,x_0,t) = \left\|\frac{1}{\alpha_t}P_k P_k^\top (x -\alpha_t x_k) + (x_k-x_0) \right\|^2.
    \end{align*}
    Notably,  given $x_0 \in U_k$, $\cT_k(x,x_0,t)$ is linear in $x\in \RR^D$, while $\cD_k(x,x_0,t)$ is quadratic in the low-dimensional representation $P_k^\top( x- \alpha_t x_k) \in \RR^d$. Since the integral region $U_k$ in $\text{Poly}^k(x;t)$ is independent of $x$, $\text{Poly}^k(x;t)$ after performing the integration is a polynomial with the form
    \begin{align*}
        \text{Poly}^k(x;t)
        =  \sum_{l=0}^{\degree_0}\sum_{j=0}^{S-1}  \frac{1}{h_t^{l+j}} \sum_{p = 0}^{2(l+j)} \alpha_t^p \sum_{|\theta| \leq l, |\gamma| \leq j} a_{l,j,p,\theta,\xi} x^\theta (P_k^\top x)^\xi,
    \end{align*}  
    where $\{a_{l,j,p,\theta,\xi} \}$ are constants in $\RR$.
    We conclude that $\text{Poly}^k(x;t)$ is a polynomial with respect to $x \in \RR^D$ up to order $\degree_0$, and $P_k^\top x \in \RR^d$ up to order $S$. 
\end{proof}

\section{Auxiliary Lemmas}

\begin{lemma}[High probability region of Gaussian Samples]\label{lemma:truncate-x}
Let $\cM$ be a manifold in $\RR^D$ and fix any $\alpha \in\RR$ and $\sigma>0$. For each point $x \in \cM$, we perturb it as $x' = \alpha x + \sigma \xi$, where $\xi \sim {\sf N}(0, I_D)$. Given an arbitrary $\delta \in (0, e^{-2})$, drawing $x'$ randomly via sampling $x \sim \dist$ on the manifold and $\xi \sim {\sf N}(0, \sigma^2 I_D)$ independently satisfies
\begin{align*}
\PP\left(x' \in \cK(\alpha\cM, 2\sigma \sqrt{D \log (1/\delta)} \right) \geq 1 - \delta^D.
\end{align*}
\end{lemma}
\begin{proof}
Due to the sampling of $x'$,  for any $R>0$, we have
\begin{align*}
\PP\left(x' \not \in \cK(\alpha\cM,R)\right) & \overset{(i)}{\leq} \int_{x \in \cM} \PP_{x' \sim {\sf N}(\alpha x, \sigma^2 I_D)} \left(x' \not \in \cB^D(\alpha x, R) \right)\ud \dist(x) \\
& = \int_{x \in \cM} \PP_{x' \sim {\sf N}(\alpha x, \sigma^2 I_D)} \left(\norm{x' - \alpha x} > R \right)\ud \dist(x).
\end{align*}
where inequality $(i)$ follows from the union bound. Plugging in $R = 2\sigma \sqrt{D \log \frac{1}{\delta}}$, we only need to verify
\begin{align}\label{eq:gaussian_tail_claim}
\PP_{x' \sim {\sf N}(x, \sigma^2 I_D)} \left(\norm{x' - \alpha x} > 2\sigma \sqrt{D \log \frac{1}{\delta}} \right) \leq \delta^D.
\end{align}
To verify inequality \eqref{eq:gaussian_tail_claim}, we note that $(x' - \alpha x)/\sigma \sim {\sf N}(0, I_D)$. Therefore, it suffices to show
\begin{align*}
\PP_{z \sim {\sf N}(0, I_D)} \left(\norm{z} > 2\sqrt{D \log \frac{1}{\delta}} \right) \leq \delta^D.
\end{align*}
In fact, we observe that $\norm{z}^2$ follows from the $\chi^2$-square distribution with freedom $D$. Hence, invoking \citet[Lemma 1]{laurent2000adaptive}, for any $u > 0$, we have
\begin{align*}
\PP_{z \sim {\sf N}(0, I_D)} \left(\norm{z}^2 > D + 2\sqrt{D u} + 2u\right) \leq \exp(-u).
\end{align*}
Setting $u = D \log \frac{1}{\delta}$, we obtain
\begin{align*}
& \PP_{z \sim {\sf N}(0, I_D)} \left(\norm{z}^2 > D + 2D \sqrt{ \log \frac{1}{\delta}} + 2 D \log \frac{1}{\delta} \right) \leq \delta^D \\
\Longrightarrow ~& \PP_{z \sim {\sf N}(0, I_D)} \left(\norm{z}^2 > 4 D \log \frac{1}{\delta} \right) \leq \delta^D,
\end{align*}
since $4 \log \frac{1}{\delta} > 2\log \frac{1}{\delta} + 2 \sqrt{ \log \frac{1}{\delta} } +1 $ when $0 < \delta < e^{-2}$. Thus, inequality~\eqref{eq:gaussian_tail_claim} is verified. The proof is complete by considering $\PP\left(x' \in \cK(\alpha \cM,R)\right) = 1 - \PP\left(x' \not \in \cK(\alpha \cM,R)\right)$.
\end{proof}

\begin{lemma}[Approximating a smooth univariate function with Taylor polynomials]\label{lemma:smooth-approx}
Let $A>0$, $\epsilon>0$ and  $f: \RR \to \RR$ be a smooth function with partial derivatives up to $\mathcal{S}_0$-th order bounded by some constant $\eta>0$. Choose $\cS \in [\max\{e^2 A, \log(\eta/\epsilon)\}, \mathcal{S}_0]$.  Then we have

\begin{equation*}
    \bigg| f(x) - \sum_{|\theta|<\cS} \frac{\partial^\theta f(0)}{\theta !}x^\theta \bigg| \leq \epsilon,
\end{equation*}
which holds for any $x \in \RR$ satisfying $\|x \| \leq A$.
\end{lemma}
\begin{proof}
    By the Taylor's theorem for multivariate functions, we have
    \begin{align*}
        f(x) = \sum_{\theta<\cS} \frac{\partial^\theta f(0)}{\theta !}x^\theta  +  \frac{h_\cS(x)}{\cS !} x^\cS,
    \end{align*}
    where $|h_\cS(x)| \leq \max_{y} |  \partial^\cS f(y)| \leq \eta $. Thereby it follows that
    \begin{align*}
        \bigg| f(x) - \sum_{\theta<\cS} \frac{\partial^\theta f(0)}{\theta !}x^\theta \bigg| \leq \frac{\eta}{\cS !} |x^\cS| .
    \end{align*}
    Using $\cS ! \geq (\cS/e)^\cS $, we further derive that
    \begin{align*}
        \bigg| f(x) - \sum_{|\theta|<\cS} \frac{\partial^\theta f(0)}{\theta !}x^\theta \bigg| 
          \leq \frac{\eta A^\cS}{(\cS/e)^\cS } = \eta\left( \frac{e A }{\cS}\right)^\cS.
    \end{align*}
    Since we choose $\cS$ such that $\cS \geq \max\{e^2 A, \log(\eta/\epsilon)\}$, we obtain that
    \begin{align*}
        \bigg| f(x) - \sum_{|\theta|<\cS} \frac{\partial^\theta f(0)}{\theta !}x^\theta \bigg|  
          \leq \eta\left( \frac{e A }{e^2 A}\right)^\cS = \eta e^{-\cS} \leq \epsilon.
    \end{align*}
    The proof is complete.
\end{proof}

\section{Average Taylor Polynomial}\label{sec:avg-taylor-poly}

In this section, we provide the definition and properties of average Taylor polynomials.

\begin{definition}[Averaged Taylor polynomials]\label{def:avg_poly}
Suppose $f \in C^{\alpha-1}(\Omega)$ where $\cB(z_0, \rho)\subseteq \Omega$. The corresponding average Taylor polynomial of degree $\alpha$ of $f$ averaged over $\cB(z_0, \rho)$ is defined as 
		\begin{align}
			Q_{z_0}^{\alpha}f(x)=\int_{\cB(z_0, \rho)} T_z^{\alpha}f(x)\phi(z)dz
		\end{align}
		with
		\begin{align}
			T_z^{\alpha}f(x)=\sum_{|\theta|<\alpha} \frac{\partial^{\theta}f(z)}{|\theta|!} (x-z)^{\theta}.
		\end{align}
		Here $\phi$ being arbitrary cut-off function satisfying
		\begin{align*}
			&\phi\in C^{\infty}(\RR^d) \mbox{ with } \phi(z)\geq 0 \mbox{ for all } z\in \RR^d, \nonumber\\
			&\supp(\phi)=\overline{\cB(z_0, \rho)} \mbox{ and } \int_{\RR^D}\phi(z)dz=1.
		\end{align*}
	\end{definition}
    For $f \in C^\alpha(\Omega) $, we define its Sobolev norm \citep[Definition 1.3.1]{brenner2008mathematical} as $\|f\|_{W^{\alpha,\infty}(\Omega)} = \max_{|\theta|\leq \alpha}\| \partial^\theta f\|_{L^\infty(\Omega)}$ with $\theta$ a multi-index.
	The averaged Taylor polynomial can approximate $f$ and its partial derivatives well. Specifically, Lemma \ref{lemma:bramble} provides an approximation guarantee in Sobolev norm. 
    
\begin{lemma}[Bramble-Hilbert, Chapter 4.1 in \cite{brenner2008mathematical}]\label{lemma:bramble}
Suppose $f \in C^{\alpha}(\cB(z_0,\rho))$. There exists a constant $C_{\alpha,d}>0$ such that
\begin{equation*}
    \left\| f- Q^\alpha_{z_0} f\right\|_{W^{p,\infty}(\cB(z_0,\rho))} \leq C_{\alpha,d} \rho^{\alpha-p} \| f\|_{W^{\alpha,\infty}(\cB(z_0,\rho))}~~\text{ for } p= 0,1,\ldots,\alpha,
\end{equation*}
where $Q^\alpha_{z_0} f$ denotes the averaged Taylor polynomial of degree $\alpha$ of $f$ averaged over $\cB(z_0,\rho)$.
\end{lemma}

Lemma \ref{lemma:atp-form} below shows that $Q_{z_0}^{\alpha}f$ can be written as a weighted sum of polynomials.
	\begin{lemma}[Proposition 4.1.12 of \cite{brenner2008mathematical}]\label{lemma:atp-form}
		Let $\alpha \in \NN^+$ and $f\in W^{\alpha-1,\infty}(\Omega)$. Let $z_0\in\Omega, \rho>0$ such that $\cB(z_0,\rho)\subseteq \Omega$. Then the averaged Taylor polynomial $Q_{z_0}^{\alpha}(f)$ can be written as
		\begin{align}
			Q_{z_0}^{\alpha}f(x)=\sum_{|\theta|<\alpha} c_{\theta}x^{\theta}, \quad\text{for } x \in \Omega.
			\label{eq.aveTayloer.poly}
		\end{align}
         Moreover, there exists a constant $C$ depending on $\alpha,d,\rho$ such that for all $|\theta| < \alpha$,
		$$
		|c_{\theta}|\leq C\|f\|_{W^{\alpha-1,\infty}}(\Omega).
		$$
		
	\end{lemma}

\section{Polynomial Size}\label{sec:poly-size}
The core of our approximation theory depends on approximating smooth functions with Taylor expansions. In this section, we develop auxiliary lemmas to control the size of the polynomials. 

We see that a function $f:\RR^d \to \RR$ can be approximated well by its Taylor polynomial 
\[
    \hat{f}(x) = \sum_{ |\alpha|  \leq d_f} {|\alpha| \choose \alpha} \frac{\partial^\alpha f(x_0)} {|\alpha|!} (x-x_0)^\alpha,
\]
where $\alpha = (\alpha_1, \dots, \alpha_d)$ with $\alpha_i \geq 0$ is a multi-index, $|\alpha| = \sum_i \alpha_i$, and 
\[
    x^\alpha = \prod_i (x_i)^{\alpha_i}, \partial^\alpha = (\partial_1)^{\alpha_1}\cdots (\partial_d)^{\alpha_d}.
\]
The combinatorial number ${|\alpha| \choose \alpha}$ is
\[
    {|\alpha| \choose \alpha} = \frac{|\alpha|!}{\alpha_1 ! \cdots \alpha_d !}
\]
For $\alpha + \beta = \gamma$, define 
\[
    {\gamma \choose \alpha \quad \beta} = \prod_i {\gamma_i \choose \alpha_i \quad \beta_i} = \prod_i \frac{\gamma_i!} { \alpha_i ! \beta_i!} 
\]

\begin{lemma}
We have the following bound
\[
    {|\alpha| \choose \alpha} \frac{2^{|\alpha|}}{|\alpha|!} \leq 2^d.
\]
\end{lemma}

\begin{lemma}[Multiplication]\label{lemma:poly:mul}
Assume we have two polynomial approximations $\hat{f}$ and $\hat{g}$, both centered at $x_0$:
\[
    \hat{f}(x) = \sum_{ |\alpha| \leq d_f} {|\alpha| \choose \alpha} \frac{\partial^\alpha f(x_0)} {|\alpha|!} (x-x_0)^\alpha,
\]
and
\[
    \hat{g}(x) = \sum_{ |\alpha|  \leq d_g} {|\alpha| \choose \alpha} \frac{\partial^\alpha g(x_0)} {|\alpha|!} (x-x_0)^\alpha.
\]
Assume that there exist constants $\{C_{f,\alpha}, C_{g,\alpha}, C_{\alpha}\}$, such that $ |{\partial^\alpha f(x_0)}| \leq C_{f,\alpha}$ and $ |{\partial^\alpha g(x_0)}| \leq C_{g,\alpha}$, and $\sup_{\alpha + \beta = \gamma} C_{f,\alpha} C_{g,\beta} \leq C_{\gamma}$. 

We have that the coefficient of the $(x-x_0)^\gamma$ term of $\hat{f}(x)\hat{g}(x)$ is bounded by 
\[
    {|\gamma| \choose \gamma} \frac{2^{|\gamma|} \cdot C_\gamma} {|\gamma|!} 
\]
The degree of  $\hat{f}(x)\hat{g}(x)$  is at most $d_f + d_g$.
\end{lemma}

\begin{lemma}[Exponentiation]\label{lemma:poly:exp}
Assume we a polynomial approximation $\hat{f}$  centered at $x_0$:
\[
    \hat{f}(x) = \sum_{ |\alpha|  \leq d_f} {|\alpha| \choose \alpha} \frac{\partial^\alpha f(x_0)} {|\alpha|!} (x-x_0)^\alpha.
\]
Assume that  $\sup |{\partial^\alpha f(x_0)}| \leq 2^{|\alpha|} C_f$. Let $\hat{g}(z)$ be the $d_g$-th order Talyor expansion of the exponential function, 
\[
    \hat{g}(z) = \sum_{k=0}^{d_g} \frac{1}{k!} z^k
\]
The coefficient of the $(x-x_0)^\gamma$ term of $\hat{g}(\hat{f}(x))$ is bounded by 
\[
     {|\gamma| \choose \gamma} \frac{1} {|\gamma|!} \cdot {\exp(2^{|\gamma|} C_f)}
\]
The degree of $\hat{g}(\hat{f}(x))$ is at most $d_fd_g$.
\end{lemma}

\begin{lemma}[Addition]\label{lemma:poly:translate}
Assume we a polynomial approximation $\hat{f}$  centered at $x_0$:
\[
    \hat{f}(x) = \sum_{ |\alpha|  \leq d_f} {|\alpha| \choose \alpha} \frac{\partial^\alpha f(x_0)} {|\alpha|!} (x-x_0)^\alpha.
\]
Consider $g(u,v) = \hat{f}(u+v)$, we have 
\[
g(u,v) =  \sum_{ |(\alpha,\beta)|  \leq d_f} {|\alpha,\beta| \choose \alpha,\beta} \frac{\partial^{(\alpha+\beta)} f(x_0)} {|\alpha,\beta|!}  (u-u_0)^\alpha (v-v_0)^\beta,
\]
for any pair of $(u_0, v_0)$ satisfying $u_0 + v_0 = x_0$, where $(\alpha,\beta)$ is a multi-index for the $2d$ input $(u,v)$.
\end{lemma}

\begin{lemma}[Partial Integral]\label{lemma:poly:int}
Assume we a polynomial approximation $\hat{g}$  centered at $(u_0, v_0)$:
\[
\hat{g}(u,v) =  \sum_{ |(\alpha,\beta)|  \leq d_g} {|\alpha,\beta| \choose \alpha,\beta} \frac{\partial^{(\alpha,\beta)} g(u_0,v_0)} {|\alpha,\beta|!}  (u-u_0)^\alpha (v-v_0)^\beta,
\]
where we assume that there exists $C_{g,\beta}$ such that $\sup_{\alpha} |{\partial^{(\alpha,\beta)} g(u_0,v_0)}| \leq C_{g,\beta}$.
Define 
\[ 
    f(v) = \int_{u \in U} \hat{g}(u,v) \mathrm{d} u.
\]
The coefficient of the $(v-v_0)^\beta$ term of $f$ is bounded by
\[
 {|\beta| \choose \beta} \frac{C_{g,\beta}}{ |\beta|!} \cdot 2^d \exp(Bd), 
\]
where $B = \sup_{u \in U} \|u-u_0\|_\infty$.
\end{lemma}

\subsection{Proofs for the lemmas}
\begin{proof}[proof of Lemma~\ref{lemma:poly:mul}]
Let's focus on the $(x-x_0)^\gamma$ term of $ \hat{f}(x) \hat{g}(x)$. The coefficient of $(x-x_0)^\gamma$ is
\[
    \sum_{|\alpha|  \leq d_f, |\beta|  \leq d_g, \alpha+\beta =\gamma}{|\alpha| \choose \alpha} \frac{\partial^\alpha f(x_0)} {|\alpha|!}{|\beta| \choose \beta} \frac{\partial^\beta g(x_0)} {|\beta|!},
\]
whose absolute value can be bounded by
\begin{align*}
    & \sum_{|\alpha|  \leq d_f, |\beta|  \leq d_g, \alpha+\beta =\gamma}{|\alpha| \choose \alpha} \frac{1} {|\alpha|!}{|\beta| \choose \beta} \frac{1} {|\beta|!} C_{f,\alpha} C_{g,\beta}  \\
    \leq & \sum_{|\alpha|  \leq d_f, |\beta|  \leq d_g, \alpha+\beta =\gamma}{|\alpha| \choose \alpha} \frac{1} {|\alpha|!}{|\beta| \choose \beta} \frac{1} {|\beta|!} C_{\gamma} \tag{Using the assumption that $\sup_{\alpha + \beta = \gamma} C_{f,\alpha} C_{g,\beta} \leq C_{\gamma}$} \\
    = & \sum_{|\alpha|  \leq d_f, |\beta|  \leq d_g, \alpha+\beta =\gamma} \frac{|\gamma|!} {|\alpha|!|\beta|!}{|\alpha| \choose \alpha} {|\beta| \choose \beta}   \frac{1}{|\gamma|!} C_{\gamma}\\
    \leq & \sum_{\alpha+\beta =\gamma} \frac{|\gamma|!} {|\alpha|!|\beta|!}{|\alpha| \choose \alpha} {|\beta| \choose \beta}   \frac{1}{|\gamma|!} C_{\gamma} \tag{relaxing the $|\alpha|  \leq d_f, |\beta|  \leq d_g$ }\\ 
    = & \sum_{\alpha+\beta =\gamma} \frac{|\gamma|!} {(\alpha_1!\beta_1!)\cdots (\alpha_d!\beta_d!)}  \frac{1}{|\gamma|!} C_{\gamma}  \\ 
    = & \sum_{\alpha+\beta =\gamma} \frac{\gamma_1!\cdots \gamma_d!} {(\alpha_1!\beta_1!)\cdots (\alpha_d!\beta_d!)}\frac{|\gamma|!} {\gamma_1!\cdots \gamma_d!}  \frac{1}{|\gamma|!} C_{\gamma}  \\
    = & 2^{\gamma_1+\dots \gamma_d} \frac{|\gamma|!} {\gamma_1!\cdots \gamma_d!}  \frac{1}{|\gamma|!} C_{\gamma} \tag{Applying the combinatorial equality $\sum_{\alpha_i + \beta_i = \gamma_i } \gamma_i!/\alpha_i!\beta_i! = 2^{\gamma_i}$} .
\end{align*} 
\end{proof}

\begin{proof}[proof of Lemma~\ref{lemma:poly:exp}]
Note that we have 
\[
\hat{g}(\hat{f}(x)) = \sum_{k=0}^{d_g} \frac{1}{k!} \Big(  \sum_{ |\alpha|  \leq d_f} {|\alpha| \choose \alpha} \frac{\partial^\alpha f(x_0)} {|\alpha|!} (x-x_0)^\alpha \Big)^k.
\]
By iteratively applying Lemma~\ref{lemma:poly:mul}, we have the coefficient of $(x-x_0)^\gamma$ in $\hat{f}(x)^k$ can be bounded by 
\[
      {|\gamma| \choose \gamma} \frac{(2^{|\gamma|})^{k}  \cdot C_f^{k}} {|\gamma|!}.
\]
Therefore, when taking summation over $k$, the coefficient of $(x-x_0)^\gamma$ in $\hat{g}(\hat{f}(x))$ can be bounded by 
\begin{align*}
      & \sum _{k\leq d_g} \frac{1}{k!} {|\gamma| \choose \gamma} \frac{(2^{|\gamma|})^{k}  \cdot C_f^{k}} {|\gamma|!} \\
      \leq & {|\gamma| \choose \gamma} \frac{1} {|\gamma|!} \cdot {\exp(2^{|\gamma|} C_f)}.
\end{align*}
\end{proof}

\begin{proof}[proof of Lemma~\ref{lemma:poly:translate}]
Note that by the binomial theorem, we have
\[
g(u,v) =  \sum_{ |\gamma|  \leq d_f} {|\gamma| \choose \gamma} \frac{\partial^\gamma f(x_0)} {|\gamma|!} \sum_{\alpha+\beta = \gamma} {\gamma \choose \alpha \quad \beta } (u-u_0)^\alpha (v-v_0)^\beta,
\]
which is the same as the formula stated in the Lemma.
\end{proof}

\begin{proof}[proof of Lemma~\ref{lemma:poly:int}]
Denote $B = \sup_{u \in U} \|u-u_0\|_\infty$, and $U' = \{u: \|u-u_0\|_\infty \leq B\}$. 
\begin{align*}
    & \Big| \int_U (u-u_0)^\alpha \mathrm{d}u \Big| \\
    \leq &\int_{U'} \Big|  (u-u_0)^\alpha \Big| \mathrm{d}u \\
    = &  \prod_i \int_{-B\leq u_i \leq B} |\Delta  u_i|^{\alpha_i}  \mathrm{d}\Delta u_i \\
    = & \prod_i \frac{2 B^{\alpha_i+1}}{\alpha_i+1} .
\end{align*}
Therefore, the coefficient of the $(v-v_0)^\beta$ term is bounded by:
\begin{align*}
  & \sum_{ \alpha: |(\alpha,\beta)|  \leq d_g}  {|\alpha,\beta| \choose \alpha,\beta} \frac{|\partial^{(\alpha,\beta)} g(u_0,v_0)|} {|\alpha,\beta|!} \Big| \int_{u\in U} (u-u_0)^\alpha \mathrm{d}u \Big | \\
  \leq & \sum_{ \alpha}  {|\alpha,\beta| \choose \alpha,\beta} \frac{C_{g,\beta}} {|\alpha,\beta|!}  \prod_i \frac{2 B^{\alpha_i+1}}{\alpha_i+1} \\
  = & \frac{C_{g,\beta}}{\prod_i \beta_i!} \sum_\alpha \prod_i \Big( \frac{1}{\alpha_i!}\cdot \frac{2 B^{\alpha_i+1}}{\alpha_i+1} \Big) \\ 
  \leq &  \frac{C_{g,\beta}}{\prod_i \beta_i!}  \prod_i \sum_{\alpha_i}  \frac{2 B^{\alpha_i+1}}{(\alpha_i+1)!}  \\ 
  \leq &  \frac{C_{g,\beta}}{\prod_i \beta_i!} 2^d \exp(Bd).  
\end{align*}

\end{proof}
\clearpage
\end{document}